\documentclass[11pt,oneside]{article}

\usepackage{epsfig,lscape}
\usepackage{amssymb,amsfonts,amsmath}
\usepackage{amsthm}
\usepackage{rotating}
\usepackage{color}
\usepackage[dvipsnames]{xcolor}
\usepackage{float}
\usepackage[hypertexnames=false]{hyperref}

\usepackage[authoryear, round]{natbib} 

\usepackage{algorithm}
\usepackage{algorithmic}

\usepackage{booktabs}
\usepackage{threeparttable}

\renewcommand{\baselinestretch}{1.6}

\def\dif{\mathrm d}

\def\Cov{\mathrm{Cov}}
\newenvironment{prf}[1][Proof]{\noindent \textbf{#1.}\ \ }{\hfill $\Box$ \vspace{.1in}}

\newtheorem{thm}{Theorem}[section]
\newtheorem{lem}{Lemma}[section]

\newtheorem{cor}{Corollary}[section]
\newtheorem{ass}{Assumption}
\theoremstyle{definition}
\newtheorem{rem}{Remark}[section]

\allowdisplaybreaks
\usepackage[medium,compact]{titlesec}
\titlespacing*{\section} {0pt}{1.5ex}{1ex}
\titlespacing*{\subsection} {0pt}{1.5ex}{1ex}

\titleformat{\part}[hang]
  {\huge\bfseries}
  {\partname\ \thepart}
  {1em}
  {}

\begin{document}

\setlength{\abovedisplayskip}{5pt}
\setlength{\belowdisplayskip}{5pt}

\begin{titlepage}

\begin{center}
{\Large Unregularized Convergence of Single-Loop, Entropy-Regularized\\ Natural Actor-Critic}

\vspace{.1in} Zhiqiang Tan\footnotemark[1]

\vspace{.1in}
\today
\end{center}

\footnotetext[1]{Department of Statistics, Rutgers University. Address: 110 Frelinghuysen Road,
Piscataway, NJ 08854. E-mail: ztan@stat.rutgers.edu.
The author thanks Yifan Hu and Koulik Khamaru for helpful discussions, and acknowledges the assistance and collaboration from Gemini and ChatGPT.}

\paragraph{Abstract.}

While entropy regularization is widely used to stabilize and accelerate Natural Policy Gradient methods, its ability to yield faster convergence rates for the unregularized objective remains underexplored. Existing analyses often rely on double-loop architectures and invoke a linear entropy penalty. To bridge the gap between theory and practice, we analyze a single-loop, entropy-regularized Natural Actor-Critic algorithm under compatible linear function approximation. By training an uncentered critic, our critic tracking can remain stable even as the training policy approaches determinism and the Fisher information matrix degenerates.
We focus on two primary regimes for the optimization landscape: a Stochastic Regime, where we fuse coupled actor-critic updates into a joint Lyapunov recurrence, and a Deterministic Regime, where we pivot to a Policy Mirror Descent framework to circumvent the collapse of Euclidean geometry. By exploiting a positive Minimal Action Gap in the unregularized Markov decision process, we introduce an Exponential Translation mechanism that maps the regularized gap to the unregularized one up to an exponentially decaying tail. By tuning the fixed temperature, our algorithm achieves accelerated unregularized convergence rates, up to approximation-error terms: $\tilde{\mathcal{O}}(T_{total}^{-1})$ in the Stochastic Regime, and $\tilde{\mathcal{O}}(T_{total}^{-2/3})$ for the average iterate alongside $\tilde{\mathcal{O}}(T_{total}^{-1/3})$ for the last iterate in the Deterministic Regime. Here, $T_{total}$ denotes the total number of stochastic critic updates (or Monte Carlo rollouts). Furthermore, in the tabular setting, our positive-action-gap analysis yields a $\tilde{\mathcal{O}}(T_{total}^{-2/3})$ average-iterate rate, surpassing the $\mathcal{O}(T_{total}^{-1/2})$ worst-case statistical barrier that applies without a positive action margin.

\vspace{.2in}

\end{titlepage}

{
\renewcommand{\baselinestretch}{1.1}\normalsize
\tableofcontents
}
\newpage

\section{Introduction} \label{sec:intro}

\textbf{Motivation.}
Policy gradient methods, particularly Natural Policy Gradient (NPG) \citep{kakade2001natural,peters2008natural}, form the backbone of modern reinforcement learning. To encourage exploration and improve convergence, these algorithms are frequently deployed in tandem with entropy regularization \citep{williams1991function, peters2010relative, mnih2016asynchronous, haarnoja2018soft}.
While the empirical benefits of entropy regularization---such as smoother optimization landscapes and faster convergence---are widely recognized, there remain substantive gaps in existing theoretical analyses to address the following question:\vspace{-.08in}
\begin{quote}
    \textit{Compared with unregularized NPG algorithms, what convergence rates (if any faster) can entropy-regularized NPG algorithms be shown to achieve in terms of unregularized performance gaps by appropriate temperature tuning alongside step-size tuning?}
\end{quote}

Historically, much of the early convergence literature focused on the \textit{tabular} Markov Decision Process (MDP). While early works established global convergence for both unregularized and regularized NPG using exact gradients \citep{mei2020global, bhandari2021linear, cen2022fast}, subsequent theoretical developments exploit the Policy Mirror Descent (PMD) framework, which is motivated by Trust Region Policy Optimization (TRPO) \citep{schulman2015trust}.
Under log-linear policy parameterizations, entropy-regularized NPG is equivalent to a PMD update with the Kullback-Leibler (KL) divergence as the proximity term. However, these PMD analyses typically study the unregularized and regularized settings in isolation \citep{shani2020adaptive, lan2023policy}, leaving the unregularized performance of regularized algorithms unevaluated. Even when regularized tabular algorithms are explicitly mapped to the unregularized objective, the translation relies on a standard linear $\mathcal{O}(\tau)$ entropy penalty. For Soft Policy Iteration (SPI) analyzed in \cite{cen2022fast},  the unregularized convergence rate is shown to be $\tilde{\mathcal{O}}(T_{total}^{-1/2})$ (see Table~\ref{tab:tabular-comparison}).

Beyond the tabular setting, guarantees for entropy regularization with function approximation remain limited.
For linear function approximation, \cite{cayci2024convergence} analyzed entropy-regularized NPG algorithms, obtaining an $\tilde{\mathcal{O}}(T_{total}^{-1/3})$ convergence rate (up to the approximation error term) for the \textit{regularized} objective, but without explicitly discussing the unregularized performance gap.
When this analysis is mapped to the unregularized objective via the linear entropy penalty as shown in Appendix \ref{app:cayci-et-al}, the resulting convergence rate would be $\tilde{\mathcal{O}}(T_{total}^{-1/6})$. This is significantly slower than the unregularized rates, $\tilde{\mathcal{O}}(T_{total}^{-1/4})$ to $\tilde{\mathcal{O}}(T_{total}^{-1/2})$, established directly for unregularized NPG algorithms \citep{agarwal2021theory, liu2020improved, yuan2023linear}  (see Table~\ref{tab:sample-complexity-comparison}).
Identifying whether entropy-regularized algorithms can match or even break the existing unregularized rates via appropriate tuning remains an open theoretical challenge.

Compounding this challenge is that the existing analyses of NPG algorithms are heavily dominated by double-loop architectures. Within an outer iteration, these algorithms freeze the actor parameter and run an inner loop of stochastic gradient descent (SGD) to solve a least-squares projection for updating the critic parameter.
While this double-loop approach simplifies the theoretical analysis, it can artificially inflate overall sample complexities or lead to highly skewed hyperparameter configurations to suppress statistical noise. Interestingly, the unbalanced configurations are required by the strongest existing theoretical rates: to achieve their $\tilde{\mathcal{O}}(T_{total}^{-1/2})$ unregularized convergence limits, both the double-loop Q-NPG under function approximation \citep{yuan2023linear} (see Appendix~\ref{app:Yuan-et-al}) and the tabular SPI \citep{cen2022fast} dictate large inner-loop iterations or batch evaluations alongside a very small number of outer policy updates. Such choices misalign with the single-loop implementations commonly used in practice.

\textbf{Algorithm: Single-Loop, Entropy-Regularized Uncentered NAC.}
To bridge the gap between theory and practice, we study an entropy-regularized NPG algorithm in an actor-critic framework. Specifically, we analyze a single-loop, entropy-regularized Natural Actor-Critic (NAC) algorithm under a log-linear policy parameterization and a linear compatible critic. In contrast to double-loop methods, our algorithm updates the actor and the critic continuously at every iteration, executing exactly one critic SGD step per actor update ($N=1$).

Notably, our algorithm evaluates the critic using uncentered features rather than action-centered features (the score function). Standard NPG algorithms typically project the advantage or Q-values onto centered features. However, such projections rely on the positive-definiteness of the centered feature covariance matrix (the Fisher information matrix).
This reliance becomes problematic in the deterministic regime: as the temperature drops to zero to recover the unregularized MDP, the optimal policy collapses into a deterministic distribution. Consequently, as the training policy tracks this optimum and approaches determinism, its action randomness vanishes and its Fisher information matrix degenerates.
By instead projecting the regularized $Q$-values onto uncentered features, our critic tracking can remain stable even when the Fisher information matrix degenerates.

\begin{table}[htbp]
\centering
\renewcommand{\arraystretch}{0.8}

\begin{threeparttable}

\caption{Comparison of sample-based NPG and NAC algorithms using linear
function approximation (log-linear policies and linear compatible critics).
$T_{total} = N \times T$ denotes the total number of stochastic updates
(or Monte Carlo rollouts) across $T$ outer-loop iterations and $N$ inner-loop
steps, translating to the total number of environment interactions up to an
additional $\mathcal{O}((1-\gamma)^{-1})$ expected-horizon factor.}
\label{tab:sample-complexity-comparison}

\footnotesize
\begin{tabular}{@{}lcccc@{}}
\toprule
\textbf{Reference} & \textbf{Algorithm} & \textbf{Setting}
& \textbf{Obj} & \textbf{Rate} \\
\midrule

\textbf{\cite{agarwal2021theory}}
& Double-loop & Discounted, & Unreg
& Avg: $\mathcal{O}(T_{total}^{-1/4}) + \mathcal{O}(\sqrt{\epsilon_{app}})$ \\
(Corollary 26) & Q-NPG & indep.\ sampling & & \\[2pt]

\textbf{\cite{liu2020improved}}
& Double-loop & Discounted, & Unreg
& Avg: $\mathcal{O}(T_{total}^{-1/3}) + \mathcal{O}(\sqrt{\epsilon_{app}})$ \\
(Theorem 4.9) & NPG & indep.\ sampling & & \\[2pt]

\textbf{\cite{yuan2023linear}}
& Double-loop & Discounted, & Unreg
& Last: $\tilde{\mathcal{O}}(T_{total}^{-1/2}) + \mathcal{O}(\sqrt{\epsilon_{app}})$ \\
(Corollary 1) & Q-NPG & indep.\ sampling & & \\[2pt]

\textbf{\cite{cayci2024convergence}}
& Double-loop & Discounted, & Reg
& Avg: $\tilde{\mathcal{O}}(T_{total}^{-1/3}) + \mathcal{O}(\sqrt{\epsilon_{app}})$ \\
(Prop. 4.1, Rem. 4.6) & ent-reg NPG & indep.\ sampling & ($\tau>0$) & \\[2pt]

\textbf{\cite{wang2024nonasymptotic}}
& Single-loop & Average-reward, & Unreg
& Avg: $\tilde{\mathcal{O}}(T_{total}^{-1/3}) + \mathcal{O}(\sqrt{\epsilon_{app}})$ \\
(Theorem 2) & NAC & Markovian & & \\[2pt]

\textbf{Ours (Part II)}
& Single-loop & Discounted, & Unreg
& Avg: $\tilde{\mathcal{O}}(T_{total}^{-2/3}) + \tilde{\mathcal{O}}(\epsilon_{app})$ \\
(Theorems \ref{thm:PMD-unreg-average}, \ref{thm:PMD-unreg-last})
& ent-reg, unc NAC & indep.\ sampling & &
Last: $\tilde{\mathcal{O}}(T_{total}^{-1/3})
+ \tilde{\mathcal{O}}(\sqrt{\epsilon_{app}})$ \\[2pt]

\textbf{Ours (Part I)}
& Single-loop & Discounted, & Unreg
& Avg: $\tilde{\mathcal{O}}(T_{total}^{-1})
+ \tilde{\mathcal{O}}(\epsilon_{app})$ \\
(Theorems \ref{thm:unreg-average-iterate-optimal},
\ref{thm:unreg-last-iterate-optimal})
& ent-reg, unc NAC & indep.\ sampling & &
Last: $\tilde{\mathcal{O}}(T_{total}^{-1})
+ \tilde{\mathcal{O}}(\epsilon_{app})$ \\

\bottomrule
\end{tabular}

\begin{tablenotes}[flushleft]
\footnotesize
\item[] \textit{Obj}: objective;
\textit{Unreg}: unregularized;
\textit{Reg}: regularized;
\textit{ent-reg}: entropy-regularized;
\textit{unc}: uncentered;
\textit{indep.\ sampling}: conditionally independent sampling from the
discounted visitation measure of the current policy;
\textit{Avg}: average-iterate bound;
\textit{Last}: last-iterate bound.
The stated Part-I bounds are understood over the admissible range for
$\epsilon_{app}$, excluding $\epsilon_{app}=0$, as described in
Remark~\ref{rem:admissible-eps-app}.
\end{tablenotes}

\end{threeparttable}
\end{table}

\begin{table}[htbp]
\centering
\renewcommand{\arraystretch}{0.8}

\begin{threeparttable}

\caption{Comparison of sample-based NPG and PMD algorithms in the
\textit{tabular} unregularized setting, free of concentrability conditions.
$T_{total}$ denotes the total number of stochastic updates, translating to
environment interactions up to an additional
$\mathcal{O}((1-\gamma)^{-1})$ expected-horizon factor.}
\label{tab:tabular-comparison}

\footnotesize
\begin{tabular}{@{}llc@{}}
\toprule
\textbf{Reference} & \textbf{Algorithm} & \textbf{Unregularized Rate} \\
\midrule

\textbf{\cite{shani2020adaptive}} (Thm. 5)
& TRPO (with $\nu$-sampling)
& Avg: $\tilde{\mathcal{O}}(T_{total}^{-1/4})$ \\[2pt]

\textbf{\cite{cen2022fast}} (Rem. 1)
& SPI (Generative Model)
& Last: $\tilde{\mathcal{O}}(T_{total}^{-1/2})$ \\[2pt]

\textbf{\cite{lan2023policy}} (Prop. 2)
& SPMD (Generative Model)
& Avg: $\tilde{\mathcal{O}}(T_{total}^{-1/2})$ \\[2pt]

\textbf{Ours} (Theorem~\ref{thm:tabular-convergence})
& ent-reg, unc NAC
& Avg: $\tilde{\mathcal{O}}(T_{total}^{-2/3})$ \\
& (with $\nu$-sampling)
& Last: $\tilde{\mathcal{O}}(T_{total}^{-1/3})$ \\

\bottomrule
\end{tabular}

\begin{tablenotes}[flushleft]
\footnotesize
\item[] \textit{Note:}
\textit{$\nu$-sampling} initializes trajectories from an exploratory restart
distribution to ensure state-action coverage, whereas a \textit{Generative
Model} assumes access to a simulator capable of drawing next-state
transitions from any arbitrary state-action query.
\end{tablenotes}

\end{threeparttable}
\end{table}

\textbf{Analysis and Results.}
To analyze our entropy-regularized algorithm as the temperature drops ($\tau \to 0$), we distinguish two optimization regimes: \textit{Stochastic Regime} and \textit{Deterministic Regime}, which have not been explicitly separated in prior literature.
We develop tailored analytical tools for each regime
and establish new sample complexities for the unregularized objective, summarized in Tables~\ref{tab:sample-complexity-comparison} and \ref{tab:tabular-comparison}.
For simplicity, we discuss convergence rates up to approximation error terms.
\begin{itemize}
    \item \textit{Coupled Lyapunov Analysis (Stochastic Regime):} When the regularized parametric optimum retains strict stochasticity as $\tau \to 0$, the Fisher information matrix remains uniformly positive-definite. By establishing a Parameter-Space Polyak-\L{}ojasiewicz (PL) condition and an Actor Progress Bound, we obtain
        both Coupled Actor Recurrence and Coupled Critic Tracking. We then fuse these two coupled bounds into a \textit{joint Lyapunov function}, proving an optimal $\tilde{\mathcal{O}}(T_{total}^{-1})$ average-iterate and last-iterate regularized convergence rate.

    \item \textit{Policy Mirror Descent (Deterministic Regime):} When the regularized parametric optimum approaches a deterministic policy as $\tau \to 0$,
    the action randomness vanishes and the Fisher information degenerates, breaking the Euclidean parameter-space geometry. To survive this limit, we pivot to a Policy Mirror Descent (PMD) geometry. By evaluating algorithmic progress directly on the probability simplex via KL divergence, we decouple the actor's progress and  the uncentered critic tracking, establishing an accelerated $\tilde{\mathcal{O}}(T_{total}^{-2/3})$ average-iterate and $\tilde{\mathcal{O}}(T_{total}^{-1/3})$ last-iterate regularized convergence rate.

    \item \textit{Exponential Translation Mechanism:} To map these regularized rates back to the unregularized MDP, we exploit the \textit{Minimal Action Gap}, a structural condition dictating a positive margin between the optimal and suboptimal unregularized action values. Under this condition, we prove the unregularized suboptimality is bounded by the regularized gap plus an \textit{exponentially} small translation tail. By tuning the fixed temperature to balance this exponential tail against the regularized optimization error, our algorithm achieves unregularized rates that  match the regularized rates up to logarithmic factors.

    \item \textit{Warm-Start Unification of Single- and Double-Loop Architectures:} To connect our single-loop algorithm with conventional double-loop implementations, we extend the PMD analysis to a warm-start architecture that performs $N$ critic SGD steps per actor update while initializing each inner loop from the previous critic output. For average-iterate convergence, the accelerated $\tilde{\mathcal{O}}(T_{total}^{-2/3})$ unregularized rate remains invariant across $N\in[1,T_{total}^{2/9}]$. Under the standard linear entropy penalty, the baseline $\tilde{\mathcal{O}}(T_{total}^{-1/4})$ rate remains invariant across the larger window $N\in[1,\sqrt{T_{total}}]$, directly connecting the single-loop and balanced double-loop endpoints.

    \item \textit{Application to Tabular Setting:} We apply our function approximation theory to the tabular setting equipped with a one-hot feature encoding. By employing an exploratory restart distribution ($\nu$-sampling) to guarantee sufficient state-action coverage during critic training, our single-loop, entropy-regularized algorithm achieves a $\tilde{\mathcal{O}}(T_{total}^{-2/3})$ average-iterate unregularized rate under the Minimal Action Gap. This instance-dependent rate surpasses the $\mathcal{O}(T_{total}^{-1/2})$ worst-case statistical barrier governing tabular MDPs without a positive action margin.
\end{itemize}

\textbf{Related Works.}
The literature on policy gradient methods is extensive.
We refer to Appendix~\ref{app:related-works} for extended discussions of the four analyses of sample-based algorithms stated in Table~\ref{tab:sample-complexity-comparison} \citep{agarwal2021theory, liu2020improved, yuan2023linear, cayci2024convergence}.
Here, we focus on the non-asymptotic literature concerning sample-based Natural Actor-Critic (NAC) algorithms. A significant portion of these works explicitly address Markovian sampling. While our current analysis utilizes conditionally independent sampling (with an exploratory $\nu$-sampling extension), we leave the extension of our theory to Markovian trajectories for future work.

In the tabular setting, \citet{khodadadian2021finite} and \citet{khodadadian2022finite} established finite-sample guarantees for double-loop and single-loop (two-timescale) NAC algorithms, respectively. Extending to linear function approximation, \citet{xu2020improving} and \citet{chen2022finite} analyzed double-loop NAC algorithms under Markovian sampling.
Recently, \citet{wang2024nonasymptotic} analyzed a single-loop NAC algorithm for average-reward MDPs, achieving an $\tilde{\mathcal{O}}(T_{total}^{-1/3})$ unregularized rate under Markovian sampling, similar to that of \citet{liu2020improved} under conditionally independent sampling. Like \citet{liu2020improved}, this analysis operates within Euclidean parameter spaces, relying on the assumption that the Fisher information matrix remains uniformly positive-definite. While appropriate within our Stochastic Regime, this approach would break down in the Deterministic Regime, where the Fisher information degenerates as the training policy approaches determinism.

\textbf{Notation.}
For a vector $x$, we use $\|x\|_1$, $\|x\|_2$, and $\|x\|_\infty$ to denote the standard $L_1$, Euclidean, and $L_\infty$ norms, respectively. For a matrix $A$, $\|A\|_2$ denotes its spectral norm. For a measurable function $f$, $\|f\|_\infty$ denotes its essential supremum (or maximum when the underlying space is finite).
We use $\mathcal{O}(\cdot)$ and $\Theta(\cdot)$ for standard asymptotic upper and matching-order bounds, respectively, and use $\tilde{\mathcal{O}}(\cdot)$ and $\tilde{\Theta}(\cdot)$ to additionally suppress logarithmic factors in the relevant asymptotic variables, such as $T$, $T_{total}$, and, when applicable, $\epsilon_{app}^{-1}$. Asymptotic notation may suppress fixed problem-dependent constants, while dependence displayed explicitly, such as those on $(1-\gamma)^{-1}$, $\kappa$, $\lambda$, and $\Delta$, is retained. All logarithms are natural logarithms.

\section{Problem Formulation} \label{sec:setup}

We consider an infinite-horizon discounted Markov Decision Process (MDP) defined by the tuple $(\mathcal{S}, \mathcal{A}, \mathcal{P}, r, \gamma, \mu)$, where $\mathcal{S}$ is the state space (allowed to be infinite), $\mathcal{A}$ is the action space (assumed to be finite of size $|\mathcal{A}|$), $\mathcal{P}(\cdot|s,a)$ is the transition probability measure over the next state given current $(s,a)$, $r(s,a) \in [0, 1]$ is the reward function, $\gamma \in (0,1)$ is the discount factor, and $\mu$ is the initial state distribution. A policy $\pi(a|s)$ defines the probability of selecting action $a$ in state $s$.

Let $(s_t, a_t)$ denote the state-action pair at time step $t \ge 0$.
Throughout the paper, we use $\mathbb{P}$ and $\mathbb{E}$ to denote probability and expectation. To avoid ambiguity, we subscript the expectation operator to indicate the random variables and their associated distributions (e.g., expectation over a policy trajectory $\mathbb{E}_{\pi}$, or over specific state-action marginals $\mathbb{E}_{s \sim d^\pi}$ or $\mathbb{E}_{a \sim \pi}$).
The discounted state visitation measure for a policy $\pi$ is denoted by $d^\pi$, defined such that for any measurable subset $\tilde{\mathcal{S}} \subset \mathcal{S}$, $d^\pi(\tilde{\mathcal{S}}) = (1-\gamma) \sum_{t=0}^\infty \gamma^t \mathbb{P}(s_t \in \tilde{\mathcal{S}} \mid s_0 \sim \mu, \pi)$. For convenience, we use $d^\pi(s)$ to denote its density (or mass in the discrete case). The dependence on $\mu$ is suppressed in the notation.

\subsection{Setup and Optimization Regimes} \label{sec:reg-objective}

The entropy-regularized objective from the initial state distribution $\mu$ is \citep{neu2017unified, geist2019theory}
\begin{align*}
    J_\tau(\pi) &= \mathbb{E}_{\pi, s_0 \sim \mu} \left[ \sum_{t=0}^\infty \gamma^t \big( r(s_t, a_t) - \tau \log \pi(a_t|s_t) \big) \right] \\
    &= \frac{1}{1-\gamma} \mathbb{E}_{s \sim d^\pi} \Big[ \mathbb{E}_{a \sim \pi}[r(s,a)] + \tau H(\pi(\cdot|s)) \Big],
\end{align*}
where $\tau > 0$ is the entropy temperature, and $H(\pi(\cdot|s)) \triangleq -\sum_{a} \pi(a|s) \log \pi(a|s)$ is the Shannon entropy.
Since we are primarily interested in the limits of entropy-regularized algorithms as the temperature decreases ($\tau \to 0$), we consider the temperature to be bounded by a base constant throughout the paper, i.e., $\tau \le \tau_{max}$ for some $\tau_{max} > 0$.

The regularized state-value function includes the instantaneous entropy penalty at all steps:
$$V_\tau^\pi(s) = \mathbb{E}_{\pi} \left[ \sum_{t=0}^\infty \gamma^t \big( r(s_t, a_t) - \tau \log \pi(a_t|s_t) \big) \;\middle|\;
s_0 = s \right].$$
The action-value function evaluates the expected return after taking a specific initial action $a$.
Therefore, it excludes the entropy penalty for that first fixed action:
$$Q_\tau^\pi(s,a) = r(s,a) + \mathbb{E}_{\pi} \left[ \sum_{t=1}^\infty \gamma^t \big( r(s_t, a_t) - \tau \log \pi(a_t|s_t) \big) \;\middle|\;
s_0 = s, a_0 = a \right].$$
These definitions satisfy the recursive soft Bellman relations:
\begin{align*}
V_\tau^\pi(s) &= \sum_a \pi(a|s) \big( Q_\tau^\pi(s,a) - \tau \log \pi(a|s) \big), \\
Q_\tau^\pi(s,a) &= r(s,a) + \gamma \mathbb{E}_{s' \sim \mathcal{P}(\cdot|s,a)}[V_\tau^\pi(s')].
\end{align*}
The regularized advantage function is defined as $A_\tau^\pi(s,a) = Q_\tau^\pi(s,a) - \tau \log\pi(a|s) - V_\tau^\pi(s)$.

We employ a log-linear softmax policy parameterized by $\omega \in \mathbb{R}^d$, operating over a given feature map $\phi(s,a) \in \mathbb{R}^d$:
$$\pi_\omega(a|s) = \frac{\exp(\omega^\top \phi(s,a))}{\sum_{a' \in \mathcal{A}} \exp(\omega^\top \phi(s,a'))}.$$
This can be written as
$\pi_\omega(a|s) = \exp(\omega^\top \phi(s,a))/ Z_\omega(s)$ using the normalizing constant $Z_\omega(s) \triangleq \sum_{a'} \exp(\omega^\top \phi(s,a'))$.
We assume that, for each $\tau\in(0,\tau_{max}]$, the parametric regularized objective admits a finite maximizer
$\omega^*_\tau\in\mathbb{R}^d$. Let $\omega^*_\tau\in\operatorname*{argmax}_{\omega\in\mathbb{R}^d}J_\tau(\pi_\omega)$
denote such an optimal parameter vector, and let $\pi_{\omega^*_\tau}$ be the corresponding regularized optimal policy within the parametric family.
By abuse of notation, we also use $J_\tau(\omega) \triangleq J_\tau(\pi_\omega)$.
For a training policy $\pi_t$, we define the parametric regularized performance gap as $\text{Gap}_t \triangleq J_\tau(\pi_{\omega^*_\tau}) - J_\tau(\pi_t)$.
We denote the discounted state visitation distribution of the regularized parametric optimum $\pi_{\omega^*_\tau}$ as $d^{\omega^*_\tau} \triangleq d^{\pi_{\omega^*_\tau}}$.

Furthermore, for any $\tau > 0$, let $\pi^*_\tau$ denote the unique globally optimal policy for the regularized MDP over the entire probability simplex, defined as the soft Bellman-optimal policy that maximizes the regularized state-value $V_\tau^\pi(s)$ for every state $s$. Thus $\pi^*_\tau$ maximizes the scalar objective $J_\tau(\pi)$ for any initial distribution, including $\mu$.
By definition, the performance of the parametric optimum is bounded by the global optimum, such that $J_\tau(\pi_{\omega^*_\tau}) \le J_\tau(\pi^*_\tau)$.
We define the global regularized performance gap as $\text{Gap}_t^\dag \triangleq J_\tau(\pi^*_\tau) - J_\tau(\pi_t)$, distinct from $\text{Gap}_t$.
We denote the discounted state visitation distribution of the regularized globally optimal policy $\pi^*_\tau$ as $d^*_\tau \triangleq d^{\pi^*_\tau}$,
and denote the state-value and action-value functions of $\pi^*_\tau$ as $V^*_\tau \triangleq V^{\pi^*_\tau}_\tau$ and $Q^*_\tau \triangleq Q^{\pi^*_\tau}_\tau$.

We aim to analyze the performance of entropy-regularized policy optimization in maximizing the unregularized objective while tuning the temperature. To formalize the limit as the temperature drops ($\tau \to 0$), let $\pi^*_0$ denote a globally optimal policy for the unregularized MDP, defined as maximizing the unregularized state-value $V_0^\pi(s)$ for every state $s$. By standard MDP theory, such a statewise global optimum always exists and can be chosen to be deterministic.
We denote the discounted state visitation distribution of $\pi^*_0$ as $d^*_0 \triangleq d^{\pi^*_0}$,
and denote the state-value and action-value functions of $\pi^*_0$ as $V^*_0 \triangleq V^{\pi^*_0}_0$ and $Q^*_0 \triangleq Q^{\pi^*_0}_0$.

We primarily evaluate algorithmic performance against the global optimum, defining the global unregularized performance gap as $J_0(\pi^*_0) - J_0(\pi_t)$. To facilitate theoretical analysis, we assume a Minimal Action Gap on the unregularized advantage function for $\pi^*_0$ (Assumption~\ref{ass:action-gap}). From this structural property, we derive an Exponential Translation Bound (Theorem~\ref{thm:suboptimal-mass}) to control the unregularized metric $J_0(\pi^*_0) - J_0(\pi_t)$ via the regularized metric $\text{Gap}_t^\dag$.

Depending on the asymptotic behavior of the regularized optimal policies as $\tau \to 0$, we focus on two primary regimes for the optimization landscape:
\begin{itemize}
    \item \textit{Stochastic Regime:} As $\tau \to 0$, the sequence of regularized optimal policies $\pi_{\omega^*_\tau}$ converges to an unregularized parametric optimum $\pi_{\omega^*_0}$ with $\omega^*_\tau \to \omega^*_0 \in \mathbb{R}^d$.
    Consequently, under Bounded Features (Assumption~\ref{ass:bounded-features}), for sufficiently small $\tau$, the distributions $\pi_{\omega^*_\tau}(\cdot|s)$ remain uniformly interior to the probability simplex for all states $s$, so that the action randomness does not vanish. To capture this persistent stochasticity and rule out redundant or action-independent feature directions, we assume that the centered feature covariance matrix (i.e., Fisher information matrix) at the parametric optimum $\pi_{\omega^*_\tau}$ remains uniformly positive-definite, independent of temperature $\tau$ (Assumption~\ref{ass:pd-fim}). In this regime, we also define the parametric unregularized performance gap as $J_0(\pi_{\omega^*_0}) - J_0(\pi_t)$, distinct from $J_0(\pi^*_0) - J_0(\pi_t)$.

    \item \textit{Deterministic Regime:} As $\tau \to 0$, the sequence of regularized optimal policies $\pi_{\omega^*_\tau}$ approaches an unregularized parametric optimum
    that is a deterministic policy in the closure of the log-linear family. This deterministic collapse causes the action randomness to vanish and the Fisher information matrix to degenerate. A highly expressive log-linear family that perfectly captures the deterministic global optimum $\pi^*_0$ naturally falls into this regime. However, deterministic collapse can also occur in a restricted family if its best unregularized policy in the closure is a suboptimal deterministic policy.
\end{itemize}

We analyze the two regimes sequentially in Parts I and II, both evaluating convergence against the global optimum $\pi^*_0$, based on the Minimal Action Gap assumption.
In Part I, we tackle the Stochastic Regime by establishing a Parameter-Space Polyak-\L{}ojasiewicz (PL) Condition (Lemma~\ref{lem:pl-condition})
and deriving an Actor Progress Bound (Lemma~\ref{lem:pl-ascent}) through the smoothness of the regularized objective and a Positive-Definite Fisher Information assumption.
In Part II, we address the Deterministic Regime, where the degeneration of the Fisher information matrix breaks the optimization geometry required for the Actor Progress Bound. To circumvent this, we pivot to a Policy Mirror Descent (PMD) framework, measuring optimization progress via Kullback-Leibler (KL) divergence (Lemma~\ref{lem:pmd-progress}) and establishing a PMD Actor Progress Bound (Lemma~\ref{lem:pmd-telescoping-l2}).
Alongside these differences in measuring actor progress, the critic tracking is also analyzed in different manners: Part I couples the critic's target drift to the objective descent via the Euclidean parameter geometry, whereas Part II bounds the target drift using a decoupled, worst-case parameter step.
Finally, complementing our primary analysis, Appendices~\ref{app:linear-entropy} and \ref{app:warm-start-linear-entropy} evaluate the unregularized performance gap in, respectively, the Stochastic and Deterministic Regimes using a standard linear entropy penalty. This alternative approach completely bypasses the Minimal Action Gap assumption, providing robust baseline convergence rates.

By the boundedness of the Shannon entropy $H(\pi(\cdot|s))$, we state as a lemma the fact that the regularized value functions are unconditionally bounded for any policy.

\begin{lem}[Bounded Regularized Values] \label{lem:bounded-values}
For any policy $\pi$, any state $s$, and any action $a$, the true regularized state-value and action-value are bounded:
$$0 \le V_\tau^\pi(s) \le \frac{1 + \tau \log |\mathcal{A}|}{1-\gamma}, \quad 0 \le Q_\tau^\pi(s,a) \le \frac{1 + \tau \log |\mathcal{A}|}{1-\gamma}.$$
Consequently for any temperature $\tau \in (0, \tau_{max}]$, we have $V_\tau^\pi(s) \le V_{max}$ and $Q_\tau^\pi(s,a) \le V_{max}$ globally,
where $V_{max} \triangleq \frac{1 + \tau_{max} \log |\mathcal{A}|}{1-\gamma}$.
\end{lem}

\subsection{Uncentered Natural Actor-Critic Algorithm}

We consider an entropy-regularized Natural Policy Gradient (NPG) algorithm with compatible function approximation in an actor-critic framework.
An explicit policy (the \textit{actor}) is maintained to select actions, while a value function estimator (the \textit{critic}) is learned to evaluate the natural policy gradient for guiding the policy update.
For a training policy $\pi_t \triangleq \pi_{\omega_t}$ parameterized by the actor parameter $\omega_t$, the ideal critic parameter $\theta^*_\tau(\omega_t) \in \mathbb{R}^d$ is defined as the minimizer of the Mean Squared Error (MSE) objective in the regularized $Q$-value projection:
$$\theta^*_\tau(\omega_t) \triangleq \arg\min_{\theta \in \mathbb{R}^d} \mathbb{E}_{s \sim d^{\pi_t}, a \sim \pi_t} \left[ \big( Q_\tau^{\pi_t}(s,a) - \theta^\top \phi(s,a) \big)^2 \right].$$
The inherent approximation error is $\epsilon_t(s,a) \triangleq Q_\tau^{\pi_t}(s,a) - \theta^*_\tau(\omega_t)^\top \phi(s,a)$.
For the entropy-regularized objective, the ideal uncentered NPG update incorporates the entropy penalty directly into the policy parameter space:
$$\omega_{t+1} = \omega_t + \eta_t \big( \theta^*_\tau(\omega_t) - \tau \omega_t \big),$$
where $\eta_t > 0$ is the actor step size. This formulation adapts standard compatible NPG \citep{kakade2001natural, peters2008natural} to use uncentered $Q$-values and features. Under exact $Q$-function realizability, this coincides with the standard entropy-regularized NPG direction. Otherwise, it incurs an approximation bias quantified later (Lemma~\ref{lem:uncentered-gradient}). See Appendix~\ref{app:critic-comparison} for a detailed derivation and comparison.
The critic parameter $\theta^*_\tau(\omega_t)$ can be viewed to both parameterize $Q_\tau^{\pi_t}(\cdot)$ and define the ideal actor update.


In practice, we concurrently train an actual critic parameter to track the ideal target $\theta^*_\tau(\omega_t)$ using stochastic samples.
Specifically, we define the filtration $\mathcal{F}_t$ to include the algorithmic trajectory up to the parameters $(\theta_t, \omega_t)$.
At each iteration $t$, conditionally on $\mathcal{F}_t$, we draw a sample $(s_t, a_t) \sim d^{\pi_t} \times \pi_t$ (referred to as \textit{conditionally independent sampling}) and obtain an unbiased estimate $\hat{Q}_t$ of the regularized $Q$-value such that $\mathbb{E}[\hat{Q}_t \mid \mathcal{F}_t, s_t, a_t] = Q_\tau^{\pi_t}(s_t, a_t)$.
We define the stochastic gradient of the critic update (corresponding to the gradient of half the MSE objective) as:
\begin{align}
g_t^{cr}(s_t, a_t) = \big( \theta_t^\top \phi(s_t, a_t) - \hat{Q}_t \big) \phi(s_t, a_t). \label{eq:grad-cr}
\end{align}
Using this sample gradient, the critic parameter is updated to $\theta_{t+1}$ via stochastic gradient descent (SGD), and the resulting critic tracking error is evaluated as $e_t \triangleq \theta_{t+1} - \theta^*_\tau(\omega_t)$.
To ensure the unbiasedness of $\hat{Q}_t$, geometrically stopped Monte Carlo rollouts can be used as described in \citet[Algorithm~3]{yuan2023linear}. Because such a rollout simulates a trajectory with multiple raw transitions until termination, translating the number of SGD updates into raw environment interactions incurs an additional $\mathcal{O}((1-\gamma)^{-1})$ expected-horizon factor, which, for simplicity, is suppressed in our sample complexity discussion.

Alternatively, as in Temporal Difference (TD) learning \citep{sutton2018reinforcement,szepesvari2010algorithms}, $\hat{Q}_t$ can be constructed using bootstrapping, which incurs an approximation bias.
The update of $\theta_{t+1}$ using these two choices for $\hat{Q}_t$ corresponds directly to the standard Monte Carlo and TD methods for action-value evaluation.
While our theoretical analysis assumes unbiased Monte Carlo estimates to isolate the core dynamics, we leave the study of bootstrapping estimates to future work.

We employ an alternating update scheme: the actor parameter is updated to $\omega_{t+1}$ using the freshly updated critic $\theta_{t+1}$. Accordingly, we define the effective gradient for the actor update as:
\begin{align}
g_t^{ac} \triangleq \theta_{t+1} - \tau \omega_t. \label{eq:grad-ac}
\end{align}
While early actor-critic theory often employs a simultaneous update scheme---where the actor evaluates its update using the concurrent critic $\theta_t$ \citep[e.g.,][]{konda2000actor, bhatnagar2009natural, wu2020finite}---our alternating update reflects the sequential execution flow of modern practical implementations. Nevertheless, we note that our theoretical analysis and convergence guarantees can be readily applied to the simultaneous update algorithm.

\begin{algorithm}[t]
\caption{Single-Loop, Entropy-Regularized, Uncentered Natural Actor-Critic}
\label{alg:uncentered-nac}
\begin{algorithmic}[1]
\REQUIRE Initial actor parameter $\omega_0$ satisfying $\|\omega_0\|_2 \le \frac{B_\theta}{\tau_{max}}$, initial critic parameter $\theta_0$ satisfying $\|\theta_0\|_2\le B_\theta$, temperature $\tau \in (0,\tau_{max}]$, critic projection radius $R_\theta \, (=B_\theta)$, step-size schedules $\{\eta_t\}_{t \ge 0}$ and $\{\alpha_t\}_{t \ge 0}$.
\FOR{$t = 0, 1, \dots, T-1$}
    \STATE \textbf{Sampling:} Draw a state-action pair $(s_t, a_t) \sim d^{\pi_t} \times \pi_t$ under the current policy $\pi_t \triangleq \pi_{\omega_t}$.
    \STATE \textbf{Evaluation:} Obtain an unbiased estimate $\hat{Q}_t$ of the regularized action-value $Q_\tau^{\pi_t}(s_t, a_t)$.
    \STATE \textbf{Critic Update:} Compute the stochastic critic gradient \eqref{eq:grad-cr} and update via constrained SGD:
    $$\theta_{t+1} = \Pi_{R_\theta} \Big[ \theta_t - \alpha_t g_t^{cr}(s_t, a_t) \Big].$$ \vspace*{-.3in}
    \STATE \textbf{Actor Update:} Update the actor parameter using the effective actor gradient \eqref{eq:grad-ac}: 
    $$\omega_{t+1}  = \omega_t + \eta_t g_t^{ac} = \omega_t + \eta_t \big( \theta_{t+1} - \tau \omega_t \big)$$  \vspace*{-.4in}
\ENDFOR
\ENSURE The output policy sequence $\{\pi_t\}_{t=0}^T$.
\end{algorithmic}
\end{algorithm}

In summary, the entropy-regularized, uncentered Natural Actor-Critic (NAC) algorithm executes these coupled updates iteratively as a \textit{single-loop} method (Algorithm~\ref{alg:uncentered-nac}). Instead of employing a double-loop architecture that freezes the actor to perform multiple critic updates, we continuously interleave a single critic update with a single actor update at every iteration. To guarantee algorithmic stability, the critic is constrained to a ball of radius $R_\theta$ (taking $R_\theta = B_\theta$ for the structural constant $B_\theta$ in Lemma~\ref{lem:structural-bounds} later). We reserve the unsubscripted expectation $\mathbb{E}[\cdot]$ to denote the expectation over the randomness of the training algorithm (i.e., the dynamic sampling process).

We emphasize two core aspects of our algorithmic design: (i) using uncentered $Q$-values and uncentered features, and (ii) realizing the NPG update directly via the action-value parameter iterate $\theta_{t+1}$, instead of solving the least-squares projection in a double-loop architecture.
While the use of $Q$-values in NPG (often referred to as Q-NPG) is considered in prior literature \citep{agarwal2021theory, yuan2023linear}, such theoretical analyses generally do not study estimating the NPG via the action-value parameter iterates in a single-loop scheme.
We refer readers to Appendix~\ref{app:critic-comparison} for a further comparison between this uncentered algorithmic choice and standard centered NPG algorithms with compatible function approximation.

In Appendix~\ref{app:nu_sampling}, we extend our theoretical analysis to accommodate exploratory $\nu$-sampling by introducing generalized off-policy mismatch bounds. While the main text focuses on on-policy sampling from the objective's initial distribution $\mu$, employing a restart distribution $\nu$ is a standard mechanism to improve the state-action coverage for training the critic \citep{agarwal2021theory, yuan2023linear}. Subtle differences emerge across our Part I and Part II analyses concerning how the concentrability and mismatch bounds are invoked. Nevertheless, subject to these generalized assumptions, the asymptotic convergence rates remain identical to those in the main text.

\subsection{Structural Assumptions} \label{sec:struc-assumptions}

Before detailing the specific conditions, we make a general note on the structural assumptions used throughout this paper. To accommodate our goal of establishing unregularized convergence as the temperature drops ($\tau \to 0$), we formulate several assumptions (Assumptions~\ref{ass:bounded-q}--\ref{ass:approximation-error} in this section, and Assumptions~\ref{ass:pd-fim}--\ref{ass:optimal-uncentered} later) requiring their respective structural constants to be independent of $\tau \in (0, \tau_{max}]$. However, for the regularized convergence with a fixed temperature $\tau > 0$, our theoretical results remain valid even if these involved constants exhibit a $\tau$-dependence.

For analytical simplicity, to control the magnitudes of the actor and critic updates, we first assume the feature representations of the state-action pairs are bounded.

\begin{ass}[Bounded Features] \label{ass:bounded-features}
$\sup_{s,a} \|\phi(s,a)\|_2 \le B_\phi$.
\end{ass}

Algorithm~\ref{alg:uncentered-nac} relies on stochastic estimates of the regularized action-values. As shown in Lemma~\ref{lem:bounded-values}, the true regularized action-values $Q_\tau^{\pi_t}(s,a)$ are bounded by the constant $V_{max}$ which is independent of $\tau \in (0, \tau_{max}]$.
For analytical simplicity, we assume a similar uniform, $\tau$-independent bound on the second moment of the unbiased estimate $\hat{Q}_t$.

\begin{ass}[Bounded Q-value Second Moment] \label{ass:bounded-q}
There exists a constant $Q_{max} > 0$, independent of the temperature $\tau \in (0, \tau_{max}]$, such that $\mathbb{E}[\hat{Q}_t^2 \mid \mathcal{F}_t, s_t, a_t] \le Q_{max}^2$ almost surely for all iterations $t \ge 0$ under Algorithm~\ref{alg:uncentered-nac}. By the Tower Property, this directly implies $\mathbb{E}[\hat{Q}_t^2 \mid \mathcal{F}_t] \le Q_{max}^2$. Consistent with the true regularized $Q$-values (Lemma~\ref{lem:bounded-values}), we assume $Q_{max} = \mathcal{O}((1-\gamma)^{-1})$.
\end{ass}

To ensure the critic's tracking error contracts stably across both the stochastic and deterministic regimes, we require the uncentered feature moment matrix to remain uniformly positive-definite for all training policies. For any policy $\pi$, let $\bar{\Sigma}_{unc}(\pi) \triangleq \mathbb{E}_{s \sim d^\pi, a \sim \pi}[\phi(s,a) \phi(s,a)^\top]$.

\begin{ass}[Positive-Definite Uncentered Feature Moment] \label{ass:pd-uncentered}
There exists a constant $\kappa > 0$, independent of the temperature $\tau \in (0, \tau_{max}]$, such that for all training policies $\{\pi_t\}_{t \ge 0}$ under Algorithm~\ref{alg:uncentered-nac}:
$$ \bar{\Sigma}_{unc}(\pi_t) \succeq \kappa I, $$
where $I$ is the identity matrix, and $A \succeq B$ denotes that $A - B$ is a positive semi-definite matrix.
\end{ass}

By combining Assumptions~\ref{ass:bounded-features}, \ref{ass:bounded-q}, \ref{ass:pd-uncentered}, and the bounded regularized values (Lemma~\ref{lem:bounded-values}), we establish temperature-independent bounds on the ideal uncentered critic, the actor parameter iterates, and the actor and critic update directions.

\begin{lem}[Structural Bounds of the Algorithm] \label{lem:structural-bounds}
(i) Under Assumptions~\ref{ass:bounded-features} and \ref{ass:pd-uncentered}, the ideal uncentered critic is bounded for all training policies by the $\tau$-independent constant $B_\theta$:
$$ \|\theta^*_\tau(\omega_t)\|_2 \le \frac{B_\phi V_{max}}{\kappa} \triangleq B_\theta. $$
Additionally under Assumption~\ref{ass:bounded-q}, the critic update direction is bounded by a temperature-independent constant $G_{cr} > 0$:
\begin{align*}
\mathbb{E}\big[\|g_t^{cr}(s_t,a_t)\|_2^2 \mid \mathcal{F}_t\big] & \le 2 B_\phi^2 \big( B_\theta^2 B_\phi^2 + Q_{max}^2 \big) \triangleq G_{cr}^2.
\end{align*}
(ii) Under Assumptions~\ref{ass:bounded-features} and \ref{ass:pd-uncentered}, by setting the critic projection radius to $R_\theta = B_\theta$, if the actor step size satisfies $\eta_t \le 1/\tau$ and the actor is initialized such that $\|\omega_0\|_2 \le \frac{B_\theta}{\tau_{max}}$ (for example, $\omega_0=0$), then the actor iterates $\omega_t$ are bounded for all $t \ge 0$ by:
$$ \|\omega_t\|_2 \le \frac{B_\theta}{\tau}. $$
Moreover, the actor update direction is bounded by a temperature-independent constant $G_{ac} > 0$:
\begin{align*}
\|g_t^{ac}\|_2 &\le 2 B_\theta \triangleq G_{ac}.
\end{align*}
\end{lem}

The following lemma establishes the Lipschitz continuity of the ideal uncentered critic mapping $\theta^*_\tau(\omega)$, subject to the positive-definiteness of the uncentered feature moment matrix at the associated parameter $\omega$. Consequently, by Assumption~\ref{ass:pd-uncentered}, this lemma can be applied to the parameters of any two training policies generated by Algorithm~\ref{alg:uncentered-nac}.
We explicitly define a Lipschitz constant through the maximum temperature $\tau_{max}$ to ensure asymptotic stability as $\tau \to 0$.

\begin{lem}[Lipschitz Ideal Uncentered Critic] \label{lem:lipschitz-critic}
For any $\kappa > 0$, define the parameter set $\Omega_\kappa \triangleq \{\omega \in \mathbb{R}^d : \bar{\Sigma}_{unc}(\pi_\omega) \succeq \kappa I\}$.
Under Assumption~\ref{ass:bounded-features}, the ideal uncentered critic mapping $\theta^*_\tau(\omega)$ is $L_{\theta,\tau}$-Lipschitz continuous over $\Omega_\kappa$. Specifically, the Lipschitz constant is bounded by:
$$L_{\theta,\tau} \triangleq \frac{2 B_\phi^2 (1 + \tau \log |\mathcal{A}|)}{\kappa (1-\gamma)^2} \left( 3 + \frac{B_\phi^2}{\kappa} \right),$$
such that for any $\omega, \omega' \in \Omega_\kappa$, $\|\theta^*_\tau(\omega) - \theta^*_\tau(\omega')\|_2 \le L_{\theta,\tau} \|\omega - \omega'\|_2$. Consequently, for any temperature $\tau \in (0, \tau_{max}]$, the mapping $\theta^*_\tau(\omega)$ is $L_\theta$-Lipschitz continuous over $\Omega_\kappa$, where $L_\theta \triangleq L_{\theta,\tau_{max}}$.
\end{lem}

Under compatible function approximation, the expressivity of the ideal uncentered critic is tied to the feature representation of the log-linear policy class. We assume this shared architecture is sufficiently expressive such that the critic $L_2$ approximation error remains uniformly bounded across the sequence of training policies $\{\pi_t\}_{t \ge 0}$. For the global convergence in Part I (Section~\ref{sec:unregularized-complexity}), we additionally require this bound to hold at the regularized parametric optimum $\pi_{\omega^*_\tau}$.

\begin{ass}[Critic $L_2$ Approximation Error] \label{ass:approximation-error}
The MSE of the ideal uncentered critic is bounded by $\epsilon_{app}$ for all training policies $\{\pi_t\}_{t \ge 0}$ under Algorithm~\ref{alg:uncentered-nac}:
$$ \mathbb{E}_{d^{\pi_t}, \pi_t} [ \epsilon_t(s,a)^2 ] \le \epsilon_{app}. $$
Additionally, for the global analysis in Part I (specifically, Lemma~\ref{lem:global-class-error}), we assume this bound holds at the regularized parametric optimum:
$$ \mathbb{E}_{d^{\omega^*_\tau}, \pi_{\omega^*_\tau}} [ \epsilon_*(s,a)^2 ] \le \epsilon_{app}, $$
where $\epsilon_*(s,a) \triangleq Q_\tau^{\pi_{\omega^*_\tau}}(s,a) - \theta^*_\tau(\omega^*_\tau)^\top \phi(s,a)$.
For analytical simplicity, we assume the bound $\epsilon_{app}$ is independent of the temperature $\tau \in (0, \tau_{max}]$.
\end{ass}

\begin{rem}[$L_\infty$ Approximation and Concentrability]
Because the sample-based critic in Algorithm~\ref{alg:uncentered-nac} is trained via stochastic gradient descent, the $L_2$ error ($\epsilon_{app}$) is the native and practically aligned metric for our analysis. Our theoretical guarantees rely on density ratio bounds (Assumptions~\ref{ass:parametric-density} and \ref{ass:density-ratios}) to manage various distribution shifts, including shifting this $L_2$ error across state measures, and to facilitate associated $\chi^2$-to-KL divergence translations. In Appendix~\ref{app:linfty-bounds}, we demonstrate that if the function approximation error is instead uniformly bounded in the $L_\infty$ norm, then the global joint concentrability constant ($C_{joint}^\dag$) can be eliminated in the Deterministic Regime (Part II). We also isolate a fundamental geometric bottleneck that prevents the full elimination of the parametric joint concentrability constant ($C_{joint}$) in the Stochastic Regime (Part I).
\end{rem}

To facilitate our study of unregularized convergence, we introduce a structural assumption regarding the advantage function of the unregularized MDP. By basic MDP theory, the unregularized optimal value functions $V^*_0$ and $Q^*_0$ are unique. We assume there exists a positive minimal action gap separating the best action from all suboptimal actions at every state. This is a standard regularity condition in tabular MDPs but may be violated in non-tabular MDPs. Due to this positive margin, there is a unique optimal action at each state. This uniqueness facilitates the derivation of the simple policy entropy bound (Lemma~\ref{lem:entropy-bound}) used in the proof of Theorem~\ref{thm:suboptimal-mass}.

\begin{ass}[Minimal Action Gap] \label{ass:action-gap}
There exists a positive minimal action gap $\Delta > 0$ such that for all states $s \in \mathcal{S}$, there is a unique optimal action $a^*(s) = \arg\max_{a \in \mathcal{A}} Q^*_0(s, a)$, and for all suboptimal actions $a \neq a^*(s)$, the unregularized advantage is bounded from below by $\Delta$:
$$ V^*_0(s) - Q^*_0(s,a) \ge \Delta. $$
Consequently, there exists a unique unregularized optimal policy $\pi^*_0$ that maximizes the value function at every state. This policy is deterministic and satisfies $\pi^*_0(a^*(s)|s) = 1$ for all $s \in \mathcal{S}$.
\end{ass}

\section{Suboptimality Identities and Translation Bounds}  \label{sec:suboptimalty-exp-translation}

\subsection{Suboptimality Identities and Decompositions} \label{sec:suboptimality}

In this section, we derive the suboptimality identities and decompositions that govern algorithmic progress.
First, we present the standard performance difference lemma, adapted for the entropy-regularized objective
\citep{kakade2002approximately, mei2020global,cayci2024convergence}.

\begin{lem}[Regularized Performance Difference] \label{lem:pdl}
For any two policies $\pi$ and $\tilde{\pi}$, the difference in their regularized objectives is given by:
\begin{align*}
J_\tau(\tilde{\pi}) - J_\tau(\pi) &= \frac{1}{1-\gamma} \mathbb{E}_{s \sim d^{\tilde{\pi}}} \left[ \sum_{a \in \mathcal{A}} \tilde{\pi}(a|s) \big( Q_\tau^\pi(s,a) - \tau \log \tilde{\pi}(a|s) \big) - V_\tau^\pi(s) \right].
\end{align*}
\end{lem}

Next, we present a geometric identity relating the regularized performance gap of $\pi$ against the regularized global optimum $\pi^*_\tau$ to the state-averaged \textit{reverse} KL divergence $\mathbb{E}_{s \sim d^\pi} [ D_{KL}(\pi  \| \pi^*_\tau)]$, as in \citet[Lemma 26]{mei2020global}.
For any two policies $\pi$ and $\tilde{\pi}$, the KL divergence at state $s$ is defined as $D_{KL}(\tilde{\pi}(\cdot|s) \| \pi(\cdot|s)) \triangleq \sum_{a \in \mathcal{A}} \tilde{\pi}(a|s) \log \frac{\tilde{\pi}(a|s)}{\pi(a|s)}$.
Recall that $\pi^*_\tau$ is the unique soft Bellman-optimal policy for the regularized MDP. Because this identity relies on the statewise optimality of $\pi^*_\tau$, it would not in general hold if $\pi^*_\tau$ is replaced by the parametric optimum $\pi_{\omega^*_\tau}$.

\begin{lem}[Global Regularized Suboptimality Identity] \label{lem:optim-identity}
For any policy $\pi$, the performance gap between $\pi$ and the regularized global optimum $\pi^*_\tau$ is proportional to the state-averaged reverse KL divergence under the state visitation distribution $d^\pi$ of policy $\pi$:
$$ J_\tau(\pi^*_\tau) - J_\tau(\pi) = \frac{\tau}{1-\gamma} \mathbb{E}_{s \sim d^\pi} \big[ D_{KL}(\pi(\cdot|s) \| \pi^*_\tau(\cdot|s)) \big]. $$
\end{lem}

Finally, we prove useful decompositions of the regularized performance difference for a training policy $\pi_t$ against any comparator policy $\pi$, such as $\pi_{\omega^*_\tau}$ or $\pi^*_\tau$. This lemma provides two distinct algebraic separations tailored to our two optimization regimes. The Ideal Critic Decomposition isolates the ideal algorithmic progress, which we bound in Part I to establish the Parameter-Space PL Condition (Lemma~\ref{lem:pl-condition}). Conversely, the Actual Critic Decomposition isolates the actor update direction alongside the critic tracking error. In Part II, we bound this algorithmic progress via the Bregman Three-Point identity (Lemma~\ref{lem:pmd-progress}) and substitute it back into the Actual Critic Decomposition to establish the PMD Actor Progress Bound (Lemma~\ref{lem:pmd-telescoping-l2}). For notational clarity, for any state $s$, we define the inner product over the discrete action space between any two functions $f$ and $g$ as $\langle f, g \rangle_{\mathcal{A}} \triangleq \sum_{a \in \mathcal{A}} f(s,a) g(s,a)$.

\begin{lem}[Regularized Suboptimality Decompositions] \label{lem:sub-decomp}
For the log-linear policy parameterization, the regularized performance difference between a training policy $\pi_t$ and a comparator policy $\pi$ admits two exact decompositions, depending on whether it is evaluated using the ideal critic $\theta^*_\tau(\omega_t)$ or the actual critic $\theta_{t+1}$:

\noindent \textit{(i) Ideal Critic Decomposition (Part I):}
\begin{align*}
(1-\gamma) \big( J_\tau(\pi) - J_\tau(\pi_t) \big) &= \underbrace{-\tau \mathbb{E}_{d^\pi}[D_{KL}(\pi \| \pi_t)]}_{\text{Restorative Entropy Force}} + \underbrace{\mathbb{E}_{d^\pi} \Big[ \langle (\theta^*_\tau(\omega_t) - \tau \omega_t)^\top \phi, \pi - \pi_t \rangle_{\mathcal{A}} \Big]}_{\text{Ideal Algorithmic Progress}} \\
&\quad + \underbrace{\mathbb{E}_{d^\pi} \big[ \langle \epsilon_t, \pi - \pi_t \rangle_{\mathcal{A}} \big]}_{\text{Approximation Bias}}.
\end{align*}

\noindent \textit{(ii) Actual Critic Decomposition (Part II):}
\begin{align*}
(1-\gamma) \big( J_\tau(\pi) - J_\tau(\pi_t) \big) &= -\tau \mathbb{E}_{d^\pi}[D_{KL}(\pi \| \pi_t)] + \underbrace{\mathbb{E}_{d^\pi} \Big[ \langle (\theta_{t+1} - \tau \omega_t)^\top \phi, \pi - \pi_t \rangle_{\mathcal{A}} \Big]}_{\text{Algorithmic Progress}} \\
&\quad + \underbrace{\mathbb{E}_{d^\pi} \Big[ \langle -e_t^\top \phi, \pi - \pi_t \rangle_{\mathcal{A}} \Big]}_{\text{Critic Tracking Error}} + \mathbb{E}_{d^\pi} \big[ \langle \epsilon_t, \pi - \pi_t \rangle_{\mathcal{A}} \big].
\end{align*}
\end{lem}

\begin{rem}[Dual Perspectives of the Performance Difference Lemma] \label{rem:dual-pdl}
The proofs of Lemma~\ref{lem:optim-identity} and Lemma~\ref{lem:sub-decomp} highlight an interesting difference in how the Regularized Performance Difference Lemma (Lemma~\ref{lem:pdl}) is applied. In Lemma~\ref{lem:optim-identity}, we evaluate the performance gap using the advantage function of the global optimum $\pi^*_\tau$, with expectations taken over the \textit{current} policy's state visitation distribution ($d^\pi$). Because the global optimum satisfies the soft Bellman equations, this ``backward'' application collapses the advantage terms, isolating the exact reverse KL divergence $D_{KL}(\pi \| \pi^*_\tau)$. In contrast, Lemma~\ref{lem:sub-decomp} applies the lemma in the standard ``forward'' direction. It evaluates the gap using the advantage function of the \textit{current} policy $\pi$, taking expectations over the \textit{comparator} policy's state visitation distribution ($d^{\tilde{\pi}}$). This forward application exposes the inner product between the comparator's action distribution and the current policy's advantages, which naturally decomposes into the forward KL divergence and the algorithmic progress terms that drive both the Part I and Part II analyses.
\end{rem}

\subsection{Exponential Translation Bounds} \label{sec:exp-translation}

In this section, we establish a bridge connecting the global regularized progress of a training policy to its ultimate unregularized performance under the Minimal Action Gap (Assumption~\ref{ass:action-gap}). By evaluating the performance via the \textit{reverse} KL divergence (as established in Lemma~\ref{lem:optim-identity}), we bound the global unregularized performance gap using the global regularized performance gap. Crucially, this yields an \textit{exponential} translation bound that bypasses the linear $\mathcal{O}(\tau)$ entropy penalty that traditionally bottlenecks unregularized bounds.

To execute this translation, we first present a lemma which bounds the Shannon entropy of any policy as a function of the probability mass it places on the suboptimal actions.

\begin{lem}[Policy Entropy Decomposition Bound] \label{lem:entropy-bound}
For any state $s$ and policy $\pi(\cdot|s)$, let $a^*(s)$ be an arbitrary designated action, and let $q(s) \triangleq \sum_{a \neq a^*(s)} \pi(a|s)$ be the probability mass on all other actions.
The pointwise Shannon entropy of the policy is bounded by:
$$ H(\pi(\cdot|s)) \le - q(s) \log q(s) - (1-q(s)) \log(1-q(s)) + q(s) \log |\mathcal{A}|. $$
\end{lem}

Next, to facilitate application of Assumption~\ref{ass:action-gap}, we provide a lemma which establishes a uniform bound on how much the regularized and unregularized optimal value functions deviate from one another, driven purely by the maximum possible entropy over the action space.
Recall that $V^*_0$ and $Q^*_0$ denote the optimal state-value and action-value functions for the unregularized MDP, and $V^*_\tau$ and $Q^*_\tau$ denote the optimal value functions for the regularized MDP.

\begin{lem}[Uniform Regularized Value Bound] \label{lem:reg-value-bound}
The regularized optimal values bound from above the unregularized optimal values, and the uniform approximation error is bounded by the maximum discounted entropy:
$$ 0 \le V^*_\tau(s) - V^*_0(s) \le \frac{\tau \log |\mathcal{A}|}{1-\gamma}, $$
$$ 0 \le Q^*_\tau(s,a) - Q^*_0(s,a) \le \frac{\tau \log |\mathcal{A}|}{1-\gamma}, $$
for all states $s \in \mathcal{S}$ and actions $a \in \mathcal{A}$.
\end{lem}

With these tools in place, we now establish the core translation mechanism. By exploiting the Minimal Action Gap for the unregularized MDP, we prove that for sufficiently small temperatures, the global unregularized performance gap is bounded by the global regularized performance gap plus an exponentially small translation tail.

\begin{thm}[Exponential Translation of Regularized Suboptimality] \label{thm:suboptimal-mass}
Under Assumption~\ref{ass:action-gap}, if the temperature is sufficiently small such that $\tau \le \frac{\Delta}{2 \log C_\gamma}$, then the unregularized performance gap is bounded in terms of the state-averaged reverse KL divergence:
$$ J_0(\pi^*_0) - J_0(\pi_t) \le \frac{A_{max}}{1-\gamma} \left( \frac{4 \tau}{\Delta} \mathbb{E}_{d^{\pi_t}} \big[ D_{KL}(\pi_t(\cdot|s) \| \pi^*_\tau(\cdot|s)) \big] + \sqrt{C_\gamma} \exp\left(-\frac{\Delta}{2\tau}\right) \right),$$
where $A_{max} \triangleq \sup_{s,a} (V^*_0(s) - Q^*_0(s,a))$ is the maximum optimality gap, satisfying $A_{max} \le \frac{1}{1-\gamma}$,
and $C_\gamma \triangleq e |\mathcal{A}|^{\frac{3-\gamma}{1-\gamma}}$ such that $\log C_\gamma = \frac{2 \log |\mathcal{A}|}{1-\gamma} + \log |\mathcal{A}| + 1$.
Equivalently, the unregularized performance gap is bounded in terms of the regularized performance gap $\text{Gap}_t^\dag \triangleq J_\tau(\pi^*_\tau) - J_\tau(\pi_t)$:
$$ J_0(\pi^*_0) - J_0(\pi_t) \le \frac{4 A_{max}}{\Delta} \text{Gap}_t^\dag + \frac{A_{max}\sqrt{C_\gamma}}{1-\gamma} \exp\left(-\frac{\Delta}{2\tau}\right). $$
\end{thm}

Theorem~\ref{thm:suboptimal-mass} requires the temperature to fall below a specific structural threshold. To ensure this translation mechanism can be safely applied even when the temperature might be high, we extend the exponential bound to hold for all $\tau > 0$.

\begin{cor}[Universal Unregularized Suboptimality Bound] \label{cor:universal-suboptimal-mass}
Under Assumption~\ref{ass:action-gap}, the exponential translation bounds established in Theorem~\ref{thm:suboptimal-mass} can be extended to hold for all $\tau > 0$.
Specifically, with the universal tail coefficient $C_{tail} \triangleq \frac{A_{max} C_\gamma}{1-\gamma}$, the unregularized performance gap is bounded in terms of the state-averaged reverse KL divergence for all $\tau > 0$:
$$ J_0(\pi^*_0) - J_0(\pi_t) \le \frac{4 A_{max} \tau}{\Delta(1-\gamma)} \mathbb{E}_{d^{\pi_t}} \big[ D_{KL}(\pi_t(\cdot|s) \| \pi^*_\tau(\cdot|s)) \big] + C_{tail} \exp\left(-\frac{\Delta}{2\tau}\right). $$
Equivalently, it is bounded in terms of the regularized performance gap $\text{Gap}_t^\dag$ for all $\tau > 0$:
$$ J_0(\pi^*_0) - J_0(\pi_t) \le \frac{4 A_{max}}{\Delta} \text{Gap}_t^\dag + C_{tail} \exp\left(-\frac{\Delta}{2\tau}\right). $$
\end{cor}

\part{Stochastic Regime} \label{part:stochastic}

In the first part of our analysis, we investigate the convergence of Algorithm~\ref{alg:uncentered-nac} in the \textit{Stochastic Regime}. As described in Section~\ref{sec:setup}, this regime is characterized by the sequence of regularized optimal policies $\pi_{\omega^*_\tau}$ converging to an unregularized parametric optimum $\pi_{\omega^*_0}$ with $\omega^*_\tau \to \omega^*_0 \in \mathbb{R}^d$ as the temperature drops ($\tau \to 0$). This behavior can naturally arise, for example, when the restricted log-linear family lacks the expressivity to represent the deterministic global optimum $\pi^*_0$, and the best available parametric approximation leads to stochastic action selection. Under Bounded Features (Assumption~\ref{ass:bounded-features}), for sufficiently small $\tau$, the distributions $\pi_{\omega^*_\tau}$ remain uniformly interior to the probability simplex, so that the action randomness does not vanish.

We formally capture this persistent stochasticity and rule out redundant or action-independent feature directions through the expected centered feature covariance matrix, defined for any policy $\pi$ as $\bar{\Sigma}_{cen}(\pi) \triangleq \mathbb{E}_{s \sim d^\pi} \big[ \Cov_{a \sim \pi(\cdot|s)}(\phi(s,a)) \big]$. This matrix corresponds to the Fisher information matrix evaluated over $d^\pi$. The positive-definiteness of this matrix at the parametric optimum $\pi_{\omega^*_\tau}$ provides the essential curvature which, when combined with the smoothness of the regularized objective, allows us to establish the Actor Progress Bound (Lemma~\ref{lem:pl-ascent}). This progress bound serves a dual purpose in our joint convergence analysis: it is used directly to control the target drift in the Coupled Critic Tracking (Lemma~\ref{lem:critic-tracking-pl}), and, when combined with a Parameter-Space PL Condition (Lemma \ref{lem:pl-condition}), it establishes a Coupled Actor Recurrence (Lemma~\ref{lem:coupled-pl-descent}) over the performance gap.

\begin{ass}[Positive-Definite Fisher Information] \label{ass:pd-fim}
Define the temperature-dependent minimum Fisher eigenvalue at the parametric optimum as
$ \lambda_\tau \triangleq \lambda_{\min}\big(\bar{\Sigma}_{cen}(\pi_{\omega^*_\tau})\big)$,
where
$\bar{\Sigma}_{cen}(\pi_{\omega^*_\tau}) = \mathbb{E}_{s \sim d^{\omega^*_\tau}} \big[ \Cov_{a \sim \pi_{\omega^*_\tau}(\cdot|s)}(\phi(s,a)) \big]$.
There exists a constant $\lambda > 0$ such that
$\lambda_\tau \ge \lambda$ for any $\tau \in (0,\tau_{max}]$.
Equivalently, for any $\tau \in (0,\tau_{max}]$,
$$ \bar{\Sigma}_{cen}(\pi_{\omega^*_\tau})= \mathbb{E}_{s \sim d^{\omega^*_\tau}} \big[ \Cov_{a \sim \pi_{\omega^*_\tau}}(\phi(s,a)) \big] \succeq  \lambda I .$$
\end{ass}

For completeness, we note that the Fisher information $\bar{\Sigma}_{cen}(\pi^*_\tau)$ at the global optimum $\pi^*_\tau$ cannot remain uniformly positive-definite as $\tau \to 0$ under the Minimal Action Gap assumption. As $\tau\to 0$, $\pi^*_\tau$ collapses to the deterministic unregularized optimum $\pi^*_0$, causing the action randomness to vanish and the Fisher information matrix to degenerate. Thus, uniform positive-definiteness (Assumption~\ref{ass:pd-fim}) is only viable for the parametric optimum $\pi_{\omega^*_\tau}$ even in the Stochastic Regime.

To connect distributional properties of the parametric optimum $\pi_{\omega^*_\tau}$ to the trajectory of policies $\{\pi_t\}_{t \ge 0}$ generated during training, we require that, under Algorithm~\ref{alg:uncentered-nac}, sufficient coverage is maintained over the target state-action visitation measure.

\begin{ass}[Parametric Joint Concentrability] \label{ass:parametric-density}
Under Algorithm~\ref{alg:uncentered-nac}, the joint state-action density ratio of the regularized parametric optimum $\pi_{\omega^*_\tau}$ with respect to the training policy $\pi_t$ is bounded for all $t \ge 0$:
$$ \sup_{s,a} \frac{d^{\omega^*_\tau}(s) \pi_{\omega^*_\tau}(a|s)}{d^{\pi_t}(s) \pi_t(a|s)} \le C_{joint}, $$
where $C_{joint} > 0$ is a constant independent of the temperature $\tau \in (0, \tau_{max}]$.
\end{ass}

We make several comments on the two preceding assumptions. First, the joint concentrability bound is naturally implied by the conjunction of separate state and policy concentrability bounds. Specifically, if the following holds for some constants $(C,W)$ independent of $\tau \in (0, \tau_{max}]$:
\begin{itemize}
    \item \textit{Parametric State Concentrability:} $\|d^{\omega^*_\tau} / d^{\pi_t}\|_\infty \le C$,
    \item \textit{Parametric Policy Concentrability:} $\sup_{s,a} \frac{\pi_{\omega^*_\tau}(a|s)}{\pi_t(a|s)} \le W$,
\end{itemize}
then by the multiplicative rule of density ratios, the joint concentrability bound is satisfied with $C_{joint} = C W$.
Conversely, because $\mathbb{E}_{a \sim \pi_t} \big[ \frac{\pi_{\omega^*_\tau}(a|s)}{\pi_t(a|s)} \big] = 1$ at each state $s$, the joint concentrability directly implies the state concentrability bound with $C = C_{joint}$: $\|d^{\omega^*_\tau} / d^{\pi_t}\|_\infty \le C_{joint}$.
However, the joint formulation allows the local policy ratio $\frac{\pi_{\omega^*_\tau}(a|s)}{\pi_t(a|s)}$ to be unbounded so long as the associated state density ratio $\frac{d^{\omega^*_\tau}(s)}{ d^{\pi_t}(s)}$ is sufficiently small.

Second, for analysis of entropy-regularized MDPs, some prior works \citep{mei2020global, cayci2024convergence} combine the state concentrability $C$ with a minimum action probability condition: $\inf_{t,s,a} \pi_t(a|s) \ge c_\tau$ for a constant $c_\tau > 0$ depending on temperature $\tau$. This condition directly implies a policy ratio bound of $\frac{\pi_{\omega^*_\tau}(a|s)}{\pi_t(a|s)} \le \frac{1}{c_\tau}$. Similar to \citet[Lemma~3.2]{cayci2024convergence}, the actor iterate bound (Lemma~\ref{lem:structural-bounds}) leads to a $\tau$-dependent bound: $c_\tau = \frac{1}{|\mathcal{A}|} \exp\left(-\frac{2 B_\theta B_\phi}{\tau}\right)$. However, combining $C$ with $W=\frac{1}{c_\tau}$ and relying on $C_{joint} = C W$ would introduce an exponential $\tau$-dependence into our structural constants as $\tau\to 0$. Alternatively, a policy concentrability bound $W$ independent of $\tau$ could be satisfied by requiring $\inf_{t,s,a} \pi_t(a|s) \ge c$ for a constant $c > 0$ independent of $\tau$. While such a uniform lower bound on $\pi_t(a|s)$ is conceptually consistent with the persistent stochasticity that characterizes the Stochastic Regime, our joint concentrability formulation acts as a crucial relaxation, allowing even unbounded local policy ratios as $\tau \to 0$.

Third, Assumption~\ref{ass:parametric-density} is utilized in exactly two places in our Part I analysis. It is first used in the proof of the Parameter-Space PL Condition (Lemma~\ref{lem:pl-condition}) to facilitate the action-space $\chi^2$-to-KL divergence translation while simultaneously shifting the state measure. It is then used in the proof of the Actor Progress Bound (Lemma~\ref{lem:pl-ascent}) in conjunction with the Positive-Definite Fisher Information (Assumption~\ref{ass:pd-fim}). Notably, this is the \textit{only} place where Assumption~\ref{ass:pd-fim} is invoked. While Assumption~\ref{ass:pd-fim} ensures the uniform positive-definiteness of the Fisher information matrix under the parametric optimum $\pi_{\omega^*_\tau}$, the joint concentrability bound dynamically extends this geometric stability to all generated policies. As shown in Step (iii) of the proof for Lemma~\ref{lem:pl-ascent}, the centered covariance for any training policy $\pi_t$ satisfies $\bar{\Sigma}_{cen}(\pi_t) \succeq \frac{1}{C_{joint}} \bar{\Sigma}_{cen}(\pi_{\omega^*_\tau})$, which implies $ \bar{\Sigma}_{cen}(\pi_t) \succeq \frac{\lambda}{C_{joint}} I$. Note that because $\bar{\Sigma}_{unc}(\pi_t) \succeq \bar{\Sigma}_{cen}(\pi_t)$, the Positive-Definite Uncentered Feature Moment (Assumption~\ref{ass:pd-uncentered}) is then satisfied by taking$\kappa = \frac{\lambda}{C_{joint}}$ in the Stochastic Regime. Nevertheless, for clarity, we keep track of the dependence on $\kappa$ and $\lambda$ separately in Part I.

\textbf{Connections to Related Works.}
Assumptions~\ref{ass:pd-fim} and \ref{ass:parametric-density} adapt standard regularity and concentrability requirements to the parametric setting. First, a positive-definite Fisher information matrix is required in several analyses \citep{liu2020improved, wang2024nonasymptotic}. As discussed in Appendix~\ref{app:Liu-et-al}, \cite{liu2020improved} require this matrix to be uniformly positive-definite for all training policies, whereas Assumption~\ref{ass:pd-fim} localizes this requirement to the parametric optimum $\pi_{\omega^*_\tau}$ and then dynamically extends the property to all training policies via Assumption~\ref{ass:parametric-density}. Second, joint concentrability bounds are implicitly or explicitly used in several analyses \citep{liu2020improved, wang2024nonasymptotic}. For example, \cite{liu2020improved} implicitly bundle a \textit{global} counterpart of this requirement into a transferred approximation error bound evaluated at the unregularized global optimum. In contrast, our analysis utilizes the \textit{parametric} joint concentrability to transfer the on-policy approximation error to that under the parametric optimum. Finally, in the on-policy case ($\nu = \mu \times \pi_t$), the concentrability assumption in \cite{agarwal2021theory} corresponds to the unregularized counterpart of our parametric joint concentrability, when the comparator policy is set to the parametric optimum. Note that in Appendix~\ref{app:Agarwal-et-al}, this is discussed when the comparator policy is set to the global optimum, to relate their assumption to our global joint concentrability (Assumption~\ref{ass:density-ratios}).

Finally, we point out an intrinsic constraint in the Stochastic Regime under the stated assumptions, particularly the conjunction of the Positive-Definite Fisher Information and the Minimal Action Gap. As $\tau\to0$, the Minimal Action Gap drives the regularized global optimum $\pi^*_\tau$ toward the unregularized deterministic optimum $\pi^*_0$, whereas the Fisher lower bound requires the regularized parametric optimum $\pi_{\omega^*_\tau}$ to retain nonvanishing action randomness. At the same time, the approximation-error assumption, together with the Fisher lower bound and the global-to-parametric joint concentrability, controls the regularized gap between $\pi^*_\tau$ and $\pi_{\omega^*_\tau}$ by a term proportional to $\epsilon_{app}/\tau$ (Lemma~\ref{lem:global-class-error}). These properties jointly imply that $\epsilon_{app}$ cannot be arbitrarily small independently of $\tau$ under the Part-I assumptions. This is formalized in Lemma~\ref{lem:part-I-compatibility},

\section{Parameter-Space Polyak-\L{}ojasiewicz Condition} \label{sec:pl-geometry}

In this section, we establish a fundamental geometric property that contributes to algorithmic convergence in the Stochastic Regime: the parameter-space Polyak-\L{}ojasiewicz (PL) condition. Unlike standard PL conditions that bound suboptimality via objective gradients, our formulation operates within the parameter space. It bounds the parametric regularized performance gap ($\text{Gap}_t$) via the Euclidean distance between the temperature-scaled actor parameter ($\tau \omega_t$) and the ideal uncentered critic target ($\theta^*_\tau(\omega_t)$).

We first introduce a general statistical inequality. Under bounded density ratios as in Assumption~\ref{ass:parametric-density}, the following lemma bounds the $\chi^2$ divergence by the KL divergence, providing the central mechanism to control action-space distribution shifts.

\begin{lem}[$\chi^2$ to KL Divergence Bound] \label{lem:chi2-KL}
For any two probability measures $P$ and $Q$ over a probability space $\mathcal{X}$, if the density ratio is bounded such that $\sup_{x \in \mathcal{X}} \frac{\dif P}{\dif Q}(x) \le M$, then:
$$\chi^2(P \| Q)\le 2 M D_{KL}(P \| Q),$$
where  $\chi^2(P \| Q) \triangleq \int_{\mathcal{X}} \left( \frac{\dif P}{\dif Q} - 1 \right)^2 \dif Q $
and $ D_{KL}(P \| Q) \triangleq \int_{\mathcal{X}} \log \left( \frac{\dif P}{\dif Q} \right) \dif P $.
\end{lem}

With this tool in place, we now establish the PL condition. We apply the Ideal Critic Decomposition from Lemma~\ref{lem:sub-decomp}(i) to the parametric optimum $\pi = \pi_{\omega^*_\tau}$. After bounding the Ideal Algorithmic Progress using H\"{o}lder's and Pinsker's inequalities, and the Approximation Bias via the Cauchy-Schwarz inequality alongside the $\chi^2$-to-KL bound, we apply Young's inequality to decouple the terms. This yields positive penalty terms that offset the Restorative Entropy Force, leading to the desired parameter-space bound. Throughout, we denote the state-averaged forward KL divergence from this parametric optimum $\pi_{\omega^*_\tau}$ as $D_t \triangleq \mathbb{E}_{d^{\omega^*_\tau}}[D_{KL}(\pi_{\omega^*_\tau} \| \pi_t)]$.

\begin{lem}[Parameter-Space PL Condition] \label{lem:pl-condition}
Under Assumptions \ref{ass:bounded-features}, \ref{ass:approximation-error}, and \ref{ass:parametric-density},
the parametric regularized performance gap satisfies the PL condition, penalized by the MSE of the approximation error:
$$\text{Gap}_t \le \frac{C_{PL}}{\tau} \|\theta^*_\tau(\omega_t) - \tau \omega_t\|_2^2 + \frac{C_{err}}{\tau} \epsilon_{app},$$
where $C_{PL} \triangleq \frac{B_\phi^2}{1-\gamma}$ and $C_{err} \triangleq \frac{C_{joint}}{1-\gamma}$.
\end{lem}

\begin{rem}[Comparison with Standard PL Conditions] \label{rem:pl-comparison}
In classical optimization, the standard PL condition bounds the suboptimality via the Euclidean gradient norm. When combined with objective smoothness, the PL condition guarantees linear convergence. In the context of \textit{tabular} MDPs with a softmax parameterization, \cite{mei2020global} established two distinct Euclidean PL conditions: first for unregularized MDPs (which requires a sufficient exploration assumption) and second for entropy-regularized MDPs. Specifically, for the regularized MDP evaluated with identical initial distributions for the objective and training ($\mu= \rho$), their PL condition (Eq.\ 27) translated to our notation is:
$$ J_\tau(\pi^*_\tau) - J_\tau(\pi_\omega) \le \frac{|\mathcal{S}| \cdot \|d^{\pi^*_\tau} / d^{\pi_\omega} \|_\infty}{2\tau \min_{s} \mu(s) \min_{s,a} \pi_\omega(a|s)^2} \|\nabla_\omega J_\tau(\omega)\|_2^2. $$
This coefficient scales with the state-space size $|\mathcal{S}|$ and the state distribution mismatch, and requires positive initial state coverage ($\min_s \mu(s) > 0$).

In contrast, Lemma \ref{lem:pl-condition} establishes a \textit{parameter-space} PL condition tailored for compatible function approximation. Instead of the objective gradient, we bound the suboptimality directly via the Euclidean distance between the $\tau$-scaled actor iterate and the ideal critic target: $\|\theta^*_\tau(\omega_t) - \tau \omega_t\|_2^2$. Under exact parameterization ($\epsilon_{app} = 0$), our PL condition reduces to
$\text{Gap}_t \le \frac{B_\phi^2}{(1-\gamma)\tau } \|\theta^*_\tau(\omega_t) - \tau \omega_t\|_2^2$, eliminating the dependencies on $|\mathcal{S}|$, $\min_s \mu(s)$, $\min_{s,a}\pi_\omega(a|s)$, and $\|d^{\pi^*_\tau} / d^{\pi_\omega} \|_\infty$.
\end{rem}

\section{Actor Progress and Critic Tracking} \label{sec:ascent-tracking}

In this section, we analyze the actor and critic updates of Algorithm~\ref{alg:uncentered-nac}. First, by combining the global smoothness of the regularized objective with the Positive-Definite Fisher Information assumption, we establish the Actor Progress Bound (Lemma~\ref{lem:pl-ascent}), bounding the step-by-step reduction in the parametric performance gap up to the uncentered bias and critic tracking penalties. Next, we demonstrate how this progress bound and the parameter-space PL condition establish a fundamental two-way coupling: the critic's tracking error penalizes the Coupled Actor Recurrence (Lemma~\ref{lem:coupled-pl-descent}), while the actor's reduction in the performance gap dynamically bounds the target drift in the Coupled Critic Tracking (Lemma~\ref{lem:critic-tracking-pl}). This bidirectional dynamic lays the groundwork for the joint convergence analyzed in Section~\ref{sec:pl-convergence}.

\subsection{Uncentered Gradients and Objective Smoothness}

A notable characteristic of our algorithm is that the ideal critic $\theta^*_\tau(\omega_t)$ is trained on uncentered features $\phi(s,a)$. Because of this uncentered representation, mapping the resulting parameter update back to the exact policy gradient introduces a structural bias. We first quantify this bias.
By the first-order optimality condition for the ideal critic $\theta^*_\tau(\omega_t)$, the approximation error $\epsilon_t(s,a)$ is orthogonal to the raw features under the training distribution:
\begin{equation} \label{eq:mse-orthogonality}
\mathbb{E}_{s \sim d^{\pi_t}, a \sim \pi_t} \big[ \phi(s,a) \epsilon_t(s,a) \big] = 0.
\end{equation}
Let $\tilde{\phi}_t(s,a) \triangleq \phi(s,a) - \mathbb{E}_{a' \sim \pi_t}[\phi(s,a')]$ denote the action-centered feature.
Then the expected centered feature covariance under the training distribution can be expressed as
$ \bar{\Sigma}_{cen}(\pi_t) = \mathbb{E}_{d^{\pi_t}, \pi_t} \big[ \tilde{\phi}_t(s,a) \tilde{\phi}_t(s,a)^\top \big]$.

\begin{lem}[Uncentered Gradient Identity] \label{lem:uncentered-gradient}
The exact policy gradient evaluated using the ideal uncentered critic decomposes into a centered gradient vector and a structural bias vector:
$$\nabla J_\tau(\omega_t) = \frac{1}{1-\gamma} \bar{\Sigma}_{cen}(\pi_t) \big( \theta^*_\tau(\omega_t) - \tau \omega_t \big) + \frac{1}{1-\gamma} E_{bias},$$
where the bias vector is defined by the feature mean $\bar{\phi}_t(s) \triangleq \mathbb{E}_{\pi_t}[\phi(s,a')]$ and the residual error:
$$E_{bias} \triangleq -\mathbb{E}_{s \sim d^{\pi_t}} \Big[ \bar{\phi}_t(s) \mathbb{E}_{a \sim \pi_t}[\epsilon_t(s,a) \mid s] \Big].$$
Under Assumption~\ref{ass:approximation-error}, its norm is bounded by the critic approximation error: $\|E_{bias}\|_2 \le B_\phi \sqrt{\epsilon_{app}}$.
\end{lem}

Next, we establish that the entropy-regularized objective $J_\tau(\omega)$ is globally smooth with respect to the log-linear policy parameters. This smoothness property, combined with the uncentered gradient identity and the positive-definite Fisher information, is instrumental in obtaining the Actor Progress Bound for analyzing algorithmic convergence. We explicitly define a smoothness constant through the maximum temperature $\tau_{max}$ to ensure asymptotic stability as $\tau \to 0$.

\begin{lem}[Smoothness of the Regularized Objective] \label{lem:smoothness}
Under Assumption~\ref{ass:bounded-features}, the entropy-regularized objective $J_\tau(\omega)$ is $L_{J,\tau}$-smooth with respect to the parameters $\omega$. Specifically, the smoothness constant is bounded by:
$$L_{J,\tau} \triangleq \frac{16 B_\phi^2 (1 + \tau \log |\mathcal{A}|) + 4 \tau B_\phi^2}{(1-\gamma)^3},$$
such that for any $\omega, \omega' \in \mathbb{R}^d$, we have $\|\nabla J_\tau(\omega) - \nabla J_{\tau} (\omega')\|_2 \le L_{J,\tau} \|\omega - \omega'\|_2$.
Consequently, for any temperature $\tau \in (0, \tau_{max}]$, the regularized objective $J_\tau(\omega)$ is $L_J$-smooth, where $L_J \triangleq L_{J,\tau_{max}}$.
\end{lem}

\subsection{Actor Progress and Algorithmic Coupling} \label{sec:actor-descent}

We now analyze the step-by-step progress of the actor parameter update. By linking the actor update direction to the policy gradient (Lemma~\ref{lem:uncentered-gradient}), we substitute this relation into the quadratic smoothness bound (Lemma~\ref{lem:smoothness}) and leverage the positive-definite Fisher information to establish a progress inequality: the actor parameter update decreases the parametric performance gap, offset only by the critic tracking error and the uncentered approximation bias.

\begin{lem}[Actor Progress Bound] \label{lem:pl-ascent}
Let $Z_t \triangleq \mathbb{E}[\|e_t\|_2^2]$ be the expected critic tracking error.
Under Assumptions \ref{ass:bounded-features}, \ref{ass:approximation-error}, \ref{ass:pd-fim}, and \ref{ass:parametric-density},
if the step size satisfies $\eta_t \le \frac{\lambda}{2 L_J C_{joint} (1-\gamma)}$, the actor's update guarantees an expected progress inequality for the parametric gap:
$$\mathbb{E}[\text{Gap}_{t+1}] \le \mathbb{E}[\text{Gap}_t] - \eta_t \rho_{ac} \mathbb{E}[\|\theta_{t+1} - \tau \omega_t\|_2^2] + \eta_t C_Z Z_t + \eta_t C_{bias} \epsilon_{app},$$
where the descent coefficient is $\rho_{ac} = \frac{\lambda}{8 C_{joint} (1-\gamma)}$, the critic tracking penalty is $C_Z = \frac{2 B_\phi^2}{1-\gamma}$, and the uncentered bias penalty is $C_{bias} = \frac{2 C_{joint} B_\phi^2}{\lambda (1-\gamma)}$.
\end{lem}

The progress bound in Lemma~\ref{lem:pl-ascent} relies on the expected parameter distance $\mathbb{E}[\|\theta_{t+1} - \tau \omega_t\|_2^2]$, which is difficult to control directly over the entire trajectory. To resolve this, we invoke the Parameter-Space PL Condition (Lemma~\ref{lem:pl-condition}) to bridge the physical parameter update back to the performance gap. This yields a coupled linear recurrence where optimization progress is intrinsically tied to the critic's tracking accuracy.

\begin{lem}[Coupled Actor Recurrence] \label{lem:coupled-pl-descent}
Under Assumptions \ref{ass:bounded-features}, \ref{ass:approximation-error}, \ref{ass:pd-fim}, and \ref{ass:parametric-density},
if the step size satisfies $\eta_t \le \frac{\lambda}{2 L_J C_{joint} (1-\gamma)}$, the parametric regularized performance gap satisfies an expected linear recurrence:
$$\mathbb{E}[\text{Gap}_{t+1}] \le (1 - \rho_{gap} \eta_t) \mathbb{E}[\text{Gap}_t] + \eta_t \tilde{C}_Z Z_t + \eta_t C_{\epsilon,gap} \epsilon_{app},$$
where the effective recurrence parameter is $\rho_{gap} \triangleq \frac{\rho_{ac} \tau}{2 C_{PL}}$, the critic-tracking coupling is $\tilde{C}_Z \triangleq \rho_{ac} + C_Z$, and the total approximation penalty is $C_{\epsilon,gap} \triangleq \frac{\rho_{ac} C_{err}}{2 C_{PL}} + C_{bias}$.
\end{lem}

\subsection{Coupled Critic Tracking} \label{sec:critic-tracking}

We turn to analyzing the critic tracking error, $e_t \triangleq \theta_{t+1} - \theta^*_\tau(\omega_t)$. Because the actor is updated concurrently with the critic, the ideal target $\theta^*_\tau(\omega_t)$ continuously shifts as the policy evolves. To bound this target drift, we invoke the Lipschitz continuity of the ideal critic (Lemma~\ref{lem:lipschitz-critic}). Furthermore, we use the Actor Progress Bound (Lemma~\ref{lem:pl-ascent}) to connect this target drift back to the objective performance gap. By enforcing a two-timescale separation ($\alpha_{t+1} = \xi \eta_t$), we combine the geometric contraction of the strongly convex critic update with this objective-coupled drift penalty to form a recursive bound on the critic tracking error.

\begin{lem}[Coupled Critic Tracking] \label{lem:critic-tracking-pl}
Let the critic step size be $\alpha_{t+1} = \xi \eta_t$ for all $t \ge 0$, with an arbitrary initialization $\alpha_0$.
Under Assumptions \ref{ass:bounded-features}--\ref{ass:approximation-error} and \ref{ass:pd-fim}--\ref{ass:parametric-density},
if the actor step size satisfies $\eta_t \le \min\left( \frac{\lambda}{2 L_J C_{joint} (1-\gamma)}, \frac{1}{2\xi\kappa} \right)$, then
the expected squared tracking error $Z_t \triangleq \mathbb{E}[\|e_t\|_2^2] = \mathbb{E}[\|\theta_{t+1} - \theta^*_\tau(\omega_t)\|_2^2]$ satisfies the coupled recurrence:
$$Z_{t+1} \le (1 - \rho_{cr} \eta_t) Z_t + C_{gap} \big( \mathbb{E}[\text{Gap}_t] - \mathbb{E}[\text{Gap}_{t+1}] \big) + \xi^2 \eta_t^2 G_{cr}^2 + \eta_t C_{\epsilon,cr} \epsilon_{app},$$
where the native recurrence rate is $\rho_{cr} \triangleq \xi \kappa - \frac{L_\theta^2 C_Z}{\xi \kappa \rho_{ac}}$, the objective coupling constant is $C_{gap} \triangleq \frac{L_\theta^2}{\xi \kappa \rho_{ac}}$, and the approximation penalty is $C_{\epsilon,cr} \triangleq C_{gap} C_{bias}$.
\end{lem}

\section{Fixed-Temperature Convergence (Stochastic Regime)} \label{sec:pl-convergence}

We show that Algorithm~\ref{alg:uncentered-nac} achieves an $\tilde{\mathcal{O}}(T^{-1})$ average-iterate and an $\mathcal{O}(T^{-1})$ last-iterate convergence rate for the regularized performance gap.
This is accomplished by fusing the Coupled Actor Recurrence and Coupled Critic Tracking established in Section~\ref{sec:ascent-tracking} into a joint Lyapunov function that drives a linear recurrence for the overall algorithm.

\subsection{Joint Actor-Critic System} \label{sec:coupled-system}

We review the main results from Section \ref{sec:actor-descent} and \ref{sec:critic-tracking} to establish the joint system. We implement a two-timescale step size scheme: $\alpha_{t+1} = \xi \eta_t$ for a timescale separation constant $\xi \ge 1$.
By the Coupled Actor Recurrence (Lemma~\ref{lem:coupled-pl-descent}), the expected performance gap evolves according to:
\begin{equation} \label{eq:actor-recursion}
\mathbb{E}[\text{Gap}_{t+1}] \le (1 - \rho_{gap} \eta_t) \mathbb{E}[\text{Gap}_t] + \eta_t \tilde{C}_Z Z_t + \eta_t C_{\epsilon,gap} \epsilon_{app},
\end{equation}
where
$\rho_{gap} \triangleq \frac{\rho_{ac} \tau}{2 C_{PL}} = \frac{\lambda \tau}{16 C_{joint} B_\phi^2}$,
$\tilde{C}_Z \triangleq \rho_{ac} + C_Z = \frac{\lambda}{8 C_{joint} (1-\gamma)} + \frac{2 B_\phi^2}{1-\gamma}$, and
$C_{\epsilon,gap} \triangleq \frac{\rho_{ac} C_{err}}{2 C_{PL}} + C_{bias} = \frac{\lambda}{16 B_\phi^2 (1-\gamma)} + \frac{2 C_{joint} B_\phi^2}{\lambda (1-\gamma)}$.
Simultaneously, by the Coupled Critic Tracking (Lemma~\ref{lem:critic-tracking-pl}), the critic's expected tracking error $Z_t$ evolves according to:
\begin{equation} \label{eq:critic-recursion}
Z_{t+1} \le (1 - \rho_{cr} \eta_t) Z_t + C_{gap} \big( \mathbb{E}[\text{Gap}_t] - \mathbb{E}[\text{Gap}_{t+1}] \big) + \xi^2 \eta_t^2 G_{cr}^2 + \eta_t C_{\epsilon,cr} \epsilon_{app},
\end{equation}
where
$\rho_{cr} \triangleq \xi \kappa - \frac{L_\theta^2 C_Z}{\xi \kappa \rho_{ac}} = \xi \kappa - \frac{16 C_{joint} B_\phi^2 L_\theta^2}{\xi \kappa \lambda}$,
$C_{gap} \triangleq \frac{L_\theta^2}{\xi \kappa \rho_{ac}} = \frac{8 C_{joint} (1-\gamma) L_\theta^2}{\xi \kappa \lambda}$,
$C_{\epsilon,cr} \triangleq C_{gap} C_{bias} = \frac{16 C_{joint}^2 B_\phi^2 L_\theta^2}{\xi \kappa \lambda^2}$, and
$G_{cr}^2$ is a constant bound on $\mathbb{E}[\Vert{}g_t^{cr}\Vert{}_2^2 \mid \mathcal{F}_t]$ (Lemma~\ref{lem:structural-bounds}).

We construct a joint Lyapunov function that linearly fuses the expected performance gap with the expected critic tracking error:
$$\mathcal{L}_t \triangleq \mathbb{E}[\text{Gap}_t] + \beta \big( Z_t + C_{gap} \mathbb{E}[\text{Gap}_t] \big),$$
where $\beta > 0$ is a coupling constant.

\begin{thm}[Coupled Lyapunov Recurrence] \label{thm:coupled-recurrence}
Suppose that Assumptions \ref{ass:bounded-features}--\ref{ass:approximation-error} and \ref{ass:pd-fim}--\ref{ass:parametric-density} hold,
and the step sizes satisfy the conditions established for the actor and critic updates (Lemmas~\ref{lem:coupled-pl-descent} and \ref{lem:critic-tracking-pl}): $\eta_t \le \min\left( \frac{\lambda}{2 L_J C_{joint} (1-\gamma)}, \frac{1}{2\xi\kappa} \right)$ and $\alpha_{t+1} = \xi\eta_t $.
By choosing a timescale ratio $\xi \ge \frac{1}{\kappa} \left( 2\rho_{gap} + 2 L_\theta \sqrt{\frac{C_Z}{\rho_{ac}}} \right)$ and setting the Lyapunov coupling constant to $\beta \triangleq \frac{2 \tilde{C}_Z}{\rho_{cr}}$, the joint Lyapunov function satisfies an expected linear recurrence:
$$\mathcal{L}_{t+1} \le (1 - \rho_{\mathcal{L}} \eta_t) \mathcal{L}_t + \eta_t^2 \Sigma_{\mathcal{L}} + \eta_t C_{\epsilon,\mathcal{L}} \epsilon_{app},$$
where the joint rate is $\rho_{\mathcal{L}} \triangleq \frac{\rho_{gap}}{1 + \beta C_{gap}}$, the joint variance is $\Sigma_{\mathcal{L}} \triangleq \beta \xi^2 G_{cr}^2$, and the joint bias is $C_{\epsilon,\mathcal{L}} \triangleq C_{\epsilon,gap} + \beta C_{\epsilon,cr}$.
\end{thm}

We point out that the algebraic recurrence in Theorem~\ref{thm:coupled-recurrence} is valid regardless of the sign of the multiplier $(1 - \rho_{\mathcal{L}} \eta_t)$. In the convergence guarantees later, the average-iterate analysis (Theorem~\ref{thm:average-iterate}) accommodates this recurrence via a telescoping sequence. In contrast, the last-iterate analysis (Theorem~\ref{thm:last-iterate}) requires the multiplier to be non-negative to ensure pointwise geometric contraction via Chung's Lemma. This condition is explicitly satisfied by enforcing a sufficiently large initial step offset $t_0 \ge \rho_{\mathcal{L}} c_\eta$ in the step-size schedule.

\begin{rem}[Structural Order Evaluation] \label{rem:order-evaluation}
To facilitate the convergence analysis, we explicitly trace the structural orders of the system's composite constants with respect to the temperature $\tau$, the uncentered feature moment minimum eigenvalue $\kappa$, the Fisher information minimum eigenvalue $\lambda$, and the effective horizon $(1-\gamma)^{-1}$. Throughout this discussion, the stated orders refer to the structural constants and upper-envelope constants defined in the preceding analysis.
\begin{itemize}
    \item \textit{Effective Gap Recurrence:} The effective recurrence parameter scales linearly the temperature and this minimum eigenvalue: $\rho_{gap} = \frac{\lambda \tau}{16 C_{joint} B_\phi^2}=\Theta(\lambda \tau)$.
    Because the horizon dependencies of $\rho_{ac}$ and $C_{PL}$ cancel, $\rho_{gap}$ is independent of $(1-\gamma)$.

\item \textit{Timescale Separation:} The joint recurrence requires $\xi \ge \frac{1}{\kappa} \big(2\rho_{gap} + 2 L_\theta \sqrt{\frac{C_Z}{\rho_{ac}}} \big)$. From Lemma~\ref{lem:lipschitz-critic}, the ideal critic's Lipschitz constant evaluates to $L_\theta = \Theta((1-\gamma)^{-2}\kappa^{-2})$. Because $C_Z=\Theta((1-\gamma)^{-1})$ and $\rho_{ac}=\Theta((1-\gamma)^{-1}\lambda)$, the target drift term $S \triangleq L_\theta \sqrt{\frac{C_Z}{\rho_{ac}}}$ scales as $\Theta((1-\gamma)^{-2}\kappa^{-2}\lambda^{-1/2})$.
        Because $\rho_{gap} = \Theta(\tau)$, the $2\rho_{gap}$ term is dominated by $S$. For concreteness, we set $\xi = \frac{c}{\kappa}(2\rho_{gap} + 2S)$ for a constant $c \ge 1$, yielding the temperature-independent order $\xi = \Theta((1-\gamma)^{-2}\kappa^{-3}\lambda^{-1/2})$.

    \item \textit{Critic Recurrence:} As $\tau \to 0$, the gap recurrence term vanishes ($\rho_{gap} \to 0$). By setting our timescale ratio to $\xi = \frac{c}{\kappa}(2\rho_{gap} + 2S)$, we obtain $\kappa \xi \to 2cS$. Substituting this limit into the native critic recurrence yields $\rho_{cr} \to 2cS - \frac{S^2}{2cS} = (2c - \frac{1}{2c})S$. Since $c \ge 1$, this guarantees $\rho_{cr} = \Theta((1-\gamma)^{-2}\kappa^{-2}\lambda^{-1/2})$, safely bounded away from zero.

    \item \textit{Coupling Constants:} Because $\rho_{cr} = \Theta((1-\gamma)^{-2}\kappa^{-2}\lambda^{-1/2})$ and $\tilde{C}_Z = \Theta((1-\gamma)^{-1})$, the required Lyapunov weight is $\beta = \frac{2\tilde{C}_Z}{\rho_{cr}} = \Theta((1-\gamma)\kappa^2\lambda^{1/2})$. Substituting $L_\theta$ and $\xi$, the objective coupling evaluates to $C_{gap} = \frac{8 C_{joint} (1-\gamma) L_\theta^2}{\xi \kappa \lambda} = \Theta((1-\gamma)^{-1}\kappa^{-2}\lambda^{-1/2})$. The horizon dependencies cancel, yielding a structural invariant: $\beta C_{gap} = \Theta(1)$.

    \item \textit{Variance and Bias:}
    The critic-gradient second moment is upper-bounded by the envelope $G_{cr}^2=2B_\phi^2(B_\theta^2B_\phi^2+Q_{max}^2)$. Because
    $B_\theta=\Theta((1-\gamma)^{-1}\kappa^{-1})$ and $Q_{max}=\mathcal{O}((1-\gamma)^{-1})$, this envelope satisfies
    $G_{cr}^2=\Theta((1-\gamma)^{-2}\kappa^{-2})$.
    Thus, the joint tracking variance evaluates to
    $\Sigma_{\mathcal{L}}=\beta\xi^2G_{cr}^2
    =\Theta((1-\gamma)^{-5}\kappa^{-6}\lambda^{-1/2})$.
        The total approximation penalty evaluates to $C_{\epsilon,gap} = \frac{\lambda}{16 B_\phi^2 (1-\gamma)} + C_{bias}$. The first term scaling as $\Theta((1-\gamma)^{-1}\lambda)$ is dominated by the second term $C_{bias}= \frac{2 C_{joint} B_\phi^2}{\lambda (1-\gamma)}$, scaling as $\Theta((1-\gamma)^{-1}\lambda^{-1})$. Thus, $C_{\epsilon,gap} = \Theta((1-\gamma)^{-1}\lambda^{-1})$.
        Because $\beta C_{gap} = \Theta(1)$, the joint bias evaluates to $C_{\epsilon,\mathcal{L}} = C_{\epsilon,gap} + \beta C_{gap} C_{bias} = \Theta((1-\gamma)^{-1}\lambda^{-1})$.

    \item \textit{Joint Recurrence Rate:} Because $\beta C_{gap} = \Theta(1)$, the denominator $(1+\beta C_{gap})$ reduces to a constant multiplier. The joint Lyapunov recurrence rate preserves the native actor's effective recurrence rate: $\rho_{\mathcal{L}} = \frac{\rho_{gap}}{1+\beta C_{gap}} = \Theta(\lambda \tau)$.
\end{itemize}
\end{rem}

\subsection{Average-Iterate Regularized Convergence}  \label{sec:pl-average}

Based on the joint Lyapunov recurrence established in Section~\ref{sec:coupled-system}, we show that Algorithm~\ref{alg:uncentered-nac} achieves an $\tilde{\mathcal{O}}(T^{-1})$ average-iterate convergence to the regularized parametric optimum (up to the approximation bias) by deploying a diminishing Robbins-Monro step-size schedule. We first establish a general telescoping bound for the resulting average sequence.

\begin{lem}[Telescoping Robbins-Monro Sequence] \label{lem:telescoping-sequence}
Let $\{x_t\}_{t=0}^\infty$ be a non-negative sequence satisfying the recursive bound:
$$x_{t+1} \le (1 - r \eta_t) x_t + a \eta_t^2 + b \eta_t, \quad t \ge 0,$$
where $r>0$, $a>0$, and $b \ge 0$.
Assume the diminishing step size is $\eta_t = \frac{c_\eta}{t + t_0}$, with an initial offset $t_0 \ge 1$ and a numerator chosen such that $c \triangleq r c_\eta > 1$.
Then, the average sequence is bounded by:
$$\frac{1}{T} \sum_{t=0}^{T-1} x_t \le \frac{t_0 - 1}{T(c - 1)} x_0 + \frac{a c_\eta^2}{c - 1} \left( \frac{1 + \log T}{T} \right) + \frac{b c_\eta}{c - 1}.$$
\end{lem}

\begin{thm}[Average-Iterate Regularized Convergence] \label{thm:average-iterate}
For Algorithm~\ref{alg:uncentered-nac}, suppose that Assumptions \ref{ass:bounded-features}--\ref{ass:approximation-error} and \ref{ass:pd-fim}--\ref{ass:parametric-density} hold.
Let the actor step size be $\eta_t = \frac{c_\eta}{t + t_0}$, with a schedule numerator $c_\eta > \frac{1}{\rho_{\mathcal{L}}}$ and an initial offset $t_0 \ge \max\left(1, \frac{c_\eta}{\eta_{max}}\right)$, and let the critic step size be $\alpha_{t+1} = \xi \eta_t$, with a timescale ratio $\xi$ chosen as in Remark~\ref{rem:order-evaluation},
where $\eta_{max} \triangleq \min\left( \frac{\lambda}{2 L_J C_{joint} (1-\gamma)}, \frac{1}{2\xi\kappa} \right)$ is the maximum allowable actor step size.
Then, the average Lyapunov sequence is bounded by:
$$\frac{1}{T} \sum_{t=0}^{T-1} \mathcal{L}_t \le \frac{t_0 - 1}{T (\rho_{\mathcal{L}} c_\eta - 1)} \mathcal{L}_0 + \frac{c_\eta^2 \Sigma_{\mathcal{L}}}{\rho_{\mathcal{L}} c_\eta - 1} \left( \frac{1 + \log T}{T} \right) + \frac{c_\eta C_{\epsilon,\mathcal{L}}}{\rho_{\mathcal{L}} c_\eta - 1} \epsilon_{app}.$$
By choosing the step-size parameters $c_\eta = \Theta\big(\frac{1}{\lambda\tau}\big)$ and $t_0 = \Theta\big( \frac{1}{(1-\gamma)^2 \lambda^2 \tau} + \frac{1}{(1-\gamma)^2 \kappa^2 \lambda^{3/2}\tau} \big)$ subject to the stated conditions, the expected average regularized performance gap satisfies the convergence rate:
$$ \frac{1}{T} \sum_{t=0}^{T-1} \mathbb{E}[\text{Gap}_t] \le \mathcal{O}\left( \frac{1}{(1-\gamma)^5 \kappa^6 \lambda^{5/2} \tau^2} \frac{\log T}{T} \right) + \mathcal{O}\left( \frac{\epsilon_{app}}{(1-\gamma) \lambda^2 \tau} \right). $$
\end{thm}

\subsection{Last-Iterate Regularized Convergence} \label{sec:pl-last}

While the average-iterate sequence provides a robust empirical estimator, we further prove that the terminal policy converges to the regularized parametric optimum at an $\mathcal{O}(T^{-1})$ rate (up to the approximation bias). Because our step-size schedule explicitly enforces a non-negative multiplier, the joint Lyapunov recurrence yields a formal linear contraction. We first present a standard variant of Chung's Lemma \citep{chung1954stochastic} to bound the resulting recursive sequence.

\begin{lem}[Chung's Lemma for Robbins-Monro Sequences] \label{lem:chungs-lemma}
Let $\{x_t\}_{t=0}^\infty$ be a non-negative sequence satisfying the recursive bound:
$$x_{t+1} \le \left(1 - \frac{c}{t + t_0}\right) x_t + \frac{a}{(t + t_0)^2} + \frac{b}{t+t_0}, \quad t \ge 0,$$
where the constants satisfy $c > 1$, $a > 0$, $b \ge 0$, and $t_0 \ge c$.
Then, the sequence is bounded by:
$$x_t \le \frac{v}{t + t_0} + \frac{b}{c-1}, \quad t \ge 0,$$
where $v \triangleq \max\left( t_0 x_0, \frac{a}{c - 1} \right)$.
\end{lem}

\begin{thm}[Last-Iterate Regularized Convergence] \label{thm:last-iterate}
For Algorithm~\ref{alg:uncentered-nac} with $T$ replaced by $T+1$, suppose that Assumptions \ref{ass:bounded-features}--\ref{ass:approximation-error} and \ref{ass:pd-fim}--\ref{ass:parametric-density} hold.
Consider the actor and critic step-size schedules defined in Theorem~\ref{thm:average-iterate}, with the additional requirement that the initial offset satisfies $t_0 \ge \rho_{\mathcal{L}} c_\eta$. The Lyapunov sequence of the last iterate is bounded by:
$$\mathcal{L}_T \le \frac{v}{T + t_0} + \frac{c_\eta C_{\epsilon,\mathcal{L}}}{\rho_{\mathcal{L}} c_\eta - 1} \epsilon_{app},$$
where the terminal constant is defined as $v \triangleq \max\left( t_0 \mathcal{L}_0, \frac{c_\eta^2 \Sigma_{\mathcal{L}}}{\rho_{\mathcal{L}} c_\eta - 1} \right)$.
By choosing the step-size parameters $c_\eta = \Theta\big(\frac{1}{\lambda\tau}\big)$ and $t_0 = \Theta\big( \frac{1}{(1-\gamma)^2 \lambda^2 \tau} + \frac{1}{(1-\gamma)^2 \kappa^2 \lambda^{3/2}\tau} \big)$ subject to the stated conditions, the expected last-iterate regularized performance gap satisfies the convergence rate:
$$ \mathbb{E}[\text{Gap}_T] \le \mathcal{O}\left( \frac{1}{(1-\gamma)^5 \kappa^6 \lambda^{5/2} \tau^2 T} \right) + \mathcal{O}\left( \frac{\epsilon_{app}}{(1-\gamma) \lambda^2 \tau} \right). $$
\end{thm}

\vspace{.1in}
The additional iteration in Theorem~\ref{thm:last-iterate} is required only by the indexing of the critic tracking error. Specifically, $\mathcal{L}_T$ contains $Z_T=\mathbb{E}[\|\theta_{T+1}-\theta^*_\tau(\omega_T)\|_2^2]$, so that one additional critic update at $t=T$ is needed to establish the bound for the policy $\pi_T$. This extra update changes the iteration count from $T$ to $T+1$ and has no effect on the stated asymptotic convergence rate.

\begin{rem}[Equivalence of Average and Last-Iterate Rates] \label{rem:regularized-equivalence}
It is noteworthy that the last-iterate convergence rate (Theorem~\ref{thm:last-iterate}) asymptotically matches the average-iterate rate (Theorem~\ref{thm:average-iterate}) up to a logarithmic factor, both achieving an optimal $\tilde{\mathcal{O}}(T^{-1})$ dependence.
Under a two-timescale step-size scheme, the actor's parameter-space PL condition and the objective smoothness combine to facilitate a joint Lyapunov recurrence for the actor-critic system.
In classical unconstrained optimization, the standard PL condition and objective smoothness are known to enable optimal $\tilde{\mathcal{O}}(T^{-1})$ or $\mathcal{O}(T^{-1})$ convergence rates for both the average and last iterates of stochastic gradient descent \citep{karimi2016linear, khaled2023better}.
For our algorithm under the $\mathcal{O}(1/t)$ diminishing step-size schedule, this joint recurrence similarly leads to optimal rates for both the last-iterate and average-iterate convergence.
\end{rem}

\section{Unregularized Global Convergence (Stochastic Regime)} \label{sec:unregularized-complexity}

While Section~\ref{sec:pl-convergence} establishes the convergence of Algorithm~\ref{alg:uncentered-nac} under a fixed regularization temperature $\tau$, our ultimate goal is to solve the unregularized MDP. Recall from Section~\ref{sec:reg-objective} that $\pi_{\omega^*_0}$ denotes an unregularized parametrically optimal policy, and $\pi^*_0$ denotes the unregularized globally optimal policy. To establish convergence against this global optimum, we first transition our convergence guarantees from being against $\pi_{\omega^*_\tau}$ to against $\pi^*_\tau$. We then feed this extended global regularized bound directly into the Exponential Translation Bounds (Section~\ref{sec:exp-translation}) under Assumption~\ref{ass:action-gap}. By tuning the temperature $\tau$ to balance the regularized optimization error against the exponentially small translation tail, this bypasses the standard linear entropy penalty, extracting the optimal $\tilde{\mathcal{O}}(T^{-1})$ global convergence rates for the unregularized problem.

For environments where the Minimal Action Gap condition (Assumption~\ref{ass:action-gap}) is violated, we provide a robust fallback analysis in Appendix~\ref{app:linear-entropy} that utilizes the standard linear entropy penalty to guarantee $\tilde{\mathcal{O}}(T^{-1/3})$ unregularized convergence rates.

\subsection{Global Convergence Bounds} \label{sec:global-convergence}

We have so far bounded the convergence of Algorithm~\ref{alg:uncentered-nac} relative to the regularized parametrically optimal policy $\pi_{\omega^*_\tau}$. We now extend these bounds to the regularized globally optimal policy $\pi^*_\tau$ over the entire probability simplex. We begin by formalizing the density ratio bounds of the regularized global optimum relative to the parametric optimum.

\begin{ass}[Global-to-Parametric Joint Concentrability] \label{ass:global-parametric-density}
The joint state-action density ratio of the regularized global optimum $\pi^*_\tau$ with respect to the parametric optimum $\pi_{\omega^*_\tau}$ is bounded:
$$ \sup_{s,a} \frac{d^*_\tau(s) \pi^*_\tau(a|s)}{d^{\omega^*_\tau}(s) \pi_{\omega^*_\tau}(a|s)} \le C_{joint}^*, $$
where $C_{joint}^*>0$ is a constant independent of the temperature $\tau \in (0, \tau_{max}]$.
\end{ass}

By leveraging the suboptimality decomposition at the parametric optimum, we prove that the regularized performance difference between the regularized global optimum and its parametric counterpart is tightly controlled by the uncentered critic's approximation error.

\begin{lem}[Global Class Approximation Error] \label{lem:global-class-error}
Under Assumptions \ref{ass:bounded-features}, \ref{ass:approximation-error}, \ref{ass:pd-fim}, and \ref{ass:global-parametric-density},
the regularized performance gap between the regularized globally optimal policy $\pi^*_\tau$ and the regularized parametrically optimal policy $\pi_{\omega^*_\tau}$ is bounded by:
$$J_\tau(\pi^*_\tau) - J_\tau(\pi_{\omega^*_\tau}) \le \frac{C_{class}}{\tau} \epsilon_{app},$$
where $C_{class} \triangleq \frac{B_\phi^4}{(1-\gamma) \lambda^2} + \frac{C_{joint}^*}{1-\gamma}$.
\end{lem}

Because the global regularized performance gap decomposes as $\text{Gap}_t^\dag = \text{Gap}_t + \big(J_\tau(\pi^*_\tau) - J_\tau(\pi_{\omega^*_\tau})\big)$, Lemma~\ref{lem:global-class-error} guarantees that the global performance gap inherits the same convergence rates established for the parametric performance gap in Theorems~\ref{thm:average-iterate} and \ref{thm:last-iterate}, subject only to an additional representation penalty.

\begin{cor}[Global Regularized Performance Bound] \label{cor:global-reg-convergence}
Under Assumptions \ref{ass:bounded-features}, \ref{ass:approximation-error}, \ref{ass:pd-fim}, and \ref{ass:global-parametric-density},
for any training policy $\pi_t$, the global regularized performance gap is bounded by the parametric regularized performance gap plus the class approximation penalty:
$$\text{Gap}_t^\dag \le \text{Gap}_t + \frac{C_{class}}{\tau} \epsilon_{app}.$$
\end{cor}

The class approximation bound (Lemma~\ref{lem:global-class-error}), together with the Positive-Definite Fisher Information in the Stochastic Regime and the Minimal Action Gap, reveals an intrinsic compatibility between the approximation error $\epsilon_{app}$ and the temperature $\tau$, as formalized by the following lemma.

\begin{lem}[Part-I Approximation-Temperature Compatibility] \label{lem:part-I-compatibility}
Under Assumptions~\ref{ass:bounded-features}, \ref{ass:approximation-error}--\ref{ass:pd-fim}, and \ref{ass:global-parametric-density}, for any temperature $\tau\in(0,\tau_{max}]$, we have
$$
\epsilon_{app}\ge C_{comp}\tau,
$$
where
$$C_{comp} \triangleq
\frac{q_0\Delta}{2(1-\gamma)C_{class}}
\min\left\{1,
\frac{q_0\Delta}{2\tau_{max}\log|\mathcal{A}|}
\right\}, \quad
q_0\triangleq\frac{d\lambda}{4B_\phi^2},
$$
and $C_{class}$ is defined in Lemma~\ref{lem:global-class-error}. Consequently, arbitrarily small $\epsilon_{app}$, independent of temperature $\tau\in (0,\tau_{max}]$, are incompatible with the stated Part-I assumptions, and the exact-realizability case $\epsilon_{app}=0$ is excluded in the global analysis under the Stochastic Regime.
\end{lem}

\subsection{Average-Iterate Unregularized Convergence}

We now map the global regularized bounds to the unregularized objective. Unlike the parametric bounds analyzed in Appendix~\ref{app:linear-entropy}---which are bottlenecked by the linear $\mathcal{O}(\tau)$ entropy penalty of the training policy---we utilize the Universal Exponential Translation Bound (Corollary~\ref{cor:universal-suboptimal-mass}). Because this translation absorbs the policy entropy, we tune the fixed temperature $\tau$ to decrease much more slowly, optimally balancing the regularized optimization error against the exponentially small translation tail to extract an $\tilde{\mathcal{O}}(T^{-1})$ convergence rate.

\begin{thm}[Unregularized Average-Iterate Convergence] \label{thm:unreg-average-iterate-optimal}
For Algorithm~\ref{alg:uncentered-nac}, suppose that Assumptions \ref{ass:bounded-features}--\ref{ass:global-parametric-density} hold.
Define the two-stage temperature:
$$
\tau_T \triangleq \max\left(
\frac{\Delta}{2 \log (C_\gamma T)},
\frac{\Delta}{2 \log(C_\gamma(1+\epsilon_{app}^{-1}))}
\right).
$$
If $T$ is sufficiently large and $\epsilon_{app}$ is sufficiently small such that $\tau_T \le \tau_{max}$,
then by choosing the step-size parameters as in Theorem~\ref{thm:average-iterate} evaluated at $\tau = \tau_T$ subject to the stated conditions, the expected average global unregularized performance gap satisfies the convergence rate:
\begin{align*}
\frac{1}{T} \sum_{t=0}^{T-1} \mathbb{E} \big[ J_0(\pi^*_0) - J_0(\pi_t) \big]
&\le \mathcal{O}\left(\frac{(\log T)/(1-\gamma)^2 + (\log T)^3}{(1-\gamma)^6 \kappa^6 \lambda^{5/2} \Delta^3 T} \right) \\
& \quad + \mathcal{O}\left( \frac{(1-\gamma)^{-1}+\log(1+\epsilon_{app}^{-1})}{(1-\gamma)^2 \lambda^2 \Delta^2} \epsilon_{app} \right) \\
&= \tilde{\mathcal{O}}\left( \frac{1}{T} \right)
+ \tilde{\mathcal{O}}(\epsilon_{app}).
\end{align*}
\end{thm}

\begin{rem}[Admissible Approximation-Error Range in Part I] \label{rem:admissible-eps-app}
The phrase ``$\epsilon_{app}$ sufficiently small'' in Theorem~\ref{thm:unreg-average-iterate-optimal} refers to the requirement that the approximation-controlled branch of the two-stage temperature remain below $\tau_{max}$. This restriction needs to be considered together with the Part-I compatibility condition in Lemma~\ref{lem:part-I-compatibility}. Indeed, by the definition of $\tau_T$, we have
$\tau_T\ge\frac{\Delta}{2\log(C_\gamma(1+\epsilon_{app}^{-1}))}$.
Then the compatibility condition
$\epsilon_{app}\ge C_{comp}\tau_T$ implies
$$
\epsilon_{app}\log\!\big(C_\gamma(1+\epsilon_{app}^{-1})\big)
\ge\frac{C_{comp}\Delta}{2},
$$
which excludes arbitrarily small $\epsilon_{app}$ because $\epsilon_{app}\log(1+\epsilon^{-1}_{app})\to 0$ as $\epsilon_{app}\to0$.
Meanwhile, the requirement $\tau_T\le\tau_{max}$ imposes an upper-side restriction on $\epsilon_{app}$:
$$
\log(C_\gamma(1+\epsilon_{app}^{-1}))
\ge\frac{\Delta}{2\tau_{max}}.
$$
Thus, the Part-I assumptions and the selected temperature jointly constrain the admissible range of $\epsilon_{app}$. The statement of Theorem~\ref{thm:unreg-average-iterate-optimal} is understood over this range.
\end{rem}

\subsection{Last-Iterate Unregularized Convergence}

We extend our optimal unregularized convergence bounds to the final output policy of Algorithm~\ref{alg:uncentered-nac}. Because the last-iterate regularized bound (Theorem~\ref{thm:last-iterate}) avoids the $\log T$ penalty present in the average-iterate bound, applying the Exponential Translation mechanism leads to a sharper unregularized convergence rate (by a factor of $\log T$) for the final iteration.

\begin{thm}[Unregularized Last-Iterate Convergence] \label{thm:unreg-last-iterate-optimal}
For Algorithm~\ref{alg:uncentered-nac} with $T$ replaced by $T+1$, suppose that Assumptions \ref{ass:bounded-features}--\ref{ass:global-parametric-density} hold.
Consider the same two-stage temperature as in Theorem~\ref{thm:unreg-average-iterate-optimal}:
$$ \tau_T \triangleq \max\left( \frac{\Delta}{2 \log(C_\gamma T)}, \frac{\Delta}{2 \log(C_\gamma(1+\epsilon_{app}^{-1}))} \right). $$
If $T$ is sufficiently large and $\epsilon_{app}$ is sufficiently small such that $\tau_T \le \tau_{max}$, then by choosing the step-size parameters as in Theorem~\ref{thm:last-iterate} evaluated at $\tau = \tau_T$ subject to the stated conditions, the expected last-iterate global unregularized performance gap satisfies the convergence rate:
\begin{align*}
\mathbb{E} \big[ J_0(\pi^*_0) - J_0(\pi_T) \big]
&\le \mathcal{O}\left( \frac{(1-\gamma)^{-2} + (\log T)^2}{(1-\gamma)^6 \kappa^6 \lambda^{5/2} \Delta^3 T} \right)
+ \mathcal{O}\left( \frac{(1-\gamma)^{-1}+\log(1+\epsilon_{app}^{-1})}{(1-\gamma)^2 \lambda^2 \Delta^2} \epsilon_{app} \right) \\
&= \tilde{\mathcal{O}}\left( \frac{1}{T} \right) + \tilde{\mathcal{O}}(\epsilon_{app}).
\end{align*}
\end{thm}

\begin{rem}[Power of Minimal Action Gap] \label{rem:unreg-action-gap-power}
Comparing Theorems~\ref{thm:unreg-average-iterate-optimal} and \ref{thm:unreg-last-iterate-optimal} against Appendix~\ref{app:linear-entropy} signifies the theoretical power of the Minimal Action Gap (Assumption~\ref{ass:action-gap}).
By leveraging the deterministic nature of the unregularized MDP to unlock the Exponential Translation Bounds (Section \ref{sec:exp-translation}), this mechanism enables the slower $\tau \propto (\log T)^{-1}$ temperature tuning and accelerates the unregularized convergence rate to $\tilde{\mathcal{O}}(T^{-1})$. Consequently, Algorithm~\ref{alg:uncentered-nac} escapes the $\mathcal{O}(\tau)$ linear entropy penalty that would otherwise restrict the unregularized convergence to $\tilde{\mathcal{O}}(T^{-1/3})$.
\end{rem}

\part{Deterministic Regime} \label{part:deterministic}

In the second part, we investigate the convergence of Algorithm~\ref{alg:uncentered-nac} in the \textit{Deterministic Regime}. As established in Section~\ref{sec:setup}, this regime is characterized by the sequence of regularized optimal policies $\pi_{\omega^*_\tau}$ approaching an unregularized parametric optimum that is a deterministic policy in the closure of the log-linear family as the temperature drops ($\tau \to 0$). This behavior naturally arises when a highly expressive log-linear family perfectly captures the deterministic unregularized global optimum $\pi^*_0$, though it can also occur in a restricted family if the best policy in its closure is a suboptimal deterministic policy.

As this deterministic collapse occurs, the target probabilities for suboptimal actions decay to zero. Therefore, the action randomness vanishes and the minimum eigenvalue of the centered feature covariance collapses ($\lambda_\tau \to 0$), resulting in a degenerate Euclidean geometry. This degeneration violates the positive-definite Fisher Information assumption used to establish the Actor Progress Bound (Lemma~\ref{lem:pl-ascent}), which alongside the Parameter-Space PL Condition (Lemma~\ref{lem:pl-condition}) is essential for establishing the two-way coupling of the actor-critic system (Lemmas~\ref{lem:coupled-pl-descent} and \ref{lem:critic-tracking-pl}). Consequently, the convergence analysis in Part I breaks down in the Deterministic Regime.

To survive the deterministic limit, our analysis abandons parameter-distance metrics, such as $\|\theta^*_\tau(\omega_t) - \tau \omega_t\|_2$ in Lemmas~\ref{lem:pl-condition} and \ref{lem:pl-ascent}.
Instead, we leverage the Policy Mirror Descent (PMD) framework for the actor parameter update to evaluate algorithmic progress against the regularized globally optimal policy $\pi^*_\tau$ via the Kullback-Leibler (KL) divergence, bypassing the need for a positive-definite Fisher information matrix at the parametric optimum $\pi_{\omega^*_\tau}$.
Furthermore, because the Actor Progress Bound is no longer available to couple the critic's target drift to the actor's performance gap, we decouple the critic tracking from the actor progress and bound the target drift using a worst-case parameter step. By carefully balancing the decoupled step sizes, we control the competing algorithmic components to guarantee convergence.

To manage distribution shifts relative to the regularized global optimum $\pi^*_\tau$, we require that, under Algorithm~\ref{alg:uncentered-nac}, sufficient coverage is maintained over the target state-action visitation measure. Note that Assumption~\ref{ass:density-ratios} is utilized in exactly one place in our Part II analysis: in Step (ii) of the proof of the PMD Actor Progress Bound (Lemma~\ref{lem:pmd-telescoping-l2}), where it facilitates applying the action-space $\chi^2$-to-KL divergence bound and shifting the state measure to control the approximation bias.

\begin{ass}[Global Joint Concentrability] \label{ass:density-ratios}
Under Algorithm~\ref{alg:uncentered-nac}, the joint state-action density ratio of the regularized global optimum $\pi^*_\tau$ with respect to the training policy $\pi_t$ is bounded for all $t\ge 0$:
$$ \sup_{s,a} \frac{d^*_\tau(s) \pi^*_\tau(a|s)}{d^{\pi_t}(s) \pi_t(a|s)} \le C_{joint}^\dag, $$
where $C_{joint}^\dag > 0$ is a constant independent of the temperature $\tau \in (0, \tau_{max}]$.
\end{ass}

We make several comments related to Assumption~\ref{ass:density-ratios}. First, similarly to Assumption~\ref{ass:parametric-density}, the global joint concentrability bound is naturally implied by the conjunction of separate global state and policy concentrability bounds. Specifically, if the following density ratio bounds hold for some constants $(C^\dag,W^\dag)$ independent of $\tau \in (0, \tau_{max}]$:
\begin{itemize}
    \item \textit{Global State Concentrability:} $\|d^*_\tau / d^{\pi_t}\|_\infty \le C^\dag$,
    \item \textit{Global Policy Concentrability:} $\sup_{s,a} \frac{\pi^*_\tau(a|s)}{\pi_t(a|s)} \le W^\dag$,
\end{itemize}
then by the multiplicative rule of density ratios, the joint concentrability bound is satisfied with $C_{joint}^\dag = C^\dag W^\dag$.
Conversely, because $\mathbb{E}_{a \sim \pi_t} \big[ \frac{\pi^*_\tau(a|s)}{\pi_t(a|s)} \big] = 1$ at each state $s$, the joint concentrability directly implies the global state concentrability bound with $C^\dag = C_{joint}^\dag$: $\|d^*_\tau / d^{\pi_t}\|_\infty \le C_{joint}^\dag$.
However, the joint formulation allows the local policy ratio $\frac{\pi^*_\tau(a|s)}{\pi_t(a|s)}$ to be unbounded so long as the associated state density ratio $\frac{d^*_\tau(s)}{ d^{\pi_t}(s)}$ is sufficiently small.

Second, a simple sufficient condition for the global policy concentrability bound $W^\dag$ is a minimum action probability along the algorithmic trajectory.
Specifically, if $\inf_{t,s,a} \pi_t(a|s) \ge c$ for some constant $c>0$ independent of $\tau$, then
the policy concentrability condition holds with $W^\dag = 1/c$. However, such a uniform lower bound on $\pi_t(a|s)$ can be unrealistic in the Deterministic Regime, where probabilities of suboptimal actions may naturally vanish as $\tau \to 0$. Our joint concentrability formulation is more flexible: it allows the local policy ratio $\frac{\pi^*_\tau(a|s)}{\pi_t(a|s)}$ to become large provided that it is compensated by a sufficiently small state density ratio $\frac{d^*_\tau(s)}{d^{\pi_t}(s)}$.

Third, while Assumption~\ref{ass:parametric-density} in Part I bounds the algorithmic trajectory relative to the parametric optimum, Assumption~\ref{ass:density-ratios} bounds it directly relative to the global optimum. If $\pi^*_\tau$ were replaced by $\pi_{\omega^*_\tau}$ in Assumption~\ref{ass:density-ratios}, the entire analysis of regularized convergence in Section \ref{sec:regularized-conv-deterministic} would remain valid, yielding convergence bounds for the parametric gap $\text{Gap}_t$. However, the bound on the regularized performance difference between $\pi^*_\tau$ and $\pi_{\omega^*_\tau}$ (Lemma~\ref{lem:global-class-error}) breaks down because $\lambda_\tau\to 0$ as $\tau \to 0$ in the Deterministic Regime. Without this structural bridge, the Exponential Translation (Theorem~\ref{thm:suboptimal-mass}) could not be applied to derive unregularized convergence bounds. Evaluating the trajectory directly against the global optimum bypasses this broken bridge. Furthermore, by the multiplicative rule of density ratios, this global assumption is naturally satisfied with $C_{joint}^\dag = C_{joint} C_{joint}^*$ if both the Parametric Joint Concentrability (Assumption~\ref{ass:parametric-density}) and the Global-to-Parametric Joint Concentrability (Assumption~\ref{ass:global-parametric-density}) hold simultaneously.

\textbf{Connections to Related Works.}
Assumption~\ref{ass:density-ratios} serves as the regularized generalization of the standard concentrability and mismatch bounds widely utilized in the unregularized MDP literature \citep{agarwal2021theory, yuan2023linear}. While \cite{agarwal2021theory} rely on $L_\infty$ bounds that directly represent the unregularized counterpart of our assumption, \cite{yuan2023linear} employ $L_2$-based relaxations of these density ratios. Additionally, as noted in Part I, some analyses implicitly bundle this global distribution shift into a transferred approximation error bound evaluated under the global optimum \citep{liu2020improved}. In Appendix~\ref{app:related-works}, we provide an extended discussion of these existing concentrability conditions, comparing them against our generalized mismatch bounds while accommodating exploratory $\nu$-sampling.

To analyze critic tracking, we still require the Positive-Definite Uncentered Feature Moment (Assumption~\ref{ass:pd-uncentered}).
For pedagogical purposes, to contrast with the Positive-Definite Fisher Information (Assumption~\ref{ass:pd-fim}) in Part I, we introduce a static assumption on the uncentered feature moment matrix evaluated at the regularized global optimum.

\begin{ass}[Positive-Definite Uncentered Feature Moment at Optimum] \label{ass:optimal-uncentered}
There exists a constant $\kappa^* > 0$, independent of the temperature $\tau \in (0, \tau_{max}]$, such that
$$ \bar{\Sigma}_{unc}(\pi^*_\tau) = \mathbb{E}_{d^*_\tau, \pi^*_\tau}[\phi(s,a) \phi(s,a)^\top] \succeq \kappa^* I. $$
\end{ass}

When combined with the Global Density Ratio Bounds (Assumption~\ref{ass:density-ratios}), this static property dynamically guarantees the required positive-definiteness along the entire algorithmic trajectory. Specifically, for any training policy $\pi_t$ and any vector $v$, Assumptions~\ref{ass:density-ratios} and
\ref{ass:optimal-uncentered} together guarantee:
$$v^\top \bar{\Sigma}_{unc}(\pi_t) v \ge \frac{1}{C_{joint}^\dag} v^\top \bar{\Sigma}_{unc}(\pi^*_\tau) v \ge \frac{\kappa^*}{C_{joint}^\dag} \|v\|_2^2.$$
Then the Positive-Definite Uncentered Feature Moment (Assumption~\ref{ass:pd-uncentered}) introduced in Section~\ref{sec:struc-assumptions} is satisfied by taking $\kappa = \frac{\kappa^*}{C_{joint}^\dag}$ in the Deterministic Regime.

This Positive-Definite Uncentered Feature Moment at Optimum differs sharply from the Positive-Definite Fisher Information (Assumption~\ref{ass:pd-fim}) used in Part I. While Assumption~\ref{ass:optimal-uncentered} evaluates the moment at the global optimum $\pi^*_\tau$ rather than the parametric optimum $\pi_{\omega^*_\tau}$, the fundamental distinction lies in the matrix structure. In Part I, positive-definiteness is assumed on the \textit{centered} feature covariance matrix, which degenerates ($\lambda_\tau \to 0$) as the temperature $\tau \to 0$ in the Deterministic Regime. However, the \textit{uncentered} feature moment matrix can remain positive-definite even for deterministic policies, safely anchoring the critic updates in this regime.

Finally, it is worth noting that the PMD analysis developed in this part is agnostic to the value of $\lambda_\tau$ and is therefore applicable to the Stochastic Regime as well.
However, when $\lambda_\tau$ is bounded away from 0 as $\tau\to 0$, the analysis in Part I leverages the positive-definite Fisher Information to establish the Actor Progress Bound, which combines with the PL condition to extract significantly faster unregularized convergence rates.
Because the PMD framework does not exploit this strict Euclidean curvature, it yields comparatively slower unregularized rates.
Therefore, this PMD analysis is exclusively required for---and specifically tailored to---the Deterministic Regime.

\section{Policy Mirror Descent} \label{sec:pmd-geometry}

To bypass the degenerate Euclidean geometry in the Deterministic Regime, we analyze Algorithm~\ref{alg:uncentered-nac} through the lens of Policy Mirror Descent (PMD) \citep{shani2020adaptive, lan2023policy, yuan2023linear}, evaluating algorithmic progress directly on the probability simplex.
Standard tabular PMD executes independent, state-by-state updates on the probability simplex. While applying this to function approximation typically incurs projection errors, the log-linear policy parameterization circumvents this issue \citep{yuan2023linear}, establishing that NPG methods correspond exactly to unconstrained PMD with an appropriate functional gradient. We demonstrate how our specific entropy-regularized, uncentered NPG update also exhibits this structural equivalence.

In the exact PMD framework, a policy is updated by moving along a functional gradient $g(s,a)$ while penalizing the KL divergence from the current policy. For a step size $\eta_t$, this executes the following functional optimization over the probability simplex $\Delta(\mathcal{A})$:
$$ \pi_{t+1}(\cdot|s) = \mathop{\arg\max}_{p \in \Delta(\mathcal{A})} \left\{ \langle g(s,\cdot), p \rangle_{\mathcal{A}} - \frac{1}{\eta_t} D_{KL}(p \| \pi_t(\cdot|s)) \right\}. $$
This yields the closed-form multiplicative update: $\pi_{t+1}(a|s) \propto \pi_t(a|s) \exp(\eta_t g(s,a))$.
As shown in Appendix~\ref{app:critic-comparison}, our actor parameter update $\omega_{t+1} = \omega_t + \eta_t (\theta_{t+1} - \tau \omega_t)$ is derived from the NPG direction.
Under the log-linear structure $\pi_t(a|s) \propto \exp(\omega_t^\top \phi(s,a))$, this translates directly to the probability simplex as:
$$ \pi_{t+1}(a|s) \propto \exp\big( (\omega_t + \eta_t (\theta_{t+1} - \tau \omega_t))^\top \phi(s,a) \big) = \pi_t(a|s) \exp\big( \eta_t (\theta_{t+1} - \tau \omega_t)^\top \phi(s,a) \big). $$
Therefore, our parameter update intrinsically executes an exact PMD step. The effective functional gradient evaluates exactly to $(g_t^{ac})^\top \phi(s,a) = (\theta_{t+1} - \tau \omega_t)^\top \phi(s,a)$.

Consequently, the \textit{algorithmic progress} term $\langle (\theta_{t+1} - \tau \omega_t)^\top \phi, \pi^*_\tau - \pi_t \rangle_{\mathcal{A}}$ from the Actual Critic Decomposition (Lemma~\ref{lem:sub-decomp}(ii)) aligns with the PMD functional gradient. Because our update embodies an exact PMD step, this progress term can be bounded using the Bregman Three-Point identity, as derived in the following lemma.

\begin{lem}[PMD Algorithmic Progress] \label{lem:pmd-progress}
Under Assumptions~\ref{ass:bounded-features} and \ref{ass:pd-uncentered}, if the actor step size satisfies $\eta_t \le 1/\tau$, the algorithmic progress inner product satisfies the Bregman Three-Point inequality for any state $s$:
$$\langle (\theta_{t+1} - \tau \omega_t)^\top \phi, \pi^*_\tau - \pi_t \rangle_{\mathcal{A}} \le \frac{D_{KL}(\pi^*_\tau(\cdot|s) \| \pi_t(\cdot|s)) - D_{KL}(\pi^*_\tau(\cdot|s) \| \pi_{t+1}(\cdot|s))}{\eta_t} + \underbrace{\frac{\eta_t}{2} G_{ac}^2 B_\phi^2}_{\text{PMD Local Error}}.$$
\end{lem}

\begin{rem}[Conceptual PMD vs Algebraic Proofs]
While we frame our actor update as conceptually executing a PMD step to motivate the shift to KL geometry in the Deterministic Regime, our mathematical proofs do not invoke the PMD representation (i.e., the $\arg\max$ projection over the probability simplex). Standard PMD analyses \citep{shani2020adaptive, lan2023policy, yuan2023linear} rely on the variational inequalities (first-order optimality conditions) of this projection, which bound algorithmic progress by evaluating the functional gradient against the \textit{updated} policy $\pi_{t+1}$ (i.e., $\langle g_t, \pi^*_\tau - \pi_{t+1} \rangle_{\mathcal{A}}$). In contrast, similar to \cite{agarwal2021theory}, our analysis in Lemma~\ref{lem:pmd-progress} substitutes the exact log-linear parameter update directly into the Bregman Three-Point identity, evaluating the gradient against the \textit{current} policy $\pi_t$. This approach isolates the exact local step error $D_{KL}(\pi_t \| \pi_{t+1})$ and bounds it via Hoeffding's Lemma, allowing us to exploit the PMD geometry without relying on the variational inequalities.
\end{rem}

\section{Fixed-Temperature Convergence (Deterministic Regime)}  \label{sec:regularized-conv-deterministic}

We show that Algorithm~\ref{alg:uncentered-nac} achieves an $\mathcal{O}(T^{-2/3})$ average-iterate and an $\tilde{\mathcal{O}}(T^{-1/3})$ last-iterate convergence rate for the global regularized performance gap. This is accomplished by synthesizing the actor's PMD geometry, established in Section~\ref{sec:pmd-geometry}, with the uncoupled critic tracking dynamics developed in this section, dynamically balancing the decoupled diminishing step sizes to control the competing forces from the actor and critic updates.

\subsection{Actor Progress and Value Telescoping}

To track global progress, we recall the global regularized performance gap, $\text{Gap}_t^\dag \triangleq J_\tau(\pi^*_\tau) - J_\tau(\pi_t)$, introduced in Section~\ref{sec:reg-objective}, and we denote the state-averaged forward KL divergence from the global optimum $\pi^*_\tau$ as $D_t^\dag \triangleq \mathbb{E}_{d^*_\tau}[D_{KL}(\pi^*_\tau \| \pi_t)]$, distinct from $D_t$ used in Part I.

A notable difference distinguishes this analysis from that in Part I. In Part I, we rely on the Positive-Definite Fisher Information at parametric optimum $\pi_{\omega^*_\tau}$ (Assumption~\ref{ass:pd-fim}) to establish the Actor Progress Bound, ensuring algorithmic progress along the Euclidean parameter update. This Euclidean progress is then combined with the parameter-space PL condition to obtain the Coupled Actor Recurrence.
In Part II, as $\tau \to 0$, the deterministic collapse forces the minimum eigenvalue of the Fisher Information to vanish ($\lambda_\tau \to 0$), causing this Euclidean geometry to fail.

In contrast, the PMD geometry directly tracks the KL divergence. This allows us to directly bound the step-by-step algorithmic progress in the Actual Critic Decomposition from Lemma~\ref{lem:sub-decomp}(ii), thereby bypassing the need for a non-degenerate Fisher information matrix. We first derive the fundamental single-step inequality, and subsequently apply a value telescoping argument to extract the average-iterate performance gap.

\begin{lem}[PMD Actor Progress Bound] \label{lem:pmd-telescoping-l2}
Let $Z_t \triangleq \mathbb{E}[\|e_t\|_2^2]$ be the expected critic tracking error.
Under Assumptions~\ref{ass:bounded-features}, \ref{ass:pd-uncentered}, \ref{ass:approximation-error}, and \ref{ass:density-ratios}, if the actor step size satisfies $\eta_t \le 1/\tau$, the expected global performance gap satisfies the single-step inequality:
$$(1-\gamma) \mathbb{E}[\text{Gap}_t^\dag] \le -\frac{3\tau}{4} \mathbb{E}[D_t^\dag] + \frac{\mathbb{E}[D_t^\dag] - \mathbb{E}[D_{t+1}^\dag]}{\eta_t} + \frac{\eta_t}{2} G_{ac}^2 B_\phi^2 + \frac{4 C_{joint}^\dag}{\tau} \epsilon_{app} + \frac{4 B_\phi^2}{\tau} Z_t.$$
Notably, the joint concentrability constant $C_{joint}^\dag$ only penalizes the approximation bias, keeping the tracking error penalty completely free of concentrability assumptions.
\end{lem}

To extract the average-iterate performance gap from the single-step PMD bound, we need to manage the accumulated tracking and approximation errors. To prevent these errors from being amplified, we choose a diminishing step-size schedule $\eta_t = \frac{c_\eta}{t+t_0}$ to cancel the restorative KL divergence drift, creating a clean telescoping sequence across the training horizon.

\begin{lem}[PMD Value-Telescoping Bound] \label{lem:telescoping-pmd}
Let the actor step size be $\eta_t = \frac{c_\eta}{t+t_0}$, with a numerator constant chosen such that $c_\eta \ge \frac{4}{3\tau}$, and an initial offset satisfying $t_0 \ge c_\eta \tau$.
Under Assumptions~\ref{ass:bounded-features}, \ref{ass:pd-uncentered}, \ref{ass:approximation-error}, and \ref{ass:density-ratios}, the expected average regularized performance gap is bounded by:
$$ (1-\gamma) \frac{1}{T} \sum_{t=0}^{T-1} \mathbb{E}[\text{Gap}_t^\dag] \le \frac{t_0-1}{c_\eta T} D_0^\dag + \frac{1}{T} \sum_{t=0}^{T-1} \mathcal{E}_{noise}(t), $$
where the noise envelope is defined as $\mathcal{E}_{noise}(t) \triangleq \frac{\eta_t}{2} G_{ac}^2 B_\phi^2 + \frac{4 C_{joint}^\dag}{\tau} \epsilon_{app} + \frac{4 B_\phi^2}{\tau} Z_t$.
\end{lem}

\subsection{Uncoupled Critic Tracking}

With the actor's PMD progress established, we bound the remaining critic tracking error. Here, another important difference from Part I arises in controlling the critic target drift. In Part I, the Actor Progress Bound---enabled by the Positive-Definite Fisher Information (Assumption~\ref{ass:pd-fim})---bounds the Euclidean parameter step by the expected reduction in the performance gap, establishing the Coupled Critic Tracking. In Part II, the deterministic collapse ($\lambda_\tau \to 0$) severs this objective coupling. Consequently, we bound the ideal critic target drift using the worst-case, uncoupled parameter step. From Lemma~\ref{lem:structural-bounds}(ii), the absolute change in the actor parameter is bounded by the temperature-independent constant $G_{ac}$: $\|\omega_{t+1} - \omega_t\|_2 = \eta_t \| \theta_{t+1} - \tau \omega_t\|_2 \le \eta_t G_{ac}$.

The telescoping bound from Lemma~\ref{lem:telescoping-pmd} dictates that the actor step size decays as $\eta_t = \mathcal{O}(t^{-1})$.
To facilitate fast convergence, the critic step size $\alpha_{t+1}$ needs to be carefully tuned.
As shown in the proof of Lemma~\ref{lem:decoupled-critic}, the expected tracking error $Z_t \triangleq \mathbb{E}[\|e_t\|_2^2]$ satisfies the recurrence
$$ Z_{t+1} \le (1 - \alpha_{t+1} \kappa) Z_t + \underbrace{\alpha_{t+1}^2 G_{cr}^2}_{\text{SGD Error}} + \underbrace{\frac{1}{\alpha_{t+1} \kappa} L_\theta^2 \eta_t^2 G_{ac}^2}_{\text{Target Drift}}. $$
This recurrence reveals two competing, uncoupled noise forces: the SGD error scaling as $\mathcal{O}(\alpha_{t+1}^2)$ and the target drift scaling as $\mathcal{O}(\eta_t^2 / \alpha_{t+1})$.
To balance these two forces ($\alpha_{t+1}^2$ vs $\eta_t^2 / \alpha_{t+1}$), we set $\alpha_{t+1} \propto \eta_t^{2/3}$. Given the actor decay $\eta_t = \mathcal{O}(t^{-1})$, this yields the critic schedule $\alpha_{t+1} = \mathcal{O}(t^{-2/3})$.
The choice $\alpha_{t+1} \propto \eta_t^{2/3}$ not only ensures that both noise components decay as $\mathcal{O}(t^{-4/3})$, but also balances their dependence on the temperature $\tau$: this fractional coupling guarantees that both the SGD error and the target drift share an identical $\Theta(\tau^{-4/3})$ temperature dependence.
See the order evaluation of the noise envelope $K$ in the proof of Theorem~\ref{thm:reg-conv-l2}.

To formally prove that the critic's tracking error achieves this balanced $\mathcal{O}(t^{-2/3})$ rate, we require a specialized stochastic approximation tool. Unlike standard Robbins-Monro sequences that decay as $\mathcal{O}(1/t)$, our critic operates on a slower fractional timescale. The following lemma provides the pointwise convergence guarantee for this fractional geometry.

\begin{lem}[Fractional Chung's Lemma] \label{lem:fractional-chung}
Let $\{x_t\}_{t=0}^\infty$ be a non-negative sequence satisfying the recursive bound:
$$x_{t+1} \le \left(1 - \frac{c}{(t+t_0)^{2/3}}\right) x_t + \frac{K}{(t+t_0)^{4/3}}, \quad t \ge 0,$$
where $c > 0$, $K \ge 0$, and the initial offset satisfies $t_0 \ge \max\left( c^{3/2}, \left(\frac{4}{3c}\right)^3 \right)$. Then, the sequence is bounded pointwise by:
$$x_t \le \frac{v}{(t+t_0)^{2/3}}, \quad t \ge 0,$$
where the bounding constant is defined as $v \triangleq \max\left(t_0^{2/3} x_0, \frac{2K}{c}\right)$.
\end{lem}

With the fractional convergence tool established, we now bound the uncoupled critic tracking error. By evaluating the constrained SGD update alongside the worst-case target drift, we construct the recursive sequence required to apply the fractional tracking limit.

\begin{lem}[Pointwise Critic Tracking Error] \label{lem:decoupled-critic}
Suppose that Assumptions~\ref{ass:bounded-features}--\ref{ass:pd-uncentered} hold.
For any independent constants $c_\eta, c_\alpha > 0$, let the actor and critic step sizes be $\eta_t = \frac{c_\eta}{t+t_0}$ and $\alpha_{t+1} = \frac{c_\alpha}{(t+t_0)^{2/3}}$ for all $t \ge 0$, with an arbitrary initialization $\alpha_0$
and an initial offset $t_0 \ge \max\left( c_\eta \tau, (2 c_\alpha \kappa)^{3/2}, (\frac{4}{3 c_\alpha \kappa} )^3 \right)$.
Then the expected tracking error $Z_t \triangleq \mathbb{E}[\|e_t\|_2^2]$ is bounded pointwise for all $t \ge 0$ by:
$$Z_t \le \frac{v}{(t+t_0)^{2/3}},$$
where the bounding coefficient is $v \triangleq \max\left( t_0^{2/3} Z_0, \frac{2K}{c_\alpha \kappa} \right)$,
and the noise envelope constant is $K \triangleq c_\alpha^2 G_{cr}^2 + \frac{L_\theta^2 c_\eta^2 G_{ac}^2}{c_\alpha \kappa}$.
\end{lem}

\subsection{Average-Iterate Regularized Convergence}

We combine the PMD Value-Telescoping Bound (Lemma~\ref{lem:telescoping-pmd}) and the Pointwise Critic Tracking Error (Lemma~\ref{lem:decoupled-critic}) to bound the total algorithmic error. The expected performance gap decomposes into an initialization error, a PMD local error, the critic tracking error, and the approximation bias.
By evaluating these terms under the selected step-size schedule,
the following theorem establishes a $\mathcal{O}(T^{-2/3})$ average-iterate convergence rate for the regularized global objective.

\begin{thm}[Average-Iterate Regularized Convergence] \label{thm:reg-conv-l2}
For Algorithm~\ref{alg:uncentered-nac}, suppose that Assumptions~\ref{ass:bounded-features}--\ref{ass:approximation-error} and \ref{ass:density-ratios} hold.
For independent constants $c_\eta \ge \frac{4}{3\tau}$ and $\zeta > 0$, let the actor and critic step sizes be $\eta_t = \frac{c_\eta}{t+t_0}$ and $\alpha_{t+1} = \frac{c_\alpha}{(t+t_0)^{2/3}}$ for all $t\ge 0$, with $c_\alpha \triangleq \zeta c_\eta^{2/3}$ and an initial offset $t_0 \ge \max\left( c_\eta \tau, (2 c_\alpha \kappa)^{3/2}, (\frac{4}{3 c_\alpha \kappa} )^3 \right)$.
Then the expected average regularized performance gap is bounded by:
\begin{align*}
\frac{1}{T} \sum_{t=0}^{T-1} \mathbb{E}[\text{Gap}_t^\dag] \le \frac{1}{1-\gamma} \Bigg[ &\underbrace{ \frac{t_0-1}{c_\eta T} D_0^\dag }_{\text{PMD Initialization Error}} + \underbrace{ \frac{c_\eta G_{ac}^2 B_\phi^2}{2 T} (1 + \log T) }_{\text{PMD Local Error}} \\
&+ \underbrace{ \frac{12 B_\phi^2 v}{\tau} T^{-2/3} \left(1 + \frac{t_0}{T}\right)^{1/3} }_{\text{Critic Tracking Error}} + \underbrace{ \frac{4 C_{joint}^\dag}{\tau} \epsilon_{app} }_{\text{Approximation Bias}} \Bigg].
\end{align*}
By choosing the step-size constants $c_\eta = \Theta(\tau^{-1})$ and $\zeta = \Theta(\kappa^{-1})$ and the initial offset $t_0 = \Theta(\tau^{-1})$ subject to the stated conditions, the asymptotic convergence rate is:
\begin{align*}
\frac{1}{T} \sum_{t=0}^{T-1} \mathbb{E}[\text{Gap}_t^\dag] &\le \underbrace{ \mathcal{O}\left( \frac{1}{(1-\gamma)^7 \kappa^6 \tau^{5/3}} T^{-2/3} \right) }_{\text{Primary Tracking Error}} + \underbrace{ \tilde{\mathcal{O}}\left( \frac{1}{(1-\gamma)^7 \kappa^6 \tau^2} T^{-1} \right) }_{\text{Transient Convergence Error}} + \underbrace{ \mathcal{O}\left( \frac{\epsilon_{app}}{(1-\gamma)\tau} \right) }_{\text{Approximation Bias}} .
\end{align*}
\end{thm}

\subsection{Last-Iterate Regularized Convergence}

While the telescoping sequence extracts an accelerated $\mathcal{O}(T^{-2/3})$ rate for the average policy, we now establish the convergence of the final output policy. Because the PMD framework bounds algorithmic progress via the forward KL divergence, we evaluate the last-iterate performance gap by passing the terminal divergence through Pinsker's inequality. To accomplish this, we establish the pointwise geometric contraction of the forward KL divergence. We first introduce a specialized variant of Chung's Lemma that handles fractional noise and an additive bias.

\begin{lem}[Chung's Lemma with Fractional Noise] \label{lem:fractional-chung-bias}
Let $\{x_t\}_{t=0}^\infty$ be a non-negative sequence satisfying the recursive bound:
$$ x_{t+1} \le \left(1 - \frac{c}{t+t_0}\right) x_t + \frac{A}{(t+t_0)^{5/3}} + \frac{B}{t+t_0}, \quad t \ge 0, $$
where $c \ge 1$, $A \ge 0$, $B \ge 0$, and the initial offset satisfies $t_0 \ge c$. Then, the sequence is bounded pointwise by:
$$ x_t \le \frac{u}{(t+t_0)^{2/3}} + \frac{B}{c}, \quad t \ge 0 ,$$
where the bounding coefficient is $u \triangleq \max\left( t_0^{2/3} x_0, \frac{3A}{3c-2} \right)$.
\end{lem}

Next, we derive the pointwise contraction of the forward KL divergence. By isolating the divergence within the single-step PMD Actor Progress Bound and substituting the Pointwise Critic Tracking Error, we construct a recurrence relation that aligns with our specialized Chung's Lemma to establish the convergence of the forward KL divergence.

\begin{lem}[Last-Iterate Forward KL Convergence] \label{lem:last-iterate-dt}
Under the identical assumptions and step-size conditions of Theorem~\ref{thm:reg-conv-l2}, the expected state-averaged forward KL divergence of the final iterate is bounded by:
$$ \mathbb{E}[D_T^\dag] \le \frac{u}{(T+t_0)^{2/3}} + \frac{16 C_{joint}^\dag}{3\tau^2} \epsilon_{app}, $$
where the bounding coefficient is $u \triangleq \max\left(t_0^{2/3} D_0^\dag, \frac{3A}{3c-2}\right)$, with $c \triangleq \frac{3\tau}{4}c_\eta$ and $A \triangleq \frac{c_\eta^2 G_{ac}^2 B_\phi^2}{2} + \frac{4 c_\eta B_\phi^2 v}{\tau}$.
By choosing the step-size constants $c_\eta = \Theta(\tau^{-1})$ and $\zeta = \Theta(\kappa^{-1})$ and the initial offset $t_0 = \Theta(\tau^{-1})$ subject to the stated conditions, the asymptotic convergence rate is:
$$ \mathbb{E}[D_T^\dag] \le \mathcal{O}\left( \frac{1}{(1-\gamma)^6 \kappa^6 \tau^{8/3}} T^{-2/3} \right) + \mathcal{O}\left( \frac{\epsilon_{app}}{\tau^2} \right). $$
\end{lem}

With the geometric contraction of the forward KL divergence established, we translate this divergence back into the global performance gap. While the Exact Suboptimality Identity (Lemma~\ref{lem:optim-identity}) captures the suboptimality via the \textit{reverse} KL divergence under the training state distribution $d^{\pi_t}$, our global convergence analysis tracks the \textit{forward} KL divergence under the optimal state distribution $d^*_\tau$ to manage distribution shift.  To bridge this analytical divide, we build on Lemma~\ref{lem:sub-decomp} and establish an upper bound on the performance gap through the forward KL divergence.

\begin{lem}[Forward KL Bound on Regularized Suboptimality] \label{lem:pinsker-gap}
Under Assumptions~\ref{ass:bounded-features} and \ref{ass:pd-uncentered},
if the actor step size satisfies $\eta_t \le 1/\tau$, then
the global regularized performance gap is bounded by the state-averaged forward KL divergence under $d^*_\tau$:
$$J_\tau(\pi^*_\tau) - J_\tau(\pi_t) \le \frac{\sqrt{2} M_\tau}{1-\gamma} \sqrt{\mathbb{E}_{s \sim d^*_\tau}[D_{KL}(\pi^*_\tau(\cdot|s) \| \pi_t(\cdot|s))]},$$
where the capacity bound is $M_\tau \triangleq V_{max} + 2 B_\theta B_\phi + \tau \log |\mathcal{A}|$.
Consequently, for any temperature $\tau \in (0,\tau_{max}]$,
$$J_\tau(\pi^*_\tau) - J_\tau(\pi_t) \le \frac{\sqrt{2} M_{max}}{1-\gamma} \sqrt{\mathbb{E}_{s \sim d^*_\tau}[D_{KL}(\pi^*_\tau(\cdot|s) \| \pi_t(\cdot|s))]},$$
where $M_{max} \triangleq V_{max} + 2 B_\theta B_\phi + \tau_{max} \log |\mathcal{A}|$.
\end{lem}

By applying this Forward KL Bound to the terminal divergence in Lemma~\ref{lem:last-iterate-dt}, we extract a $\mathcal{O}(T^{-1/3})$ last-iterate convergence rate for the global regularized performance gap.

\begin{thm}[Last-Iterate Regularized Convergence] \label{thm:reg-last-iterate-gap}
Under the identical assumptions and step-size conditions of Theorem~\ref{thm:reg-conv-l2}, the expected last-iterate regularized performance gap is bounded by:
$$ \mathbb{E}[\text{Gap}_T^\dag] \le \frac{\sqrt{2 u} M_{max}}{1-\gamma} \frac{1}{(T+t_0)^{1/3}} + \frac{4 M_{max} \sqrt{2 C_{joint}^\dag}}{\sqrt{3} \tau (1-\gamma)} \sqrt{\epsilon_{app}}, $$
where $M_{max}$ is defined in Lemma~\ref{lem:pinsker-gap}, and $u$ is the structural constant defined in Lemma~\ref{lem:last-iterate-dt}.
By choosing the step-size parameters $c_\eta = \Theta(\tau^{-1})$ and $\zeta = \Theta(\kappa^{-1})$ and the initial offset $t_0 = \Theta(\tau^{-1})$ subject to the stated conditions, the asymptotic convergence rate is:
$$ \mathbb{E}[\text{Gap}_T^\dag] \le \mathcal{O}\left(\frac{1}{(1-\gamma)^5 \kappa^4 \tau^{4/3}} T^{-1/3}\right) + \mathcal{O}\left(\frac{1}{(1-\gamma)^2 \kappa \tau} \sqrt{\epsilon_{app}}\right). $$
\end{thm}

\begin{rem}[Methodological Divergence of Average and Last-Iterate Bounds] \label{rem:methodological-divergence}
The single PMD inequality derived in Lemma~\ref{lem:pmd-telescoping-l2} serves as the common foundation for both our average-iterate and last-iterate convergence guarantees, but algebraic constraints force the two metrics to take different paths. The core expected inequality is structured as:
$$ (1-\gamma) \mathbb{E}[\text{Gap}_t^\dag] + \frac{3\tau}{4} \mathbb{E}[D_t^\dag] \le \frac{\mathbb{E}[D_t^\dag] - \mathbb{E}[D_{t+1}^\dag]}{\eta_t} + \mathcal{E}_{noise}(t), $$
where the noise envelope $\mathcal{E}_{noise}(t) = \mathcal{O}(\eta_t) + \mathcal{O}(Z_t/\tau) + \mathcal{O}(\epsilon_{app}/\tau)$ encapsulates the PMD local error, the critic tracking error, and the approximation bias.

For the \textit{average-iterate bound}, telescoping this inequality over $T$ steps utilizing the diminishing schedule $\eta_t = \frac{c_\eta}{t+t_0}$ neutralizes the step-size division by multiplying through by $(t+t_0)$ as shown in the proof. Because the uncancelled boundary term from the telescoping resolves to the initial divergence scaled by the iteration budget ($\frac{t_0-1}{c_\eta T} D_0^\dag$), we balance this transient divergence penalty against the noise, unlocking the accelerated $\mathcal{O}(T^{-2/3})$ regularized rate.

For the \textit{last-iterate bound}, however, direct extraction is impossible. Without telescoping, the terminal performance gap $\mathbb{E}[\text{Gap}_T^\dag]$ would be bounded by $\frac{\mathbb{E}[D_T^\dag]}{\eta_T}$. As the policy sequence converges, the divergence $\mathbb{E}[D_T^\dag]$ only decays to a noise floor of $\mathcal{O}(T^{-2/3}) + \mathcal{O}(\epsilon_{app})$ as shown in Lemma~\ref{lem:last-iterate-dt}. Dividing this floor by the decaying step size $\eta_T = \mathcal{O}(T^{-1})$ is catastrophic: the amplified noise contains components scaling as $T^{1/3}$ and $T \epsilon_{app}$, blowing up to infinity.

To achieve pointwise convergence, we abandon the $\frac{1}{\eta_t}$ translation. Instead, we first prove that the divergence $\mathbb{E}[D_T^\dag]$ undergoes a geometric contraction to the noise floor via Chung's Lemma ($\mathbb{E}[D_T^\dag] \le \mathcal{O}(T^{-2/3}) + \mathcal{O}(\epsilon_{app})$), and subsequently translate this divergence back to the value gap via Pinsker's inequality ($\mathbb{E}[\text{Gap}_T^\dag] \le \mathcal{O}(1) \sqrt{\mathbb{E}[D_T^\dag]}$). This sidesteps the $\eta_T^{-1}$ amplification, but taking the square root slows the last-iterate regularized rate to $\mathcal{O}(T^{-1/3})$.
\end{rem}

\section{Unregularized Global Convergence (Deterministic Regime)}  \label{sec:unregularized-conv-deterministic}

We show that Algorithm~\ref{alg:uncentered-nac} achieves an $\tilde{\mathcal{O}}(T^{-2/3})$ average-iterate and an $\tilde{\mathcal{O}}(T^{-1/3})$ last-iterate convergence rate for the unregularized performance gap. This is accomplished by applying the exponential translation mechanism alongside a two-stage temperature schedule, which balances the regularized optimization error against an exponentially decaying target tail to successfully bypass the standard linear entropy bottleneck.

\subsection{Average-Iterate Unregularized Convergence}

We now map the global regularized bounds to the unregularized objective. By utilizing the Universal Exponential Translation Bound (Corollary~\ref{cor:universal-suboptimal-mass}), we absorb the policy entropy into an exponentially small translation tail, allowing us to tune the fixed temperature $\tau$ to decrease much more slowly. By optimally balancing the regularized optimization error against this exponential term, we bypass the standard linear entropy bottleneck to extract an accelerated $\tilde{\mathcal{O}}(T^{-2/3})$ average-iterate convergence rate for the unregularized problem.

\begin{thm}[Unregularized Average-Iterate Convergence] \label{thm:PMD-unreg-average}
For Algorithm~\ref{alg:uncentered-nac}, suppose that Assumptions~\ref{ass:bounded-features}--\ref{ass:action-gap} and \ref{ass:density-ratios} hold.
Define the two-stage temperature:
$$ \tau_T \triangleq \max\left( \frac{\Delta}{2 \log(C_\gamma T^{2/3})}, \frac{\Delta}{2 \log(C_\gamma(1+\epsilon_{app}^{-1}))} \right). $$
If $T$ is sufficiently large and $\epsilon_{app}$ is sufficiently small such that $\tau_T \le \tau_{max}$, then by choosing the step-size parameters as in Theorem~\ref{thm:reg-conv-l2} evaluated at $\tau = \tau_T$ subject to the stated conditions, the expected average global unregularized performance gap satisfies the convergence rate:
\begin{align*}
\frac{1}{T} \sum_{t=0}^{T-1} \mathbb{E} \big[ J_0(\pi^*_0) - J_0(\pi_t) \big]
&\le \mathcal{O}\left( \frac{(1-\gamma)^{-5/3} + (\log T)^{5/3}}{(1-\gamma)^8 \kappa^6 \Delta^{8/3} T^{2/3}} \right) \\
&\quad + \mathcal{O}\left( \frac{(1-\gamma)^{-1} + \log(1+\epsilon_{app}^{-1})}{(1-\gamma)^2 \Delta^2} \epsilon_{app} \right) \\
&= \tilde{\mathcal{O}}\left( T^{-2/3} \right) + \tilde{\mathcal{O}}(\epsilon_{app}).
\end{align*}
In the case of $\epsilon_{app}=0$, we interpret $\tau_T = \frac{\Delta}{2\log(C_\gamma T^{2/3})}$ and $\epsilon_{app}\log(1+\epsilon_{app}^{-1})=0$.
\end{thm}

\begin{rem}[Unification with Double-Loop Architectures] \label{rem:double-loop-unification}
Balancing the average-iterate regularized convergence rate established in Theorem~\ref{thm:reg-conv-l2} with the standard linear entropy penalty leads to an $\mathcal{O}(T^{-1/4})$ average-iterate unregularized convergence rate for our single-loop algorithm. Interestingly, the same $\mathcal{O}(T_{total}^{-1/4})$ rate was established by \citet{agarwal2021theory} for a double-loop Q-NPG algorithm (see the extended discussion in Appendix~\ref{app:Agarwal-et-al}). However, their algorithm utilizes a \textit{cold-start} inner critic, which requires a specific balancing of the outer-loop iterations $T$ and inner-loop iterations $N$ such that $N = T = T_{total}^{1/2}$ to achieve this convergence rate.

To reconcile our single-loop algorithm ($N=1$) with their double-loop configuration ($N=T_{total}^{1/2}$), we provide a unified analysis of a \textit{warm-start} double-loop architecture in Appendix~\ref{app:warm-start-double-loop}. We demonstrate that by initializing the inner critic with the parameter from the previous outer loop, the convergence rate is no longer restricted to a single loop split. Under the standard linear entropy penalty, the warm-start analysis yields a uniform $\tilde{\mathcal{O}}(T_{total}^{-1/4})$ convergence envelope across $N\in[1,\sqrt{T_{total}}]$, thereby connecting the single-loop and balanced double-loop endpoints. Under our Exponential Translation mechanism, the accelerated $\tilde{\mathcal{O}}(T_{total}^{-2/3})$ convergence rate remains invariant across the window $N\in[1,T_{total}^{2/9}]$.
\end{rem}

\subsection{Last-Iterate Unregularized Convergence}

We now extend the unregularized convergence bounds to the final output policy of Algorithm~\ref{alg:uncentered-nac}.
By applying the Exponential Translation mechanism to the terminal policy using the two-stage temperature schedule as in Theorem~\ref{thm:PMD-unreg-average}, we demonstrate an $\tilde{\mathcal{O}}(T^{-1/3})$ last-iterate convergence rate for the unregularized problem, bypassing the standard linear entropy penalty.

\begin{thm}[Unregularized Last-Iterate Convergence] \label{thm:PMD-unreg-last}
For Algorithm~\ref{alg:uncentered-nac}, suppose that Assumptions~\ref{ass:bounded-features}--\ref{ass:action-gap} and \ref{ass:density-ratios} hold.
Consider the same two-stage temperature as in Theorem~\ref{thm:PMD-unreg-average}:
$$ \tau_T \triangleq \max\left( \frac{\Delta}{2 \log(C_\gamma T^{2/3})}, \frac{\Delta}{2 \log(C_\gamma(1+\epsilon_{app}^{-1}))} \right). $$
If $T$ is sufficiently large and $\epsilon_{app}$ is sufficiently small such that $\tau_T \le \tau_{max}$, then by choosing the step-size parameters as in Theorem~\ref{thm:reg-last-iterate-gap} evaluated at $\tau = \tau_T$ subject to the stated conditions, the expected last-iterate unregularized performance gap satisfies the convergence rate:
\begin{align*}
\mathbb{E} \big[ J_0(\pi^*_0) - J_0(\pi_T) \big]
&\le \mathcal{O}\left( \frac{(1-\gamma)^{-4/3} + (\log T)^{4/3}}{(1-\gamma)^6 \kappa^4 \Delta^{7/3} T^{1/3}} \right)
 + \mathcal{O}\left( \frac{(1-\gamma)^{-1} + \log(1+\epsilon_{app}^{-1})}{(1-\gamma)^3 \kappa \Delta^2} \sqrt{\epsilon_{app}} \right) \\
&= \tilde{\mathcal{O}}\left( T^{-1/3} \right) + \tilde{\mathcal{O}}(\sqrt{\epsilon_{app}}).
\end{align*}
In the case of $\epsilon_{app}=0$, we interpret $\tau_T = \frac{\Delta}{2\log(C_\gamma T^{2/3})}$ and $\sqrt{\epsilon_{app}}\log(1+\epsilon_{app}^{-1})=0$.
\end{thm}

\subsection{Application to Tabular Setting} \label{sec:tabular}

To illustrate the strength of our theory, we establish new unregularized convergence guarantees for the tabular setting. In this setting, both the state space $\mathcal{S}$ and the action space $\mathcal{A}$ are finite, and the log-linear policy uses the one-hot feature encoding $\phi(s,a) \in \mathbb{R}^{|\mathcal{S}||\mathcal{A}|}$, i.e., the basis vector with $1$ at the index corresponding to $(s,a)$ and $0$ elsewhere. We synthesize the unregularized convergence bounds (Theorems~\ref{thm:PMD-unreg-average} and \ref{thm:PMD-unreg-last}) with the exploratory $\nu$-sampling extension (Appendix~\ref{app:nu_sampling}) and the concentrability-free $L_\infty$ approximation analysis (Appendix~\ref{app:linfty-bounds}). We demonstrate that Algorithm~\ref{alg:uncentered-nac} with a uniformly positive restart distribution $\nu$ achieves $\tilde{\mathcal{O}}(T^{-2/3})$ average-iterate and $\tilde{\mathcal{O}}(T^{-1/3})$ last-iterate convergence rates, under the Minimal Action Gap (Assumption~\ref{ass:action-gap}) as the only primitive condition, but without relying on any concentrability assumptions.

\begin{thm}[Unregularized Convergence in the Tabular Setting] \label{thm:tabular-convergence}
Consider a tabular MDP equipped with the one-hot feature representation. For Algorithm~\ref{alg:uncentered-nac} with an exploratory restart distribution $\nu$ satisfying $\min_{s,a}\nu(s,a)>0$ and the unbiased $Q$-value estimator $\hat{Q}_t$ constructed via geometrically stopped Monte Carlo rollouts, suppose that the Minimal Action Gap (Assumption~\ref{ass:action-gap}) holds.
Define
$ \kappa_{exp} \triangleq (1-\gamma)\min_{s,a}\nu(s,a)$
and the $T$-dependent temperature
$$ \tau_T \triangleq \frac{\Delta}{2\log(C_\gamma T^{2/3})}. $$
If $T$ is sufficiently large such that $\tau_T\le\tau_{max}$, then by choosing the step-size parameters as in Theorem~\ref{thm:linfty-reg-conv}, with $\kappa$ replaced by $\kappa_{exp}$ and evaluated at $\tau=\tau_T$, the expected global unregularized performance gap satisfies:
\begin{align*}
\frac{1}{T}\sum_{t=0}^{T-1}\mathbb{E}\big[J_0(\pi^*_0)-J_0(\pi_t)\big]
&\le \mathcal{O}\left(
\frac{(1-\gamma)^{-5/3}+(\log T)^{5/3}}
{(1-\gamma)^8\kappa_{exp}^6\Delta^{8/3}}
T^{-2/3}
\right), \\
\mathbb{E}\big[J_0(\pi^*_0)-J_0(\pi_T)\big]
&\le \mathcal{O}\left(
\frac{(1-\gamma)^{-4/3}+(\log T)^{4/3}}
{(1-\gamma)^6\kappa_{exp}^4\Delta^{7/3}}
T^{-1/3}
\right).
\end{align*}
\end{thm}

\begin{rem}[Breaking Tabular Stochastic Barrier] \label{rem:tabular-comparison}
For unregularized tabular MDPs, standard sample-based NPG and PMD algorithms---such as those analyzed by \citet{shani2020adaptive}, \citet{agarwal2021theory}, and \citet{lan2023policy}---yield an $\mathcal{O}(T_{total}^{-1/2})$ convergence rate (or $\mathcal{O}(1/\epsilon^2)$ sample complexity). While the minimax lower bound for sample-based tabular MDPs is known to be $\mathcal{O}(T_{total}^{-1/2})$ by worst-case environments with vanishing margins \citep{azar2013minimax}, our analysis operates in the instance-dependent regime. By running the stochastic updates in the entropy-regularized space (extracting the restorative $-\tau D_t^\dag$ contraction) and exploiting a positive Minimal Action Gap ($\Delta > 0$), our Exponential Translation mechanism bypasses this worst-case statistical barrier, mapping the regularized progress back to the unregularized objective to accelerate the average-iterate convergence rate to $\tilde{\mathcal{O}}(T_{total}^{-2/3})$. It remains an interesting open question whether the last-iterate rate can also be improved to $\tilde{\mathcal{O}}(T_{total}^{-2/3})$ via alternative analytical techniques or algorithmic designs.
\end{rem}

\addcontentsline{toc}{section}{References} 

\bibliographystyle{apalike}
\bibliography{NPG-draft}

\newpage

\appendix

\section{Algorithm Comparison} \label{app:critic-comparison}

In this section, we provide a detailed mathematical comparison between our uncentered algorithm and standard entropy-regularized Natural Policy Gradient (NPG) algorithms operating with compatible function approximation. Specifically, we contrast our approach with the standard advantage projection and the $Q$-value projection onto centered features in the entropy-regularized NPG with averaging studied by \cite{cayci2024convergence}.

\subsection{Standard NPG with Advantage-Driven, Centered Critic}
For a log-linear policy $\pi_{\omega_t} \propto \exp(\omega_t^\top \phi)$, the score function corresponds exactly to the centered feature vector $\tilde{\phi}_t(s,a) \triangleq \phi(s,a) - \mathbb{E}_{a' \sim \pi_t(\cdot|s)}[\phi(s,a')]$. For compatible function approximation, the standard approach evaluates the update direction by finding the optimal parameter $\theta^{*A}_\tau(\omega_t)$ that minimizes the mean squared error between the true regularized advantage and the centered feature approximation:
$$ \theta^{*A}_\tau (\omega_t) \triangleq \arg\min_{\theta \in \mathbb{R}^d} \mathbb{E}_{s \sim d^{\pi_t}, a \sim \pi_t} \left[ \big( A_\tau^{\pi_t}(s,a) - \theta^\top \tilde{\phi}_t(s,a) \big)^2 \right]. $$
Because the state-dependent baseline $V_\tau^{\pi_t}(s)$ is orthogonal to the centered features $\tilde{\phi}_t(s,a)$ under the action expectation, this projection can be equivalently defined directly without the state-value function:
\begin{align}
\theta^{*A}_\tau (\omega_t) = \arg\min_{\theta \in \mathbb{R}^d} \mathbb{E}_{s \sim d^{\pi_t}, a \sim \pi_t} \left[ \big( Q_\tau^{\pi_t}(s,a) - \tau \log \pi_t(a|s) - \theta^\top \tilde{\phi}_t(s,a) \big)^2 \right]. \label{eq:adv-proj-Q}
\end{align}
Throughout our discussion, when the relevant least-squares minimizers are not unique, the stated parameter equalities should be interpreted as corresponding relations between the minimizer sets (or the induced fitted functions).

By the Policy Gradient Theorem, the exact Euclidean gradient of the entropy-regularized objective evaluates to \citep{sutton1999policy}:
\begin{align}
 \nabla J_\tau(\omega_t) = \frac{1}{1-\gamma} \bar{\Sigma}_{cen}(\pi_t) \theta^{*A}_\tau(\omega_t),  \label{eq:policy-grad}
\end{align}
where $\bar{\Sigma}_{cen}(\pi_t) = \mathbb{E}_{s \sim d^{\pi_t}, a \sim \pi_t} \big[ \tilde{\phi}_t(s,a) \tilde{\phi}_t(s,a)^\top \big]$
is the Fisher information matrix (the centered feature covariance).
By preconditioning, the natural policy gradient is $\bar{\Sigma}_{cen}(\pi_t)^{-1} \nabla J_\tau(\omega_t) \propto \theta^{*A}_\tau(\omega_t)$
if the Fisher information matrix is invertible (or a generalized inverse is used if not invertible).
Thus, the standard NPG parameter update is defined directly by this advantage-driven, centered critic \citep{kakade2001natural, peters2008natural}:
$$ \omega_{t+1} = \omega_t + \eta_t \theta^{*A}_\tau(\omega_t). $$

\subsection{NPG with Q-Value-Driven, Centered Critic}
A closely related variant projects the regularized action-values directly onto the centered features. We define the $Q$-value-driven, centered critic as:
$$ \theta^{*Q}_\tau (\omega_t) \triangleq \arg\min_{\theta \in \mathbb{R}^d} \mathbb{E}_{s \sim d^{\pi_t}, a \sim \pi_t} \left[ \big( Q_\tau^{\pi_t}(s,a) - \theta^\top \tilde{\phi}_t(s,a) \big)^2 \right]. $$
Because any state-dependent mean $\mathbb{E}_{a'}[Q_\tau^{\pi_t}(s,a')]$ is orthogonal to $\tilde{\phi}_t(s,a)$, projecting the uncentered $Q$-values is identical to projecting the centered $Q$-values ($\tilde{Q}_\tau^{\pi_t}(s,a) \triangleq Q_\tau^{\pi_t}(s,a) - \mathbb{E}_{a'}[Q_\tau^{\pi_t}(s,a')]$) in the entropy-regularized NPG with averaging analyzed by \cite{cayci2024convergence}:
$$ \theta^{*CHS}_\tau (\omega_t) \triangleq \arg\min_{\theta \in \mathbb{R}^d} \mathbb{E}_{s \sim d^{\pi_t}, a \sim \pi_t} \left[ \big( \tilde{Q}_\tau^{\pi_t}(s,a) - \theta^\top \tilde{\phi}_t(s,a) \big)^2 \right]. $$
Therefore, $\theta^{*Q}_\tau(\omega_t) = \theta^{*CHS}_\tau(\omega_t)$, even in the presence of approximation error.

Next, we substitute the expression $\tau \log \pi_t(a|s) = \tau \omega_t^\top \phi(s,a) - \tau \log Z_t(s)$ into the equivalent advantage projection in \eqref{eq:adv-proj-Q}.
The orthogonality of $ \log Z_t(s)$ with the centered features $\tilde{\phi}_t(s,a)$ reveals that projecting the entropy penalty onto $\tilde{\phi}_t(s,a)$ yields exactly $-\tau \omega_t$. By the linearity of the least-squares projection, this establishes a geometric relationship between the advantage and $Q$-value projections (onto centered features):
\begin{align}
\theta^{*A}_\tau (\omega_t) = \theta^{*Q}_\tau (\omega_t) - \tau \omega_t.  \label{eq:proj-A-Q}
\end{align}
Consequently, the standard NPG update can be equivalently written in terms of the centered $Q$-value projection as:
$$ \omega_{t+1} = \omega_t + \eta_t \big( \theta^{*Q}_\tau(\omega_t) - \tau \omega_t \big). $$

\subsection{Equivalence under Exact Parameterization}
While the centered critic parameters maintain algebraic relationships with each other unconditionally, they can be related to our uncentered critic $\theta^*_\tau(\omega_t)$ under the idealized assumption of zero approximation error.

\begin{lem}[Critic Equivalence under Exact Parameterization] \label{lem:critic-equiv}
Assume there is no approximation error in the uncentered critic, such that $Q_\tau^{\pi_t}(s,a) = \theta^*_\tau(\omega_t)^\top \phi(s,a)$ holds exactly for all states and actions. Then the centered projections satisfy:
$$ \theta^{*Q}_\tau(\omega_t) = \theta^{*CHS}_\tau(\omega_t)= \theta^*_\tau(\omega_t) \quad
\text{and} \quad \theta^{*A}_\tau(\omega_t) = \theta^*_\tau(\omega_t) - \tau \omega_t. $$
\end{lem}

\begin{prf}
Substituting the exact parameterization into the centered $Q$-value yields:
\begin{align*}
\tilde{Q}_\tau^{\pi_t}(s,a) &= \theta^*_\tau(\omega_t)^\top \phi(s,a) - \mathbb{E}_{a' \sim \pi_t} \big[ \theta^*_\tau(\omega_t)^\top \phi(s,a') \big] \\
&= \theta^*_\tau(\omega_t)^\top \big( \phi(s,a) - \mathbb{E}_{a' \sim \pi_t} [\phi(s,a')] \big) \\
&= \theta^*_\tau(\omega_t)^\top \tilde{\phi}_t(s,a).
\end{align*}
Because $\tilde{Q}_\tau^{\pi_t}(s,a)$ is perfectly captured by the centered features with parameter $\theta^*_\tau(\omega_t)$, the projection trivially yields $\theta^{*CHS}_\tau(\omega_t) = \theta^*_\tau(\omega_t)$.
As established in the text,
$\theta^{*Q}_\tau(\omega_t) =\theta^{*CHS}_\tau(\omega_t)$ and $\theta^{*A}_\tau (\omega_t) = \theta^{*Q}_\tau (\omega_t) - \tau \omega_t$.
Substituting these identities yields $\theta^{*Q}_\tau(\omega_t) = \theta^*_\tau(\omega_t)$
and $\theta^{*A}_\tau(\omega_t) = \theta^*_\tau(\omega_t) - \tau \omega_t$, completing the proof.
\end{prf}

It is important to note that this exact parameter-space equivalence between the uncentered and centered critics relies on the absence of approximation error. In general, under compatible function approximation with non-zero approximation error, this equivalence does not hold, and the respective critic parameters
may differ because they solve different projection problems.

\subsection{Algorithmic Design and Comparison}
The equivalence relation $\theta^{*A}_\tau(\omega_t) = \theta^*_\tau(\omega_t) - \tau \omega_t$ holds exactly under exact parameterization. Motivated by this idealized equivalence, our algorithm explicitly chooses to define the parameter update using the uncentered critic parameter directly:
$$ \omega_{t+1} = \omega_t + \eta_t \big( \theta^*_\tau(\omega_t) - \tau \omega_t \big). $$

In practice, operating with the uncentered critic parameter $\theta^*_\tau(\omega_t)$ incurs a structural approximation bias relative to the exact NPG direction.
However, as we established in the Uncentered Gradient Identity (Lemma~\ref{lem:uncentered-gradient}), the true policy gradient decomposes as:
$$ \nabla J_\tau(\omega_t) = \frac{1}{1-\gamma} \bar{\Sigma}_{cen}(\pi_t) \big( \theta^*_\tau(\omega_t) - \tau \omega_t \big) + \frac{1}{1-\gamma} E_{bias}. $$
Multiplying by the inverse Fisher information matrix $\bar{\Sigma}_{cen}(\pi_t)^{-1}$ reveals that our chosen update direction, $\theta^*_\tau(\omega_t) - \tau \omega_t$,
differs from the scaled exact natural gradient by the residual term $\bar{\Sigma}_{cen}(\pi_t)^{-1} E_{bias}$, which is bounded by the approximation error of the uncentered representation $\theta^*_\tau(\omega_t)^\top \phi(s,a)$ in the Actor Progress Bound (Lemma~\ref{lem:pl-ascent}).

\begin{rem}[Centered vs. Uncentered Critic] \label{rem:centered-vs-uncentered}
Our analysis in Part I can be readily extended to a centered NAC algorithm (e.g., utilizing the centered $Q$-value critic $\theta^{*Q}_\tau$). By using the centered critic, the policy gradient evaluates to the exact natural policy gradient update direction without structural bias (from \eqref{eq:policy-grad} and \eqref{eq:proj-A-Q}):
$$ \nabla J_\tau(\omega_t) = \frac{1}{1-\gamma} \bar{\Sigma}_{cen}(\pi_t) \big( \theta^{*Q}_\tau(\omega_t) - \tau \omega_t \big). $$
Consequently, the Actor Progress Bound can be established free of the $\epsilon_{app}$ bias term. However, the corresponding PL condition still involves an $\epsilon_{app}$ term. Therefore, the resulting Coupled Actor Recurrence inequality retains a similar $\epsilon_{app}$ dependence, yielding similar convergence rates to those of the uncentered NAC algorithm in the Stochastic Regime.

In contrast, the centered and uncentered formulations behave differently in the Deterministic Regime considered in Part II. As the temperature drops ($\tau \to 0$), the deterministic collapse forces the minimum eigenvalue of the centered feature covariance to vanish ($\lambda_\tau \to 0$), eliminating the uniform positive-definiteness needed to control the centered least-squares projection. Consequently, the centered NAC formulation would not be covered by our Part II analysis under the current assumptions. Our use of the uncentered critic circumvents this degenerating geometry by relying instead on the Positive-Definite Uncentered Feature Moment (Assumption~\ref{ass:pd-uncentered}), while introducing a residual term that can be controlled by the critic approximation error.
\end{rem}

\section{Generalization to Exploratory $\nu$-Sampling}
\label{app:nu_sampling}

In the main text, Algorithm~\ref{alg:uncentered-nac} generates trajectories starting from the objective's initial state distribution $\mu$, effectively sampling from the on-policy visitation measure. In this section, we generalize our theoretical framework to accommodate algorithms that utilize an exploratory initial state-action restart distribution, $\nu(s,a)$, a standard mechanism of using off-policy data to improve action-space exploration for the critic \citep{agarwal2021theory, yuan2023linear}.

\subsection{Generalized Sampling Measure and Critic Assumptions} \label{app:generalized-sampling}

Under generalized $\nu$-sampling, at each iteration $t$, Algorithm~\ref{alg:uncentered-nac} draws an initial state-action pair $(s_0, a_0) \sim \nu$, executes $a_0$, and thereafter follows the current policy $\pi_t$. This induces the generalized discounted training measure:
$$d^{\pi_t}_\nu(s,a) \triangleq (1-\gamma) \sum_{k=0}^\infty \gamma^k \mathbb{P} \big(s_k = s, a_k = a \mid (s_0, a_0) \sim \nu, a_{1:k} \sim \pi_t \big).$$
Because all subsequent probabilities in the geometric sum are non-negative, this definition provides a simple, policy-independent lower bound on state-action coverage by the first step: $d^{\pi_t}_\nu(s,a) \ge (1-\gamma)\nu(s,a)$.
In the on-policy specialization where the restart distribution at iteration $t$ is $\nu_t(s,a)=\mu(s)\pi_t(a|s)$, the training measure reduces to the standard on-policy distribution: $d^{\pi_t}_{\nu_t}(s,a)=d^{\pi_t}(s)\pi_t(a|s)$.

Because the sample-based critic is trained via SGD over trajectories drawn from $d^{\pi_t}_\nu$, it natively minimizes the least-squares error weighted by this specific distribution. Consequently, the ideal tracking target is redefined to align with the training measure:
$$\theta^*_\tau(\omega_t) \triangleq \arg\min_{\theta \in \mathbb{R}^d} \mathbb{E}_{(s,a) \sim d^{\pi_t}_\nu} \left[ \big( Q_\tau^{\pi_t}(s,a) - \theta^\top \phi(s,a) \big)^2 \right].$$
By the first-order optimality condition, the approximation error $\epsilon_t(s,a) \triangleq Q_\tau^{\pi_t}(s,a) - \theta^*_\tau(\omega_t)^\top \phi(s,a)$ is orthogonal to the features under the generalized training measure: $\mathbb{E}_{d^{\pi_t}_\nu}[\phi \epsilon_t] = 0$.
The critic tracking error is defined as before, $e_t \triangleq \theta_{t+1} - \theta^*_\tau(\omega_t)$, but now relative to the redefined ideal critic.

By the critic's new training measure, the following assumption replaces the Critic $L_2$ Approximation Error (Assumption~\ref{ass:approximation-error}) from the main text.

\begin{ass}[Critic $L_2$ Approximation Error under Exploratory Training Measure] \label{ass:generalized-approximation-error}
The MSE of the ideal uncentered critic under the exploratory training measure is bounded by $\epsilon_{app}$ for all training policies $\{\pi_t\}_{t\ge0}$ under Algorithm~\ref{alg:uncentered-nac}:
$$
\mathbb{E}_{d^{\pi_t}_\nu}[\epsilon_t(s,a)^2]\le\epsilon_{app}.
$$
Additionally, for the global analysis in Part I, we assume the same bound holds at the regularized parametric optimum:
$$
\mathbb{E}_{d^{\omega^*_\tau}_\nu}[\epsilon_*(s,a)^2]\le\epsilon_{app},
$$
where $d^{\omega^*_\tau}_\nu \triangleq d^{\pi_{\omega^*_\tau}}_\nu$ and $\epsilon_*(s,a)\triangleq Q_\tau^{\pi_{\omega^*_\tau}}(s,a)-\theta^*_\tau(\omega^*_\tau)^\top\phi(s,a)$. As in Assumption~\ref{ass:approximation-error}, we assume $\epsilon_{app}$ is independent of $\tau\in(0,\tau_{max}]$.
\end{ass}

For any state-action distribution $d$, such as $\nu$ and $d^{\pi_t}_\nu$, we define the uncentered feature moment matrices evaluated under $d$ as:
$$ \bar{\Sigma}_{unc}(d) \triangleq \mathbb{E}_{(s,a) \sim d} \big[ \phi(s,a) \phi(s,a)^\top \big].$$
The following assumption replaces the Positive-Definite Uncentered Feature Moment (Assumption~\ref{ass:pd-uncentered}) from the main text.
Note that the Positive-Definite Fisher Information (Assumption~\ref{ass:pd-fim}) remains necessary in Part I, as the actor's update is governed by the objective's on-policy distribution and unaffected by the critic's exploratory training measure.

\begin{ass}[Positive-Definite Moment under Exploratory Training Measure] \label{ass:generalized-nu-exp}
We assume that the uncentered feature moment matrix evaluated under the training measure is uniformly positive-definite. Specifically, there exists a constant $\kappa_{exp} > 0$, independent of the temperature $\tau \in (0, \tau_{max}]$,
such that for all $t \ge 0$ under Algorithm~\ref{alg:uncentered-nac}:
$$ \bar{\Sigma}_{unc}(d^{\pi_t}_\nu) \succeq \kappa_{exp} I. $$
\end{ass}

Because $d^{\pi_t}_\nu \ge (1-\gamma)\nu$, this assumption holds natively if the restart distribution itself supports the feature space, i.e., $\bar{\Sigma}_{unc}(\nu) \succeq \kappa_0 I$, yielding $\kappa_{exp} = (1-\gamma)\kappa_0$.

\textbf{Impact of $\kappa_{exp}$ on Structural Bounds:}
Because the transition to $\nu$-sampling modifies the underlying training measure, we briefly note its impact on the structural bounds:
\begin{itemize}
    \item[(i)] \textit{Critic and Actor Bounds:} In the main text, the bound on the ideal critic is $B_\theta = \frac{B_\phi V_{max}}{\kappa}$. Under $\nu$-sampling, this definition updates to $\frac{B_\phi V_{max}}{\kappa_{exp}}$. The actor gradient bound, $G_{ac} = 2B_\theta$, changes accordingly. Both bounds remain finite, temperature-independent constants.

\item[(ii)] \textit{Lipschitz Critic ($L_\theta$):}
The general derivative bound \eqref{eq:gen-deriv-limit} used in the proof of Lemma~\ref{lem:lipschitz-critic} continues to hold when the on-policy measure is replaced by $d_\nu^{\pi_\omega}$. Indeed, the restart distribution $\nu$ is independent of $\omega$, and hence only the subsequent policy-dependent actions contribute score-function terms, whose discounted sum is bounded by the same $\frac{2B_\phi}{1-\gamma}$ factor.
Therefore, the ideal critic $\theta^*_\tau(\omega_t)$ remains Lipschitz continuous as in Lemma~\ref{lem:lipschitz-critic}, with $\kappa$ replaced by $\kappa_{exp}$ in the Lipschitz constant $L_{\theta,\tau}$ and $L_\theta \triangleq L_{\theta,\tau_{max}}$.
\end{itemize}

\subsection{Generalized Off-Policy Mismatches}

Under exploratory $\nu$-sampling, the critic approximation error is controlled under the training measure $d^{\pi_t}_\nu$, whereas the actor progress and comparator-based performance bounds are evaluated under different state-action measures. We therefore introduce generalized off-policy mismatch conditions to control the resulting changes of measure.

\begin{ass}[Generalized Off-Policy Mismatches] \label{ass:generalized-off-policy}
We assume the following density ratios are bounded by finite constants independent of the temperature $\tau\in(0,\tau_{max}]$:
\begin{enumerate}
    \item[(i)] \textit{Training and Optimum Off-Policy Mismatch (For Part I):}
$$
\sup_{t\ge0}\sup_{s,a}
\frac{d^{\pi_t}(s)\pi_t(a|s)}{d^{\pi_t}_\nu(s,a)}
\le C_{\mu,\nu},
\qquad
\sup_{s,a}
\frac{d^{\omega^*_\tau}(s)\pi_{\omega^*_\tau}(a|s)}
{d^{\omega^*_\tau}_\nu(s,a)}
\le C_{\mu,\nu}.
$$
    \item[(ii)] \textit{Parametric Off-Policy Joint Mismatch (For Part I):}
    $$ \sup_{t\ge0}\sup_{s, a} \left[ \frac{d^{\omega^*_\tau}(s) \pi_t(a|s)}{d^{\pi_t}_\nu(s,a)} \max_{a'} \frac{\pi_{\omega^*_\tau}(a'|s)}{\pi_t(a'|s)} \right] \le C_{joint,\nu}. $$
    \item[(iii)] \textit{Global-to-Parametric Off-Policy Joint Mismatch (For Part I):}
    $$\sup_{s,a}
    \left[
    \frac{d^*_\tau(s)\pi_{\omega^*_\tau}(a|s)}
    {d^{\omega^*_\tau}_\nu(s,a)}
    \max_{a'}\frac{\pi^*_\tau(a'|s)}
    {\pi_{\omega^*_\tau}(a'|s)}
    \right]
    \le C_{joint,\nu}^*.
    $$
    \item[(iv)] \textit{Global Off-Policy Joint Mismatch (For Part II):}
    $$ \sup_{t\ge0}\sup_{s, a} \left[ \frac{d^*_\tau(s) \pi_t(a|s)}{d^{\pi_t}_\nu(s,a)} \max_{a'} \frac{\pi^*_\tau(a'|s)}{\pi_t(a'|s)} \right] \le C_{joint,\nu}^\dag. $$
\end{enumerate}
\end{ass}

In Part I, the off-policy mismatch bounds $C_{\mu, \nu}$, $C_{joint,\nu}$, and $C_{joint,\nu}^*$ operate alongside the parametric joint concentrability bound $C_{joint}$ (Assumption~\ref{ass:parametric-density}).
In Part II, the off-policy mismatch bound $C_{joint,\nu}^\dag$ generalizes the global joint concentrability bound $C_{joint}^\dag$ (Assumption~\ref{ass:density-ratios}).
We make three comments on the generalized off-policy mismatches.

\textbf{Reduction to the Main Text.}
Our main-text setting is recovered by choosing the on-policy restart distribution $\nu_t(s,a)=\mu(s)\pi_t(a|s)$ for each training policy and, at the parametric optimum, $\nu^*_\tau(s,a)=\mu(s)\pi_{\omega^*_\tau}(a|s)$. Then
$d^{\pi_t}_{\nu_t}(s,a)=d^{\pi_t}(s)\pi_t(a|s)$ and
$d^{\pi_{\omega^*_\tau}}_{\nu^*_\tau}(s,a)=d^{\omega^*_\tau}(s)\pi_{\omega^*_\tau}(a|s)$.
Consequently, the Training and Optimum Off-Policy Mismatch holds with $C_{\mu,\nu}=1$, while the Parametric, Global-to-Parametric, and Global Off-Policy Joint Mismatches reduce to our original joint concentrability bounds $C_{joint}$, $C_{joint}^*$, and $C_{joint}^\dag$ in Assumptions~\ref{ass:parametric-density}, \ref{ass:global-parametric-density}, and \ref{ass:density-ratios} respectively.

\textbf{Relation to Separate Mismatch Bounds:} Similarly to our discussions of Assumption~\ref{ass:parametric-density} and \ref{ass:density-ratios} in the main text, the parametric off-policy joint mismatch (Assumption~\ref{ass:generalized-off-policy}(ii)) is naturally implied by the conjunction of separate off-policy mismatch and policy concentrability bounds. Specifically, if the following holds for some constants $(C_\nu, W)$ independent of $\tau \in (0, \tau_{max}]$:
\begin{itemize}
    \item \textit{Parametric Off-Policy Mismatch:}
    $ \sup_{s, a} \frac{d^{\omega^*_\tau}(s) \pi_t(a|s)}{d^{\pi_t}_\nu(s,a)} \le C_\nu$,
    \item \textit{Parametric Policy Concentrability:}
    $ \sup_{s,a} \frac{\pi_{\omega^*_\tau}(a|s)}{\pi_t(a|s)} \le W$,
\end{itemize}
then by the multiplicative rule, the joint bound is satisfied with $C_{joint,\nu} = C_\nu W$.
Conversely, because $\mathbb{E}_{a' \sim \pi_t} \big[ \frac{\pi_{\omega^*_\tau}(a'|s)}{\pi_t(a'|s)} \big] = 1$, Assumption~\ref{ass:generalized-off-policy}(ii)
implies the parametric off-policy mismatch bound with $C_\nu = C_{joint,\nu}$:
$$ \sup_{s,a} \frac{d^{\omega^*_\tau}(s) \pi_t(a|s)}{d^{\pi_t}_\nu(s,a)} \le C_{joint,\nu}. $$
Furthermore, by taking $a' = a$ inside the supremum, Assumption~\ref{ass:generalized-off-policy}(ii) also implies:
$$ \sup_{s, a} \frac{d^{\omega^*_\tau}(s) \pi_{\omega^*_\tau}(a|s)}{d^{\pi_t}_\nu(s,a)} \le C_{joint,\nu}. $$

Identical relations hold for the global-to-parametric off-policy joint mismatch (Assumption~\ref{ass:generalized-off-policy}(iii))
and for the global off-policy joint mismatch (Assumption~\ref{ass:generalized-off-policy}(iv)).
In particular, the latter joint bound is satisfied with $C_{joint,\nu}^\dag = C_\nu^\dag W^\dag$ under the conjunction of the following density ratio bounds for some constants $(C_\nu^\dag, W^\dag)$ independent of $\tau \in (0, \tau_{max}]$:
\begin{itemize}
    \item \textit{Global Off-Policy Mismatch:}
    $ \sup_{s, a} \frac{d^*_\tau(s) \pi_t(a|s)}{d^{\pi_t}_\nu(s,a)} \le C_\nu^\dag$,
    \item \textit{Global Policy Concentrability:}
    $ \sup_{s,a} \frac{\pi^*_\tau(a|s)}{\pi_t(a|s)} \le W^\dag$.
\end{itemize}
Conversely, similar to the parametric case, Assumption~\ref{ass:generalized-off-policy}(iv) implies the global off-policy mismatch bound with $C_\nu^\dag = C_{joint,\nu}^\dag$:
\begin{align}
\sup_{s, a} \frac{d^*_\tau(s) \pi_t(a|s)}{d^{\pi_t}_\nu(s,a)} \le C_{joint,\nu}^\dag.  \label{eq:global-mismatch-nu-a}
\end{align}
By taking $a' = a$, Assumption~\ref{ass:generalized-off-policy}(iv) also implies:
\begin{align}
\sup_{s, a} \frac{d^*_\tau(s) \pi^*_\tau(a|s)}{d^{\pi_t}_\nu(s,a)} \le C_{joint,\nu}^\dag. \label{eq:global-mismatch-nu-b}
\end{align}

\textbf{Elimination of Independent Policy Concentrability:}
By embedding the local policy mismatches
$$
W(s)\triangleq\max_{a'}\frac{\pi_{\omega^*_\tau}(a'|s)}{\pi_t(a'|s)},\qquad
W^*(s)\triangleq\max_{a'}\frac{\pi^*_\tau(a'|s)}{\pi_{\omega^*_\tau}(a'|s)},\qquad
W^\dag(s)\triangleq\max_{a'}\frac{\pi^*_\tau(a'|s)}{\pi_t(a'|s)}
$$
into the generalized joint mismatch assumptions
$C_{joint,\nu}$, $C_{joint,\nu}^*$, and $C_{joint,\nu}^\dag$ respectively, we avoid requiring separate policy concentrability bounds, for example $W$ and $W^\dag$, under exploratory $\nu$-sampling. This mirrors the consolidation in the main text: the local policy ratios themselves need not be uniformly bounded, provided the evaluation measures, such as $d^{\omega^*_\tau}(s) \pi_t(a|s)$ and $d^*_\tau(s) \pi_t(a|s)$, are sufficiently small relative to the exploratory training measure $d^{\pi_t}_\nu(s,a)$.

\subsection{Impact on Part I (Stochastic Regime)}

In Part I, algorithmic progress is evaluated under the on-policy discounted state-action measure $d^{\pi_t}\times\pi_t$, induced by the objective’s initial state distribution $\mu$. If the critic is instead trained with an exploratory restart distribution $\nu$, a distribution mismatch occurs between the actor's on-policy evaluation metrics and the critic's training measure.

\textbf{Extension of Uncentered Gradient Identity (Lemma~\ref{lem:uncentered-gradient}):}
Under exploratory $\nu$-sampling, the critic orthogonality condition $\mathbb{E}_{d^{\pi_t}_\nu}[\phi\epsilon_t]=0$ holds under the training measure rather than the objective's on-policy measure $d^{\pi_t}\times\pi_t$. Consequently, the term $\mathbb{E}_{d^{\pi_t},\pi_t}[\phi\epsilon_t]$ no longer vanishes in the policy gradient decomposition, yielding
$$
\nabla J_\tau(\omega_t)
=\frac{1}{1-\gamma}\bar{\Sigma}_{cen}(\pi_t)\big(\theta^*_\tau(\omega_t)-\tau\omega_t\big)
+\frac{1}{1-\gamma}E_{bias}
+\frac{1}{1-\gamma}E_{mismatch},
$$
where the mismatch term is defined as $E_{mismatch} \triangleq \mathbb{E}_{d^{\pi_t}, \pi_t} \big[ \phi(s,a) \epsilon_t(s,a) \big]$.

\textbf{Propagation of $C_{\mu, \nu}$ into Actor Progress Bound (Lemma~\ref{lem:pl-ascent}):}
Because both the new mismatch term $E_{mismatch}$ and the original approximation bias $E_{bias}$ are defined as expectations under the objective's on-policy measure $d^{\pi_t} \times \pi_t$, we shift the measure to $d^{\pi_t}_\nu$ using the Training Off-Policy Mismatch $C_{\mu, \nu}$ (Assumption~\ref{ass:generalized-off-policy}(i)). The mismatch term is bounded as:
$$\|E_{mismatch}\|_2 \le B_\phi \sqrt{C_{\mu, \nu} \mathbb{E}_{d^{\pi_t}_\nu} [\epsilon_t^2]} \le B_\phi \sqrt{C_{\mu, \nu} \epsilon_{app}}.$$
Similarly, the original bias term is bounded by $\|E_{bias}\|_2 \le B_\phi \sqrt{C_{\mu, \nu} \epsilon_{app}}$.
Bounding the inner products $\langle E_{mismatch}, \Delta \omega_t\rangle$ and $\langle E_{bias}, \Delta \omega_t\rangle$ using Young's inequality leads to a progress bound where the overall bias constant $C_{bias}$ is inflated by a factor proportional to $C_{\mu, \nu}$.
In contrast, the joint concentrability bound $C_{joint}$ is used to lower-bound the on-policy Fisher information matrix ($\bar{\Sigma}_{cen}(\pi_t) \succeq \frac{\lambda}{C_{joint}} I$), which involves evaluating the current policy directly against the parametric optimum and is independent of the critic's exploratory measure $\nu$. Thus the original $C_{joint}$ bound (Assumption~\ref{ass:parametric-density}) applies without modification.

\textbf{Propagation into PL Condition (Lemma~\ref{lem:pl-condition}):}
In Step (ii) of the PL Condition proof, when bounding the approximation bias term, we first define the local policy mismatch $W(s) \triangleq \max_{a'} \frac{\pi_{\omega^*_\tau}(a'|s)}{\pi_t(a'|s)}$. Applying the Cauchy-Schwarz inequality over the action and state spaces, alongside the $\chi^2$ bound, isolates the error term $\mathbb{E}_{d^{\omega^*_\tau}} \big[ W(s) \mathbb{E}_{\pi_t}[\epsilon_t^2] \big]$. To control this term, we shift the expectation from the measure $d^{\omega^*_\tau} \times \pi_t$ to the training measure $d^{\pi_t}_\nu$ via the Parametric Off-Policy Joint Mismatch $C_{joint,\nu}$ (Assumption~\ref{ass:generalized-off-policy}(ii)):
\begin{align*} \mathbb{E}_{d^{\omega^*_\tau}} \big[ W(s) \mathbb{E}_{\pi_t}[\epsilon_t^2] \big] &= \mathbb{E}_{d^{\pi_t}_\nu} \left[ \left( \frac{d^{\omega^*_\tau}(s) \pi_t(a|s)}{d^{\pi_t}_\nu(s,a)} W(s) \right) \epsilon_t(s,a)^2 \right] \\
& \le C_{joint,\nu} \mathbb{E}_{d^{\pi_t}_\nu} [ \epsilon_t(s,a)^2 ] \le C_{joint,\nu} \epsilon_{app}.
\end{align*}
For the ideal algorithmic progress term, the bound in Step (i) relies on H\"{o}lder's and Pinsker's inequalities and remains unaffected.
Repeating the remaining steps in the proof of Lemma~\ref{lem:pl-condition} therefore yields the same PL condition with
$C_{err}$ replaced by
$C_{err,\nu}\triangleq\frac{C_{joint,\nu}}{1-\gamma}$.

\textbf{Propagation into Global Class Approximation Error (Lemma~\ref{lem:global-class-error}):}
At the parametric optimum $\pi_{\omega^*_\tau}$, the extended Uncentered Gradient Identity gives
$$
0=\bar{\Sigma}_{cen}(\pi_{\omega^*_\tau})
\big(\theta^*_\tau(\omega^*_\tau)-\tau\omega^*_\tau\big)
+E_{bias}^*+E_{mismatch}^*,
$$
where $E_{bias}^*$ and $E_{mismatch}^*$ are defined at $\pi_{\omega^*_\tau}$, analogously to $E_{bias}$ and $E_{mismatch}$.
By the Optimum Off-Policy Mismatch $C_{\mu,\nu}$ and Assumption~\ref{ass:generalized-approximation-error},
$\|E_{bias}^*\|_2\le B_\phi\sqrt{C_{\mu,\nu}\epsilon_{app}}$ and
$\|E_{mismatch}^*\|_2\le B_\phi\sqrt{C_{\mu,\nu}\epsilon_{app}}$,
and hence the Positive-Definite Fisher Information assumption gives
$$
\|\theta^*_\tau(\omega^*_\tau)-\tau\omega^*_\tau\|_2
\le\frac{2B_\phi}{\lambda}\sqrt{C_{\mu,\nu}\epsilon_{app}}.
$$
For the approximation bias term in Step (iii) of Lemma~\ref{lem:global-class-error}, defining
$W^*(s)\triangleq\max_{a'}\frac{\pi^*_\tau(a'|s)}{\pi_{\omega^*_\tau}(a'|s)}$,
the Global-to-Parametric Off-Policy Joint Mismatch $C_{joint,\nu}^*$ yields
\begin{align*}
\mathbb{E}_{d^*_\tau}\big[W^*(s)\mathbb{E}_{\pi_{\omega^*_\tau}}[\epsilon_*^2]\big]
&=\mathbb{E}_{d^{\omega^*_\tau}_\nu}
\left[
\left(
\frac{d^*_\tau(s)\pi_{\omega^*_\tau}(a|s)}
{d^{\omega^*_\tau}_\nu(s,a)}
W^*(s)
\right)\epsilon_*(s,a)^2
\right] \\
&\le C_{joint,\nu}^*
\mathbb{E}_{d^{\omega^*_\tau}_\nu}[\epsilon_*^2]
\le C_{joint,\nu}^*\epsilon_{app}.
\end{align*}
Repeating the remaining steps in the proof of Lemma~\ref{lem:global-class-error} therefore yields
$$
J_\tau(\pi^*_\tau)-J_\tau(\pi_{\omega^*_\tau})
\le\frac{C_{class,\nu}}{\tau}\epsilon_{app},
$$
where
$$
C_{class,\nu}
\triangleq
\frac{4B_\phi^4C_{\mu,\nu}}{(1-\gamma)\lambda^2}
+\frac{C_{joint,\nu}^*}{1-\gamma}.
$$

The lower-bound argument in Lemma~\ref{lem:part-I-compatibility} is unaffected by the exploratory critic training measure, yielding
$$
\epsilon_{app}\ge C_{comp,\nu}\tau, \quad \tau\in(0,\tau_{max}],
$$
where
$$C_{comp,\nu} \triangleq
\frac{q_0\Delta}{2(1-\gamma)C_{class,\nu}}
\min\left\{1, \frac{q_0\Delta}{2\tau_{max}\log|\mathcal{A}|} \right\},
\quad
q_0\triangleq\frac{d\lambda}{4B_\phi^2}.
$$
Thus, the Part-I approximation-temperature compatibility persists under exploratory $\nu$-sampling, with the compatibility constant modified through $C_{class,\nu}$.

\textbf{Impact on Critic Tracking:}
The critic tracking analysis relies on the orthogonality based on the ideal critic. Taking the conditional expectation of the critic SGD step yields:
$$\mathbb{E}[g_{t+1}^{cr} \mid \mathcal{F}_{t+1}] = \bar{\Sigma}_{unc}(d^{\pi_{t+1}}_\nu) \tilde{e}_t - \mathbb{E}_{d^{\pi_{t+1}}_\nu} \big[ \phi(s,a) \epsilon_{t+1}(s,a) \big],$$
where $\tilde{e}_t \triangleq \theta_{t+1} - \theta^*_\tau(\omega_{t+1})$.
Redefining the ideal critic and the associated errors to match the $\nu$-measure (Appendix~\ref{app:generalized-sampling}) zeros out the cross-term, ensuring the expected update remains unbiased. By Assumption~\ref{ass:generalized-nu-exp}, the tracking argument in the main text applies with $\kappa$ replaced by $\kappa_{exp}$, yielding the corresponding contraction factor $1-\alpha_{t+1}\kappa_{exp}$ in the final tracking recurrence.

\textbf{Conclusion for Part I:}
Under exploratory $\nu$-sampling, the uncentered feature moment lower bound $\kappa$ is replaced by $\kappa_{exp}$. The off-policy shifts associated with critic approximation along the algorithmic trajectory are controlled by $C_{\mu,\nu}$ and $C_{joint,\nu}$, while the global class approximation error is controlled by $C_{\mu,\nu}$ and $C_{joint,\nu}^*$. The original joint concentrability $C_{joint}$ remains responsible for lower-bounding the on-policy Fisher information matrix. The Part-I approximation-temperature compatibility persists, with $C_{class}$ and $C_{comp}$ replaced by $C_{class,\nu}$ and $C_{comp,\nu}$, respectively. Subject to these generalized assumptions, the corresponding replacement of the critic-side structural constants, and the resulting admissible range for $\epsilon_{app}$, the Part I analysis carries over. With respect to the iteration budget $T$, the final average-iterate and last-iterate unregularized convergence rates remain $\tilde{\mathcal{O}}(T^{-1})$ as in the main text.

\subsection{Impact on Part II (Deterministic Regime)}

In Part II, algorithmic progress is evaluated against the globally optimal policy $\pi^*_\tau$, rather than the parametrically optimal policy $\pi_{\omega^*_\tau}$ analyzed in Part I. Consequently, when shifting the expectation of the approximation error from the target evaluation measure to the training measure, the analysis relies on the Global Off-Policy Joint Mismatch evaluated under $d^*_\tau$ (Assumption~\ref{ass:generalized-off-policy}(iv)).

\textbf{Extension of PMD Actor Progress Bound (Lemma~\ref{lem:pmd-telescoping-l2}):}
In Step (ii) of the PMD Actor Progress Bound proof, when bounding the approximation bias term, we first define the local policy mismatch $W^\dag(s) \triangleq \max_{a'} \frac{\pi^*_\tau(a'|s)}{\pi_t(a'|s)}$. Applying the Cauchy-Schwarz inequality over the action and state spaces, alongside the $\chi^2$ bound, isolates the error term $\mathbb{E}_{d^*_\tau} \big[ W^\dag(s) \mathbb{E}_{\pi_t}[\epsilon_t^2] \big]$. To control this term, we shift the expectation from the measure $d^*_\tau \times \pi_t$ to the training measure $d^{\pi_t}_\nu$ via the Global Off-Policy Joint Mismatch $C_{joint,\nu}^\dag$ (Assumption~\ref{ass:generalized-off-policy}(iv)):
\begin{align*}
\mathbb{E}_{d^*_\tau} \big[ W^\dag(s) \mathbb{E}_{\pi_t}[\epsilon_t^2] \big] &= \mathbb{E}_{d^{\pi_t}_\nu} \left[ \left( \frac{d^*_\tau(s) \pi_t(a|s)}{d^{\pi_t}_\nu(s,a)} W^\dag(s) \right) \epsilon_t(s,a)^2 \right] \\
&\le C_{joint,\nu}^\dag \mathbb{E}_{d^{\pi_t}_\nu} [ \epsilon_t(s,a)^2 ] \le C_{joint,\nu}^\dag \epsilon_{app}.
\end{align*}
For the tracking error penalty, the bound in Step (iii) remains unaffected without requiring any state or policy mismatch bounds.

\textbf{Extension of Critic Tracking Bound (Lemma~\ref{lem:decoupled-critic}):}
The analysis of critic tracking relies on the same orthogonality mechanism established in Part I. As shown above, redefining the ideal critic zeroes out the cross-term. The proof of Lemma~\ref{lem:decoupled-critic} therefore applies with the original feature moment lower bound $\kappa$ replaced by the exploratory restart parameter $\kappa_{exp}$, yielding the corresponding contraction factor $1-\alpha_{t+1}\kappa_{exp}$ in the final tracking recurrence.

\textbf{Conclusion for Part II:}
Under exploratory $\nu$-sampling, the critic's training geometry is governed by the exploratory feature moment lower bound $\kappa_{exp}$ in place of $\kappa$, while the off-policy shift of the critic approximation error is controlled by the Global Off-Policy Joint Mismatch $C_{joint,\nu}^\dag$ in place of the joint concentrability bound $C_{joint}^\dag$. By these replacements, the PMD Actor Progress and critic tracking arguments hold as in the main text, with only the corresponding structural constants modified. Consequently, the subsequent Part II analysis applies directly. With respect to the iteration budget $T$, the final average-iterate and last-iterate unregularized convergence rates remain $\tilde{\mathcal{O}}(T^{-2/3})$ and $\tilde{\mathcal{O}}(T^{-1/3})$, respectively as in the main text.

\section{Concentrability-Free Analysis via $L_\infty$ Approximation} \label{app:linfty-bounds}

In the main text, our analysis relies on bounding the critic's approximation error via the $L_2$ mean squared error ($\epsilon_{app}$ in Assumption~\ref{ass:approximation-error}), which aligns natively with the stochastic gradient descent updates of the critic. However, evaluating this $L_2$ approximation bias incurs penalties from the joint concentrability bounds ($C_{joint}$ and $C_{joint}^\dag$) to handle the state measure shifts and the action-space $\chi^2$-to-KL divergence translations.
Note that these joint concentrability constants, $C_{joint}$ and $C_{joint}^\dag$, are absent from, respectively, the algorithmic progress term ($\|\theta^*_\tau(\omega_t) - \tau \omega_t\|_2^2$)
and the tracking error term ($Z_t$) in Lemmas~\ref{lem:pl-condition} and \ref{lem:pmd-telescoping-l2}.

In this section, we demonstrate that by substituting the $L_2$ error assumption with an $L_\infty$ error bound, we completely eliminate the remaining joint concentrability constants attached to the approximation bias, freeing the Deterministic Regime (Part II) from all concentrability dependencies. We also isolate the fundamental bottlenecks that prevent full elimination in our current analysis of the Stochastic Regime (Part I).

\begin{ass}[$L_\infty$ Critic Approximation Error] \label{ass:linfty-error}
The absolute approximation error of the ideal uncentered critic is uniformly bounded by $\epsilon_\infty$ across the state-action space for all training policies $\{\pi_t\}_{t \ge 0}$ under Algorithm~\ref{alg:uncentered-nac}:
$$ \sup_{s,a} |\epsilon_t(s,a)| \le \epsilon_\infty. $$
\end{ass}

\subsection{Complete Concentrability Elimination in Part II}

In Part II, algorithmic progress is driven by the PMD geometry over the probability simplex. We show that under the $L_\infty$ approximation error assumption, the global joint concentrability bound $C_{joint}^\dag$ is removed from the PMD Actor Progress Bound (Lemma~\ref{lem:pmd-telescoping-l2}).

\begin{lem}[Concentrability-Free PMD Actor Progress] \label{lem:linfty-pmd}
Under Assumptions~\ref{ass:bounded-features}, \ref{ass:pd-uncentered}, and \ref{ass:linfty-error}, if the actor step size satisfies $\eta_t \le 1/\tau$, the expected global performance gap satisfies the single-step inequality without any dependence on $C_{joint}^\dag$:
$$(1-\gamma) \mathbb{E}[\text{Gap}_t^\dag] \le -\frac{7\tau}{8} \mathbb{E}[D_t^\dag] + \frac{\mathbb{E}[D_t^\dag] - \mathbb{E}[D_{t+1}^\dag]}{\eta_t} + \frac{\eta_t}{2} G_{ac}^2 B_\phi^2 + 2\epsilon_\infty + \frac{4 B_\phi^2}{\tau} Z_t,$$
where $Z_t \triangleq \mathbb{E}[\|e_t\|_2^2]$ is the expected critic tracking error.
Notably, this completely avoids the $\tau^{-1}$ penalty on the bias seen in the $L_2$ analysis.
\end{lem}

\begin{prf}
We revisit inequality \eqref{eq:PMD-progress} from the proof of Lemma~\ref{lem:pmd-telescoping-l2}. As shown in Step (iii) of that proof, the tracking error penalty is bounded by $\frac{\tau}{8}D_t^\dag + \frac{4 B_\phi^2}{\tau} \|e_t\|_2^2$, without any concentrability constants. Therefore, we only need to re-evaluate the approximation bias term.

By H\"{o}lder's inequality over the action space and Assumption~\ref{ass:linfty-error}, the approximation bias term is bounded:
$$ \mathbb{E}_{d^*_\tau} \Big[ \langle \epsilon_t, \pi^*_\tau - \pi_t \rangle_{\mathcal{A}} \Big] \le \mathbb{E}_{d^*_\tau} \Big[ \max_a |\epsilon_t(s,a)| \|\pi^*_\tau - \pi_t\|_1 \Big] \le 2\epsilon_\infty, $$
where we use the fact that the $L_1$ distance is bounded by $\|\pi^*_\tau - \pi_t\|_1 \le 2$. This bound requires no concentrability bounds (i.e., it is completely free of $C_{joint}^\dag$).

Substituting this approximation bias bound and the established tracking error bound back into inequality \eqref{eq:PMD-progress} (which provides $-\tau D_t^\dag$ and the algorithmic progress) yields:
$$(1-\gamma) \text{Gap}_t^\dag \le -\tau D_t^\dag + \frac{D_t^\dag - D_{t+1}^\dag}{\eta_t} + \frac{\eta_t}{2} G_{ac}^2 B_\phi^2 + 2\epsilon_\infty + \left( \frac{\tau}{8}D_t^\dag + \frac{4 B_\phi^2}{\tau} \|e_t\|_2^2 \right).$$
Combining the $D_t^\dag$ terms, $-\tau D_t^\dag + \frac{\tau}{8} D_t^\dag=-\frac{7\tau}{8} D_t^\dag$, and taking the total expectation $\mathbb{E}[\cdot]$ over the algorithm trajectory completes the proof.
\end{prf}

By replacing Lemma~\ref{lem:pmd-telescoping-l2} with Lemma~\ref{lem:linfty-pmd} in the Part II analysis, the subsequent convergence bounds hold without requiring any concentrability assumptions. We formalize these regularized convergence bounds below.

\begin{thm}[Concentrability-Free Regularized Convergence] \label{thm:linfty-reg-conv}
For Algorithm~\ref{alg:uncentered-nac}, suppose that Assumptions~\ref{ass:bounded-features}--\ref{ass:pd-uncentered} and \ref{ass:linfty-error} hold.
For independent constants $c_\eta \ge \frac{4}{3\tau}$ and $\zeta > 0$, let the actor and critic step sizes be $\eta_t = \frac{c_\eta}{t+t_0}$ and $\alpha_{t+1} = \frac{c_\alpha}{(t+t_0)^{2/3}}$, with $c_\alpha \triangleq \zeta c_\eta^{2/3}$ and an initial offset $t_0 \ge \max\left( c_\eta \tau, (2 c_\alpha \kappa)^{3/2}, (\frac{4}{3 c_\alpha \kappa})^3 \right)$.
By choosing the step-size constants $c_\eta = \Theta(\tau^{-1})$ and $\zeta = \Theta(\kappa^{-1})$ and the initial offset $t_0 = \Theta(\tau^{-1})$ subject to the stated conditions, the expected global regularized performance gap satisfies the following asymptotic rates for the average iterate and the last iterate, respectively:
\begin{align*}
\frac{1}{T} \sum_{t=0}^{T-1} \mathbb{E}[\text{Gap}_t^\dag]
&\le \mathcal{O}\left( \frac{1}{(1-\gamma)^7 \kappa^6 \tau^{5/3}} T^{-2/3} \right)
+ \tilde{\mathcal{O}}\left( \frac{1}{(1-\gamma)^7 \kappa^6 \tau^2} T^{-1} \right)
+ \mathcal{O}\left( \frac{\epsilon_\infty}{1-\gamma} \right), \\
\mathbb{E}[\text{Gap}_T^\dag]
&\le \mathcal{O}\left( \frac{1}{(1-\gamma)^5 \kappa^4 \tau^{4/3}} T^{-1/3} \right)
+ \mathcal{O}\left( \frac{\sqrt{\epsilon_\infty} }{(1-\gamma)^2 \kappa \tau^{1/2}} \right).
\end{align*}
\end{thm}

\begin{prf}
(i) \textit{Average-Iterate Bound.}
We repeat the PMD value-telescoping argument used in Theorem~\ref{thm:reg-conv-l2}, replacing the PMD Actor Progress Bound (Lemma~\ref{lem:pmd-telescoping-l2}) by its concentrability-free counterpart, Lemma~\ref{lem:linfty-pmd}. The only substantive difference is that the $L_2$ approximation term $\mathcal{O}(\epsilon_{app}/\tau)$ in the single-step PMD inequality is replaced by the direct $L_\infty$ penalty $2\epsilon_\infty$. Consequently, the average-iterate convergence bound remains the same as in Theorem~\ref{thm:reg-conv-l2}, except with $\mathcal{O}(\frac{\epsilon_{app}}{(1-\gamma)\tau})$ replaced by $\mathcal{O}(\frac{\epsilon_\infty}{1-\gamma})$.

(ii) \textit{Last-Iterate Bound.}
For the last iterate, the $L_\infty$ approximation error enters differently because the performance gap is recovered through the terminal forward KL divergence. Dropping the non-negative performance gap in Lemma~\ref{lem:linfty-pmd} and multiplying by $\eta_t$ gives
\begin{align*}
\mathbb{E}[D_{t+1}^\dag]
\le \left(1-\frac{7\tau}{8}\eta_t\right)\mathbb{E}[D_t^\dag]
+ \frac{\eta_t^2}{2}G_{ac}^2B_\phi^2
+ \frac{4B_\phi^2\eta_t}{\tau}Z_t
+ 2\eta_t\epsilon_\infty.
\end{align*}
Using $\eta_t=\frac{c_\eta}{t+t_0}$ and $Z_t\le\frac{v}{(t+t_0)^{2/3}}$, this recurrence takes the form of Lemma~\ref{lem:fractional-chung-bias} with parameters $c=c_\infty$, $A=A_\infty$, and $B=B_\infty$:
$$ c_\infty \triangleq \frac{7\tau c_\eta}{8}, \qquad A_\infty \triangleq \frac{c_\eta^2G_{ac}^2B_\phi^2}{2}+\frac{4c_\eta B_\phi^2v}{\tau}, \qquad B_\infty \triangleq 2c_\eta\epsilon_\infty. $$
Since $c_\infty>1$ (implied by $c_\eta\ge\frac{4}{3\tau}$) and $t_0 \ge c_\infty$ (implied by $t_0\ge c_\eta\tau$), Chung's Lemma yields
$$ \mathbb{E}[D_T^\dag]
\le \frac{u_\infty}{(T+t_0)^{2/3}}+\frac{B_\infty}{c_\infty}
= \frac{u_\infty}{(T+t_0)^{2/3}}+\frac{16}{7\tau}\epsilon_\infty, $$
where $u_\infty \triangleq \max\left(t_0^{2/3}D_0^\dag,\frac{3A_\infty}{3c_\infty-2}\right)$. Under the stated step-size choices, the same structural calculation as in Lemma~\ref{lem:last-iterate-dt} gives
$$ u_\infty=\mathcal{O}\left(\frac{1}{(1-\gamma)^6\kappa^6\tau^{8/3}}\right). $$

Finally, applying the Forward KL Bound (Lemma~\ref{lem:pinsker-gap}), Jensen's inequality, and $\sqrt{x+y}\le\sqrt{x}+\sqrt{y}$ gives
$$ \mathbb{E}[\text{Gap}_T^\dag]
\le \frac{\sqrt{2u_\infty}M_{max}}{1-\gamma}\frac{1}{(T+t_0)^{1/3}}
+ \frac{4\sqrt{2}M_{max}}{\sqrt{7}(1-\gamma)\tau^{1/2}}\sqrt{\epsilon_\infty}. $$
Using $M_{max}=\mathcal{O}((1-\gamma)^{-1}\kappa^{-1})$ from the proof of Theorem~\ref{thm:reg-last-iterate-gap} yields
the stated last-iterate convergence bound, completing the proof.
\end{prf}

For the average iterate, the $L_\infty$ approximation bias in Theorem~\ref{thm:linfty-reg-conv} is independent of the temperature $\tau$. Therefore, unlike the $L_2$ analysis in the main text, a one-stage temperature can be chosen solely to balance the regularized convergence error against the exponential translation tail. For the last iterate, however, the forward-KL analysis introduces a residual $\tau^{-1/2}\sqrt{\epsilon_\infty}$ approximation term. To prevent this factor from growing as the temperature decreases with $T$, we employ a two-stage temperature for the last-iterate guarantee.

\begin{thm}[Concentrability-Free Unregularized Convergence] \label{thm:linfty-unreg-conv}
For Algorithm~\ref{alg:uncentered-nac}, suppose that Assumptions~\ref{ass:bounded-features}--\ref{ass:pd-uncentered}, \ref{ass:action-gap}, and \ref{ass:linfty-error} hold.

For the average iterate, define the one-stage temperature
$$ \tau_T^{avg} \triangleq \frac{\Delta}{2\log(C_\gamma T^{2/3})}. $$
If $T$ is sufficiently large such that $\tau_T^{avg}\le\tau_{max}$, then by choosing the step-size parameters as in Theorem~\ref{thm:linfty-reg-conv} evaluated at $\tau=\tau_T^{avg}$ subject to the stated conditions, the expected average global unregularized performance gap satisfies
\begin{align*}
\frac{1}{T}\sum_{t=0}^{T-1}\mathbb{E}\big[J_0(\pi^*_0)-J_0(\pi_t)\big]
&\le \mathcal{O}\left( \frac{(1-\gamma)^{-5/3}+(\log T)^{5/3}}{(1-\gamma)^8\kappa^6\Delta^{8/3}T^{2/3}} \right)
+ \mathcal{O}\left( \frac{\epsilon_\infty}{(1-\gamma)^2\Delta} \right) \\
&= \tilde{\mathcal{O}}(T^{-2/3})+\mathcal{O}(\epsilon_\infty).
\end{align*}

For the last iterate, define the two-stage temperature
$$ \tau_T^{last} \triangleq \max\left( \frac{\Delta}{2\log(C_\gamma T^{2/3})}, \frac{\Delta}{2\log(C_\gamma(1+\epsilon_\infty^{-1}))} \right). $$
If $T$ is sufficiently large and $\epsilon_\infty$ is sufficiently small such that $\tau_T^{last}\le\tau_{max}$, then by choosing the step-size parameters as in Theorem~\ref{thm:linfty-reg-conv} evaluated at $\tau=\tau_T^{last}$ subject to the stated conditions, the expected last-iterate global unregularized performance gap satisfies
\begin{align*}
\mathbb{E}\big[J_0(\pi^*_0)-J_0(\pi_T)\big]
&\le \mathcal{O}\left( \frac{(1-\gamma)^{-4/3}+(\log T)^{4/3}}{(1-\gamma)^6\kappa^4\Delta^{7/3}T^{1/3}} \right) \\
&\quad + \mathcal{O}\left( \frac{(1-\gamma)^{-1/2}+\sqrt{\log(1+\epsilon_\infty^{-1})}}{(1-\gamma)^3\kappa\Delta^{3/2}}\sqrt{\epsilon_\infty} \right) \\
&= \tilde{\mathcal{O}}(T^{-1/3})+\tilde{\mathcal{O}}(\sqrt{\epsilon_\infty}).
\end{align*}
In the case of $\epsilon_\infty=0$, we interpret $\tau_T^{last}=\frac{\Delta}{2\log(C_\gamma T^{2/3})}$ and $\epsilon_\infty\log(1+\epsilon_\infty^{-1})=0$.
\end{thm}

\begin{prf}
(i) \textit{Average-Iterate Bound.}
By the Universal Unregularized Suboptimality Bound (Corollary~\ref{cor:universal-suboptimal-mass}) and the average-iterate bound in Theorem~\ref{thm:linfty-reg-conv},
\begin{align*}
\frac{1}{T}\sum_{t=0}^{T-1}\mathbb{E}\big[J_0(\pi^*_0)-J_0(\pi_t)\big]
\le &\ \mathcal{O}\left(\frac{T^{-2/3}}{(1-\gamma)^8\kappa^6\Delta(\tau_T^{avg})^{5/3}}\right)
+\tilde{\mathcal{O}}\left(\frac{T^{-1}}{(1-\gamma)^8\kappa^6\Delta(\tau_T^{avg})^2}\right) \\
&+\mathcal{O}\left(\frac{\epsilon_\infty}{(1-\gamma)^2\Delta}\right)
+C_{tail}\exp\left(-\frac{\Delta}{2\tau_T^{avg}}\right).
\end{align*}
Substituting $\tau_T^{avg}=\frac{\Delta}{2\log(C_\gamma T^{2/3})}$ and
following similar calculation as in the proof of Theorem~\ref{thm:PMD-unreg-average}
leads to the stated average-iterate bound.

(ii) \textit{Last-Iterate Bound.}
We first consider $\epsilon_\infty>0$. When $\epsilon_\infty=0$, the approximation term vanishes and the result follows from the sample-limited temperature $\tau_T^{last}=\frac{\Delta}{2\log(C_\gamma T^{2/3})}$.

By Corollary~\ref{cor:universal-suboptimal-mass} and the last-iterate bound in Theorem~\ref{thm:linfty-reg-conv},
\begin{align*}
\mathbb{E}\big[J_0(\pi^*_0)-J_0(\pi_T)\big]
\le &\ \mathcal{O}\left(\frac{T^{-1/3}}{(1-\gamma)^6\kappa^4\Delta(\tau_T^{last})^{4/3}}\right)
+\mathcal{O}\left(\frac{\sqrt{\epsilon_\infty}}{(1-\gamma)^3\kappa\Delta(\tau_T^{last})^{1/2}}\right) \\
&+C_{tail}\exp\left(-\frac{\Delta}{2\tau_T^{last}}\right).
\end{align*}
Substituting the two-stage temperature $\tau_T^{last}$ and following similar calculation as in the proof of Theorem~\ref{thm:PMD-unreg-last}
leads to the stated last-iterate bound.
\end{prf}

\subsection{Generalization to Exploratory $\nu$-Sampling (Deterministic Regime)}

We extend the concentrability-free unregularized convergence bounds in the Deterministic Regime (Part II) to accommodate the exploratory $\nu$-sampling introduced in Appendix~\ref{app:nu_sampling}. Because the $L_\infty$ approximation bound bypasses the need for state measure shifts, this extension does not require the Global Off-Policy Joint Mismatch (Assumption~\ref{ass:generalized-off-policy}(iv)),
while it relies on the positive-definite moment condition in Assumption~\ref{ass:generalized-nu-exp}, replacing the main-text lower bound $\kappa$ by $\kappa_{exp}$. Therefore, Theorem~\ref{thm:linfty-unreg-conv} applies directly with this replacement, retaining the one-stage temperature for the average iterate and the two-stage temperature for the last iterate.

\begin{cor}[Concentrability-Free Unregularized Convergence under $\nu$-Sampling] \label{cor:linfty-nu-sampling}
For Algorithm~\ref{alg:uncentered-nac} with exploratory $\nu$-sampling, suppose that Assumptions~\ref{ass:bounded-features}--\ref{ass:bounded-q}, \ref{ass:action-gap}, \ref{ass:generalized-nu-exp}, and \ref{ass:linfty-error} hold.

For the average iterate, define the one-stage temperature
$$ \tau_T^{avg} \triangleq \frac{\Delta}{2\log(C_\gamma T^{2/3})}. $$
If $T$ is sufficiently large such that $\tau_T^{avg}\le\tau_{max}$, then by choosing the step-size parameters as in Theorem~\ref{thm:linfty-reg-conv} with $\kappa$ replaced by $\kappa_{exp}$ and evaluated at $\tau=\tau_T^{avg}$, the expected average global unregularized performance gap satisfies
\begin{align*}
\frac{1}{T}\sum_{t=0}^{T-1}\mathbb{E}\big[J_0(\pi^*_0)-J_0(\pi_t)\big]
&\le \mathcal{O}\left( \frac{(1-\gamma)^{-5/3}+(\log T)^{5/3}}{(1-\gamma)^8\kappa_{exp}^6\Delta^{8/3}T^{2/3}} \right)
+ \mathcal{O}\left( \frac{\epsilon_\infty}{(1-\gamma)^2\Delta} \right) \\
&= \tilde{\mathcal{O}}(T^{-2/3})+\mathcal{O}(\epsilon_\infty).
\end{align*}

For the last iterate, define the two-stage temperature
$$ \tau_T^{last} \triangleq \max\left( \frac{\Delta}{2\log(C_\gamma T^{2/3})}, \frac{\Delta}{2\log(C_\gamma(1+\epsilon_\infty^{-1}))} \right). $$
If $T$ is sufficiently large and $\epsilon_\infty$ is sufficiently small such that $\tau_T^{last}\le\tau_{max}$, then by choosing the step-size parameters as in Theorem~\ref{thm:linfty-reg-conv} with $\kappa$ replaced by $\kappa_{exp}$ and evaluated at $\tau=\tau_T^{last}$, the expected last-iterate global unregularized performance gap satisfies
\begin{align*}
\mathbb{E}\big[J_0(\pi^*_0)-J_0(\pi_T)\big]
&\le \mathcal{O}\left( \frac{(1-\gamma)^{-4/3}+(\log T)^{4/3}}{(1-\gamma)^6\kappa_{exp}^4\Delta^{7/3}T^{1/3}} \right) \\
&\quad + \mathcal{O}\left( \frac{(1-\gamma)^{-1/2}+\sqrt{\log(1+\epsilon_\infty^{-1})}}{(1-\gamma)^3\kappa_{exp}\Delta^{3/2}}\sqrt{\epsilon_\infty} \right) \\
&= \tilde{\mathcal{O}}(T^{-1/3})+\tilde{\mathcal{O}}(\sqrt{\epsilon_\infty}).
\end{align*}
In the case of $\epsilon_\infty=0$, we interpret $\tau_T^{last}=\frac{\Delta}{2\log(C_\gamma T^{2/3})}$ and $\epsilon_\infty\log(1+\epsilon_\infty^{-1})=0$.
\end{cor}

\subsection{Impact on Part I and Geometric Bottleneck}

Applying the same $L_\infty$ mechanism to the Stochastic Regime removes the joint concentrability assumption from the Parameter-Space PL Condition (Lemma~\ref{lem:pl-condition}) and improves the bias scaling, bypassing the $\tau^{-1}$ amplification associated with $\epsilon_{app}$.

\begin{lem}[Concentrability-Free PL Condition] \label{lem:linfty-pl-condition}
Under Assumptions~\ref{ass:bounded-features} and \ref{ass:linfty-error}, the parametric regularized performance gap satisfies the PL condition without any dependence on $C_{joint}$:
$$\text{Gap}_t \le \frac{B_\phi^2}{2(1-\gamma)\tau} \|\theta^*_\tau(\omega_t) - \tau \omega_t\|_2^2 + \frac{2}{1-\gamma} \epsilon_\infty.$$
\end{lem}

\begin{prf}
By the Ideal Critic Decomposition (Lemma~\ref{lem:sub-decomp}(i)) evaluated at $\pi = \pi_{\omega^*_\tau}$, the gap is:
$$(1-\gamma) \text{Gap}_t = -\tau D_t + \mathbb{E}_{d^{\omega^*_\tau}} \Big[ \langle (\theta^*_\tau(\omega_t) - \tau \omega_t)^\top \phi, \pi_{\omega^*_\tau} - \pi_t \rangle_{\mathcal{A}} \Big] + \mathbb{E}_{d^{\omega^*_\tau}} \big[ \langle \epsilon_t, \pi_{\omega^*_\tau} - \pi_t \rangle_{\mathcal{A}} \big].$$
By H\"{o}lder's inequality and Assumption~\ref{ass:linfty-error}, the approximation bias is bounded:
$$ \mathbb{E}_{d^{\omega^*_\tau}} \big[ \langle \epsilon_t, \pi_{\omega^*_\tau} - \pi_t \rangle_{\mathcal{A}} \big] \le \mathbb{E}_{d^{\omega^*_\tau}} \Big[ \max_a |\epsilon_t(s,a)| \|\pi_{\omega^*_\tau} - \pi_t\|_1 \Big] \le 2\epsilon_\infty. $$
For the ideal algorithmic progress, we revisit the bound established in Step (i) of the proof of Lemma~\ref{lem:pl-condition}:
$$\mathbb{E}_{d^{\omega^*_\tau}} \Big[ \langle (\theta^*_\tau(\omega_t) - \tau \omega_t)^\top \phi, \pi_{\omega^*_\tau} - \pi_t \rangle_{\mathcal{A}} \Big] \le B_\phi \|\theta^*_\tau(\omega_t) - \tau \omega_t\|_2 \sqrt{2 D_t}.$$
We decouple this product using Young's inequality ($xy \le \frac{\tau}{2}x^2 + \frac{1}{2\tau}y^2$) with $x = \sqrt{2D_t}$ and $y = B_\phi \|\theta^*_\tau(\omega_t) - \tau \omega_t\|_2$:
$$ B_\phi \|\theta^*_\tau(\omega_t) - \tau \omega_t\|_2 \sqrt{2 D_t} \le \tau D_t + \frac{B_\phi^2}{2\tau} \|\theta^*_\tau(\omega_t) - \tau \omega_t\|_2^2. $$
Substituting these bounds back into the decomposition, the $\tau D_t$ term cancels the restorative entropy force $-\tau D_t$. Dividing the entire inequality by $1-\gamma$ yields the stated result.
\end{prf}

\textbf{The $C_{joint}$ Geometric Bottleneck:} While Lemma~\ref{lem:linfty-pl-condition} eliminates $C_{joint}$ from the PL condition, the Stochastic Regime cannot be completely freed from this concentrability bound under our current assumptions. As derived in the Actor Progress Bound (Lemma~\ref{lem:pl-ascent}), establishing Euclidean parameter progress requires lower-bounding the eigenvalues of the on-policy Fisher information matrix: $\bar{\Sigma}_{cen}(\pi_t)$. Because the curvature guarantee $\lambda$ (Assumption~\ref{ass:pd-fim}) is anchored at the parametric optimum $\pi_{\omega^*_\tau}$, we are forced to shift the measure: $\bar{\Sigma}_{cen}(\pi_t) \succeq \frac{1}{C_{joint}} \bar{\Sigma}_{cen}(\pi_{\omega^*_\tau}) \succeq \frac{\lambda}{C_{joint}} I$. This mapping of the current policy's state-action measure to the optimal policy's measure introduces the joint concentrability constant $C_{joint}$ as a geometric bottleneck. However, if it were assumed that $\bar{\Sigma}_{cen}(\pi_t)$ is uniformly positive-definite across the training trajectory (analogous to the uncentered feature moment in Assumption~\ref{ass:pd-uncentered}), this reliance on $C_{joint}$ would be bypassed.

\section{Related Works}  \label{app:related-works}

\subsection{Agarwal et al. (2021)} \label{app:Agarwal-et-al}

\cite{agarwal2021theory} studied an unregularized Q-NPG algorithm with a double-loop architecture and an exploratory $\nu$-sampling measure.
We mainly discuss their Corollary 21 and compare it with our Part II results under generalized $\nu$-sampling (Appendix \ref{app:nu_sampling}).

\textbf{Unregularized Q-NPG Algorithm.}
Translated to our notation, the unregularized Q-NPG algorithm updates the actor parameter $\omega_t$ via:
$$ \omega_{t+1} = \omega_t + \eta_t \theta_{t+1}, $$
where the critic parameter $\theta_{t+1}$ approximates the ideal minimizer
$ \theta^*_0(\omega_t) \triangleq \arg\min_{\theta \in \mathbb{R}^d} L(\theta; \omega_t, d^{\pi_t}_\nu)$,
and the MSE objective in the unregularized $Q$-value projection is
$$ L(\theta; \omega_t, d^{\pi_t}_\nu) \triangleq \mathbb{E}_{(s,a) \sim d^{\pi_t}_\nu} \Big[ \big(Q_0^{\pi_t}(s,a) - \theta^\top \phi(s,a) \big)^2 \Big] $$
under the exploratory $\nu$-sampling measure as described in Appendix \ref{app:nu_sampling}.

To estimate this $Q$-value projection from samples, \cite{agarwal2021theory} employ a double-loop architecture (Algorithm~\ref{alg:double-loop-npg}). The algorithm performs $T$ outer-loop actor updates. Within each outer-loop iteration $t$, the actor parameter $\omega_t$ is frozen, and an inner loop executes $N$ stochastic gradient descent (SGD) steps to train the critic.
The total number of SGD updates is therefore $T_{total} = T \times N$. As with our Algorithm~\ref{alg:uncentered-nac}, translating this into raw environment interactions incurs an additional $\mathcal{O}((1-\gamma)^{-1})$ expected-horizon factor due to the Monte Carlo rollouts.

\begin{algorithm}[t]
\caption{Double-Loop Unregularized Q-NPG \citep{agarwal2021theory,yuan2023linear}}
\label{alg:double-loop-npg}
\begin{algorithmic}[1]
\REQUIRE Initial actor parameter $\omega_0$, exploratory distribution $\nu$, inner-loop iterations $N$, outer-loop iterations $T$, step sizes $\{\eta_t\}_{t \ge 0}$ and $\{\alpha_t\}_{t \ge 0}$.
\FOR{$t = 0, 1, \dots, T-1$}
    \STATE \textbf{Actor Freeze:} Fix the current policy $\pi_t \triangleq \pi_{\omega_t}$.
    \STATE Initialize inner critic parameter $\theta_t^{(0)}$.
    \FOR{$k = 0, 1, \dots, N-1$}
        \STATE \textbf{Sampling:} Draw a state-action pair $(s_k, a_k) \sim d^{\pi_t}_\nu$.
        \STATE \textbf{Evaluation:} Obtain an unbiased estimate $\hat{Q}_k$ for the unregularized $Q$-value $Q_0^{\pi_t}(s_k, a_k)$.
        \STATE \textbf{Critic SGD:} Update the inner critic parameter (possibly with additional projection):
        $$\theta_t^{(k+1)} = \theta_t^{(k)} + \alpha_t \big( \hat{Q}_k - \theta_t^{(k)\top} \phi(s_k, a_k) \big) \phi(s_k, a_k)$$
    \ENDFOR
    \STATE \textbf{Critic Assignment:} Set the outer-loop critic parameter $\theta_{t+1} = \theta_t^{(N)}$ (or an average of the inner iterates). 
    \STATE \textbf{Actor Update:} Update the actor parameter:
    $$\omega_{t+1} = \omega_t + \eta_t \theta_{t+1}$$
\ENDFOR
\ENSURE The output policy sequence $\{\pi_t\}_{t=0}^T$.
\end{algorithmic}
\end{algorithm}

\textbf{Notation Mapping.}
To bridge \cite{agarwal2021theory} to our generalized $\nu$-sampling analysis (Appendix~\ref{app:nu_sampling}), we map the error bounds used in \cite{agarwal2021theory} to our notation.
\begin{itemize}
\item \textit{Statistical Estimation Error ($\epsilon_{stat}$):} \cite{agarwal2021theory} define $\epsilon_{stat}$ to bound the excess risk of the sample-based critic parameter $\theta_{t+1}$ (denoted $w^{(t)}$ in their work) compared to the ideal least-squares minimizer $\theta^*_0(\omega_t)$, evaluated under the $\nu$-induced training distribution:
$$ \mathcal{E}_{stat}(t) \triangleq \mathbb{E} \left[ L(\theta_{t+1}; \omega_t, d^{\pi_t}_\nu) - L(\theta^*_0(\omega_t); \omega_t, d^{\pi_t}_\nu) \right] \le \epsilon_{stat}. $$
Because $\theta^*_0(\omega_t)$ is a minimum of this quadratic objective, the excess risk evaluates exactly to the expected squared tracking error weighted by the uncentered feature moment matrix under the $\nu$-measure:
$$ \mathcal{E}_{stat}(t)
=\mathbb{E} \Big[ (\theta_{t+1} - \theta^*_0(\omega_t))^\top \bar{\Sigma}_{unc}(d^{\pi_t}_\nu) (\theta_{t+1} - \theta^*_0(\omega_t)) \Big] . $$
Under the feature bound $\bar{\Sigma}_{unc}(d^{\pi_t}_\nu)\preceq B_\phi^2 I$ and the exploratory positive-definiteness condition $\bar{\Sigma}_{unc}(d^{\pi_t}_\nu)\succeq\kappa_{exp}I$, we have
$$
\frac{\mathcal{E}_{stat}(t)}{B_\phi^2}
\le
\mathbb{E}\big[\|\theta_{t+1}-\theta^*_0(\omega_t)\|_2^2\big]
\le
\frac{\mathcal{E}_{stat}(t)}{\kappa_{exp}}
\le
\frac{\epsilon_{stat}}{\kappa_{exp}}.
$$
Thus, their $\epsilon_{stat}$ directly controls our expected critic tracking error.

For our comparison, we focus on the strengthened statistical setting in Remark 27 of \cite{agarwal2021theory}, where a positive lower bound on the minimum eigenvalue of the training feature moment matrix yields $\epsilon_{stat}=\mathcal{O}(d/N)$. This is consistent with our exploratory positive-definiteness assumption $\bar{\Sigma}_{unc}(d^{\pi_t}_\nu)\succeq\kappa_{exp}I$, and we absorb the feature dimension $d$ into the structural constants to write $\epsilon_{stat}=\mathcal{O}(1/N)$.

\item \textit{Approximation Error ($\epsilon_{approx}$):} \cite{agarwal2021theory} also define $\epsilon_{approx}$ as the inherent mean squared error of the ideal linear fit under the $\nu$-induced training distribution:
$$ \mathbb{E}_{(s,a) \sim d^{\pi_t}_\nu} \Big[ \big( Q_0^{\pi_t}(s,a) - \theta^*_0(\omega_t)^\top \phi(s,a) \big)^2 \Big] = \mathbb{E}_{d^{\pi_t}_\nu} \big[ \epsilon_t(s,a)^2 \big] \le \epsilon_{approx}. $$
Hence, their $\epsilon_{approx}$ bound corresponds directly to our $\epsilon_{app}$ bound.
\end{itemize}

\textbf{Structural Assumptions.}
To control the distribution mismatch, Corollary 21 of \cite{agarwal2021theory} requires a single uniform bound:
$$ \sup_{s,a} \frac{d^*_0(s) \frac{1}{|\mathcal{A}|}}{\nu(s,a)} \le C_{AKLM}. $$
While \cite{agarwal2021theory} allow $d^*_0$ to be an arbitrary comparator state visitation distribution, for concreteness, we take $d^*_0$ to be that associated with the unregularized global optimum $\pi^*_0$.
Note that because bounding the density ratio inherently bounds the feature moment matrix, their $C_{AKLM}$ assumption implies their Assumption 6.2, which requires a finite relative condition number $\sup_{\theta} \frac{\theta^\top \bar{\Sigma}_{unc}(d^*_0 \times \frac{1}{|\mathcal{A}|}) \theta}{\theta^\top \bar{\Sigma}_{unc}(\nu) \theta} < \infty$.

As shown in their analysis and replicated in Eq.~\eqref{eq:Agarwal-PDL}, bounding the error term $\mathbb{E}_{d^*_0} \big[ \langle err_t, \pi^*_0 - \pi_t \rangle_{\mathcal{A}} \big]$ only requires two specific mismatch bounds:
$$ \sup_{s,a} \frac{d^*_0(s) \pi_t(a|s) }{\nu(s,a)} \le \tilde{C}_{AKLM} \quad
\text{and} \quad \sup_{s,a} \frac{d^*_0(s) \pi^*_0(a|s) }{\nu(s,a)} \le \tilde{C}_{AKLM}. $$
Because $d^{\pi_t}_\nu(s,a) \ge (1-\gamma)\nu(s,a)$, these mismatch bounds yield
\begin{align}
\sup_{s,a} \frac{d^*_0(s) \pi_t(a|s) }{d^{\pi_t}_\nu(s,a)} \le \frac{\tilde{C}_{AKLM}}{1-\gamma} \quad
\text{and} \quad \sup_{s,a} \frac{d^*_0(s) \pi^*_0(a|s) }{d^{\pi_t}_\nu(s,a)} \le \frac{\tilde{C}_{AKLM}}{1-\gamma}. \label{eq:Agarwal-et-al-concen}
\end{align}
Their original $1/|\mathcal{A}|$ formulation acts as a conservative, policy-independent sufficient condition for these bounds.
The two mismatch bounds in \eqref{eq:Agarwal-et-al-concen} correspond to the unregularized limits of those in \eqref{eq:global-mismatch-nu-a} and \eqref{eq:global-mismatch-nu-b} respectively, deduced from our Assumption~\ref{ass:generalized-off-policy}(iv). Due to the presence of the additional local policy mismatch term, the unregularized limit of our Assumption~\ref{ass:generalized-off-policy}(iv) is, in general, more restrictive than the two bounds in \eqref{eq:Agarwal-et-al-concen}.

In the on-policy specialization where, at each iteration $t$, the restart distribution is chosen as $\nu_t(s,a)=\mu(s)\pi_t(a|s)$, the induced training measure reduces to$
d^{\pi_t}_{\nu_t}(s,a)=d^{\pi_t}(s)\pi_t(a|s)$.
Under this specialization, the two bounds reduce to:
$$
\sup_s\frac{d^*_0(s)}{d^{\pi_t}(s)}
\le\frac{\tilde{C}_{AKLM}}{1-\gamma}
\quad\text{and}\quad
\sup_{s,a}\frac{d^*_0(s)\pi^*_0(a|s)}
{d^{\pi_t}(s)\pi_t(a|s)}
\le\frac{\tilde{C}_{AKLM}}{1-\gamma},
$$
respectively. The second bound, which implies the first, represents the exact unregularized counterpart to our global joint concentrability assumption (Assumption~\ref{ass:density-ratios}).

\textbf{Convergence Bounds.}
\cite{agarwal2021theory}, Corollary 21, establishes the following bound for the best-iterate performance gap 
by employing a constant actor step size $\eta_t = \eta \propto 1/\sqrt{T}$:
\begin{align}
\mathbb{E}\left[\min_{t<T} \big\{ J_0(\pi^*_0) - J_0(\pi_t) \big\}\right] \le \mathcal{O}\left(\frac{1}{\sqrt{T}}\right) + \mathcal{O}\left(\sqrt{\epsilon_{stat}}\right) + \mathcal{O}\left(\sqrt{\epsilon_{app}}\right), \label{eq:conv-bound-Agarwal}
\end{align}
where the proof also yields the corresponding average-iterate bound.
Substituting the strongly-convex inner-loop decay $\epsilon_{stat} = \mathcal{O}(1/N)$, the bound evaluates to:
$$ \mathbb{E}\left[\min_{t<T} \big\{ J_0(\pi^*_0) - J_0(\pi_t) \big\}\right] \le \mathcal{O}\left(\frac{1}{\sqrt{T}}\right) + \mathcal{O}\left(\frac{1}{\sqrt{N}}\right) + \mathcal{O}\left(\sqrt{\epsilon_{app}}\right), $$
as mentioned in \cite{agarwal2021theory}, Remark 27.

To minimize the total algorithmic error subject to a fixed budget of SGD updates $T_{total} = T \times N$, the outer-loop decay is balanced against the inner-loop statistical error by setting $1/\sqrt{T} \approx 1/\sqrt{N}$, yielding $T = N = \sqrt{T_{total}}$.
Substituting these balanced sizes back shows that the unregularized performance gap of the double-loop Q-NPG algorithm is bounded by:
$$ \mathbb{E}\left[\min_{t<T} \big\{ J_0(\pi^*_0) - J_0(\pi_t) \big\}\right] \le \mathcal{O}\left(\frac{1}{T_{total}^{1/4}}\right) + \mathcal{O}\left(\sqrt{\epsilon_{app}}\right). $$

\textbf{Main Steps of Analysis.}
The analysis in \cite{agarwal2021theory} relies on the Performance Difference Lemma (PDL) and the Bregman Three-Point Identity. By PDL, the unregularized performance gap decomposes into the algorithmic update direction and the error $err_t(s,a) \triangleq Q_0^{\pi_t}(s,a) - \theta_{t+1}^\top \phi(s,a)$, notably lacking a negative KL divergence term:
\begin{align}
(1-\gamma) \big( J_0(\pi^*_0) - J_0(\pi_t) \big) = \mathbb{E}_{d^*_0} \Big[ \langle \theta_{t+1}^\top \phi, \pi^*_0 - \pi_t \rangle_{\mathcal{A}} \Big] + \mathbb{E}_{d^*_0} \Big[ \langle err_t, \pi^*_0 - \pi_t \rangle_{\mathcal{A}} \Big]. \label{eq:Agarwal-PDL}
\end{align}
By substituting the log-linear parameterization into the Bregman Three-Point Identity evaluated against the \textit{current} policy $\pi_t$ (i.e., $\langle \log \pi_{t+1} - \log \pi_t, \pi^*_0 - \pi_t \rangle_{\mathcal{A}} = D_{KL}(\pi^*_0 \| \pi_t) - D_{KL}(\pi^*_0 \| \pi_{t+1}) + D_{KL}(\pi_t \| \pi_{t+1})$),
the first term is directly related to the reduction in KL divergences:
$$ \eta_t \langle \theta_{t+1}^\top \phi, \pi^*_0 - \pi_t \rangle_{\mathcal{A}} = D_{KL}(\pi^*_0 \| \pi_t) - D_{KL}(\pi^*_0 \| \pi_{t+1}) + D_{KL}(\pi_t \| \pi_{t+1}). $$
Bounding the local distance $D_{KL}(\pi_t \| \pi_{t+1})$ by $\mathcal{O}(\eta_t^2)$ leads to the inequality (which acts similarly to our PMD Algorithmic Progress bound):
$$ \langle \theta_{t+1}^\top \phi, \pi^*_0 - \pi_t \rangle_{\mathcal{A}} \le \frac{1}{\eta_t} \Big( D_{KL}(\pi^*_0 \| \pi_t) - D_{KL}(\pi^*_0 \| \pi_{t+1}) \Big) + \mathcal{O}(\eta_t). $$
Taking the expectation over $d^*_0$ and substituting this into the PDL equation \eqref{eq:Agarwal-PDL} yields a progress bound (mirroring inequality \eqref{eq:PMD-progress} in our proof of PMD Actor Progress Bound):
\begin{align}
(1-\gamma) \big( J_0(\pi^*_0) - J_0(\pi_t) \big) \le \frac{D_t^\dag-D_{t+1}^\dag}{\eta_t} + \mathbb{E}_{d^*_0} \Big[ \langle err_t, \pi^*_0 - \pi_t \rangle_{\mathcal{A}} \Big] + \mathcal{O}(\eta_t), \label{eq:PMD-progress-Agarwal}
\end{align}
where $D_t^\dag \triangleq \mathbb{E}_{d^*_0}[D_{KL}(\pi^*_0 \| \pi_t)]$, serving as the unregularized counterpart to $D_t^\dag$ in our analysis.
Averaging this inequality over $T$ iterations with a constant step size $\eta$ yields a bound of $\frac{D_0^\dag}{\eta T} + \mathcal{O}(\eta)$ up to the remaining error term.
For this error term, \cite{agarwal2021theory} apply Cauchy-Schwarz over the state-action space (e.g., $\mathbb{E}_{d^*_0, \pi^*_0}[err_t] \le \sqrt{\mathbb{E}_{d^*_0, \pi^*_0}[err_t^2]}$), trapping both the statistical estimation and approximation errors under a square root ($\mathcal{O}(\sqrt{\epsilon_{stat}}) + \mathcal{O}(\sqrt{\epsilon_{app}})$).
Combining these components yields exactly the convergence bound stated in \eqref{eq:conv-bound-Agarwal}.

\textbf{Comparison.}
Because both the analysis in \cite{agarwal2021theory} and our Part II analysis operate within the PMD geometry, we compare the resulting average-iterate bounds. In contrast to their $\mathcal{O}(T_{total}^{-1/4})+\mathcal{O}(\sqrt{\epsilon_{app}})$ bound, our single-loop average-iterate bound (Theorem~\ref{thm:PMD-unreg-average}) achieves an accelerated $\tilde{\mathcal{O}}(T_{total}^{-2/3})$ rate and a linear $\tilde{\mathcal{O}}(\epsilon_{app})$ bias dependence.

This improvement is driven by three core analytical choices in Part II:
\begin{itemize}
\item relying on entropy regularization to extract a Restorative Entropy Force term, which introduces a geometric contraction that enables PMD Value-Telescoping with diminishing step sizes ($\eta_t = \mathcal{O}(t^{-1})$) to bypass the $1/\sqrt{T}$ barrier;

\item exploiting the Global Joint Concentrability assumption or, with $\nu$-sampling, the Generalized Off-Policy Mismatch assumption, to simultaneously manage the state measure shift and facilitate the action-space $\chi^2$-to-KL divergence translation, decoupling the error inner product via Young's inequality after Cauchy-Schwarz and absorbing the resulting divergence penalty back into the Restorative Entropy Force; and

\item deriving the exponential translation mechanism under the Minimal Action Gap assumption to map regularized progress back to the unregularized objective, bypassing the standard linear entropy penalty.
\end{itemize}

\subsection{Yuan et al. (2023)} \label{app:Yuan-et-al}

\cite{yuan2023linear} analyzed the unregularized Q-NPG algorithm similar to that in \cite{agarwal2021theory}. The algorithm utilizes a double-loop architecture and trains the critic using an exploratory $\nu$-sampling measure as shown in Algorithm~\ref{alg:double-loop-npg}.

\textbf{Notation and Assumptions.}
To bridge \cite{yuan2023linear} to our framework, we align their structural assumptions and error definitions with our notation:
\begin{itemize}
    \item \textit{Distribution Mismatch ($\vartheta_\mu$):} They define a state distribution mismatch coefficient $\vartheta_\mu \triangleq \frac{1}{1-\gamma} \| d^*_0 / \mu \|_\infty$, which satisfies $\vartheta_\mu \ge 1/(1-\gamma)$. While \cite{yuan2023linear} allow $d^*_0$ to be an arbitrary comparator state visitation distribution, for concreteness, we take $d^*_0$ to be that associated with the unregularized global optimum $\pi^*_0$. Because $d^{\pi_t}(s) \ge (1-\gamma)\mu(s)$, this coefficient also satisfies $\vartheta_\mu \ge \| d^*_0 / d^{\pi_t} \|_\infty$. This serves as the unregularized counterpart to the implied state concentrability bound $\| d^*_\tau / d^{\pi_t} \|_\infty \le C_{joint}^\dag$ from our Assumption~\ref{ass:density-ratios}.

    \item \textit{$L_2$-Concentrability:} Their concentrability assumption (Assumption 6 in their paper) is postulated for several state-action distributions in terms of the $L_2$ norm of the density ratio over $d^{\pi_t}_\nu$ instead of the $L_\infty$ norm, including among others
        $$ \mathbb{E}_{(s,a)\sim d^{\pi_t}_\nu} \left[\left( \frac{d^*_0(s) \pi_t(a|s)}{d^{\pi_t}_\nu(s,a)} \right)^2 \right] \le C_{YDGLX} ,
        \quad \mathbb{E}_{(s,a)\sim d^{\pi_t}_\nu} \left[\left( \frac{d^*_0(s) \pi^*_0(a|s)}{d^{\pi_t}_\nu(s,a)} \right)^2 \right] \le C_{YDGLX} .$$
        These bounds correspond to the unregularized $L_2$ limits of the $L_\infty$ mismatch bounds in \eqref{eq:global-mismatch-nu-a} and \eqref{eq:global-mismatch-nu-b} respectively, deduced from our Assumption~\ref{ass:generalized-off-policy}(iv).

    \item \textit{Error Terms ($\epsilon_{stat}$ and $\epsilon_{app}$):} Their statistical estimation error $\epsilon_{stat}$ and linear approximation error $\epsilon_{approx}$ map identically to those defined in the \cite{agarwal2021theory} above.
\end{itemize}

\textbf{Convergence Bounds.}
Under these structural assumptions, their Theorem 3 establishes that for a geometrically increasing step size satisfying $\eta_{t+1} \ge \frac{1}{\gamma}\eta_t$ (with an initialization $\eta_0 \ge \frac{1-\gamma}{\gamma} D_{KL}(\pi^*_0 \| \pi_0)$), the expected unregularized performance gap at iteration $T$ is bounded by:
\begin{align}
\mathbb{E} \big[ J_0(\pi^*_0) - J_0(\pi_T) \big] \le \frac{2}{1-\gamma} \left(1 - \frac{1}{\vartheta_\mu}\right)^T + \mathcal{O}\left( \sqrt{\epsilon_{stat}} + \sqrt{\epsilon_{app}} \right). \label{eq:conv-bound-Yuan}
\end{align}
Substituting the standard inner-loop critic decay $\epsilon_{stat} = \mathcal{O}(1/N)$ into Theorem 3 yields the performance gap in terms of the algorithmic loops:
$$ \mathbb{E} \big[ J_0(\pi^*_0) - J_0(\pi_T) \big] \le \mathcal{O}\left( \left(1 - \frac{1}{\vartheta_\mu}\right)^T \right) + \mathcal{O}\left( \frac{1}{\sqrt{N}} \right) + \mathcal{O}\left( \sqrt{\epsilon_{app}} \right). $$

To extract the final sample complexity (Corollary 1), Yuan et al. tune the outer-loop size $T$ and inner-loop size $N$ to achieve a target convergence error $\epsilon$:
\begin{itemize}
    \item \textit{Outer Loop ($T$):} To ensure the exponentially decaying convergence error falls below $\epsilon$, the number of actor steps scales as: $T = \mathcal{O}\left( \log(1/\epsilon) \right)$.
    \item \textit{Inner Loop ($N$):} To ensure the square-rooted statistical error $\mathcal{O}(1/\sqrt{N})$ falls below $\epsilon$, the number of critic SGD steps per actor update grows as: $N = \mathcal{O}(1/\epsilon^2)$.
\end{itemize}
From these choices, the total budget of SGD updates required is $T_{total} = T \times N = \mathcal{O}\left( \frac{1}{\epsilon^2} \log\frac{1}{\epsilon} \right)$. Inverting this relationship with respect to $T_{total}$, their Corollary 1 implies that the unregularized performance gap of the double-loop Q-NPG algorithm is bounded by:
$$ \mathbb{E} \big[ J_0(\pi^*_0) - J_0(\pi_T) \big] \le \tilde{\mathcal{O}}\left( \frac{1}{\sqrt{T_{total}}} \right) + \mathcal{O}\left( \sqrt{\epsilon_{app}} \right). $$

\textbf{Main Steps of Analysis.}
To bypass the $1/\sqrt{T}$ bottleneck, \cite{yuan2023linear} analyze the unregularized PMD update via its first-order optimality condition (i.e., $\eta_t \langle \theta_{t+1}^\top \phi, \pi^*_0 - \pi_{t+1} \rangle_{\mathcal{A}} \le \langle \log \pi_{t+1} - \log \pi_t, \pi^*_0 - \pi_{t+1} \rangle_{\mathcal{A}}$).
Applying the Bregman Three-Point Identity evaluated against the \textit{next} policy $\pi_{t+1}$ (i.e., $\langle \log \pi_{t+1} - \log \pi_t, \pi^*_0 - \pi_{t+1} \rangle_{\mathcal{A}} = D_{KL}(\pi^*_0 \| \pi_t) - D_{KL}(\pi^*_0 \| \pi_{t+1}) - D_{KL}(\pi_{t+1} \| \pi_t)$)
and dropping the non-positive local distance, $-D_{KL}(\pi_{t+1} \| \pi_t)$, yields a one-step bound:
\begin{align*}
\eta_t \langle \theta_{t+1}^\top \phi, \pi^*_0 - \pi_{t+1} \rangle_{\mathcal{A}}
& \le D_{KL}(\pi^*_0 \| \pi_t) - D_{KL}(\pi^*_0 \| \pi_{t+1}) - D_{KL}(\pi_{t+1} \| \pi_t) \\
& \le D_{KL}(\pi^*_0 \| \pi_t) - D_{KL}(\pi^*_0 \| \pi_{t+1}).
\end{align*}
Then they split the inner product: $\langle \theta_{t+1}^\top \phi, \pi^*_0 - \pi_{t+1} \rangle_{\mathcal{A}}=
\langle \theta_{t+1}^\top \phi, \pi^*_0 - \pi_t \rangle_{\mathcal{A}} + \langle \theta_{t+1}^\top \phi, \pi_t - \pi_{t+1} \rangle_{\mathcal{A}}$.
Taking the expectation of the first term over $d_0^*$ and applying PDL yields
$(1-\gamma) \text{Gap}_t^\dag - \mathbb{E}_{d^*_0} \Big[ \langle err_t, \pi^*_0 - \pi_t \rangle_{\mathcal{A}} \Big]$
as in \eqref{eq:Agarwal-PDL}, where $\text{Gap}_t^\dag = J_0(\pi^*_0) - J_0(\pi_t)$.
For the second term $\langle \theta_{t+1}^\top \phi, \pi_t - \pi_{t+1} \rangle_{\mathcal{A}}$, which is non-positive, they shift the state measure to $d^{\pi_{t+1}}$ via the mismatch coefficient $\vartheta_\mu$. Applying PDL again to this shifted term yields the bound:
$$ \mathbb{E}_{d^*_0} \Big[ \langle \theta_{t+1}^\top \phi, \pi_t - \pi_{t+1} \rangle_{\mathcal{A}} \Big] \ge \vartheta_\mu \left( (1-\gamma)(\text{Gap}_{t+1}^\dag - \text{Gap}_t^\dag) - \mathbb{E}_{d^{\pi_{t+1}}} \Big[ \langle err_t, \pi_t - \pi_{t+1} \rangle_{\mathcal{A}} \Big] \right). $$
Substituting these back and rearranging yields a contractive recurrence for the \textit{next} gap:
\begin{align*}
(1-\gamma) \text{Gap}_{t+1}^\dag \le &\left(1 - \frac{1}{\vartheta_\mu}\right) (1-\gamma) \text{Gap}_t^\dag + \frac{1}{\eta_t \vartheta_\mu} \Big( D_{KL}(\pi^*_0 \| \pi_t) - D_{KL}(\pi^*_0 \| \pi_{t+1}) \Big) \\
&+ \frac{1}{\vartheta_\mu} \mathbb{E}_{d^*_0} \Big[ \langle err_t, \pi^*_0 - \pi_t \rangle_{\mathcal{A}} \Big] + \mathbb{E}_{d^{\pi_{t+1}}} \Big[ \langle err_t, \pi_t - \pi_{t+1} \rangle_{\mathcal{A}} \Big].
\end{align*}
By deploying geometrically increasing step sizes ($\eta_t \propto \gamma^{-t}$), \cite{yuan2023linear} telescope this recursion to derive a linear $\mathcal{O}((1 - 1/\vartheta_\mu)^T)$ outer-loop convergence rate as stated in \eqref{eq:conv-bound-Yuan}.
By choosing a highly unbalanced $N$ and $T$ ($N \gg T$), they achieve an overall last-iterate convergence rate of $\mathcal{O}(1/\sqrt{T_{total}})$ in terms of the total number of SGD updates $T_{total} = N \times T$.


\textbf{Comparison.}
We compare the analysis in \cite{yuan2023linear} with our Part II analysis. In contrast to their unbalanced double-loop structure and exponentially growing step-size schedule ($\eta_t \propto \gamma^{-t}$), our single-loop algorithm ($N=1, T=T_{total}$) uses a stable, diminishing $\mathcal{O}(t^{-1})$ step-size schedule. Furthermore, compared to their $\tilde{\mathcal{O}}(T_{total}^{-1/2})+\mathcal{O}(\sqrt{\epsilon_{app}})$ bound, our theory leads to an $\tilde{\mathcal{O}}(T_{total}^{-2/3})+\tilde{\mathcal{O}}(\epsilon_{app})$ average-iterate bound, and an $\tilde{\mathcal{O}}(T_{total}^{-1/3})+\tilde{\mathcal{O}}(\sqrt{\epsilon_{app}})$ last-iterate bound for the unregularized MDP.

As discussed in the comparison with \cite{agarwal2021theory}, our theoretical analysis in Part II achieves this by (1) exploiting entropy regularization to extract a Restorative Entropy Force, (2) utilizing the Global Joint Concentrability assumption or, with $\nu$-sampling, the Generalized Off-Policy Mismatch assumption, to simultaneously manage the state measure shift and facilitate the action-space $\chi^2$-to-KL divergence translation, and (3) applying the exponential translation mechanism to bypass the standard linear entropy penalty.

\subsection{Liu et al. (2020)} \label{app:Liu-et-al}

\cite{liu2020improved} analyzed an unregularized NPG algorithm.
Similar to the NPG algorithm in \cite{agarwal2021theory}, their algorithm utilizes a double-loop architecture but operates with a horizon-independent, constant outer-loop step size by exploiting a positive-definite Fisher information matrix assumption to achieve faster convergence.

\textbf{Unregularized NPG Algorithm.}
Translated to our notation, the double-loop NPG algorithm studied by \cite{liu2020improved} updates the actor parameter $\omega_t$ via:
$$ \omega_{t+1} = \omega_t + \eta \theta_{t+1}, $$
where the critic parameter $\theta_{t+1}$ approximates the ideal minimizer in the unregularized advantage projection (similar to $\theta^{*A}_\tau(\omega_t)$ in Appendix~\ref{app:critic-comparison}):
$$ \theta^{*A}_0(\omega_t) \triangleq \arg\min_{\theta \in \mathbb{R}^d} \mathbb{E}_{d^{\pi_t}, \pi_t} \left[ \big( A_0^{\pi_t}(s,a) - \theta^\top \tilde{\phi}_t(s,a) \big)^2 \right]. $$
By the Policy Gradient Theorem, this projection is proportional to the natural policy gradient direction: $\theta^{*A}_0(\omega_t) = (1-\gamma) \bar{\Sigma}_{cen}(\pi_t)^{-1} \nabla J_0(\omega_t)$. The algorithm estimates this direction by solving the regression problem over $N$ inner-loop SGD steps via conditionally independent sampling with fixed policy $\pi_t$ (Algorithm~\ref{alg:double-loop-npg-liu}). The total number of SGD updates is $T_{total} = T \times N$.

\begin{algorithm}[!t]
\caption{Double-Loop Unregularized NPG \citep{liu2020improved}}
\label{alg:double-loop-npg-liu}
\begin{algorithmic}[1]
\REQUIRE Initial actor parameter $\omega_0$, inner-loop iterations $N$, outer-loop iterations $T$, constant actor step size $\eta$, inner critic step size $\alpha$.
\FOR{$t = 0, 1, \dots, T-1$}
    \STATE \textbf{Actor Freeze:} Fix the current policy $\pi_t \triangleq \pi_{\omega_t}$.
    \STATE Initialize inner critic parameter $\theta_t^{(0)} = 0$.
    \FOR{$k = 0, 1, \dots, N-1$}
        \STATE \textbf{Sampling:} Draw state-action pairs $(s_k, a_k) \sim d^{\pi_t} \times \pi_t$.
        \STATE \textbf{Evaluation:} Obtain an unbiased estimate $\hat{A}_k$ for the advantage $A_0^{\pi_t}(s_k, a_k)$.
        \STATE \textbf{Critic SGD:} Update the critic parameter:
        $$\theta_t^{(k+1)} = \theta_t^{(k)} + \alpha \Big( \hat{A}_k - \theta_t^{(k)\top} \tilde{\phi}_t(s_k, a_k)\Big) \tilde{\phi}_t(s_k, a_k)$$
    \ENDFOR
    \STATE \textbf{Critic Assignment:} Set the outer-loop critic parameter as the average of the inner iterates: $\theta_{t+1} = \frac{1}{N} \sum_{k=1}^N \theta_t^{(k)}$.
    \STATE \textbf{Actor Update:} Update the actor parameter:
    $$\omega_{t+1} = \omega_t + \eta \theta_{t+1}$$
\ENDFOR
\ENSURE The output policy sequence $\{\pi_t\}_{t=0}^T$.
\end{algorithmic}
\end{algorithm}

\textbf{Notation and Assumptions.}
To bridge \cite{liu2020improved} to our framework, we align their structural assumptions with our notation:
\begin{itemize}
    \item \textit{Positive-Definite Fisher Information:} Their Assumption 2.1 requires the Fisher information matrix $\bar{\Sigma}_{cen}(\pi_\omega)$ to be uniformly positive-definite for all $\omega\in\mathbb{R}^d$.
        This serves a similar purpose to Assumption~\ref{ass:pd-fim} in Part I.\ In fact, their proof also holds provided
        $\bar{\Sigma}_{cen}(\pi_t) \succeq \mu_F I$ for all training policies $\{\pi_t\}_{t \ge 0}$.
        This condition is satisfied if $\bar{\Sigma}_{cen}(\pi_{\omega^*_0})$ is positive-definite at the parametric optimum $\pi_{\omega^*_0}$ and the joint concentrability bound holds: $\sup_{s,a} \frac{d^{\omega^*_0}(s)\pi_{\omega^*_0}(a|s)}{d^{\pi_t}(s) \pi_t(a|s)} \le C_{joint, 0}$ for a constant $C_{joint, 0}>0$, where $d^{\omega^*_0} \triangleq d^{\pi_{\omega^*_0}}$.
        As noted below Assumption~\ref{ass:pd-fim}, the Fisher information $\bar{\Sigma}_{cen}(\pi^*_0)$ is usually degenerate even in the Stochastic Regime
        because the global optimum $\pi^*_0$ is deterministic.

    \item \textit{Transferred Approximation Error ($\epsilon_{bias}$):} Their Assumption 4.4 defines the approximation error directly under the state-action distribution of the global optimum $\pi^*_0$:
    $$ \mathbb{E}_{d^*_0, \pi^*_0} \Big[ \big( A_0^{\pi_t}(s,a) - \theta^{*A}_0(\omega_t)^\top \tilde{\phi}_t(s,a) \big)^2 \Big] \le \epsilon_{bias}. $$
    This implicitly bundles the $L_2$ approximation error ($\epsilon_{app}$) and the global joint concentrability bound ($\sup_{s,a} \frac{d^*_0(s)\pi^*_0(a|s)}{d^{\pi_t}(s) \pi_t(a|s)} \le C_{joint, 0}^\dag$ for a constant $C_{joint, 0}^\dag>0$) into a single ``transferred'' constant, explaining the absence of explicit distribution mismatch coefficients in their convergence bounds.
    Here, $C_{joint, 0}$ parallels $C_{joint}$ in our Part I analysis, while the direct global-to-current mismatch $C_{joint, 0}^\dag$ corresponds to the combined shift controlled by $C_{joint} C_{joint}^*$.

    \item \textit{Statistical Estimation Error:} The convergence error of their inner-loop SGD is bounded by $\mathbb{E}[\|\theta_{t+1} - \theta^{*A}_0(\omega_t)\|_2^2] = \mathcal{O}(1/N)$, relying on the strong convexity provided by the positive-definite Fisher information in the advantage regression problem.
\end{itemize}

\textbf{Convergence Bounds.}
Under these assumptions, their Theorem 4.9 establishes the following bound for the average unregularized performance gap with a constant step size $\eta = \Theta(1)$:
\begin{align}
\frac{1}{T} \sum_{t=0}^{T-1} \mathbb{E} \big[ J_0(\pi^*_0) - J_0(\pi_t) \big] \le \mathcal{O}\left( \frac{1}{\eta T} \right) + \mathcal{O}\left( \frac{1}{\sqrt{N}} \right) + \mathcal{O}\left( \sqrt{\epsilon_{bias}} \right) . \label{eq:conv-bound-Liu}
\end{align}
To balance the optimization error against the inner-loop statistical error for a fixed budget of SGD updates $T_{total} = T \times N$, they tune the loop sizes such that $T = \mathcal{O}(T_{total}^{1/3})$ and $N = \mathcal{O}(T_{total}^{2/3})$. This yields the overall convergence bound:
$$ \frac{1}{T} \sum_{t=0}^{T-1} \mathbb{E} \big[ J_0(\pi^*_0) - J_0(\pi_t) \big] \le \mathcal{O}\left( T_{total}^{-1/3} \right) + \mathcal{O}\left( \sqrt{\epsilon_{bias}} \right). $$

\textbf{Main Steps of Analysis.}
The theoretical analysis in \cite{liu2020improved} employs an interesting set of techniques, mixing Policy Mirror Descent (PMD) bounding (akin to our Part II) with objective smoothness and the positive-definite Fisher information (akin to our Part I).

First, they derive a bound on the performance gap (Proposition 4.5) that is structurally similar to averaging the inequality \eqref{eq:PMD-progress-Agarwal} in \cite{agarwal2021theory} and \eqref{eq:PMD-progress} in our analysis:
\begin{align}
\frac{1}{T} \sum_{t=0}^{T-1} \big( J_0(\pi^*_0) - J_0(\pi_t) \big) \le \frac{\sqrt{\epsilon_{bias}}}{1-\gamma} + \frac{D_0^\dag}{\eta T} + \frac{G}{T}\sum_{t=0}^{T-1} \|\theta_{t+1} - \theta^{*A}_0(\omega_t)\|_2 + \frac{M \eta}{2 T} \sum_{t=0}^{T-1} \|\theta_{t+1}\|_2^2, \label{eq:PMD-progress-Liu}
\end{align}
where $G$ is an upper bound of $\|\nabla \log \pi_\omega\|$ and $M$ is a Lipschitz constant of $\nabla \log \pi_\omega$ in $\omega$.

Next, to control the accumulated step penalty $\sum_{t=0}^{T-1} \|\theta_{t+1}\|_2^2$, they employ a smoothness expansion over the objective $J_0$. After substituting the actor update $\omega_{t+1} = \omega_t + \eta \theta_{t+1}$, the step-wise progress is bounded by:
$$ J_0(\omega_{t+1}) - J_0(\omega_t) \ge \eta \langle \nabla J_0(\omega_t), \theta_{t+1} \rangle - \frac{L_J \eta^2}{2} \|\theta_{t+1}\|_2^2, $$
where $L_J$ is a smoothness constant of the objective $J_0(\omega)$ (i.e., a Lipschitz constant of its gradient).
Then they decompose the actual update $\theta_{t+1}$ into the ideal target $\theta^{*A}_0(\omega_t) = (1-\gamma) \bar{\Sigma}_{cen}(\pi_t)^{-1} \nabla J_0(\omega_t)$ and the tracking error $\theta_{t+1} - \theta^{*A}_0(\omega_t)$. By invoking the positive-definite Fisher information ($\bar{\Sigma}_{cen}(\pi_t) \succeq \mu_F I$) and bounded centered features ($\bar{\Sigma}_{cen}(\pi_t) \preceq G^2 I$), the ideal progress is lower-bounded ($\langle \nabla J_0(\omega_t), \theta^{*A}_0(\omega_t) \rangle \ge \frac{1-\gamma}{G^2} \|\nabla J_0(\omega_t)\|_2^2$), and the ideal step penalty is upper-bounded ($\|\theta^{*A}_0(\omega_t)\|_2^2 \le \frac{(1-\gamma)^2}{\mu_F^2} \|\nabla J_0(\omega_t)\|_2^2$).

Substituting these bounds and telescoping the step-wise progress bound controls the average gradient norm $\frac{1}{T} \sum_{t=0}^{T-1} \|\nabla J_0(\omega_t)\|_2^2$ by the total objective progress $\mathcal{O}\big(\frac{J_0(\omega_T) - J_0(\omega_0)}{T}\big)$ up to the average expected tracking error $\frac{1}{T} \sum_{t=0}^{T-1} \mathbb{E}[\|\theta_{t+1} - \theta^{*A}_0(\omega_t)\|_2^2]$.
Because the inner-loop SGD drives this tracking error down, this bound on the average gradient norm subsequently shows that the accumulated step penalty in the PMD bound \eqref{eq:PMD-progress-Liu}, $\frac{\eta}{T} \sum_{t=0}^{T-1} \|\theta_{t+1}\|_2^2$, is $\mathcal{O}\big(\frac{1}{\eta T}\big)$ (up to the tracking error). With a constant step size $\eta = \Theta(1)$, this leads to the stated convergence bound \eqref{eq:conv-bound-Liu}, thereby accelerating the outer-loop convergence to $\mathcal{O}(1/T)$ and the overall rate to $\mathcal{O}(T_{total}^{-1/3})$.

\textbf{Comparison.}
While \cite{liu2020improved} improves upon the $\mathcal{O}(T_{total}^{-1/4})$ rate of \cite{agarwal2021theory}, their analysis is appropriate only within the Stochastic Regime due to the reliance on a uniformly positive-definite Fisher information matrix. Hence, it is meaningful to compare their theoretical framework with our Part I analysis.
\begin{itemize}
    \item \textit{Roles of the Fisher Information.} As NPG is used, the positive-definite Fisher information serves two distinct purposes in their framework. First, it provides the strong convexity when solving the advantage value projection. Second, it provides a mechanism to control the step penalty from the smoothness expansion. In contrast, for our uncentered algorithm (Q-NPG), we rely on the positive-definite uncentered feature moment ($\kappa > 0$) for the inner-loop critic convergence,
    while employing the Fisher information ($\lambda > 0$) for one purpose: to lower-bound the algorithmic progress in the smoothness expansion for the outer loop.

    \item \textit{Gradient Space vs Parameter Space.} \cite{liu2020improved} choose to maintain and bound the average gradient norm $\|\nabla J_0(\omega_t)\|_2^2$ in the gradient space. In contrast, our proof relates this objective progress to the parameter distance $\| \theta_{t+1}-\tau \omega_t\|_2^2$ as in the Actor Progress Bound (Lemma~\ref{lem:pl-ascent}). In addition, we establish a Parameter-Space PL Condition, relating the performance gap directly to the parameter distance $\| \theta^*_\tau(\omega_t)-\tau \omega_t\|_2^2$.

    \item \textit{Telescoping vs Joint Lyapunov.} To finalize their convergence guarantees, \cite{liu2020improved} still use telescoping to bound the average gap. Our approach takes a fundamentally different path: we use the Actor Progress Bound to derive a Coupled Actor Recurrence and a Coupled Critic Tracking, leading to a joint Lyapunov analysis.
\end{itemize}
Ultimately, via entropy regularization and the Exponential Translation mechanism, our Part I analysis achieves an optimal $\tilde{\mathcal{O}}(T_{total}^{-1})$ unregularized rate for both average-iterate and last-iterate convergence (up to the approximation error) in the Stochastic Regime, accelerating beyond the $\mathcal{O}(T_{total}^{-1/3})$ rate established in \cite{liu2020improved}.

\subsection{Cayci et al. (2024)} \label{app:cayci-et-al}

\cite{cayci2024convergence} studied an entropy-regularized NPG algorithm with averaging. Translated to our notation, their analysis operates under a fixed regularization temperature $\tau > 0$ and projects the regularized action-values onto centered features.
See Appendix \ref{app:critic-comparison} for a discussion of $Q$-value projections in entropy-regularized NPG algorithms.

\textbf{Entropy-Regularized NPG Algorithm.}
\cite{cayci2024convergence} employ a double-loop architecture as shown in Algorithm~\ref{alg:cayci-npg}: an inner Temporal Difference (TD) learning subroutine to estimate the centered $Q$-values ($\tilde{Q}_\tau^{\pi_t}$), followed by an inner, constrained SGD subroutine to project these centered $Q$-values onto the centered features ($\tilde\phi_t$).

The exact minimizer in this centered $Q$-value projection is defined as
$$\theta^{*R}_\tau (\omega_t) = \arg\min_{\|\theta\|_2 \le R} \overline{L}(\theta, \omega_t)\triangleq \mathbb{E}_{d^{\pi_t} \times \pi_t} \big[ \big( \tilde{Q}_\tau^{\pi_t}(s,a) -\theta^\top \tilde{\phi}_t(s,a)\big)^2 \big],$$
where $R>0$ is a projection radius,
$\tilde{\phi}_t(s,a) \triangleq \nabla \log \pi_t(a|s)=\phi(s,a) - \mathbb{E}_{a' \sim \pi_t}[\phi(s,a')]$ denotes the centered features
and
$\tilde{Q}_\tau^{\pi_t}(s,a) \triangleq Q_\tau^{\pi_t}(s,a) - \mathbb{E}_{a' \sim \pi_t}[Q_\tau^{\pi_t}(s,a')]$ denotes the centered  $Q$-value.
See Appendix \ref{app:critic-comparison} for a discussion on the relationship between our ideal critic $\theta^*_\tau (\omega_t)$
and the unconstrained version of $\theta^{*R}_\tau (\omega_t)$, denoted as $\theta^{*CHS}_\tau (\omega_t)$.

\begin{algorithm}[t]
\caption{Sample-Based Entropy-Regularized NPG with Averaging \citep{cayci2024convergence}}
\label{alg:cayci-npg}
\begin{algorithmic}[1]
\REQUIRE Projection radius $R > 0$, regularization $\tau > 0$, TD iterations $K$, projection iterations $N$, step-sizes $\eta_t = \frac{1}{\tau(t+1)}$.
\STATE \textbf{Initialization:} $\omega_0 = 0$.
\FOR{$t = 0, 1, \dots, T-1$}
    \STATE \textbf{Actor Freeze:} Fix the current policy $\pi_t \triangleq \pi_{\omega_t}$.
    \STATE \textbf{TD Learning (Critic):} Run $K$ steps of sample-based TD learning to obtain an estimator $\hat{Q}_\tau^{\pi_t}$ for the regularized $Q$-value $Q_\tau^{\pi_t}$. Compute the centered estimator $\hat{\tilde{Q}}_\tau^{\pi_t}(s,a) = \hat{Q}_\tau^{\pi_t}(s,a) - \mathbb{E}_{a'\sim \pi_t}\hat{Q}_\tau^{\pi_t}(s,a')$
    for the centered $Q$-value $\tilde{Q}_\tau^{\pi_t}(s,a)$.
    \STATE \textbf{SGD Projection (Critic):} Run $N$ steps of constrained SGD using $\hat{\tilde{Q}}_\tau^{\pi_t}$ to obtain an estimate $\theta_{t+1}$ for the exact minimizer $\theta^{*R}_\tau (\omega_t)$, constrained to the ball $\|\theta_{t+1}\|_2 \le R$.
    \STATE \textbf{Actor Update:} Update the actor parameter:
    $$\omega_{t+1} = \omega_t + \eta_t (\theta_{t+1} - \tau \omega_t) .$$
\ENDFOR
\ENSURE The output policy sequence $\{\pi_t\}_{t=0}^T$.
\end{algorithmic}
\end{algorithm}

\textbf{Notation and Assumptions.}
To bridge \cite{cayci2024convergence} to our framework, we align their structural assumptions and error definitions with our notation:
\begin{itemize}
    \item \textit{Projection Radius ($R$):} In their analysis of constrained SGD (Proposition 4.5), the projection radius $R$ is assumed to satisfy a structural lower bound ($R > \bar{R}$). This is similar to our constrained SGD, where the projection radius $R_\theta$ is tied to the structural parameter $B_\theta$.

    \item \textit{Minimum Action Probability ($p_{min}$):} The global minimum action probability over the trajectory is defined as $p_{min} \triangleq \inf_{t,s,a} \pi_t(a|s)$.
    For a fixed $\tau$, because the critic is constrained in the inner SGD subroutine ($\|\theta_{t+1}\|_2 \le R$), the actor's update rule structurally guarantees that the actor parameter iterate remains bounded: $\|\omega_t\|_2 \le R/\tau$ for all $t \ge 0$. As a result, $p_{min}$ is bounded by $p_{min} \ge \frac{1}{|\mathcal{A}|} \exp(-2R/\tau) > 0$. See a related discussion after Assumption~\ref{ass:parametric-density}.

\item \textit{Distribution Mismatch ($\|d^{\omega^*_\tau} / \mu\|_\infty$):} \cite{cayci2024convergence} use a state distribution mismatch coefficient $\|d^{\omega^*_\tau} / \mu\|_\infty$, relative to the parametric optimum $\pi_{\omega^*_\tau}$.
        Because $d^{\pi_t}(s) \ge (1-\gamma)\mu(s)$, this coefficient upper bounds (up to a scaling factor) the implied state concentrability bound $\| d^{\omega^*_\tau} / d^{\pi_t} \|_\infty \le C_{joint}$ from our Assumption~\ref{ass:parametric-density}.

    \item \textit{Constrained Approximation Error ($\epsilon(R)$):}
    Their approximation bias is defined as $\epsilon(R) \triangleq \sup_{t \ge 0} \overline{L}(\theta^{*R}_t, \omega_t)$.
    This is comparable to our approximation error $\epsilon_{app}$.

    \item \textit{Statistical Error ($\epsilon_{stat}$):}
    Their $\epsilon_{critic}$ bounds the mean squared error of the inner TD evaluation subroutine, and $\epsilon_{actor}$ bounds the convergence error of the inner SGD projection subroutine. The sum of these two variances ($\epsilon_{critic} + \epsilon_{actor}$) maps directly to our total expected statistical tracking error, $\epsilon_{stat}$.
\end{itemize}

\textbf{Convergence Bounds.}
Under these definitions, their Proposition 4.1 establishes the following bound for the best-iterate, parametric regularized performance gap (which also applies to the average-iterate gap)
by employing the step-size schedule $\eta_t = \frac{1}{\tau(t+1)}$:
\begin{align}
\mathbb{E} \Big[ \min_{t < T} \big\{ J_\tau(\pi_{\omega^*_\tau}) - J_\tau(\pi_t) \big\} \Big] \le \mathcal{O}\left( \frac{\log T}{\tau T} \right) +  \mathcal{O}\left(  \sqrt{ \frac{\epsilon(R) + \epsilon_{stat}}{p_{min}}  } \right). \label{eq:conv-bound-Cayci}
\end{align}

To optimize the total error subject to a fixed budget of stochastic updates, let $K$ and $N$ denote the numbers of inner TD and SGD steps, respectively. For simplicity, we take $K=\Theta(N)$, so that the total stochastic-update budget satisfies $T_{total}=T(K+N)=\Theta(TN)$. Under the full-rank assumptions described in their Remark 4.6, taking $K=\Theta(N)$ gives $\epsilon_{stat}=\tilde{\mathcal{O}}(1/N)$. Balancing the outer-loop decay $\mathcal{O}(T^{-1})$ against the inner-loop statistical error $\mathcal{O}(N^{-1/2})$ requires $T\approx\sqrt{N}$.
This yields $N \approx T_{total}^{2/3}$ and $T \approx T_{total}^{1/3}$. Substituting these balanced loop sizes back into the gap bound establishes an overall convergence rate for the regularized objective:
$$\mathbb{E} \Big[ \min_{t < T} \big\{ J_\tau(\pi_{\omega^*_\tau}) - J_\tau(\pi_t) \big\} \Big] \le \mathcal{O}\left( \frac{\log T}{\tau T_{total}^{1/3}} \right) +  \mathcal{O}\left(  \sqrt{ \frac{\epsilon(R)}{p_{min}}  } \right),$$
as mentioned in \cite{cayci2024convergence}, Remark 4.6.

\textbf{Main Steps of Analysis.}
The core analysis in \cite{cayci2024convergence} (Theorem~3.6) relies on a Lyapunov drift inequality (their Lemma~3.5):
\begin{align}
D_{t+1} - D_t \le -\eta_t \tau D_t - \eta_t (1-\gamma) \big( J_\tau(\pi_{\omega^*_\tau}) - J_\tau(\pi_t) \big) + \eta_t \mathbb{E}_{d^{\omega^*_\tau}} \Big[ \langle err_t, \pi_{\omega^*_\tau} - \pi_t \rangle_{\mathcal{A}} \Big] + \mathcal{O}(\eta_t^2), \label{eq:PMD-progress-Cayci}
\end{align}
where $D_t \triangleq \mathbb{E}_{d^{\omega^*_\tau}} [D_{KL}(\pi_{\omega^*_\tau} \| \pi_t)]$ serves as the potential function, and $err_t(s,a) \triangleq \tilde{Q}_\tau^{\pi_t}(s,a) - \theta_{t+1}^\top \tilde{\phi}_t(s,a)$ encompasses both the approximation bias and the statistical error.

To bound the error inner product, \cite{cayci2024convergence} apply the Cauchy-Schwarz inequality over the state-action space and, in the first term, shift the state measure to the training distribution via the mismatch coefficient $\|d^{\omega^*_\tau} / \mu\|_\infty$:
$$ \mathbb{E}_{d^{\omega^*_\tau}} \Big[ \langle err_t, \pi_{\omega^*_\tau} - \pi_t \rangle_{\mathcal{A}} \Big] \le \sqrt{ \frac{\| d^{\omega^*_\tau}/ \mu\|_\infty }{1-\gamma} \cdot \big( \epsilon(R) + \epsilon_{stat} \big) } \cdot \sqrt{ \mathbb{E}_{d^{\omega^*_\tau}} \big[ \chi^2(\pi_{\omega^*_\tau}(\cdot|s) \| \pi_t(\cdot|s)) \big] }. $$
Then, they control the $\chi^2$-divergence term via the inequality $\chi^2(\pi_{\omega^*_\tau}(\cdot|s) \| \pi_t(\cdot|s)) \le 2/p_{min}$ for any $s$, depending on the minimum action probability $p_{min}$. Substituting this bound into the drift inequality \eqref{eq:PMD-progress-Cayci} and telescoping over $T$ iterations with the diminishing step-size $\eta_t = \mathcal{O}((\tau t)^{-1})$ leads to their final bound described in \eqref{eq:conv-bound-Cayci}.

\textbf{Comparison for Regularized MDP.}
Up to the distinction between $D_t$ and $D_t^\dag$ and the precise definitions of $err_t(s,a)$,
the drift inequality \eqref{eq:PMD-progress-Cayci} is mathematically equivalent to inequality \eqref{eq:PMD-progress} in our proof of the PMD Actor Progress Bound, both of which mirror the Bregman progress bound \eqref{eq:PMD-progress-Agarwal} in \cite{agarwal2021theory}. A crucial difference between \cite{cayci2024convergence} and our analysis for the regularized MDP lies in how the error inner product is decoupled and bounded. Their Cauchy-Schwarz argument traps both the approximation and statistical errors under a square root, which is then amplified by the $1/\sqrt{p_{min}}$ factor.

In contrast, our Part II analysis avoids both the square-root barrier and $1/\sqrt{p_{min}}$ amplification in deriving the PMD Actor Progress Bound. By exploiting the global joint concentrability assumption (Assumption~\ref{ass:density-ratios}), we simultaneously manage the state measure shift and facilitate the action-space $\chi^2$-to-KL divergence translation. This translation allows us to fully decouple the error inner product via Young's inequality after Cauchy-Schwarz, absorbing the resulting divergence penalty back into the restorative entropy force $-\eta_t \tau D_t^\dag$. This approach frees the approximation and statistical errors from the square-root trap and removes the $1/\sqrt{p_{min}}$ amplification.

Furthermore, to circumvent the double-loop architecture, our analysis employs the fractional Chung's lemma to balance the actor and critic step sizes and derive an uncoupled Critic Tracking Error bound. For our single-loop algorithm ($N=1, T=T_{total}$), our theory in Part II establishes an accelerated average-iterate $\mathcal{O}(T_{total}^{-2/3}) + \mathcal{O}(\epsilon_{app})$ bound on the regularized performance gap. This improves upon their double-loop $\tilde{\mathcal{O}}(T_{total}^{-1/3}) + \mathcal{O}(\sqrt{\epsilon_{app}})$ bound.

\textbf{Implications for Unregularized MDP.}
\cite{cayci2024convergence} restrict their analysis entirely to the fixed-$\tau$ regularized MDP, and do not provide explicit bounds on the unregularized performance gap.
Nevertheless, we discuss implications of their analysis for solving the unregularized MDP.

Extending their guarantees to the unregularized objective in our Deterministic Regime encounters two degeneracies as $\tau\to0$. First, the optimal policy $\pi_{\omega^*_\tau}$ approaches a deterministic distribution, so that $p_{min}\to0$ and their $1/\sqrt{p_{min}}$ error factor deteriorates. Second, the minimum eigenvalue of the centered feature covariance satisfies $\lambda_\tau\to0$, eliminating a temperature-uniform strong-convexity constant for the centered least-squares $Q$-value projection and causing the constants underlying the inner-loop rate $\epsilon_{stat} = \mathcal{O}(1/N)$ to deteriorate.

In the Stochastic Regime, provided $p_{min}$ and $\lambda_\tau$ remain bounded away from zero,
their fixed-temperature regularized bound could be combined with a temperature-tuning argument.
The resulting convergence rate, based on their $\tilde{\mathcal{O}}(T_{total}^{-1/3})$ regularized rate up to the approximation error term, would still be slower than ours in Part I: either $\tilde{\mathcal{O}}(T_{total}^{-1/6})$ under the standard linear entropy penalty, or $\tilde{\mathcal{O}}(T_{total}^{-1/3})$ if combined with our Exponential Translation Bounds.

\section{Technical Details} \label{app:tech-details}

\subsection{Proofs for Section~\ref{sec:setup}}

\begin{prf}[Proof of Lemma \ref{lem:bounded-values} (Bounded Regularized Values)]
By definition, the regularized state-value is:
\begin{align*}
V_\tau^\pi(s) &= \mathbb{E}_{\pi} \left[ \sum_{t=0}^\infty \gamma^t \big( r(s_t, a_t) - \tau \log \pi(a_t|s_t) \big) \;\middle|\; s_0 = s \right] \\
& = \mathbb{E}_{\pi} \left[ \sum_{t=0}^\infty \gamma^t \Big( \mathbb{E}_{a_t \sim \pi}[r(s_t, a_t)|s_t] + \tau H(\pi(\cdot|s_t)) \Big) \;\middle|\; s_0 = s \right].
\end{align*}
Since the reward is bounded in $[0,1]$ and the entropy of a discrete distribution over $\mathcal{A}$ is bounded by $0 \le H \le \log |\mathcal{A}|$, the inner term for each step is bounded between $0$ and $1 + \tau \log |\mathcal{A}|$. Evaluating the geometric series yields $0\le V_\tau^\pi(s) \le \frac{1 + \tau \log |\mathcal{A}|}{1-\gamma}$.

For the action-value $Q_\tau^\pi(s,a)$, the initial step $t=0$ uses a fixed action $a$ without an instantaneous entropy penalty, meaning the $t=0$ term is simply $0 \le r(s,a) \le 1 \,(\le 1 + \tau \log |\mathcal{A}|)$.
Summing the remaining discounted terms yields $0 \le Q_\tau^\pi(s,a) \le \frac{1+\tau\log|\mathcal{A}|}{1-\gamma}$, the same as for $V_\tau^\pi(s)$.
\end{prf}

\begin{prf}[Proof of Lemma \ref{lem:structural-bounds} (Structural Bounds of the Algorithm)]
(i) \textit{Ideal Critic Bound.} By definition, the ideal uncentered critic is a solution to the uncentered least-squares objective: $\theta^*_\tau(\omega_t) = \bar{\Sigma}_{unc}(\pi_t)^{-1} \mathbb{E}_{d^{\pi_t}, \pi_t} [\phi(s,a) Q_\tau^{\pi_t}(s,a)]$.
By Assumption~\ref{ass:pd-uncentered}, the spectral norm of the inverse feature moment matrix is bounded by $\|\bar{\Sigma}_{unc}(\pi_t)^{-1}\|_2 \le \frac{1}{\kappa}$.
For the target vector, applying Jensen's inequality alongside the Bounded Features (Assumption~\ref{ass:bounded-features}) and Bounded Regularized Values (Lemma~\ref{lem:bounded-values}) yields:
$$ \big\| \mathbb{E}_{d^{\pi_t}, \pi_t} [\phi(s,a) Q_\tau^{\pi_t}(s,a)] \big\|_2 \le \mathbb{E}_{d^{\pi_t}, \pi_t} \Big[ \|\phi(s,a)\|_2 \big| Q_\tau^{\pi_t}(s,a) \big| \Big] \le B_\phi V_{max}. $$
Taking the norm of the ideal critic identity and applying the sub-multiplicative property establishes the bound:
$$ \|\theta^*_\tau(\omega_t)\|_2 \le \|\bar{\Sigma}_{unc}(\pi_t)^{-1}\|_2 \big\| \mathbb{E}_{d^{\pi_t}, \pi_t} [\phi(s,a) Q_\tau^{\pi_t}(s,a)] \big\|_2 \le \frac{B_\phi V_{max}}{\kappa} = B_\theta. $$

(ii) \textit{Actor Iterate Bound.} Because the actual critic is projected onto $R_\theta = B_\theta$, we have $\|\theta_{t+1}\|_2 \le B_\theta$.
The actor update rule is $\omega_{t+1} = \omega_t + \eta_t (\theta_{t+1} - \tau \omega_t) = (1 - \tau \eta_t)\omega_t + \eta_t \theta_{t+1}$.
Assuming inductively that $\|\omega_t\|_2 \le \frac{B_\theta}{\tau}$, and noting $(1 - \tau \eta_t) \ge 0$ since $\eta_t \le 1/\tau$, the triangle inequality yields:
$$\|\omega_{t+1}\|_2 \le (1 - \tau \eta_t)\|\omega_t\|_2 + \eta_t \|\theta_{t+1}\|_2 \le (1 - \tau \eta_t) \frac{B_\theta}{\tau} + \eta_t B_\theta = \frac{B_\theta}{\tau}.$$
Since $\tau \le \tau_{max}$, the initialization satisfies $\|\omega_0\|_2 \le \frac{B_\theta}{\tau_{max}} \le \frac{B_\theta}{\tau}$. Therefore, by induction, the bound holds for all $t \ge 0$.

(iii) \textit{Update Direction Bounds.} For the critic gradient, by the bounded feature assumption ($\|\phi(s,a)\|_2 \le B_\phi$) and the projection step ($\|\theta_t\|_2 \le B_\theta$), we bound the squared norm using the inequality $(x+y)^2 \le 2x^2 + 2y^2$:
$$ \|g_t^{cr}(s_t,a_t)\|_2^2 = \big|\theta_t^\top \phi_t - \hat{Q}_t\big|^2 \|\phi_t\|_2^2 \le 2 B_\phi^2 \left( (\theta_t^\top \phi_t)^2 + \hat{Q}_t^2 \right) \le 2 B_\phi^2 \left( B_\theta^2 B_\phi^2 + \hat{Q}_t^2 \right). $$
Taking the conditional expectation and substituting the $Q$-value second moment bound $\mathbb{E}[\hat{Q}_t^2 \mid \mathcal{F}_t] \le Q_{max}^2$ yields the bound $G_{cr}^2$:
$$ \mathbb{E}\big[\|g_t^{cr}(s_t,a_t)\|_2^2 \mid \mathcal{F}_t \big] \le 2 B_\phi^2 \left( B_\theta^2 B_\phi^2 + Q_{max}^2 \right) = G_{cr}^2. $$
For the actor's update direction, we apply the triangle inequality and leverage the actor iterate bound $\|\omega_t\|_2 \le \frac{B_\theta}{\tau}$ from the previous step:
\begin{align*}
\|g_t^{ac}\|_2 = \|\theta_{t+1} - \tau \omega_t\|_2 &\le \|\theta_{t+1}\|_2 + \tau \|\omega_t\|_2  \\
&  \le B_\theta + \tau \frac{B_\theta}{\tau} = 2 B_\theta = G_{ac}.
\end{align*}
Crucially, the temperature $\tau$ cancels out the $1/\tau$ dependence in the actor iterate bound.
\end{prf}

\begin{prf}[Proof of Lemma \ref{lem:lipschitz-critic} (Lipschitz Ideal Uncentered Critic)]
By definition, the ideal uncentered critic is $\theta^*_\tau(\omega) = \bar{\Sigma}(\omega)^{-1} b(\omega)$, where $\bar{\Sigma}(\omega) \triangleq \bar{\Sigma}_{unc}(\pi_\omega)$ and $b(\omega) \triangleq \mathbb{E}_{d^{\pi_\omega}, \pi_\omega}[\phi(s,a) Q_\tau^{\pi_\omega}(s,a)]$.
By the Policy Gradient Theorem, the derivative of any expectation over the discounted state-action visitation measure evaluates as:
\begin{align*}
\nabla_\omega \mathbb{E}_{d^{\pi_\omega}, \pi_\omega}[f(s,a)] &= \mathbb{E}_{d^{\pi_\omega}, \pi_\omega}[\nabla_\omega f(s,a)] + \mathbb{E}_{d^{\pi_\omega}, \pi_\omega} \Big[ \nabla_\omega \log \pi_\omega(a|s) Q_f^{\pi_\omega}(s,a) \Big]
\end{align*}
where $Q_f^{\pi_\omega}(s,a) \triangleq \mathbb{E}_{\pi_\omega} \big[ \sum_{t=0}^\infty \gamma^t f(s_t,a_t) \mid s_0=s, a_0=a \big]$.
For a log-linear policy, the score function is bounded by $\|\nabla_\omega \log \pi_\omega\|_2 = \|\phi(s,a) - \mathbb{E}_{\pi_\omega}[\phi]\|_2 \le 2 B_\phi$. Furthermore, the accumulated sum is bounded by the supremum over the entire state-action space: $\|Q_f^{\pi_\omega}(s,a)\|_2 \le \frac{1}{1-\gamma} \sup_{s',a'} \|f(s',a')\|_2$. Taking the norm and applying the triangle inequality yields the general derivative bound:
\begin{equation} \label{eq:gen-deriv-limit}
\|\nabla_\omega \mathbb{E}_{d^{\pi_\omega}, \pi_\omega}[f]\|_2 \le \mathbb{E}_{d^{\pi_\omega}, \pi_\omega}[\|\nabla_\omega f\|_2] + \frac{2 B_\phi}{1-\gamma} \sup_{s,a} \|f(s,a)\|_2.
\end{equation}

We bound the global Lipschitz continuity of the individual components $\bar{\Sigma}(\omega)$ and $b(\omega)$:
\begin{itemize}
    \item \textit{Moment Matrix:} Because the raw feature $\phi(s,a)$ is independent of $\omega$, we have $\nabla_\omega (\phi \phi^\top) = 0$. Applying \eqref{eq:gen-deriv-limit} yields a global bound on the gradient: $\|\nabla_\omega \bar{\Sigma}(\omega)\|_2 \le 0 + \frac{2 B_\phi}{1-\gamma} \sup_{s,a} \|\phi \phi^\top\|_2 \le \frac{2 B_\phi^3}{1-\gamma} \triangleq L_\Sigma$. By the Mean Value Theorem, for any $\omega, \omega' \in \mathbb{R}^d$, $\|\bar{\Sigma}(\omega) - \bar{\Sigma}(\omega')\|_2 \le L_\Sigma \|\omega - \omega'\|_2$.
    \item \textit{Target Vector:} For $f(s,a) = \phi(s,a) Q_\tau^{\pi_\omega}(s,a)$, we have $\nabla_\omega f = \phi (\nabla_\omega Q_\tau^{\pi_\omega})^\top$. Let $r_{max} \triangleq 1 + \tau \log |\mathcal{A}|$. From Lemma~\ref{lem:bounded-values}, $|Q_\tau^{\pi_\omega}| \le \frac{r_{max}}{1-\gamma}$.
        From Step (ii) in the proof of Lemma~\ref{lem:smoothness}, $\|\nabla_\omega Q_\tau^{\pi_\omega}\|_2 \le \frac{4 \gamma B_\phi r_{max}}{(1-\gamma)^2}$. Applying \eqref{eq:gen-deriv-limit} gives a global bound: $\|\nabla_\omega b(\omega)\|_2 \le \frac{4 \gamma B_\phi^2 r_{max}}{(1-\gamma)^2} + \frac{2 B_\phi^2 r_{max}}{(1-\gamma)^2} \le \frac{6 B_\phi^2 r_{max}}{(1-\gamma)^2} \triangleq L_b$. Thus, for any $\omega, \omega' \in \mathbb{R}^d$, $\|b(\omega) - b(\omega')\|_2 \le L_b \|\omega - \omega'\|_2$.
\end{itemize}

To compute the difference in the ideal critic, we use the algebraic resolvent identity $A^{-1} - B^{-1} = A^{-1}(B - A)B^{-1}$ to expand $\theta^*_\tau(\omega) - \theta^*_\tau(\omega')$:
\begin{align*}
\theta^*_\tau(\omega) - \theta^*_\tau(\omega') &= \bar{\Sigma}(\omega)^{-1} b(\omega) - \bar{\Sigma}(\omega')^{-1} b(\omega') \\
&= \bar{\Sigma}(\omega)^{-1} \big[b(\omega) - b(\omega')\big] + \big[\bar{\Sigma}(\omega)^{-1} - \bar{\Sigma}(\omega')^{-1}\big] b(\omega') \\
&= \bar{\Sigma}(\omega)^{-1} \big[b(\omega) - b(\omega')\big] + \bar{\Sigma}(\omega)^{-1} \big[\bar{\Sigma}(\omega') - \bar{\Sigma}(\omega)\big] \bar{\Sigma}(\omega')^{-1} b(\omega').
\end{align*}
Taking the norm and applying the triangle and sub-multiplicative properties:
\begin{align*}
\|\theta^*_\tau(\omega) - \theta^*_\tau(\omega')\|_2 &\le \|\bar{\Sigma}(\omega)^{-1}\|_2 \|b(\omega) - b(\omega')\|_2 + \|\bar{\Sigma}(\omega)^{-1}\|_2 \|\bar{\Sigma}(\omega') - \bar{\Sigma}(\omega)\|_2 \|\bar{\Sigma}(\omega')^{-1}\|_2 \|b(\omega')\|_2.
\end{align*}
For two parameters $\omega$ and $\omega'$ satisfying $\bar{\Sigma}(\omega) \succeq \kappa I$ and $\bar{\Sigma}(\omega') \succeq \kappa I$, we have the spectral limits $\|\bar{\Sigma}(\omega)^{-1}\|_2 \le \frac{1}{\kappa}$ and $\|\bar{\Sigma}(\omega')^{-1}\|_2 \le \frac{1}{\kappa}$. The target vector is bounded by $\|b(\omega')\|_2 \le \frac{B_\phi r_{max}}{1-\gamma}$.
Substituting our Lipschitz components yields:
\begin{align*}
\|\theta^*_\tau(\omega) - \theta^*_\tau(\omega')\|_2 &\le \left( \frac{1}{\kappa} \right) L_b \|\omega - \omega'\|_2 + \left( \frac{1}{\kappa} \right)^2 L_\Sigma \|\omega - \omega'\|_2 \left( \frac{B_\phi r_{max}}{1-\gamma} \right) \\
&= \left[ \frac{1}{\kappa} \frac{6 B_\phi^2 r_{max}}{(1-\gamma)^2} + \frac{1}{\kappa^2} \frac{2 B_\phi^3}{1-\gamma} \frac{B_\phi r_{max}}{1-\gamma} \right] \|\omega - \omega'\|_2.
\end{align*}
Factoring out the shared constant $\frac{2 B_\phi^2 r_{max}}{\kappa (1-\gamma)^2}$ yields exactly $L_{\theta,\tau}$, completing the proof.
\end{prf}

\subsection{Proofs for Section \ref{sec:suboptimalty-exp-translation}}

\begin{prf}[Proof of Lemma \ref{lem:pdl} (Regularized Performance Difference)]
For any policy $\tilde{\pi}$, the soft state-value function satisfies the regularized Bellman equation:
$$ V^{\tilde{\pi}}_\tau(s) = \mathbb{E}_{a \sim \tilde{\pi}(\cdot|s)} \left[ r(s,a) - \tau \log \tilde{\pi}(a|s) + \gamma \mathbb{E}_{s' \sim \mathcal{P}(\cdot|s,a)} [V^{\tilde{\pi}}_\tau(s')] \right]. $$
We subtract the soft value function of a reference policy $V^\pi_\tau(s)$ from both sides. To create a telescoping sum, we add and subtract $\gamma \mathbb{E}_{s'} [V^\pi_\tau(s')]$ inside the expectation:
$$ V^{\tilde{\pi}}_\tau(s) - V^\pi_\tau(s) = \mathbb{E}_{a \sim \tilde{\pi}} \left[ r(s,a) - \tau \log \tilde{\pi}(a|s) + \gamma \mathbb{E}_{s'}[V^{\tilde{\pi}}_\tau(s') - V^\pi_\tau(s')] + \gamma \mathbb{E}_{s'}[V^\pi_\tau(s')] \right] - V^\pi_\tau(s). $$
By the definition of the soft action-value function, $Q^\pi_\tau(s,a) = r(s,a) + \gamma \mathbb{E}_{s'}[V^\pi_\tau(s')]$. Substituting this simplifies the expression to:
$$ V^{\tilde{\pi}}_\tau(s) - V^\pi_\tau(s) = \mathbb{E}_{a \sim \tilde{\pi}} \big[ Q^\pi_\tau(s,a) - \tau \log \tilde{\pi}(a|s) - V^\pi_\tau(s) \big] + \gamma \mathbb{E}_{a \sim \tilde{\pi}, s' \sim \mathcal{P}} \big[ V^{\tilde{\pi}}_\tau(s') - V^\pi_\tau(s') \big]. $$
By unrolling this recursive relationship over the infinite horizon trajectory generated by following $\tilde{\pi}$ from a fixed initial state $s_0$, the discounted future differences evaluate to the expectation over $d_{s_0}^{\tilde{\pi}}$, the discounted state visitation distribution conditioned on $s_0$:
$$ V^{\tilde{\pi}}_\tau(s_0) - V^\pi_\tau(s_0) = \frac{1}{1-\gamma} \mathbb{E}_{s \sim d_{s_0}^{\tilde{\pi}}} \left[ \sum_{a \in \mathcal{A}} \tilde{\pi}(a|s) \big( Q^\pi_\tau(s,a) - \tau \log \tilde{\pi}(a|s) - V^\pi_\tau(s) \big) \right]. $$
Taking the expectation over the initial state distribution $s_0 \sim \mu$ yields the global regularized objective difference $J_\tau(\tilde{\pi}) - J_\tau(\pi)$ on the left-hand side. Because $d^{\tilde{\pi}}(s) = \mathbb{E}_{s_0 \sim \mu}[d_{s_0}^{\tilde{\pi}}(s)]$, we obtain the final identity:
$$ J_\tau(\tilde{\pi}) - J_\tau(\pi) = \frac{1}{1-\gamma} \mathbb{E}_{s \sim d^{\tilde{\pi}}} \left[ \sum_{a \in \mathcal{A}} \tilde{\pi}(a|s) \big( Q^\pi_\tau(s,a) - \tau \log \tilde{\pi}(a|s) - V^\pi_\tau(s) \big) \right], $$
completing the proof.
\end{prf}

\begin{prf}[Proof of Lemma \ref{lem:optim-identity} (Global Regularized Suboptimality Identity)]
(i) \textit{Optimal Soft Advantage and Boltzmann Target.}
By definition, the regularized optimal value function maximizes the soft Bellman equation over all possible probability distributions $\tilde{\pi}(\cdot|s)$:
$$ V^*_\tau(s) = \max_{\tilde{\pi}} \sum_a \tilde{\pi}(a|s) \big[ Q^*_\tau(s,a) - \tau \log \tilde{\pi}(a|s) \big]. $$
By factoring out $\tau$ and introducing the normalizing partition function $Z(s) \triangleq \sum_{a'} \exp(Q^*_\tau(s,a')/\tau)$, we
rewrite the term inside the brackets:
$$  \tau \sum_a \tilde{\pi}(a|s) \left[ \log \left( \frac{\exp(Q^*_\tau(s,a)/\tau) / Z(s)}{\tilde{\pi}(a|s)} \right) + \log Z(s) \right]. $$
Recognizing that $\frac{\exp(Q^*_\tau(s,a)/\tau)}{Z(s)}$ is exactly the Boltzmann distribution, and noting that $\sum_a \tilde{\pi}(a|s) \log Z(s) = \log Z(s)$ because probabilities sum to $1$, the maximum reduces to:
$$ V^*_\tau(s) = \max_{\tilde{\pi}} \left[ \tau \log Z(s) - \tau D_{KL}\left(\tilde{\pi}(\cdot|s) \;\middle\|\; \frac{\exp(Q^*_\tau(s,\cdot)/\tau)}{Z(s)}\right) \right]. $$
Because the Kullback-Leibler divergence is non-negative and is zero if and only if the distributions match, this expression is uniquely maximized when $\tilde{\pi}$ is exactly the Boltzmann policy $\pi^*_\tau(a|s) = \frac{\exp(Q^*_\tau(s,a)/\tau)}{Z(s)}$.
Substituting this optimal policy back zeroes out the KL divergence, yielding the identity for the regularized optimal state-value and action-value:
$$ V^*_\tau(s) = \tau \log Z(s) = \tau \log \sum_{a'} \exp \left( \frac{Q^*_\tau(s,a')}{\tau} \right). $$
Consequently, the optimal Boltzmann policy can be equivalently written as $\pi^*_\tau(a|s) = \frac{\exp(Q^*_\tau(s,a)/\tau)}{\exp(V^*_\tau(s)/\tau)}$. Taking the logarithm and multiplying by $\tau$ yields an exact relation for the optimal log-probability:
$$ \tau \log \pi^*_\tau(a|s) = Q^*_\tau(s,a) - V^*_\tau(s). $$
That is, the regularized advantage is zero: $A_\tau^{\pi^*_\tau}(s,a)=Q^*_\tau(s,a) - \tau \log \pi^*_\tau(a|s) - V^*_\tau(s) = 0$.

(ii) \textit{Synthesis via KL Divergence.}
By the Regularized Performance Difference (Lemma \ref{lem:pdl}) evaluated between $\pi$ and $\pi^*_\tau$, the objective difference is:
$$ J_\tau(\pi) - J_\tau(\pi^*_\tau) = \frac{1}{1-\gamma} \mathbb{E}_{s \sim d^\pi} \left[ \sum_a \pi(a|s) \big( Q^*_\tau(s,a) - \tau \log \pi(a|s) - V^*_\tau(s) \big) \right]. $$
We substitute the preceding identity for $Q^*_\tau(s,a) - V^*_\tau(s)$ into this objective difference:
$$ J_\tau(\pi) - J_\tau(\pi^*_\tau) = \frac{1}{1-\gamma} \mathbb{E}_{s \sim d^\pi} \left[ \sum_a \pi(a|s) \big( \tau \log \pi^*_\tau(a|s) - \tau \log \pi(a|s) \big) \right]. $$
Factoring out $-\tau$ isolates the definition of the Kullback-Leibler divergence:
$$ J_\tau(\pi) - J_\tau(\pi^*_\tau) = - \frac{\tau}{1-\gamma} \mathbb{E}_{s \sim d^\pi} \left[ \sum_a \pi(a|s) \log \frac{\pi(a|s)}{\pi^*_\tau(a|s)} \right] = - \frac{\tau}{1-\gamma} \mathbb{E}_{s \sim d^\pi} \big[ D_{KL}(\pi(\cdot|s) \| \pi^*_\tau(\cdot|s)) \big]. $$
This matches the stated identity, completing the proof.
\end{prf}

\begin{prf}[Proof of Lemma \ref{lem:sub-decomp} (Regularized Suboptimality Decompositions)]
(i) \textit{Extracting KL Divergence.} Starting from the Regularized Performance Difference (Lemma \ref{lem:pdl}) evaluated between the comparator $\pi$ and the training policy $\pi_t$, we add and subtract $\tau \log \pi_t$ to extract the KL divergence:
\begin{align*}
(1-\gamma) \big( J_\tau(\pi) - J_\tau(\pi_t) \big) &= \mathbb{E}_{d^\pi} \left[ \sum_a \pi \big( Q_\tau^{\pi_t} - \tau \log \pi \big) - V_\tau^{\pi_t} \right] \\
&= \mathbb{E}_{d^\pi} \left[ \sum_a \pi \big( Q_\tau^{\pi_t} - \tau \log \pi_t \big) - \tau D_{KL}(\pi \| \pi_t) - V_\tau^{\pi_t} \right].
\end{align*}
Substituting $V_\tau^{\pi_t}(s) = \sum_a \pi_t(a|s) \big( Q_\tau^{\pi_t}(s,a) - \tau \log \pi_t(a|s) \big)$ combines the action summations into a single inner product:
$$(1-\gamma) \big( J_\tau(\pi) - J_\tau(\pi_t) \big) = -\tau \mathbb{E}_{d^\pi}[D_{KL}(\pi \| \pi_t)] + \mathbb{E}_{d^\pi} \Big[ \langle Q_\tau^{\pi_t} - \tau \log \pi_t, \pi - \pi_t \rangle_{\mathcal{A}} \Big].$$

(ii) \textit{Injecting the Policy Parameterization.} We substitute the log-linear identity of the actor policy, $\tau \log \pi_t(a|s) = \tau \omega_t^\top \phi(s,a) - \tau \log Z_t(s)$. Because $\tau \log Z_t(s)$ is a state-dependent baseline, it acts as a constant in the action inner product and cancels against $\sum_a (\pi - \pi_t) = 0$:
$$(1-\gamma) \big( J_\tau(\pi) - J_\tau(\pi_t) \big) = -\tau \mathbb{E}_{d^\pi}[D_{KL}(\pi \| \pi_t)] + \mathbb{E}_{d^\pi} \Big[ \langle Q_\tau^{\pi_t} - \tau \omega_t^\top \phi, \pi - \pi_t \rangle_{\mathcal{A}} \Big].$$

(iii) \textit{Ideal Critic Substitution.} We substitute the ideal uncentered critic $Q_\tau^{\pi_t}(s,a) = \theta^*_\tau(\omega_t)^\top \phi(s,a) + \epsilon_t(s,a)$ into the inner product, which yields the first stated decomposition:
$$(1-\gamma) \big( J_\tau(\pi) - J_\tau(\pi_t) \big) = -\tau \mathbb{E}_{d^\pi}[D_{KL}(\pi \| \pi_t)] + \mathbb{E}_{d^\pi} \Big[ \langle (\theta^*_\tau(\omega_t) - \tau \omega_t)^\top \phi, \pi - \pi_t \rangle_{\mathcal{A}} \Big] + \mathbb{E}_{d^\pi} \big[ \langle \epsilon_t, \pi - \pi_t \rangle_{\mathcal{A}} \big].$$

(iv) \textit{Actual Critic Separation.} By definition, the critic tracking error is $e_t = \theta_{t+1} - \theta^*_\tau(\omega_t)$. We substitute $\theta^*_\tau(\omega_t) = \theta_{t+1} - e_t$ directly into the ideal algorithmic progress:
$$(\theta^*_\tau(\omega_t) - \tau \omega_t)^\top \phi = (\theta_{t+1} - \tau \omega_t)^\top \phi - e_t^\top \phi.$$
Substituting this identity into the ideal critic decomposition separates the algorithmic progress driven by the physical update $\theta_{t+1}$ from the statistical tracking deviation $e_t$, yielding the second stated decomposition and completing the proof.
\end{prf}

\begin{prf}[Proof of Lemma \ref{lem:entropy-bound} (Policy Entropy Decomposition Bound)]
We evaluate the entropy of the action distribution by considering a two-step hierarchical decision process for a given state $s$.
We define a binary indicator variable $I \in \{0, 1\}$ such that $I = 0$ if the selected action is $a^*(s)$, and $I = 1$ if the selected action is any $a \neq a^*(s)$.
The probabilities given state $s$ are $P(I=0 \mid s) = 1 - q(s)$ and $P(I=1 \mid s) = q(s)$.
Because $I$ is a deterministic function of the chosen action, the joint entropy equals the entropy of the policy: $H(\pi(\cdot|s), I) = H(\pi(\cdot|s))$.
By the Chain Rule of Entropy, the joint entropy can be decomposed into the marginal entropy of the indicator and the conditional entropy of the action:
$$ H(\pi(\cdot|s)) = H(I) + H(\pi(\cdot|s) \mid I). $$

We bound each term independently:
\begin{itemize}
    \item \textit{Binary Entropy ($H(I)$):} The indicator is a Bernoulli random variable with parameter $q(s)$.
    Its entropy is exactly the binary entropy function:
    $$ H(I) = - q(s) \log q(s) - (1 - q(s)) \log (1 - q(s)). $$
    \item \textit{Conditional Entropy ($H(\pi(\cdot|s) \mid I)$):} By definition, $H(\pi(\cdot|s) \mid I) = P(I=0 \mid s)H(\pi(\cdot|s) \mid I=0) + P(I=1 \mid s)H(\pi(\cdot|s) \mid I=1)$.
    If $I=0$, the action is known to be $a^*(s)$, meaning the uncertainty $H(\pi(\cdot|s) \mid I=0) = 0$.
    If $I=1$, the action is distributed among the $|\mathcal{A}| - 1$ remaining actions.
    The maximum possible entropy for a distribution over $|\mathcal{A}| - 1$ elements is $\log(|\mathcal{A}| - 1)$.
    Therefore, the conditional entropy is bounded by:
    $$ H(\pi(\cdot|s) \mid I) \le (1 - q(s)) \cdot 0 + q(s) \log(|\mathcal{A}| - 1) \le q(s) \log |\mathcal{A}|. $$
\end{itemize}
Substituting these two bounds back into the Chain Rule decomposition yields the stated result.
\end{prf}

\begin{prf}[Proof of Lemma \ref{lem:reg-value-bound} (Uniform Regularized Value Bound)]
For any arbitrary policy $\pi$, let $V^\pi_0(s)$ be its standard unregularized expected return, defined as:
$$ V^\pi_0(s) = \mathbb{E}_\pi \left[ \sum_{t=0}^\infty \gamma^t r(s_t, a_t) \;\middle|\; s_0=s \right]. $$
Let $V^\pi_\tau(s)$ be its entropy-regularized expected return, defined as:
$$ V^\pi_\tau(s) = V^\pi_0(s) + \tau \mathbb{E}_\pi \left[ \sum_{t=0}^\infty \gamma^t H(\pi(\cdot|s_t)) \;\middle|\; s_0=s \right]. $$
Because the Shannon entropy is bounded by $0 \le H(\pi(\cdot|s)) \le \log |\mathcal{A}|$, the discounted sum of entropy is bounded between $0$ and $\frac{\log |\mathcal{A}|}{1-\gamma}$.
Thus, for any policy $\pi$, the regularized and unregularized values satisfy:
\begin{equation} \label{eq:reg-value-bound}
V^\pi_0(s) \le V^\pi_\tau(s) \le V^\pi_0(s) + \frac{\tau \log |\mathcal{A}|}{1-\gamma}.
\end{equation}

We evaluate the optimal policies against each other to establish tight upper and lower bounds on $V^*_\tau(s)$.
\begin{itemize}
    \item \textit{Lower Bound:} Because $\pi^*_\tau$ maximizes the regularized objective, its regularized value is at least as high as the regularized value of the unregularized globally optimal policy $\pi^*_0$.
    Thus, $V^*_\tau(s) = V^{\pi^*_\tau}_\tau(s) \ge V^{\pi^*_0}_\tau(s)$. Applying \eqref{eq:reg-value-bound} to $\pi^*_0$ yields $V^{\pi^*_0}_\tau(s) \ge V^{\pi^*_0}_0(s) = V^*_0(s)$.
    Therefore, $V^*_\tau(s) \ge V^*_0(s)$.
    \item \textit{Upper Bound:} Because $\pi^*_0$ maximizes the unregularized objective, its unregularized value is at least as high as the unregularized value of the regularized globally optimal policy $\pi^*_\tau$.
    Thus, $V^*_0(s) = V^{\pi^*_0}_0(s) \ge V^{\pi^*_\tau}_0(s)$. Applying \eqref{eq:reg-value-bound} to the soft policy $\pi^*_\tau$ gives $V^*_\tau(s) \le V^{\pi^*_\tau}_0(s) + \frac{\tau \log |\mathcal{A}|}{1-\gamma}$.
    Substituting the first inequality yields $V^*_\tau(s) \le V^*_0(s) + \frac{\tau \log |\mathcal{A}|}{1-\gamma}$.
\end{itemize}
Combining the bounds proves the result for the state-value functions.
For the action-value functions, we invoke the standard Bellman equations: $Q^\pi_0(s,a) = r(s,a) + \gamma \mathbb{E}_{s'}[V^\pi_0(s')]$.
Because the reward and transition dynamics are identical, subtracting the two equations yields:
$$ Q^*_\tau(s,a) - Q^*_0(s,a) = \gamma \mathbb{E}_{s' \sim \mathcal{P}(\cdot|s,a)} \big[ V^*_\tau(s') - V^*_0(s') \big]. $$
Since $0 \le V^*_\tau(s') - V^*_0(s') \le \frac{\tau \log |\mathcal{A}|}{1-\gamma}$ for all states, the expectation preserves the bound.
Thus $0 \le Q^*_\tau(s,a) - Q^*_0(s,a) \le \frac{\gamma \tau \log |\mathcal{A}|}{1-\gamma} \le \frac{\tau \log |\mathcal{A}|}{1-\gamma}$, completing the proof.
\end{prf}

\begin{prf}[Proof of Theorem \ref{thm:suboptimal-mass} (Exponential Translation of Regularized Suboptimality)]
The proof proceeds in four steps.
First, we upper-bound the unregularized performance gap by the expected suboptimal probability mass.
Second, we express the regularized performance gap as the reverse KL divergence and lower-bound its cross-entropy and policy entropy components using the suboptimal mass.
Third, we combine these lower bounds and invert the resulting non-linear inequality to isolate the suboptimal mass.
Finally, we synthesize these results to translate regularized algorithmic progress into bounds on the unregularized performance gap.

(i) \textit{Unregularized Performance Difference.}
By the standard Performance Difference Lemma applied to the unregularized objective $J_0$:
$$ J_0(\pi^*_0) - J_0(\pi_t) = \frac{1}{1-\gamma} \mathbb{E}_{s \sim d^{\pi_t}} \left[ \sum_a \pi_t(a|s) \big( V^*_0(s) - Q^*_0(s,a) \big) \right]. $$
Because the rewards are bounded in $[0,1]$, we have $V^*_0(s) \le \sum_{t=0}^\infty \gamma^t \cdot 1 = \frac{1}{1-\gamma}$ and $Q^*_0(s,a) \ge 0$.
The maximum optimality gap $A_{max} \triangleq \sup_{s,a} (V^*_0(s) - Q^*_0(s,a))$ is bounded by $\frac{1}{1-\gamma}$.
Under the Minimal Action Gap (Assumption~\ref{ass:action-gap}), the unregularized globally optimal policy $\pi^*_0$ concentrates all probability mass on the unique optimal action $a^*(s)$, yielding $V^*_0(s) - Q^*_0(s,a^*(s)) = 0$.
For all suboptimal actions $a \neq a^*(s)$, we substitute the maximum gap $A_{max}$.
Factoring this out bounds the unregularized performance gap by:
$$ J_0(\pi^*_0) - J_0(\pi_t) \le \frac{A_{max}}{1-\gamma} \mathbb{E}_{d^{\pi_t}} \left[ \sum_{a \neq a^*(s)} \pi_t(a|s) \right] = \frac{A_{max}}{1-\gamma} \mathbb{E}_{d^{\pi_t}} [q_t(s)], $$
where $q_t(s) \triangleq \sum_{a \neq a^*(s)} \pi_t(a|s)$ is the probability mass on suboptimal actions.

(ii) \textit{Reverse KL Identity and Bounds.}
By the Global Regularized Suboptimality Identity (Lemma~\ref{lem:optim-identity}), the global regularized performance gap is proportional to the state-averaged reverse KL divergence:
$$ \text{Gap}_t^\dag = \frac{\tau}{1-\gamma} \mathbb{E}_{d^{\pi_t}} \big[ D_{KL}(\pi_t(\cdot|s) \| \pi^*_\tau(\cdot|s)) \big]. $$
To link this divergence to the suboptimal mass $q_t(s)$, we expand the pointwise reverse KL divergence into the cross-entropy and the negative policy entropy:
$$ D_{KL}(\pi_t \| \pi^*_\tau) = \sum_a \pi_t(a|s) \log \frac{\pi_t(a|s)}{\pi^*_\tau(a|s)} = - \sum_a \pi_t(a|s) \log \pi^*_\tau(a|s) - H(\pi_t(\cdot|s)). $$

We independently lower-bound the cross-entropy and negative entropy terms using the suboptimal probability mass.
\begin{itemize}
    \item \textit{Cross-Entropy:}
    The regularized globally optimal policy $\pi^*_\tau(a|s) \propto \exp(Q^*_\tau(s,a)/\tau)$ naturally assigns exponentially small probability to suboptimal actions.
    Under Assumption~\ref{ass:action-gap}, the action gap satisfies $Q^*_0(s,a^*(s)) - Q^*_0(s,a) \ge \Delta$ for $a \neq a^*(s)$.
    By the Uniform Regularized Value Bound (Lemma~\ref{lem:reg-value-bound}), the regularized and unregularized $Q$-values are bounded by $\|Q^*_\tau - Q^*_0\|_\infty \le \frac{\tau \log |\mathcal{A}|}{1-\gamma}$.
    Then the regularized globally optimal policy's probability of a suboptimal action $a$ is bounded by:
    \begin{align}
    \pi^*_\tau(a|s) &\le \frac{\pi^*_\tau(a|s)}{\pi^*_\tau(a^*(s)|s)} = \frac{\exp(Q^*_\tau(s,a)/\tau)}{\exp(Q^*_\tau(s,a^*(s))/\tau)} \le \exp\left( - \frac{\Delta}{\tau} + \frac{2 \log |\mathcal{A}|}{1-\gamma} \right). \label{eq:action-prob-min-gap}
    \end{align}
    Taking the negative logarithm yields $-\log \pi^*_\tau(a|s) \ge \frac{\Delta}{\tau} - \frac{2 \log |\mathcal{A}|}{1-\gamma}$.
    The unregularized globally optimal action term is $-\log \pi^*_\tau(a^*(s)|s) \ge 0$. Dropping it lower-bounds the cross-entropy:
    $$ - \sum_a \pi_t(a|s) \log \pi^*_\tau(a|s) \ge - \sum_{a \neq a^*(s)} \pi_t(a|s) \log \pi^*_\tau(a|s) \ge q_t(s) \left( \frac{\Delta}{\tau} - \frac{2 \log |\mathcal{A}|}{1-\gamma} \right). $$

    \item \textit{Negative Entropy:} We invoke the Policy Entropy Decomposition Bound (Lemma~\ref{lem:entropy-bound}), designating the unique optimal action $a^*(s)$ under Assumption~\ref{ass:action-gap} as the single anchor:
    $$ - H(\pi_t(\cdot|s)) \ge q_t(s) \log q_t(s) + (1-q_t(s)) \log(1-q_t(s)) - q_t(s) \log |\mathcal{A}|. $$
    Using the inequality $(1-x)\log(1-x) \ge -x$ for $x \in [0,1]$, we bound the binary entropy term:
    $$ - H(\pi_t(\cdot|s)) \ge q_t(s) \log q_t(s) - q_t(s) - q_t(s) \log |\mathcal{A}|. $$
\end{itemize}

(iii) \textit{Inverting KL Bound for Suboptimal Mass.}
Summing the cross-entropy and negative entropy lower bounds yields the combined inequality:
$$ D_{KL}(\pi_t \| \pi^*_\tau) \ge q_t(s) \log q_t(s) + q_t(s) \left( \frac{\Delta}{\tau} - \left[ \frac{2 \log |\mathcal{A}|}{1-\gamma} + \log |\mathcal{A}| + 1 \right] \right). $$
Substituting our definition $\log C_\gamma \triangleq \frac{2 \log |\mathcal{A}|}{1-\gamma} + \log |\mathcal{A}| + 1$ and letting $B \triangleq \frac{\Delta}{\tau} - \log C_\gamma$, we obtain $D_{KL}(\pi_t \| \pi^*_\tau) \ge q_t(s) B + q_t(s) \log q_t(s)$.
Then we bound $q_t \triangleq q_t(s)$ by partitioning its domain $[0,1]$:
\begin{itemize}
    \item \textit{Case 1:} $q_t \ge e^{-B/2}$.
    Then $\log q_t \ge -B/2$. Substituting this yields $D_{KL} \ge q_t B - q_t(B/2) = q_t (B/2)$, which implies $q_t \le \frac{2}{B} D_{KL}(\pi_t \| \pi^*_\tau)$.
    \item \textit{Case 2:} $q_t < e^{-B/2}$. The mass is trivially bounded by the exponential threshold $q_t < e^{-B/2} = e^{\frac{1}{2} \log C_\gamma} \exp(-\frac{\Delta}{2\tau}) = \sqrt{C_\gamma} \exp(-\frac{\Delta}{2\tau})$.
\end{itemize}
Combining these cases yields $q_t(s) \le \frac{2}{\frac{\Delta}{\tau} - \log C_\gamma} D_{KL}(\pi_t \| \pi^*_\tau) + \sqrt{C_\gamma} \exp\left(-\frac{\Delta}{2\tau}\right)$.

(iv) \textit{Synthesis.}
From the threshold assumption $\tau \le \frac{\Delta}{2 \log C_\gamma}$, the denominator satisfies $\frac{\Delta}{\tau} - \log C_\gamma \ge \frac{\Delta}{2\tau}$.
Substituting this and taking the expectation over $d^{\pi_t}$ yields:
$$ \mathbb{E}_{d^{\pi_t}}[q_t(s)] \le \frac{4 \tau}{\Delta} \mathbb{E}_{d^{\pi_t}} \big[ D_{KL}(\pi_t \| \pi^*_\tau) \big] + \sqrt{C_\gamma} \exp\left(-\frac{\Delta}{2\tau}\right). $$
Substituting this mass bound into the unregularized performance difference from Step (i) yields the first stated result.
Finally, substituting the Global Regularized Suboptimality Identity ($\frac{\tau}{1-\gamma} \mathbb{E}[D_{KL}] = \text{Gap}_t^\dag$) into this $J_0$ bound completes the proof.
\end{prf}

\begin{prf}[Proof of Corollary \ref{cor:universal-suboptimal-mass} (Universal Unregularized Suboptimality Bound)]
We partition the global temperature space into two regimes.
\begin{itemize}
\item \textit{Case 1 (Cold Regime):} $\tau \le \frac{\Delta}{2 \log C_\gamma}$.
The bounds hold by Theorem~\ref{thm:suboptimal-mass}.
Because $C_\gamma \ge 1$, replacing the $\frac{A_{max} \sqrt{C_\gamma}}{1-\gamma}$ tail coefficient with the larger coefficient $C_{tail} \triangleq \frac{A_{max} C_\gamma}{1-\gamma}$ preserves both inequalities.
\item \textit{Case 2 (Hot Regime):} $\tau > \frac{\Delta}{2 \log C_\gamma}$.
In this regime, the temperature is too high to guarantee the positive linear relationship when inverting the KL bound.
However, as established in Step (i) of Theorem~\ref{thm:suboptimal-mass}, $ J_0(\pi^*_0) - J_0(\pi_t) \le \frac{A_{max}}{1-\gamma} \mathbb{E}_{d^{\pi_t}} [q_t(s)] $.
Since probability mass evaluates to at most 1, $J_0(\pi^*_0) - J_0(\pi_t) \le \frac{A_{max}}{1-\gamma}$ trivially holds.
We rearrange the regime condition $\tau > \frac{\Delta}{2 \log C_\gamma}$ to $1 < C_\gamma \exp\left(-\frac{\Delta}{2\tau}\right)$, and substitute this into our trivial global bound:
$$ J_0(\pi^*_0) - J_0(\pi_t) \le \frac{A_{max}}{1-\gamma} \le \frac{A_{max} C_\gamma}{1-\gamma} \exp\left(-\frac{\Delta}{2\tau}\right) = C_{tail} \exp\left(-\frac{\Delta}{2\tau}\right). $$
Because the state-averaged KL divergence and the regularized performance gap are both non-negative, adding either the $\frac{4 A_{max} \tau}{\Delta(1-\gamma)} \mathbb{E}[D_{KL}]$ term or the $\frac{4 A_{max}}{\Delta} \text{Gap}_t^\dag$ term to the right side naturally maintains the respective inequality.
\end{itemize}
Combining the two cases completes the proof.
\end{prf}

\subsection{Proofs for Section \ref{sec:pl-geometry}}

\begin{prf}[Proof of Lemma~\ref{lem:chi2-KL} ($\chi^2$ to KL Divergence Bound)]
We use the analytic inequality that for any $x \in [0, M]$,
$$(x-1)^2 \le 2 M (x \log x - x + 1).$$
We substitute $x = \frac{\dif P}{\dif Q}$ and integrate with respect to $Q$ over the space $\mathcal{X}$:
$$\int_{\mathcal{X}} \left( \frac{\dif P}{\dif Q} - 1 \right)^2 \dif Q \le 2 M \int_{\mathcal{X}} \left( \frac{\dif P}{\dif Q} \log \frac{\dif P}{\dif Q} - \frac{\dif P}{\dif Q} + 1 \right) \dif Q.$$
By definition, the left-hand side is the $\chi^2(P\|Q)$ divergence, whereas the right-hand side evaluates to $2 M D_{KL}(P \| Q)$, completing the proof.
\end{prf}

\begin{prf}[Proof of Lemma \ref{lem:pl-condition} (Parameter-Space PL Condition)]
We separately bound the Ideal Algorithmic Progress and the Approximation Bias from the Ideal Critic Decomposition evaluated at $\pi = \pi_{\omega^*_\tau}$ in Lemma~\ref{lem:sub-decomp}(i):
\begin{align*}
(1-\gamma)\text{Gap}_t = -\tau D_t + \underbrace{\mathbb{E}_{d^{\omega^*_\tau}} \Big[ \langle (\theta^*_\tau(\omega_t) - \tau \omega_t)^\top \phi, \pi_{\omega^*_\tau} - \pi_t \rangle_{\mathcal{A}} \Big]}_{\text{Ideal Algorithmic Progress}}
+ \underbrace{\mathbb{E}_{d^{\omega^*_\tau}} \big[ \langle \epsilon_t, \pi_{\omega^*_\tau} - \pi_t \rangle_{\mathcal{A}} \big]}_{\text{Approximation Bias}}.
\end{align*}

(i) \textit{Bounding Ideal Algorithmic Progress.} For any state $s$, we apply H\"{o}lder's inequality over the action space and Pinsker's inequality ($\|\pi_{\omega^*_\tau} - \pi_t\|_1 \le \sqrt{2 D_{KL}(\pi_{\omega^*_\tau} \| \pi_t)}$) to the inner product:
\begin{align*}
\langle (\theta^*_\tau(\omega_t) - \tau \omega_t)^\top \phi, \pi_{\omega^*_\tau} - \pi_t \rangle_{\mathcal{A}} &\le \max_a \big| (\theta^*_\tau(\omega_t) - \tau \omega_t)^\top \phi(s,a) \big| \|\pi_{\omega^*_\tau} - \pi_t\|_1 \\
&\le B_\phi \|\theta^*_\tau(\omega_t) - \tau \omega_t\|_2 \sqrt{2 D_{KL}(\pi_{\omega^*_\tau} \| \pi_t)}.
\end{align*}
Taking the expectation over $d^{\omega^*_\tau}(s)$ and applying Jensen's Inequality, $\mathbb{E}_{d^{\omega^*_\tau}}[\sqrt{D_{KL}}] \le \sqrt{\mathbb{E}_{d^{\omega^*_\tau}}[D_{KL}]} = \sqrt{D_t}$, yields:
$$\mathbb{E}_{d^{\omega^*_\tau}} \Big[ \langle (\theta^*_\tau(\omega_t) - \tau \omega_t)^\top \phi, \pi_{\omega^*_\tau} - \pi_t \rangle_{\mathcal{A}} \Big] \le B_\phi \|\theta^*_\tau(\omega_t) - \tau \omega_t\|_2 \sqrt{2 D_t}.$$
We decouple the resulting product using Young's Inequality ($xy \le \frac{\tau}{2}x^2 + \frac{1}{2\tau}y^2$) with $x = \sqrt{D_t}$ and $y = \sqrt{2} B_\phi \|\theta^*_\tau(\omega_t) - \tau \omega_t\|_2$:
$$\mathbb{E}_{d^{\omega^*_\tau}} \Big[ \langle (\theta^*_\tau(\omega_t) - \tau \omega_t)^\top \phi, \pi_{\omega^*_\tau} - \pi_t \rangle_{\mathcal{A}} \Big] \le \frac{\tau}{2} D_t + \frac{B_\phi^2}{\tau} \|\theta^*_\tau(\omega_t) - \tau \omega_t\|_2^2.$$

(ii) \textit{Bounding Approximation Bias.} For any state $s$, we define the local policy mismatch $W(s) \triangleq \max_a \frac{\pi_{\omega^*_\tau}(a|s)}{\pi_t(a|s)}$.
We apply Action-Space Cauchy-Schwarz inequality:
\begin{align*}
\langle \epsilon_t, \pi_{\omega^*_\tau} - \pi_t \rangle_{\mathcal{A}} &= \sum_a \pi_t(a|s) \epsilon_t(s,a) \left( \frac{\pi_{\omega^*_\tau}(a|s)}{\pi_t(a|s)} - 1 \right) \\
&\le \sqrt{ \mathbb{E}_{a \sim \pi_t}[\epsilon_t(s,a)^2] } \sqrt{ \sum_a \pi_t(a|s) \left( \frac{\pi_{\omega^*_\tau}(a|s)}{\pi_t(a|s)} - 1 \right)^2 }.
\end{align*}
The second square root yields the $\chi^2$-divergence. By applying the $\chi^2$ bound (Lemma~\ref{lem:chi2-KL})
$\chi^2(\pi_{\omega^*_\tau}(\cdot|s) \| \pi_t(\cdot|s)) \le 2 W(s) D_{KL}(\pi_{\omega^*_\tau}(\cdot|s) \| \pi_t(\cdot|s))$
pointwise at $s$, we obtain:
$$\langle \epsilon_t, \pi_{\omega^*_\tau} - \pi_t \rangle_{\mathcal{A}} \le \sqrt{ \mathbb{E}_{a \sim \pi_t}[\epsilon_t(s,a)^2] } \sqrt{ 2 W(s) D_{KL}(\pi_{\omega^*_\tau}(\cdot|s) \| \pi_t(\cdot|s)) }.$$
Taking the expectation over $d^{\omega^*_\tau}(s)$ and applying State-Space Cauchy-Schwarz:
\begin{align*}
\mathbb{E}_{d^{\omega^*_\tau}} \big[ \langle \epsilon_t, \pi_{\omega^*_\tau} - \pi_t \rangle_{\mathcal{A}} \big] &\le \mathbb{E}_{d^{\omega^*_\tau}} \Big[ \sqrt{ 2 D_{KL}(\pi_{\omega^*_\tau}(\cdot|s) \| \pi_t(\cdot|s)) } \sqrt{ W(s) \mathbb{E}_{a \sim \pi_t}[\epsilon_t(s,a)^2] } \Big] \\
&\le \sqrt{ 2 D_t } \sqrt{ \mathbb{E}_{d^{\omega^*_\tau}} \big[ W(s) \mathbb{E}_{\pi_t}[\epsilon_t^2] \big] }.
\end{align*}
To control the second root, we expand the expectation over the state space and apply the joint concentrability bound $\frac{d^{\omega^*_\tau}(s)}{d^{\pi_t}(s)} W(s) = \max_a \frac{d^{\omega^*_\tau}(s) \pi_{\omega^*_\tau}(a|s)}{d^{\pi_t}(s) \pi_t(a|s)} \le C_{joint}$ (Assumption~\ref{ass:parametric-density}):
\begin{align*}
\mathbb{E}_{d^{\omega^*_\tau}} \big[ W(s) \mathbb{E}_{\pi_t}[\epsilon_t^2] \big] &= \mathbb{E}_{d^{\pi_t}} \left[ \left( \frac{d^{\omega^*_\tau}(s)}{d^{\pi_t}(s)} W(s) \right) \mathbb{E}_{a \sim \pi_t}[\epsilon_t(s,a)^2] \right] \\
&\le C_{joint} \mathbb{E}_{d^{\pi_t}} \Big[ \mathbb{E}_{a \sim \pi_t}[\epsilon_t(s,a)^2] \Big] \\
&= C_{joint} \mathbb{E}_{d^{\pi_t}, \pi_t}[\epsilon_t^2] \le C_{joint} \epsilon_{app},
\end{align*}
where the last step uses Assumption~\ref{ass:approximation-error}.
Applying Young's Inequality ($xy \le \frac{\tau}{2}x^2 + \frac{1}{2\tau}y^2$) with $x = \sqrt{D_t}$ and $y = \sqrt{2 C_{joint} \epsilon_{app}}$ yields:
$$\mathbb{E}_{d^{\omega^*_\tau}} \big[ \langle \epsilon_t, \pi_{\omega^*_\tau} - \pi_t \rangle_{\mathcal{A}} \big] \le \frac{\tau}{2} D_t + \frac{C_{joint}}{\tau} \epsilon_{app}.$$

(iii) \textit{Synthesis.} We substitute the decoupled ideal algorithmic progress and approximation bias bounds back into the Ideal Critic Decomposition (Lemma \ref{lem:sub-decomp}) for $\pi = \pi_{\omega^*_\tau}$:
\begin{align*}
(1-\gamma) \text{Gap}_t &\le -\tau D_t + \left( \frac{\tau}{2} D_t + \frac{B_\phi^2}{\tau} \|\theta^*_\tau(\omega_t) - \tau \omega_t\|_2^2 \right) + \left( \frac{\tau}{2} D_t + \frac{C_{joint}}{\tau} \epsilon_{app} \right) \\
&= \frac{B_\phi^2}{\tau} \|\theta^*_\tau(\omega_t) - \tau \omega_t\|_2^2 + \frac{C_{joint}}{\tau} \epsilon_{app}.
\end{align*}
Remarkably, the positive $\tau D_t$ penalties from Young's inequality perfectly cancel the $-\tau D_t$ restorative entropy force. Dividing both sides by $1-\gamma$ yields the stated PL condition.
\end{prf}

\subsection{Proofs for Section \ref{sec:ascent-tracking}}

\begin{prf}[Proof of Lemma~\ref{lem:uncentered-gradient} (Uncentered Gradient Identity)]
(i) \textit{Gradient Expansion.} By the Policy Gradient Theorem for entropy-regularized MDPs, the exact gradient is:
$$\nabla J_\tau(\omega_t) = \frac{1}{1-\gamma} \mathbb{E}_{d^{\pi_t}, \pi_t} \Big[ \tilde{\phi}_t(s,a) \big( Q_\tau^{\pi_t}(s,a) - \tau \log \pi_t(a|s) \big) \Big].$$
We substitute the ideal uncentered critic $Q_\tau^{\pi_t} = \theta^*_\tau(\omega_t)^\top \phi + \epsilon_t$ and the log-linear policy identity $\tau \log \pi_t = \tau \omega_t^\top \phi - \tau \log Z_t$.
Because $\tau \log Z_t(s)$ depends only on the state, it vanishes when multiplied by the action-centered features $\tilde{\phi}_t(s,a)$ and taking the action expectation:
$$\nabla J_\tau(\omega_t) = \frac{1}{1-\gamma} \mathbb{E}_{d^{\pi_t}, \pi_t} \Big[ \tilde{\phi}_t(s,a) \big( (\theta^*_\tau(\omega_t) - \tau \omega_t)^\top \phi(s,a) + \epsilon_t(s,a) \big) \Big].$$

(ii) \textit{Feature Decomposition.} We separate the uncentered features via $\phi(s,a) = \tilde{\phi}_t(s,a) + \bar{\phi}_t(s)$.
For the first term in the expectation:
$$\mathbb{E}_{d^{\pi_t}, \pi_t} \big[ \tilde{\phi}_t \phi^\top \big] (\theta^*_\tau(\omega_t) - \tau \omega_t) = \mathbb{E}_{d^{\pi_t}, \pi_t} \big[ \tilde{\phi}_t (\tilde{\phi}_t + \bar{\phi}_t)^\top \big] (\theta^*_\tau(\omega_t) - \tau \omega_t) = \bar{\Sigma}_{cen}(\pi_t) (\theta^*_\tau(\omega_t) - \tau \omega_t),$$
where the cross-term $\mathbb{E}_{d^{\pi_t}, \pi_t}[\tilde{\phi}_t \bar{\phi}_t^\top]$ is zero because $\mathbb{E}_{a\sim\pi_t}[\tilde{\phi}_t] = 0$.

(iii) \textit{Bias Extraction.} For the second term, we again substitute $\tilde{\phi}_t = \phi - \bar{\phi}_t$:
$$\mathbb{E}_{d^{\pi_t}, \pi_t} \big[ \tilde{\phi}_t \epsilon_t \big] = \mathbb{E}_{d^{\pi_t}, \pi_t} \big[ \phi \epsilon_t \big] - \mathbb{E}_{d^{\pi_t}, \pi_t} \big[ \bar{\phi}_t \epsilon_t \big].$$
By the uncentered orthogonality condition \eqref{eq:mse-orthogonality}, the first term is exactly zero.
The remaining term evaluates to $E_{bias}$. Applying Cauchy-Schwarz and Jensen's inequality to the bias definition yields $\|E_{bias}\|_2 \le \mathbb{E}[\|\bar{\phi}_t\|_2 |\epsilon_t|] \le B_\phi \sqrt{\mathbb{E}[\epsilon_t^2]} \le B_\phi \sqrt{\epsilon_{app}}$, completing the proof.
\end{prf}

\begin{prf}[Proof of Lemma~\ref{lem:smoothness} (Smoothness of the Regularized Objective)]
We establish the smoothness bound via recursive differentiation of the value function.

(i) \textit{Local Policy Derivatives.} For the log-linear softmax policy, the gradient of the log-probability is the centered feature: $\nabla \log \pi_\omega(a|s) = \phi(s,a) - \mathbb{E}_{\pi_\omega}[\phi]$.
Because $\|\phi\|_2 \le B_\phi$, we have $\|\nabla \log \pi_\omega\|_2 \le 2 B_\phi$.
Consequently, the gradient of the policy itself is bounded by:
$$\sum_a \|\nabla \pi_\omega(a|s)\|_2 = \sum_a \pi_\omega(a|s) \|\nabla \log \pi_\omega(a|s)\|_2 \le 2 B_\phi.$$
For the Hessian, applying the product rule to $\nabla \pi_\omega = \pi_\omega \nabla \log \pi_\omega$ yields $\nabla^2 \pi_\omega = \pi_\omega \nabla^2 \log \pi_\omega + \frac{1}{\pi_\omega} \nabla \pi_\omega \nabla \pi_\omega^\top$.
The Hessian of the log-policy is the negative feature covariance: $\nabla^2 \log \pi_\omega(a|s) = - \Cov_{a' \sim \pi_\omega}(\phi(s,a'))$.
Since $\|\phi\|_2 \le B_\phi$, the spectral norm of this covariance matrix is bounded by $B_\phi^2$.
Thus, the local policy Hessian is bounded by:
$$\sum_a \|\nabla^2 \pi_\omega(a|s)\|_2 \le \sum_a \pi_\omega(a|s) \big( B_\phi^2 + (2 B_\phi)^2 \big) = 5 B_\phi^2.$$
Finally, by Lemma~\ref{lem:bounded-values}, the regularized action-value is globally bounded in magnitude by $|Q_\tau^{\pi_\omega}(s,a)| \le \frac{1 + \tau \log |\mathcal{A}|}{1-\gamma} \triangleq \frac{r_{max}}{1-\gamma}$.

(ii) \textit{Recursive Value Gradient.} By definition, $V_\tau^{\pi_\omega}(s) = \sum_a \pi_\omega(a|s) \big( Q_\tau^{\pi_\omega}(s,a) - \tau \log \pi_\omega(a|s) \big)$.
By the product rule, its gradient satisfies the recursive identity:
\begin{align*}
\nabla V_\tau^{\pi_\omega}(s) &= \sum_a \nabla \pi_\omega(a|s) \big( Q_\tau^{\pi_\omega}(s,a) - \tau \log \pi_\omega(a|s) \big) + \sum_a \pi_\omega(a|s) \big( \nabla Q_\tau^{\pi_\omega}(s,a) - \tau \nabla \log \pi_\omega(a|s) \big)\\
& = \sum_a \nabla \pi_\omega(a|s) \big( Q_\tau^{\pi_\omega}(s,a) - \tau \log \pi_\omega(a|s) \big) + \gamma \sum_a \pi_\omega(a|s) \mathbb{E}_{s' \sim \mathcal{P}(\cdot|s,a)} [ \nabla V_\tau^{\pi_\omega}(s') ],
\end{align*}
where $\sum_a \pi_\omega(a|s) \nabla \log \pi_\omega(a|s) = \sum_a \nabla \pi_\omega(a|s) = 0$.
Taking the vector norm and substituting the identity $\nabla \pi_\omega = \pi_\omega \nabla \log \pi_\omega$:
$$\|\nabla V_\tau^{\pi_\omega}(s)\|_2 \le \sum_a \pi_\omega(a|s) \|\nabla \log \pi_\omega(a|s)\|_2 \big| Q_\tau^{\pi_\omega}(s,a) - \tau \log \pi_\omega(a|s) \big| + \gamma \sup_{s'} \|\nabla V_\tau^{\pi_\omega}(s')\|_2.$$
From the score bound, we have $\|\nabla \log \pi_\omega(a|s)\|_2 \le 2 B_\phi$. To bound the absolute value term, we apply the triangle inequality. First, $|Q_\tau^{\pi_\omega}(s,a)| \le \frac{r_{max}}{1-\gamma}$. Second, the expected absolute entropy evaluates to the Shannon entropy: $\sum_a \pi_\omega(a|s) |\tau \log \pi_\omega(a|s)| = \tau H(\pi_\omega(\cdot|s)) \le \tau \log |\mathcal{A}| \le \frac{r_{max}}{1-\gamma}$. Thus, the entire first term is bounded by $2 B_\phi \big( \frac{2 r_{max}}{1-\gamma} \big)$, yielding:
$$\|\nabla V_\tau^{\pi_\omega}(s)\|_2 \le \frac{4 B_\phi r_{max}}{1-\gamma} + \gamma \sup_{s'} \|\nabla V_\tau^{\pi_\omega}(s')\|_2.$$
Solving this recurrence relation across the state space geometrically bounds the state-value gradient by $\|\nabla V_\tau^{\pi_\omega}\|_2 \le \frac{4 B_\phi r_{max}}{(1-\gamma)^2}$.
Because the transition dynamics are independent of $\omega$, the Q-function gradient is $\nabla Q_\tau^{\pi_\omega}(s,a) = \gamma \mathbb{E}_{s' \sim \mathcal{P}(\cdot|s,a)} [ \nabla V_\tau^{\pi_\omega}(s')]$.
Taking the norm yields
$\|\nabla Q_\tau^{\pi_\omega}\|_2 \le \frac{4 \gamma B_\phi r_{max}}{(1-\gamma)^2}$.

(iii) \textit{Recursive Value Hessian.} By the multivariate product rule, the Hessian gains one additional term from the derivative of $\nabla \pi_\omega$:
\begin{align*}
\nabla^2 V_\tau^{\pi_\omega}(s) &= \sum_a \nabla^2 \pi_\omega(a|s) \big( Q_\tau^{\pi_\omega}(s,a) - \tau \log \pi_\omega(a|s) \big) \\
&\quad + \sum_a \nabla \pi_\omega(a|s) \nabla Q_\tau^{\pi_\omega}(s,a)^\top + \sum_a \nabla Q_\tau^{\pi_\omega}(s,a) \nabla \pi_\omega(a|s)^\top\\
&\quad - \tau \sum_a \pi_\omega(a|s) \nabla \log \pi_\omega(a|s) \nabla \log \pi_\omega(a|s)^\top + \gamma \sum_a \pi_\omega(a|s) \mathbb{E}_{s' \sim \mathcal{P}(\cdot|s,a)} [ \nabla^2 V_\tau^{\pi_\omega}(s') ].
\end{align*}
We take the spectral norm of both sides via the triangle inequality.
The second and third are transposes of each other, meaning their spectral norms are identical and contribute equally.
The fourth term is the negative feature covariance scaled by $\tau$, bounded by $\tau \|\Cov_{\pi_\omega}(\phi)\|_2 \le 4 \tau B_\phi^2$.
Substituting the uniform bounds from Steps (i) and (ii) into this recurrence yields:
$$\|\nabla^2 V_\tau^{\pi_\omega}(s)\|_2 \le 5 B_\phi^2 \left( \frac{2 r_{max}}{1-\gamma} \right) + 2 (2 B_\phi) \frac{4 \gamma B_\phi r_{max}}{(1-\gamma)^2} + 4 \tau B_\phi^2 + \gamma \sup_{s'} \|\nabla^2 V_\tau^{\pi_\omega}(s')\|_2.$$
Solving this recurrence relation gives:
$$\|\nabla^2 V_\tau^{\pi_\omega}\|_2 \le \frac{10 B_\phi^2 r_{max}}{(1-\gamma)^2} + \frac{16 \gamma B_\phi^2 r_{max}}{(1-\gamma)^3} + \frac{4 \tau B_\phi^2}{1-\gamma}
\le \frac{10 + 6\gamma}{(1-\gamma)^3} B_\phi^2 r_{max} + \frac{4 \tau B_\phi^2}{(1-\gamma)^3}.$$
Since $\gamma < 1$, we have $10 + 6\gamma \le 16$, yielding the final uniform Hessian bound:
$$\|\nabla^2 V_\tau^{\pi_\omega}\|_2 \le \frac{16 B_\phi^2 r_{max} + 4 \tau B_\phi^2}{(1-\gamma)^3}.$$

(iv) \textit{Objective Smoothness.} The algorithm objective is the expected value under the initial state distribution $\mu$, so $J_\tau(\omega) = \mathbb{E}_{s_0 \sim \mu} [V_\tau^{\pi_\omega}(s_0)]$.
By Jensen's inequality, the spectral norm of the objective's Hessian is bounded by the maximum spectral norm of the value Hessian over the state space.
Therefore, $\|\nabla^2 J_\tau(\omega)\|_2 \le L_J$, guaranteeing the objective is globally $L_J$-smooth.
\end{prf}

\begin{prf}[Proof of Lemma~\ref{lem:pl-ascent} (Actor Progress Bound)]
(i) \textit{Smoothness Expansion.}
By the global smoothness of the regularized objective (Lemma~\ref{lem:smoothness}), the step-wise progress is governed by:
$$J_\tau(\omega_{t+1}) - J_\tau(\omega_t) \ge \eta_t \langle \nabla J_\tau(\omega_t), \Delta \omega_t \rangle - \frac{L_J \eta_t^2}{2} \|\Delta \omega_t\|_2^2,$$
where $\Delta \omega_t =\theta_{t+1} - \tau \omega_t$ such that $\omega_{t+1}-\omega_t = \eta_t \Delta \omega_t$.
We substitute the Uncentered Gradient Identity (Lemma~\ref{lem:uncentered-gradient}) into the inner product:
$$\langle \nabla J_\tau, \Delta \omega_t \rangle = \frac{1}{1-\gamma} \langle \bar{\Sigma}_{cen}(\pi_t) (\theta^*_\tau(\omega_t) - \tau \omega_t), \Delta \omega_t \rangle + \frac{1}{1-\gamma} \langle E_{bias}, \Delta \omega_t \rangle.$$

(ii) \textit{Decoupling Centered Matrix Term.} For any positive semi-definite symmetric matrix $\Sigma$, $a^\top \Sigma b = \frac{1}{2}\|b\|_\Sigma^2 + \frac{1}{2}\|a\|_\Sigma^2 - \frac{1}{2}\|b-a\|_\Sigma^2 \ge \frac{1}{2}\|b\|_\Sigma^2 - \frac{1}{2}\|b-a\|_\Sigma^2$.
Setting $a = \theta^*_\tau(\omega_t) - \tau \omega_t$ and $b = \Delta \omega_t$, we have $b-a = \theta_{t+1} - \theta^*_\tau(\omega_t) = e_t$.
Applying this yields:
$$\frac{1}{1-\gamma} \langle \bar{\Sigma}_{cen}(\pi_t) (\theta^*_\tau(\omega_t) - \tau \omega_t), \Delta \omega_t \rangle \ge \underbrace{ \frac{1}{2(1-\gamma)} \|\Delta \omega_t\|_{\bar{\Sigma}_{cen}(\pi_t)}^2 }_{\text{Algorithmic Progress}} - \underbrace{ \frac{1}{2(1-\gamma)} \|e_t\|_{\bar{\Sigma}_{cen}(\pi_t)}^2 }_{\text{Tracking Penalty}}.$$

(iii) \textit{Lower-Bounding Algorithmic Progress and Upper-Bounding Tracking Penalty.}
To bound the eigenvalues of $\bar{\Sigma}_{cen}(\pi_t)$ for the algorithmic progress, we shift the measure from $d^{\pi_t} \times \pi_t$ to $d^{\omega^*_\tau} \times \pi_{\omega^*_\tau}$ via the joint concentrability bound $C_{joint}$ (Assumption~\ref{ass:parametric-density}). For any vector $v$, we have:
$$ v^\top \bar{\Sigma}_{cen}(\pi_t) v = \mathbb{E}_{d^{\pi_t}, \pi_t} \big[ (v^\top \tilde{\phi}_t)^2 \big]
\ge \frac{1}{C_{joint}} \mathbb{E}_{d^{\omega^*_\tau}, \pi_{\omega^*_\tau}} \big[ (v^\top \tilde{\phi}_t)^2 \big].$$
Note that $\tilde{\phi}_t(s,a) = \phi(s,a) - \bar{\phi}_t(s)$ is centered around the training mean $\bar{\phi}_t(s)$, distinct from the feature mean under the parametric optimum $\bar{\phi}^{\omega^*_\tau}(s) \triangleq \mathbb{E}_{a \sim \pi_{\omega^*_\tau}(\cdot|s)}[\phi(s,a)]$. For any state $s$, because the expected squared distance of a random variable is minimized when centered around its true mean, this mismatched centering yields a lower bound over the action expectation:
$$\mathbb{E}_{a \sim \pi_{\omega^*_\tau}} \Big[ \big(v^\top (\phi(s,a) - \bar{\phi}_t(s))\big)^2 \;\Big|\; s \Big] \ge \mathbb{E}_{a \sim \pi_{\omega^*_\tau}} \Big[ \big(v^\top (\phi(s,a) - \bar{\phi}^{\omega^*_\tau}(s))\big)^2 \;\Big|\; s \Big].$$
Taking the outer expectation over the state distribution $s \sim d^{\omega^*_\tau}$ preserves this pointwise inequality. Therefore, $\bar{\Sigma}_{cen}(\pi_t) \succeq \frac{1}{C_{joint}} \bar{\Sigma}_{cen}(\pi_{\omega^*_\tau})$. Applying the Positive-Definite Fisher Information (Assumption~\ref{ass:pd-fim}), $\bar{\Sigma}_{cen}(\pi_{\omega^*_\tau}) \succeq \lambda I$, yields the lower bound $ \bar{\Sigma}_{cen}(\pi_t) \succeq \frac{\lambda}{C_{joint}} I$.

For the tracking penalty term, by the bounded feature assumption (Assumption~\ref{ass:bounded-features}), the centered covariance is directly upper-bounded: $\bar{\Sigma}_{cen}(\pi_t) \preceq 4 B_\phi^2 I$.
Substituting these respective spectrum limits into the decoupling inequality yields:
$$\frac{1}{1-\gamma} \langle \bar{\Sigma}_{cen}(\pi_t) (\theta^*_\tau(\omega_t) - \tau \omega_t), \Delta \omega_t \rangle \ge \frac{\lambda}{2 C_{joint} (1-\gamma)} \|\Delta \omega_t\|_2^2 - \frac{2 B_\phi^2}{1-\gamma} \|e_t\|_2^2.$$

(iv) \textit{Decoupling Uncentered Bias Term.} We decouple the uncentered bias inner product using Young's inequality ($xy \le \frac{c}{2}x^2 + \frac{1}{2c}y^2$) with carefully chosen $c = \frac{\lambda}{4 C_{joint}}$:
$$\left| \frac{1}{1-\gamma} \langle E_{bias}, \Delta \omega_t \rangle \right| \le \frac{1}{1-\gamma} \|E_{bias}\|_2 \|\Delta \omega_t\|_2 \le \frac{\lambda}{8 C_{joint} (1-\gamma)} \|\Delta \omega_t\|_2^2 + \frac{2 C_{joint}}{\lambda (1-\gamma)} \|E_{bias}\|_2^2.$$
Using the bound $\|E_{bias}\|_2^2 \le B_\phi^2 \epsilon_{app}$ from Lemma~\ref{lem:uncentered-gradient} gives:
$$\left| \frac{1}{1-\gamma} \langle E_{bias}, \Delta \omega_t \rangle \right| \le \frac{\lambda}{8 C_{joint} (1-\gamma)} \|\Delta \omega_t\|_2^2 + \frac{2 C_{joint} B_\phi^2}{\lambda (1-\gamma)} \epsilon_{app}.$$

(v) \textit{Synthesis.} We substitute the decoupled components back into the smoothness bound:
\begin{align*}
& \quad J_\tau(\omega_{t+1}) - J_\tau(\omega_t) \\
& \ge \eta_t \left( \frac{\lambda}{2 C_{joint} (1-\gamma)} - \frac{\lambda}{8 C_{joint} (1-\gamma)} - \frac{L_J \eta_t}{2} \right) \|\Delta \omega_t\|_2^2 - \eta_t C_Z \|e_t\|_2^2 - \eta_t C_{bias} \epsilon_{app}.
\end{align*}
By setting the step size $\eta_t \le \frac{\lambda}{2 L_J C_{joint} (1-\gamma)}$, we guarantee $\frac{L_J \eta_t}{2} \le \frac{\lambda}{4 C_{joint} (1-\gamma)}$.
The coefficient associated with $\eta_t \|\Delta \omega_t\|_2^2$ is lower bounded by:
$$\frac{\lambda}{2 C_{joint} (1-\gamma)} - \frac{\lambda}{8 C_{joint} (1-\gamma)} - \frac{\lambda}{4 C_{joint} (1-\gamma)} = \frac{\lambda}{8 C_{joint} (1-\gamma)} \triangleq \rho_{ac}.$$
Substituting $J_\tau(\omega_{t+1}) - J_\tau(\omega_t) = \text{Gap}_t - \text{Gap}_{t+1}$ and taking the expectation completes the proof.
\end{prf}

\begin{prf}[Proof of Lemma~\ref{lem:coupled-pl-descent} (Coupled Actor Recurrence)]
We establish the expected linear recurrence by coupling the Actor Progress Bound with the objective's PL geometry and the actor's parameter update.

(i) \textit{Inverting PL Geometry.} The term $\|\theta^*_\tau(\omega_t) - \tau \omega_t\|_2^2$ measures the distance of the actor's scaled parameter to the ideal uncentered critic parameter.
We invert the Parameter-Space PL Condition (Lemma \ref{lem:pl-condition}) to lower-bound this ideal target distance:
$$\|\theta^*_\tau(\omega_t) - \tau \omega_t\|_2^2 \ge \frac{\tau}{C_{PL}} \text{Gap}_t - \frac{C_{err}}{C_{PL}} \epsilon_{app}.$$

(ii) \textit{Bridging to Actor's Parameter Update.} We relate the ideal target distance to the actor's update direction $\Delta \omega_t=\theta_{t+1} - \tau \omega_t$ via the triangle inequality: $\|\theta^*_\tau(\omega_t) - \tau \omega_t\|_2^2 \le 2 \|\theta_{t+1} - \tau \omega_t\|_2^2 + 2 \|e_t\|_2^2$, where the tracking residual is $e_t = \theta_{t+1} - \theta^*_\tau(\omega_t)$.
Substituting this into our inverted PL bound and rearranging provides a lower bound on the actor's actual update norm:
$$\|\theta_{t+1} - \tau \omega_t\|_2^2 \ge \frac{\tau}{2 C_{PL}} \text{Gap}_t - \frac{C_{err}}{2 C_{PL}} \epsilon_{app} - \|e_t\|_2^2.$$

(iii) \textit{Synthesis.} We take the expectation of this bound and substitute it directly into the Actor Progress Bound (Lemma~\ref{lem:pl-ascent}):
$$\mathbb{E}[\text{Gap}_{t+1}] \le \mathbb{E}[\text{Gap}_t] - \eta_t \rho_{ac} \left( \frac{\tau}{2 C_{PL}} \mathbb{E}[\text{Gap}_t] - \frac{C_{err}}{2 C_{PL}} \epsilon_{app} - Z_t \right) + \eta_t C_Z Z_t + \eta_t C_{bias} \epsilon_{app}.$$
Grouping the coefficients for $\mathbb{E}[\text{Gap}_t]$, $Z_t$, and $\epsilon_{app}$ exactly yields the composite constants $\rho_{gap}$, $\tilde{C}_Z$, and $C_{\epsilon,gap}$, establishing the stated linear descent.
\end{prf}

\begin{prf}[Proof of Lemma~\ref{lem:critic-tracking-pl} (Coupled Critic Tracking)]
(i) \textit{Unprojected Tracking Error.} We bound the expected tracking error at the next step, $Z_{t+1} = \mathbb{E}[\|e_{t+1}\|_2^2]$, where $e_{t+1} = \theta_{t+2} - \theta^*_\tau(\omega_{t+1})$.
By definition, $\theta_{t+2} = \Pi_{R_\theta} \big( \theta_{t+1} - \alpha_{t+1} g_{t+1}^{cr}(s_{t+1},a_{t+1}) \big)$.
By Lemma~\ref{lem:structural-bounds} and the choice $R_\theta = B_\theta$, the ideal critic resides within the capacity ball, $\theta^*_\tau(\omega_{t+1}) = \Pi_{R_\theta}(\theta^*_\tau(\omega_{t+1}))$.
Because Euclidean projection onto a convex set is non-expansive, the error is bounded by the unprojected distance:
$$ \|e_{t+1}\|_2^2 \le \left\| \theta_{t+1} - \alpha_{t+1} g_{t+1}^{cr}(s_{t+1},a_{t+1}) - \theta^*_\tau(\omega_{t+1}) \right\|_2^2 = \left\| \tilde{e}_t - \alpha_{t+1} g_{t+1}^{cr}(s_{t+1},a_{t+1}) \right\|_2^2. $$
where $\tilde{e}_t \triangleq \theta_{t+1} - \theta^*_\tau(\omega_{t+1})$ denotes the intermediate tracking error prior to the next critic update.

(ii) \textit{SGD Contraction.} We take the conditional expectation with respect to the filtration $\mathcal{F}_{t+1}$.
Because the parameters $\theta_{t+1}$ and $\omega_{t+1}$ are determined by $\mathcal{F}_{t+1}$, the intermediate distance prior to the SGD update, $\tilde{e}_t = \theta_{t+1} - \theta^*_\tau(\omega_{t+1})$, is deterministic conditioned on $\mathcal{F}_{t+1}$.
Expanding the squared norm yields:
$$ \left\| \tilde{e}_t - \alpha_{t+1} g_{t+1}^{cr} \right\|_2^2 = \|\tilde{e}_t\|_2^2 - 2 \alpha_{t+1} \langle \tilde{e}_t, g_{t+1}^{cr} \rangle + \alpha_{t+1}^2 \|g_{t+1}^{cr}\|_2^2. $$
Because the state-action pair is sampled from the current policy's visitation measure $(s_{t+1}, a_{t+1}) \sim d^{\pi_{t+1}} \times \pi_{t+1}$ and $\hat{Q}_{t+1}$ is a conditionally unbiased estimator such that $\mathbb{E}[\hat{Q}_{t+1} \mid \mathcal{F}_{t+1}, s_{t+1}, a_{t+1}] = Q_\tau^{\pi_{t+1}}(s_{t+1}, a_{t+1})$, we apply the Tower Property to evaluate the expected stochastic gradient:
\begin{align*}
&\quad \mathbb{E}[g_{t+1}^{cr}(s_{t+1},a_{t+1}) \mid \mathcal{F}_{t+1}] \\
&= \mathbb{E} \Big[ \mathbb{E} \big[ \big( \theta_{t+1}^\top \phi(s_{t+1}, a_{t+1}) - \hat{Q}_{t+1} \big) \phi(s_{t+1}, a_{t+1}) \mid \mathcal{F}_{t+1}, s_{t+1}, a_{t+1} \big] \;\Big|\; \mathcal{F}_{t+1} \Big] \\
&= \mathbb{E} \Big[ \phi(s_{t+1}, a_{t+1}) \phi(s_{t+1}, a_{t+1})^\top \theta_{t+1} - \phi(s_{t+1}, a_{t+1}) Q_\tau^{\pi_{t+1}}(s_{t+1}, a_{t+1}) \;\Big|\; \mathcal{F}_{t+1} \Big].
\end{align*}
Because given $\mathcal{F}_{t+1}$, $(s_{t+1}, a_{t+1})$ is sampled from $d^{\pi_{t+1}} \times \pi_{t+1}$, this expected value evaluates to the gradient of the expected least-squares objective:
\begin{align*}
\mathbb{E}[g_{t+1}^{cr}(s_{t+1},a_{t+1}) \mid \mathcal{F}_{t+1}] &= \mathbb{E}_{(s,a) \sim d^{\pi_{t+1}} \times \pi_{t+1}} \big[ \phi(s,a) \phi(s,a)^\top \theta_{t+1} - \phi(s,a) Q_\tau^{\pi_{t+1}}(s,a) \big] \\
&= \bar{\Sigma}_{unc}(\pi_{t+1}) \theta_{t+1} - b(\pi_{t+1}),
\end{align*}
where $\bar{\Sigma}_{unc}(\pi_{t+1}) \triangleq \mathbb{E}_{d^{\pi_{t+1}}, \pi_{t+1}}[\phi(s,a) \phi(s,a)^\top]$ and $b(\pi_{t+1}) \triangleq \mathbb{E}_{d^{\pi_{t+1}}, \pi_{t+1}}[\phi(s,a) Q_\tau^{\pi_{t+1}}(s,a)]$.
By the first-order condition for the ideal critic, $\bar{\Sigma}_{unc}(\pi_{t+1}) \theta^*_\tau(\omega_{t+1}) = b(\pi_{t+1})$.
Substituting this identity yields $\mathbb{E}[g_{t+1}^{cr}(s_{t+1},a_{t+1}) \mid \mathcal{F}_{t+1}] = \bar{\Sigma}_{unc}(\pi_{t+1}) (\theta_{t+1} - \theta^*_\tau(\omega_{t+1})) = \bar{\Sigma}_{unc}(\pi_{t+1}) \tilde{e}_t$.

Substituting this expected gradient back into the inner product yields the quadratic form $\langle \tilde{e}_t, \bar{\Sigma}_{unc}(\pi_{t+1}) \tilde{e}_t \rangle$. By the Uncentered Feature Positive-Definiteness (Assumption~\ref{ass:pd-uncentered}), this is bounded below by $\kappa \|\tilde{e}_t\|_2^2$.
Applying the gradient second moment bound $\mathbb{E}[\|g_{t+1}^{cr}\|_2^2 \mid \mathcal{F}_{t+1}] \le G_{cr}^2$ (Lemma~\ref{lem:structural-bounds}(i)), the expected SGD step contracts as:
$$ \mathbb{E}[\|e_{t+1}\|_2^2 \mid \mathcal{F}_{t+1}] \le (1 - 2\alpha_{t+1} \kappa) \|\tilde{e}_t\|_2^2 + \alpha_{t+1}^2 G_{cr}^2. $$

(iii) \textit{Young's Inequality.} We decompose the intermediate error $\tilde{e}_t$ by separating $e_t = \theta_{t+1} - \theta^*_\tau(\omega_t)$ from the target drift:
$$ \tilde{e}_t = e_t + \big(\theta^*_\tau(\omega_t) - \theta^*_\tau(\omega_{t+1})\big). $$
Applying Young's Inequality ($\|x+y\|_2^2 \le (1+p)\|x\|_2^2 + (1+1/p)\|y\|_2^2$) with parameter $p = \alpha_{t+1} \kappa$ yields:
$$ \|\tilde{e}_t\|_2^2 \le (1 + \alpha_{t+1} \kappa) \|e_t\|_2^2 + \left(1 + \frac{1}{\alpha_{t+1} \kappa}\right) \|\theta^*_\tau(\omega_t) - \theta^*_\tau(\omega_{t+1})\|_2^2. $$
From the step-size condition $\alpha_{t+1} = \xi \eta_t \le \frac{1}{2\kappa}$, the SGD contraction multiplier is non-negative ($1 - 2\alpha_{t+1}\kappa \ge 0$). This allows us to substitute the upper bound for $\|\tilde{e}_t\|_2^2$ into the SGD contraction from Step (ii). Applying the relaxations $(1 - 2\alpha_{t+1}\kappa)(1 + \alpha_{t+1}\kappa) \le 1 - \alpha_{t+1}\kappa$ and $(1 - 2\alpha_{t+1}\kappa)(1 + \frac{1}{\alpha_{t+1} \kappa}) \le \frac{1-\alpha_{t+1}\kappa}{\alpha_{t+1}\kappa} \le \frac{1}{\alpha_{t+1} \kappa}$ gives:
$$ \mathbb{E}[\|e_{t+1}\|_2^2 \mid \mathcal{F}_{t+1}] \le (1 - \alpha_{t+1} \kappa) \|e_t\|_2^2 + \frac{1}{\alpha_{t+1} \kappa} \|\theta^*_\tau(\omega_t) - \theta^*_\tau(\omega_{t+1})\|_2^2 + \alpha_{t+1}^2 G_{cr}^2. $$
Taking the total expectation $\mathbb{E}[\cdot]$ over the algorithmic trajectory yields the single-step recurrence:
$$ Z_{t+1} \le \underbrace{ (1 - \alpha_{t+1} \kappa) Z_t + \alpha_{t+1}^2 G_{cr}^2 }_{\text{SGD Error}} + \underbrace{ \frac{1}{\alpha_{t+1} \kappa} \mathbb{E} [ \|\theta^*_\tau(\omega_t) - \theta^*_\tau(\omega_{t+1})\|_2^2 ] }_{\text{Target Drift}} . $$

(iv) \textit{Controlling Target Drift.} Because the ideal critic is $L_\theta$-Lipschitz continuous (Lemma~\ref{lem:lipschitz-critic}), the target drift is bounded by the actor's update:
$$ \|\theta^*_\tau(\omega_t) - \theta^*_\tau(\omega_{t+1})\|_2^2 \le L_\theta^2 \|\omega_{t+1} - \omega_t\|_2^2 = L_\theta^2 \eta_t^2 \|\theta_{t+1} - \tau \omega_t\|_2^2. $$
We rearrange the Actor Progress Bound (Lemma~\ref{lem:pl-ascent}) to bound the expected magnitude of the actor update by the objective progress:
$$ \eta_t \mathbb{E}[\|\theta_{t+1} - \tau \omega_t\|_2^2] \le \frac{1}{\rho_{ac}} \big( \mathbb{E}[\text{Gap}_t] - \mathbb{E}[\text{Gap}_{t+1}] \big) + \frac{C_Z \eta_t}{\rho_{ac}} Z_t + \frac{C_{bias} \eta_t}{\rho_{ac}} \epsilon_{app}. $$
Combining these two bounds leads to:
$$ \mathbb{E} [ \|\theta^*_\tau(\omega_t) - \theta^*_\tau(\omega_{t+1})\|_2^2 ]
\le L_\theta^2 \eta_t\left( \frac{1}{\rho_{ac}} \big( \mathbb{E}[\text{Gap}_t] - \mathbb{E}[\text{Gap}_{t+1}] \big) + \frac{C_Z \eta_t}{\rho_{ac}} Z_t + \frac{C_{bias} \eta_t}{\rho_{ac}} \epsilon_{app} \right).$$

(v) \textit{Synthesis.} Substituting the target drift bound into the tracking recurrence yields:
$$ Z_{t+1} \le (1 - \alpha_{t+1} \kappa) Z_t + \frac{L_\theta^2 \eta_t}{\alpha_{t+1} \kappa} \left( \frac{1}{\rho_{ac}} \big( \mathbb{E}[\text{Gap}_t] - \mathbb{E}[\text{Gap}_{t+1}] \big) + \frac{C_Z \eta_t}{\rho_{ac}} Z_t + \frac{C_{bias} \eta_t}{\rho_{ac}} \epsilon_{app} \right) + \alpha_{t+1}^2 G_{cr}^2. $$
By setting $\alpha_{t+1} = \xi \eta_t$, the ratio evaluates to $\frac{\eta_t}{\alpha_{t+1}} = \frac{1}{\xi}$, and the critic decay evaluates to $\alpha_{t+1} \kappa = \xi \kappa \eta_t$. Substituting this simplifies the inequality to:
\begin{align*}
Z_{t+1} &\le \left(1 - \xi \kappa \eta_t\right) Z_t + \frac{L_\theta^2}{\xi \kappa \rho_{ac}} \big( \mathbb{E}[\text{Gap}_t] - \mathbb{E}[\text{Gap}_{t+1}] \big) \\
&\quad + \eta_t \frac{L_\theta^2 C_Z}{\xi \kappa \rho_{ac}} Z_t + \eta_t \frac{L_\theta^2 C_{bias}}{\xi \kappa \rho_{ac}} \epsilon_{app} + \xi^2 \eta_t^2 G_{cr}^2.
\end{align*}
Grouping the coefficients for $Z_t$, the objective difference, and $\epsilon_{app}$ exactly forms the composite constants $\rho_{cr}$, $C_{gap}$, and $C_{\epsilon,cr}$, completing the proof.
\end{prf}

\subsection{Proofs for Section~\ref{sec:pl-convergence}}

\begin{prf}[Proof of Theorem~\ref{thm:coupled-recurrence} (Coupled Lyapunov Recurrence)]
(i) \textit{Algebraic System Fusion.} We rearrange the Coupled Critic Tracking sequence (Lemma~\ref{lem:critic-tracking-pl}) to align the expected performance gap terms:
$$Z_{t+1} + C_{gap} \mathbb{E}[\text{Gap}_{t+1}] \le (1 - \rho_{cr} \eta_t) Z_t + C_{gap} \mathbb{E}[\text{Gap}_t] + \xi^2 \eta_t^2 G_{cr}^2 + \eta_t C_{\epsilon,cr} \epsilon_{app}.$$
We multiply this by the coupling constant $\beta > 0$ and add it to the Coupled Actor Recurrence inequality (Lemma~\ref{lem:coupled-pl-descent}):
\begin{align*}
\mathcal{L}_{t+1} &\le \mathbb{E}[\text{Gap}_t] - \rho_{gap} \eta_t \mathbb{E}[\text{Gap}_t] + \eta_t \tilde{C}_Z Z_t + \beta (1 - \rho_{cr} \eta_t) Z_t + \beta C_{gap} \mathbb{E}[\text{Gap}_t] \\
&\quad + \beta \xi^2 \eta_t^2 G_{cr}^2 + \eta_t (C_{\epsilon,gap} + \beta C_{\epsilon,cr}) \epsilon_{app}.
\end{align*}

(ii) \textit{Constraints for Joint Recurrence.}
To prove the joint recurrence $\mathcal{L}_{t+1} \le (1 - \rho_{\mathcal{L}} \eta_t) \mathcal{L}_t + \dots$, we group the coefficients for $\mathbb{E}[\text{Gap}_t]$ and $Z_t$ and demand that they are bounded by $(1 - \rho_{\mathcal{L}} \eta_t) \mathcal{L}_t$:
$$(1 - \rho_{\mathcal{L}} \eta_t) \mathcal{L}_t = (1 - \rho_{\mathcal{L}} \eta_t)(1 + \beta C_{gap}) \mathbb{E}[\text{Gap}_t] + (1 - \rho_{\mathcal{L}} \eta_t) \beta Z_t.$$
For this to hold, the effective decay rates for both components must match or exceed the target rate $\rho_{\mathcal{L}}$, yielding two independent constraints:
\begin{itemize}
    \item \textit{Actor Constraint:} $1 + \beta C_{gap} - \rho_{gap} \eta_t \le (1 - \rho_{\mathcal{L}} \eta_t)(1 + \beta C_{gap}) \implies \rho_{\mathcal{L}} \le \frac{\rho_{gap}}{1 + \beta C_{gap}}$.
    \item \textit{Critic Constraint:} $\beta - \beta \rho_{cr} \eta_t + \eta_t \tilde{C}_Z \le \beta - \beta \rho_{\mathcal{L}} \eta_t \implies \rho_{\mathcal{L}} \le \rho_{cr} - \frac{\tilde{C}_Z}{\beta}$.
\end{itemize}

(iii) \textit{Derivation of Joint Rate and Coupling Weight.} To maximize progress, we set the joint rate exactly to the actor's bottleneck: $\rho_{\mathcal{L}} \triangleq \frac{\rho_{gap}}{1 + \beta C_{gap}}$.
To satisfy the critic constraint, we ensure $\rho_{\mathcal{L}} \le \rho_{cr} - \frac{\tilde{C}_Z}{\beta}$.
For the algorithm to converge ($\rho_{\mathcal{L}} > 0$), we require $\beta > \frac{\tilde{C}_Z}{\rho_{cr}}$.
To guarantee stability with a safe margin, we double this minimum requirement by setting $\beta \triangleq \frac{2 \tilde{C}_Z}{\rho_{cr}}$.
Substituting this choice simplifies the critic constraint to $\rho_{\mathcal{L}} \le \frac{\rho_{cr}}{2}$.
Because $\rho_{\mathcal{L}} \le \rho_{gap}$, a sufficient condition for joint recurrence is simply $\rho_{cr} \ge 2\rho_{gap}$.

(iv) \textit{Timescale Separation Limit.} We substitute the native critic rate $\rho_{cr} \triangleq \xi \kappa - \frac{S^2}{\xi \kappa}$, where $S \triangleq L_\theta \sqrt{\frac{C_Z}{\rho_{ac}}}$, into the stability requirement:
$$\xi \kappa - \frac{S^2}{\xi \kappa} \ge 2\rho_{gap} \implies (\kappa \xi)^2 - 2\rho_{gap} (\kappa \xi) - S^2 \ge 0.$$
By the quadratic formula and $\sqrt{x+y} \le \sqrt{x}+\sqrt{y}$, this is satisfied for any $\kappa \xi \ge 2\rho_{gap} + S$.
However, the choice of $\xi$ dictates the joint tracking variance $\Sigma_{\mathcal{L}} = \beta \xi^2 G_{cr}^2$.
Substituting $\beta$, the variance physically scales as $\Sigma_{\mathcal{L}} \propto \frac{\xi^2}{\rho_{cr}} = \frac{\xi^3 \kappa}{(\kappa \xi)^2 - S^2}$.
To minimize this variance, we differentiate with respect to $\xi$ and set the numerator to zero:
$$3\xi^2\big((\kappa \xi)^2 - S^2\big) - \xi^3(2\kappa^2 \xi) = 0 \Longleftrightarrow (\kappa \xi)^2 = 3S^2.$$
This reveals that the theoretical minimum occurs at $\kappa \xi = \sqrt{3}S$.
For a simple and conservative structural bound, we explicitly define the threshold as $\kappa \xi \ge 2\rho_{gap} + 2S$.
Substituting the definition of $S$ yields the required structural lower bound on the timescale ratio $\xi$.

(v) \textit{Synthesis.} With $\beta$ and $\xi$ satisfying the joint recurrence constraints, we consolidate the remaining additive terms.
The joint variance evaluates to $\Sigma_{\mathcal{L}} = \beta \xi^2 G_{cr}^2$, and the joint bias groups to $C_{\epsilon,\mathcal{L}} \epsilon_{app} = (C_{\epsilon,gap} + \beta C_{\epsilon,cr}) \epsilon_{app}$, completing the proof.
\end{prf}

\begin{prf}[Proof of Lemma~\ref{lem:telescoping-sequence} (Telescoping Robbins-Monro Sequence)]
We rearrange the recursive bound to isolate the linear decay:
$$r \eta_t x_t \le x_t - x_{t+1} + a \eta_t^2 + b \eta_t.$$
Dividing both sides by $r \eta_t$ yields $x_t \le \frac{1}{r \eta_t} x_t - \frac{1}{r \eta_t} x_{t+1} + \frac{a}{r} \eta_t + \frac{b}{r}$.
We sum this inequality from $t=0$ to $T-1$. Substituting $\frac{1}{r\eta_t} = \frac{t+t_0}{c}$ and expanding the sum:
$$\sum_{t=0}^{T-1} x_t \le \frac{t_0}{c} x_0 - \frac{T-1+t_0}{c} x_T + \sum_{t=1}^{T-1} x_t \left( \frac{t+t_0}{c} - \frac{t-1+t_0}{c} \right) + \frac{a}{r} \sum_{t=0}^{T-1} \eta_t + \frac{b T}{r}.$$
Because $x_T \ge 0$, we drop the terminal term.
The difference inside the summation simplifies to $\frac{1}{c}$. To consolidate the $x_t$ sequences, we add and subtract $\frac{1}{c} x_0$:
$$\sum_{t=0}^{T-1} x_t \le \frac{t_0 - 1}{c} x_0 + \frac{1}{c} \sum_{t=0}^{T-1} x_t + \frac{a c_\eta}{r} \sum_{t=0}^{T-1} \frac{1}{t+t_0} + \frac{b T}{r}.$$
We subtract the inner sum from both sides to get $(1 - \frac{1}{c}) \sum x_t$, and multiply the entire inequality by $\frac{c}{c - 1}$:
$$\sum_{t=0}^{T-1} x_t \le \frac{t_0 - 1}{c - 1} x_0 + \frac{a c c_\eta}{r(c - 1)} \sum_{t=0}^{T-1} \frac{1}{t+t_0} + \frac{b c T}{r(c - 1)}.$$
Substituting $c = r c_\eta$, the coefficients simplify to $\frac{a c_\eta^2}{c-1}$ and $\frac{b c_\eta T}{c-1}$.
Because $t_0 \ge 1$, the harmonic sum is bounded by $\sum_{t=0}^{T-1} \frac{1}{t+t_0} \le 1 + \log T$.
Dividing the entire expression by $T$ completes the proof.
\end{prf}

\begin{prf}[Proof of Theorem~\ref{thm:average-iterate} (Average-Iterate Regularized Convergence)]
Because the initial offset $t_0 \ge \frac{c_\eta}{\eta_{max}}$ and the step size schedule is decreasing, the maximum step size occurs at $t=0$ and satisfies $\eta_t \le \eta_0 = \frac{c_\eta}{t_0} \le \eta_{max}$ for all $t \ge 0$. By the definition of $\eta_{max}$, this ensures the required step-size condition $\eta_t \le \min\left(\frac{\lambda}{2 L_J C_{joint} (1-\gamma)}, \frac{1}{2\xi\kappa} \right)$.
Therefore, the joint Lyapunov recurrence (Theorem~\ref{thm:coupled-recurrence}) is valid, meaning the Lyapunov sequence satisfies $\mathcal{L}_{t+1} \le (1 - \rho_{\mathcal{L}} \eta_t) \mathcal{L}_t + a \eta_t^2 + b \eta_t$, where $a = \Sigma_{\mathcal{L}}$ and $b = C_{\epsilon,\mathcal{L}} \epsilon_{app}$.

The remaining conditions on $t_0$ and $c_\eta$ guarantee the structural prerequisites for the Telescoping Robbins-Monro Sequence bound (Lemma~\ref{lem:telescoping-sequence}). The choice $c_\eta > \frac{1}{\rho_{\mathcal{L}}}$ satisfies the numerator requirement $c \triangleq \rho_{\mathcal{L}} c_\eta > 1$, and $t_0 \ge 1$ provides the required initial offset. Applying Lemma~\ref{lem:telescoping-sequence} with $x_t = \mathcal{L}_t$ and $r = \rho_{\mathcal{L}}$ establishes the bound on the average Lyapunov sequence:
$$\frac{1}{T} \sum_{t=0}^{T-1} \mathcal{L}_t \le \frac{t_0 - 1}{T (\rho_{\mathcal{L}} c_\eta - 1)} \mathcal{L}_0 + \frac{c_\eta^2 \Sigma_{\mathcal{L}}}{\rho_{\mathcal{L}} c_\eta - 1} \left( \frac{1 + \log T}{T} \right) + \frac{c_\eta C_{\epsilon,\mathcal{L}}}{\rho_{\mathcal{L}} c_\eta - 1} \epsilon_{app}.$$

To evaluate the structural orders, we recall from Remark~\ref{rem:order-evaluation} that $\rho_{\mathcal{L}} = \Theta(\lambda \tau)$, $\Sigma_{\mathcal{L}} = \Theta((1-\gamma)^{-5}\kappa^{-6}\lambda^{-1/2})$, and $C_{\epsilon,\mathcal{L}} = \Theta((1-\gamma)^{-1}\lambda^{-1})$.
We choose the schedule numerator $c_\eta = \frac{c}{\rho_{\mathcal{L}}}$ for a constant $c > 1$, meaning $c_\eta = \Theta(\frac{1}{\lambda \tau})$. Then the bounding factor $\rho_{\mathcal{L}} c_\eta - 1$ evaluates to $c - 1 = \Theta(1)$.
From Lemma~\ref{lem:smoothness}, the objective smoothness constant evaluates to $L_J = \Theta((1-\gamma)^{-3})$,
and the actor progress threshold evaluates to $\frac{\lambda}{2 L_J C_{joint} (1-\gamma)} = \Theta((1-\gamma)^2 \lambda)$.
Because the timescale ratio satisfies $\xi = \Theta((1-\gamma)^{-2}\kappa^{-3}\lambda^{-1/2})$, we have $\frac{1}{2\xi\kappa} = \Theta((1-\gamma)^2 \kappa^2 \lambda^{1/2})$.
Thus, the maximum allowable step size satisfies $\frac{1}{\eta_{max}} = \Theta\big( \frac{1}{(1-\gamma)^2 \lambda} + \frac{1}{(1-\gamma)^2 \kappa^2 \lambda^{1/2}} \big)$.
This allows us to satisfy the initial offset condition by choosing $t_0 = \Theta(\frac{c_\eta}{\eta_{max}}) = \Theta\big( \frac{1}{(1-\gamma)^2 \lambda^2 \tau}  + \frac{1}{(1-\gamma)^2 \kappa^2 \lambda^{3/2}\tau} \big)$.
\begin{itemize}
    \item \textit{Initialization Error:} The initial Lyapunov function is $\mathcal{L}_0 = \mathbb{E}[\text{Gap}_0] + \beta(Z_0 + C_{gap} \mathbb{E}[\text{Gap}_0]) = (1 + \beta C_{gap}) \mathbb{E}[\text{Gap}_0] + \beta Z_0$.
    By definition, $Z_0 = \mathbb{E}[\|\theta_1 - \theta^*_\tau(\omega_0)\|_2^2]$,
    which satisfies $Z_0 \le (2B_\theta)^2 = G_{ac}^2$
    because $\theta_1$ (by constrained SGD) and $\theta^*_\tau(\omega_0)$ (by Lemma~\ref{lem:structural-bounds}(i)) are in the ball of radius $B_\theta$.
    Since $G_{ac}^2 = \Theta(\kappa^{-2}(1-\gamma)^{-2})$, we have $Z_0=\mathcal{O}(\kappa^{-2}(1-\gamma)^{-2})$.
    Substituting $\beta = \Theta((1-\gamma)\kappa^2\lambda^{1/2})$ and $\beta C_{gap} = \Theta(1)$, we obtain $\mathcal{L}_0 = \mathcal{O}((1-\gamma)^{-1}) + \mathcal{O}((1-\gamma)^{-1}\lambda^{1/2}) = \mathcal{O}((1-\gamma)^{-1})$.
    The initialization multiplier evaluates to $\frac{t_0 - 1}{\rho_{\mathcal{L}} c_\eta - 1} = \Theta(t_0) = \Theta\big( \frac{1}{(1-\gamma)^2 \lambda^2 \tau} + \frac{1}{(1-\gamma)^2 \kappa^2 \lambda^{3/2}\tau} \big)$. Thus, the expected initialization term scales as $\mathcal{O}\big( \frac{1}{(1-\gamma)^3 \lambda^2 \tau T} + \frac{1}{(1-\gamma)^3 \kappa^2 \lambda^{3/2} \tau T} \big)$.

    \item \textit{Optimization Variance:} The variance multiplier evaluates to:
    $$ \frac{c_\eta^2 \Sigma_{\mathcal{L}}}{\rho_{\mathcal{L}} c_\eta - 1} = \Theta\left( \frac{1}{\lambda^2 \tau^2} \frac{1}{(1-\gamma)^5 \kappa^6 \lambda^{1/2}} \right) = \Theta\left( \frac{1}{(1-\gamma)^5 \kappa^6 \lambda^{5/2} \tau^2} \right). $$
    Multiplying by $\frac{1 + \log T}{T}$ yields an $\mathcal{O}(T^{-1} \log T)$ variance term, which dominates the $\mathcal{O}(T^{-1})$ initialization error above, taking account of the dependence on $\tau$, $\kappa$, $\lambda$, and $(1-\gamma)^{-1}$.

    \item \textit{Approximation Bias:} The bias multiplier evaluates to:
    $$ \frac{c_\eta C_{\epsilon,\mathcal{L}}}{\rho_{\mathcal{L}} c_\eta - 1} = \Theta\left( \frac{1}{\lambda \tau} \frac{1}{(1-\gamma) \lambda} \right) = \Theta\left( \frac{1}{(1-\gamma) \lambda^2 \tau} \right). $$
\end{itemize}

Because the expected regularized performance gap is upper-bounded by the joint Lyapunov function ($\mathbb{E}[\text{Gap}_t] \le \mathcal{L}_t$), summing these evaluated components and dropping the dominated initialization term yields the stated asymptotic convergence rate.
\end{prf}

\begin{prf}[Proof of Theorem~\ref{thm:last-iterate} (Last-Iterate Regularized Convergence)]
As established in the proof of Theorem~\ref{thm:average-iterate}, the step-size schedule guarantees the required structural condition for the actor update for all $t \ge 0$. Substituting $\eta_t = \frac{c_\eta}{t + t_0}$ directly into the joint Lyapunov recurrence (Theorem~\ref{thm:coupled-recurrence}) yields:
$$\mathcal{L}_{t+1} \le \left(1 - \frac{\rho_{\mathcal{L}} c_\eta}{t + t_0}\right) \mathcal{L}_t + \frac{c_\eta^2 \Sigma_{\mathcal{L}}}{(t + t_0)^2} + \frac{c_\eta C_{\epsilon,\mathcal{L}} \epsilon_{app}}{t+t_0}.$$
We apply Chung's Lemma (Lemma~\ref{lem:chungs-lemma}) by identifying the sequence $x_t = \mathcal{L}_t$, the contraction constant $c = \rho_{\mathcal{L}} c_\eta$, the variance constant $a = c_\eta^2 \Sigma_{\mathcal{L}}$, and the bias constant $b = c_\eta C_{\epsilon,\mathcal{L}} \epsilon_{app}$.
The choice $c_\eta > \frac{1}{\rho_{\mathcal{L}}}$ guarantees $c \triangleq \rho_{\mathcal{L}} c_\eta > 1$, and the additional schedule condition $t_0 \ge \rho_{\mathcal{L}} c_\eta$ satisfies the initial offset requirement $t_0 \ge c$, establishing the bound $\mathcal{L}_T \le \frac{v}{T + t_0} + \frac{b}{c-1}$.

To evaluate the structural orders, we rely on the identical choices of step-size parameters from Theorem~\ref{thm:average-iterate}: we choose $c_\eta = \frac{c}{\rho_{\mathcal{L}}}$ for a constant $c > 1$, meaning $c_\eta = \Theta(\frac{1}{\lambda \tau})$, and choose the initial offset $t_0 = \Theta\big( \frac{1}{(1-\gamma)^2 \lambda^2 \tau} + \frac{1}{(1-\gamma)^2 \kappa^2 \lambda^{3/2}\tau} \big)$, with the scaling constant for $t_0$ sufficiently large to guarantee the additional requirement $t_0 \ge \rho_{\mathcal{L}} c_\eta \,(=c)$.
\begin{itemize}
    \item \textit{Bounding Constant ($v$):} The terminal constant is $v = \max\left( t_0 \mathcal{L}_0, \frac{c_\eta^2 \Sigma_{\mathcal{L}}}{\rho_{\mathcal{L}} c_\eta - 1} \right)$.
    Substituting the bound $\mathcal{L}_0 = \mathcal{O}((1-\gamma)^{-1})$ established in the proof of Theorem~\ref{thm:average-iterate}, the initialization argument evaluates to $t_0 \mathcal{L}_0 = \mathcal{O}\Big( \frac{1}{(1-\gamma)^3 \lambda^2 \tau} + \frac{1}{(1-\gamma)^3 \kappa^2 \lambda^{3/2}\tau} \Big)$.
    As shown in the proof of Theorem~\ref{thm:average-iterate}, the variance argument evaluates to:
    $$ \frac{c_\eta^2 \Sigma_{\mathcal{L}}}{\rho_{\mathcal{L}} c_\eta - 1} = \Theta\left( \frac{1}{\lambda^2 \tau^2} \frac{1}{(1-\gamma)^5 \kappa^6 \lambda^{1/2}} \right) = \Theta\left( \frac{1}{(1-\gamma)^5 \kappa^6 \lambda^{5/2} \tau^2} \right). $$
    Taking account of the dependence on $\tau$, $\kappa$, $\lambda$, and $(1-\gamma)^{-1}$, the variance argument dominates the initialization argument. Thus, the bounding constant evaluates to $v = \Theta\left( \frac{1}{(1-\gamma)^5 \kappa^6 \lambda^{5/2} \tau^2} \right)$.

    \item \textit{Optimization Variance:} Dividing the bounding constant by $T + t_0$ yields the variance term, which is upper-bounded by $\frac{v}{T} = \mathcal{O}\left( \frac{1}{(1-\gamma)^5 \kappa^6 \lambda^{5/2} \tau^2 T} \right)$.

    \item \textit{Approximation Bias:} The bias term matches the average-iterate bound:
    $$ \frac{c_\eta C_{\epsilon,\mathcal{L}}}{\rho_{\mathcal{L}} c_\eta - 1} \epsilon_{app} = \Theta\left( \frac{\epsilon_{app}}{(1-\gamma) \lambda^2 \tau} \right). $$
\end{itemize}

Because $\mathbb{E}[\text{Gap}_T] \le \mathcal{L}_T$, summing the optimization variance and approximation bias yields the stated asymptotic last-iterate convergence rate.
\end{prf}

\subsection{Proofs for Section~\ref{sec:unregularized-complexity}}

\begin{prf}[Proof of Lemma~\ref{lem:global-class-error} (Global Class Approximation Error)]
We mirror the structure of the PL Condition (Lemma \ref{lem:pl-condition}) to bound the representation gap. Let $D_* \triangleq \mathbb{E}_{d^*_\tau}[D_{KL}(\pi^*_\tau \| \pi_{\omega^*_\tau})]$ denote the state-averaged forward KL divergence from the global optimum to the parametric optimum.

(i) \textit{Ideal Critic Decomposition.} We apply the Regularized Suboptimality Decompositions (Lemma~\ref{lem:sub-decomp}) evaluated at the training policy $\pi_t = \pi_{\omega^*_\tau}$ against the comparator $\pi = \pi^*_\tau$. Let $\epsilon_*(s,a)$ denote the approximation error of the ideal critic $\theta^*_\tau(\omega^*_\tau)$ trained under the distribution $d^{\omega^*_\tau} \times \pi_{\omega^*_\tau}$. The exact decomposition is:
\begin{align*}
(1-\gamma) \big( J_\tau(\pi^*_\tau) - J_\tau(\pi_{\omega^*_\tau}) \big) &= -\tau D_* + \mathbb{E}_{d^*_\tau} \Big[ \langle (\theta^*_\tau(\omega^*_\tau) - \tau \omega^*_\tau)^\top \phi, \pi^*_\tau - \pi_{\omega^*_\tau} \rangle_{\mathcal{A}} \Big] \\
&\quad + \mathbb{E}_{d^*_\tau} \big[ \langle \epsilon_*, \pi^*_\tau - \pi_{\omega^*_\tau} \rangle_{\mathcal{A}} \big].
\end{align*}

(ii) \textit{Bounding the Ideal Target Distance at Optimum.} For any fixed temperature $\tau > 0$, the regularized optimal parameter $\omega^*_\tau \in \mathbb{R}^d$ is finite.
As a finite, unconstrained maximizer of the parametric objective $J_\tau(\omega)$, its policy gradient is zero: $\nabla J_\tau(\omega^*_\tau) = 0$. Substituting this into the Uncentered Gradient Identity (Lemma~\ref{lem:uncentered-gradient}), we have:
$$0 = \bar{\Sigma}_{cen}(\pi_{\omega^*_\tau}) \big( \theta^*_\tau(\omega^*_\tau) - \tau \omega^*_\tau \big) + E_{bias}^*,$$
where $E_{bias}^*\triangleq-\mathbb{E}_{s\sim d^{\omega^*_\tau}}
\big[\bar{\phi}_*(s)\mathbb{E}_{a\sim\pi_{\omega^*_\tau}}[\epsilon_*(s,a)\mid s]\big]$, with $\bar{\phi}_*(s)\triangleq\mathbb{E}_{a\sim\pi_{\omega^*_\tau}}[\phi(s,a)]$.
Because this covariance is evaluated at the parametric optimum, the Positive-Definite Fisher Information (Assumption~\ref{ass:pd-fim}) applies directly without any density shift: $ \bar{\Sigma}_{cen}(\pi_{\omega^*_\tau}) \succeq \lambda I$. Rearranging and bounding the target distance yields:
$$\|\theta^*_\tau(\omega^*_\tau) - \tau \omega^*_\tau\|_2 \le \|\bar{\Sigma}_{cen}(\pi_{\omega^*_\tau})^{-1}\|_2 \|E_{bias}^*\|_2 \le \frac{B_\phi}{\lambda} \sqrt{\epsilon_{app}},$$
where the uniform bias bound is applied, $\|E_{bias}^*\|_2 \le B_\phi \sqrt{\epsilon_{app}}$.

(iii) \textit{Decoupling Inner Products.}
We apply H\"{o}lder's inequality over the action space and Pinsker's inequality ($\|\pi^*_\tau - \pi_{\omega^*_\tau}\|_1 \le \sqrt{2 D_{KL}(\pi^*_\tau \| \pi_{\omega^*_\tau})}$) to the ideal algorithmic progress term:
\begin{align*}
\langle (\theta^*_\tau(\omega^*_\tau) - \tau \omega^*_\tau)^\top \phi, \pi^*_\tau - \pi_{\omega^*_\tau} \rangle_{\mathcal{A}} &\le \max_a \big| (\theta^*_\tau(\omega^*_\tau) - \tau \omega^*_\tau)^\top \phi(s,a) \big| \|\pi^*_\tau - \pi_{\omega^*_\tau}\|_1 \\
&\le B_\phi \|\theta^*_\tau(\omega^*_\tau) - \tau \omega^*_\tau\|_2 \sqrt{2 D_{KL}(\pi^*_\tau \| \pi_{\omega^*_\tau})}.
\end{align*}
Taking the expectation over $d^*_\tau(s)$ and applying Jensen's Inequality ($\mathbb{E}_{d^*_\tau}[\sqrt{D_{KL}}] \le \sqrt{D_*}$):
$$ \mathbb{E}_{d^*_\tau} \Big[ \langle (\theta^*_\tau(\omega^*_\tau) - \tau \omega^*_\tau)^\top \phi, \pi^*_\tau - \pi_{\omega^*_\tau} \rangle_{\mathcal{A}} \Big] \le B_\phi \|\theta^*_\tau(\omega^*_\tau) - \tau \omega^*_\tau\|_2 \sqrt{2 D_*}. $$
Applying Young's Inequality ($xy \le \frac{\tau}{2}x^2 + \frac{1}{2\tau}y^2$) with $x = \sqrt{D_*}$ yields:
$$ \mathbb{E}_{d^*_\tau} \Big[ \langle (\theta^*_\tau(\omega^*_\tau) - \tau \omega^*_\tau)^\top \phi, \pi^*_\tau - \pi_{\omega^*_\tau} \rangle_{\mathcal{A}} \Big] \le \frac{\tau}{2} D_* + \frac{B_\phi^2}{\tau} \|\theta^*_\tau(\omega^*_\tau) - \tau \omega^*_\tau\|_2^2. $$

For the approximation bias term, for any state $s$, define the local policy mismatch $W^*(s) \triangleq \max_a \frac{\pi^*_\tau(a|s)}{\pi_{\omega^*_\tau}(a|s)}$. We apply Action-Space Cauchy-Schwarz and the $\chi^2$ bound (Lemma~\ref{lem:chi2-KL}):
$$ \langle \epsilon_*, \pi^*_\tau - \pi_{\omega^*_\tau} \rangle_{\mathcal{A}} \le \sqrt{ 2 W^*(s) D_{KL}(\pi^*_\tau(\cdot|s) \| \pi_{\omega^*_\tau}(\cdot|s)) } \sqrt{ \mathbb{E}_{a \sim \pi_{\omega^*_\tau}}[\epsilon_*(s,a)^2] }. $$
Taking the expectation over $d^*_\tau(s)$ and applying State-Space Cauchy-Schwarz:
\begin{align*}
\mathbb{E}_{d^*_\tau} \big[ \langle \epsilon_*, \pi^*_\tau - \pi_{\omega^*_\tau} \rangle_{\mathcal{A}} \big] &\le \mathbb{E}_{d^*_\tau} \Big[ \sqrt{ 2 D_{KL}(\pi^*_\tau(\cdot|s) \| \pi_{\omega^*_\tau}(\cdot|s)) } \sqrt{ W^*(s) \mathbb{E}_{a \sim \pi_{\omega^*_\tau}}[\epsilon_*(s,a)^2] } \Big] \\
&\le \sqrt{ 2 D_* } \sqrt{ \mathbb{E}_{d^*_\tau} \big[ W^*(s) \mathbb{E}_{a \sim \pi_{\omega^*_\tau}}[\epsilon_*^2] \big] }.
\end{align*}
To control the second root, we expand the expectation over the state space and apply the joint concentrability bound $\frac{d^*_\tau(s)}{d^{\omega^*_\tau}(s)} W^*(s) = \max_a \frac{d^*_\tau(s) \pi^*_\tau(a|s)}{d^{\omega^*_\tau}(s) \pi_{\omega^*_\tau}(a|s)} \le C_{joint}^*$ (Assumption~\ref{ass:global-parametric-density}):
\begin{align*}
\mathbb{E}_{d^*_\tau} \big[ W^*(s) \mathbb{E}_{a \sim \pi_{\omega^*_\tau}}[\epsilon_*^2] \big] &= \mathbb{E}_{d^{\omega^*_\tau}} \left[ \left( \frac{d^*_\tau(s)}{d^{\omega^*_\tau}(s)} W^*(s) \right) \mathbb{E}_{a \sim \pi_{\omega^*_\tau}}[\epsilon_*(s,a)^2] \right] \\
&\le C_{joint}^* \mathbb{E}_{d^{\omega^*_\tau}} \Big[ \mathbb{E}_{a \sim \pi_{\omega^*_\tau}}[\epsilon_*(s,a)^2] \Big] \\
&= C_{joint}^* \mathbb{E}_{d^{\omega^*_\tau}, \pi_{\omega^*_\tau}}[\epsilon_*^2] \le C_{joint}^* \epsilon_{app},
\end{align*}
where the last step uses Assumption~\ref{ass:approximation-error}.
Applying Young's Inequality ($xy \le \frac{\tau}{2}x^2 + \frac{1}{2\tau}y^2$) with $x = \sqrt{D_*}$ and $y = \sqrt{2 C_{joint}^* \epsilon_{app}}$ yields:
$$ \mathbb{E}_{d^*_\tau} \big[ \langle \epsilon_*, \pi^*_\tau - \pi_{\omega^*_\tau} \rangle_{\mathcal{A}} \big] \le \frac{\tau}{2} D_* + \frac{C_{joint}^*}{\tau} \epsilon_{app}. $$

(iv) \textit{Synthesis.} We substitute the decoupled algorithmic progress and approximation bias back into the suboptimality decomposition. The two positive $\frac{\tau}{2} D_*$ penalties cancel the restorative entropy force $-\tau D_*$:
\begin{align*}
(1-\gamma) \big( J_\tau(\pi^*_\tau) - J_\tau(\pi_{\omega^*_\tau}) \big) &\le \frac{B_\phi^2}{\tau} \|\theta^*_\tau(\omega^*_\tau) - \tau \omega^*_\tau\|_2^2 + \frac{C_{joint}^*}{\tau} \epsilon_{app}.
\end{align*}
Substituting the bound $\|\theta^*_\tau(\omega^*_\tau) - \tau \omega^*_\tau\|_2^2 \le \frac{B_\phi^2}{\lambda^2} \epsilon_{app}$ into the first term and dividing the entire inequality by $1-\gamma$ yields the stated constant $C_{class}$, completing the proof.
\end{prf}

\begin{prf}[Proof of Corollary~\ref{cor:global-reg-convergence} (Global Regularized Performance Bound)]
This follows directly from the algebraic decomposition $J_\tau(\pi^*_\tau) - J_\tau(\pi_t) = \big(J_\tau(\pi^*_\tau) - J_\tau(\pi_{\omega^*_\tau})\big) + \big(J_\tau(\pi_{\omega^*_\tau}) - J_\tau(\pi_t)\big)$. Substituting the bound from Lemma~\ref{lem:global-class-error} for the first term and the definition of $\text{Gap}_t$ for the second completes the proof.
\end{prf}

\begin{prf}[Proof of Lemma~\ref{lem:part-I-compatibility} (Part-I Approximation-Temperature Compatibility)]
(i) \textit{Lower-Bounding Suboptimal Mass.}
Define the suboptimal-action mass at state $s$ by $q_\tau(s)\triangleq 1-\pi_{\omega^*_\tau}(a^*(s)|s)$. We first show that the uniform Fisher lower bound forces $\pi_{\omega^*_\tau}$ to retain a nonvanishing suboptimal-action probability.
By Assumption~\ref{ass:pd-fim},
$\bar{\Sigma}_{cen}(\pi_{\omega^*_\tau})\succeq\lambda I$,
taking traces yields
$$d\lambda\le\mathbb{E}_{s\sim d^{\omega^*_\tau}}\left[\operatorname{tr}\Cov_{a\sim \pi_{\omega^*_\tau}(\cdot|s)}(\phi(s,a))\right].$$
For any fixed state $s$, the expected squared distance from the mean is no larger than that from any fixed vector. Taking the feature vector of the optimal action $\phi(s,a^*(s))$ as this reference and using Assumption~\ref{ass:bounded-features} gives
\begin{align*}
\operatorname{tr}\Cov_{a\sim \pi_{\omega^*_\tau}(\cdot|s)}(\phi(s,a))
&=\mathbb{E}_{a\sim \pi_{\omega^*_\tau}(\cdot|s)}\left[\|\phi(s,a)-\mathbb{E}_{\pi_{\omega^*_\tau}}[\phi(s,\cdot)]\|_2^2\right] \\
&\le\mathbb{E}_{a\sim \pi_{\omega^*_\tau}(\cdot|s)}\left[\|\phi(s,a)-\phi(s,a^*(s))\|_2^2\right] \\
&\le 4B_\phi^2q_\tau(s).
\end{align*}
Therefore,
$$\mathbb{E}_{s\sim d^{\omega^*_\tau}}[q_\tau(s)]\ge\frac{d\lambda}{4B_\phi^2}\triangleq q_0.$$

(ii) \textit{Lower-Bounding Global-to-Parametric Gap.}
Next, by the Performance Difference Lemma for the unregularized global optimum and the Minimal Action Gap (Assumption~\ref{ass:action-gap}),
\begin{align*}
J_0(\pi^*_0)-J_0(\pi_{\omega^*_\tau})
&=\frac{1}{1-\gamma}\mathbb{E}_{d^{\omega^*_\tau},\pi_{\omega^*_\tau}}\left[V^*_0(s)-Q^*_0(s,a)\right] \\
&\ge\frac{\Delta}{1-\gamma}\mathbb{E}_{d^{\omega^*_\tau}}[q_\tau(s)]
\ge\frac{q_0\Delta}{1-\gamma}.
\end{align*}
Because $\pi^*_\tau$ is globally optimal for the regularized objective and $\pi^*_0$ is deterministic,
\begin{align*}
J_\tau(\pi^*_\tau)-J_\tau(\pi_{\omega^*_\tau})
&\ge J_\tau(\pi^*_0)-J_\tau(\pi_{\omega^*_\tau}) \\
&=J_0(\pi^*_0)-J_0(\pi_{\omega^*_\tau})-\frac{\tau}{1-\gamma}\mathbb{E}_{s\sim d^{\omega^*_\tau}}[H(\pi_{\omega^*_\tau}(\cdot|s))] \\
&\ge\frac{q_0\Delta-\tau\log|\mathcal{A}|}{1-\gamma}.
\end{align*}
Hence, whenever $\tau\log|\mathcal{A}|\le q_0\Delta/2$,
$$J_\tau(\pi^*_\tau)-J_\tau(\pi_{\omega^*_\tau})\ge\frac{q_0\Delta}{2(1-\gamma)}.$$

(iii) \textit{Uniform Compatibility Bound.}
Define the reference temperature
$$
\bar{\tau}
\triangleq
\min\left\{
\tau_{max},
\frac{q_0\Delta}{2\log|\mathcal{A}|}
\right\}.
$$
By construction,
$\bar{\tau}\log|\mathcal{A}|\le q_0\Delta/2$.
Applying the preceding lower bound at $\tau=\bar{\tau}$ gives
$$
J_{\bar{\tau}}(\pi^*_{\bar{\tau}}) -
J_{\bar{\tau}}(\pi_{\omega^*_{\bar{\tau}}})
\ge
\frac{q_0\Delta}{2(1-\gamma)}.
$$
On the other hand, Lemma~\ref{lem:global-class-error} gives
$$
J_{\bar{\tau}}(\pi^*_{\bar{\tau}}) -
J_{\bar{\tau}}(\pi_{\omega^*_{\bar{\tau}}})
\le
\frac{C_{class}}{\bar{\tau}}\epsilon_{app}.
$$
Combining these inequalities yields
$$
\epsilon_{app} \ge
\frac{q_0\Delta}{2(1-\gamma)C_{class}}\bar{\tau}.
$$
Because $\epsilon_{app}$ and the structural constants are uniform over
$\tau\in(0,\tau_{max}]$, this bound obtained at the reference temperature
$\bar{\tau}$ constrains the same quantities throughout this temperature range.
For any $\tau\in(0,\tau_{max}]$, we have
$\bar{\tau}\ge(\bar{\tau}/\tau_{max})\tau$.
Therefore,
$$
\epsilon_{app} \ge
\frac{q_0\Delta}{2(1-\gamma)C_{class}}
\frac{\bar{\tau}}{\tau_{max}}\tau
= C_{comp}\tau,
$$
which completes the proof.
\end{prf}

\begin{prf}[Proof of Theorem~\ref{thm:unreg-average-iterate-optimal} (Unregularized Average-Iterate Convergence)]
We consider $\epsilon_{app}>0$ throughout, as required by the Part-I compatibility condition in Lemma~\ref{lem:part-I-compatibility}.

(i) \textit{Global Translation Bound.} Because the temperature $\tau_T$ may initially exceed the structural threshold required in Theorem~\ref{thm:suboptimal-mass}, we apply the Universal Unregularized Suboptimality Bound (Corollary~\ref{cor:universal-suboptimal-mass}). This guarantees the following inequality holds for any $\tau_T > 0$:
$$ J_0(\pi^*_0) - J_0(\pi_t) \le \frac{4 A_{max}}{\Delta} \text{Gap}_t^\dag(\tau_T) + C_{tail} \exp\left(-\frac{\Delta}{2\tau_T}\right). $$
Taking the expectation and averaging this bound over the trajectory from $t=0$ to $T-1$ yields:
$$ \frac{1}{T} \sum_{t=0}^{T-1} \mathbb{E} \big[ J_0(\pi^*_0) - J_0(\pi_t) \big] \le \frac{4 A_{max}}{\Delta} \left( \frac{1}{T} \sum_{t=0}^{T-1} \mathbb{E}[\text{Gap}_t^\dag(\tau_T)] \right) + C_{tail} \exp\left(-\frac{\Delta}{2\tau_T}\right). $$

(ii) \textit{Regularized Order Substitution.} By the Global Regularized Performance Bound (Corollary~\ref{cor:global-reg-convergence}), the expected global regularized gap is bounded by the parametric gap plus the representation penalty: $\mathbb{E}[\text{Gap}_t^\dag] \le \mathbb{E}[\text{Gap}_t] + \frac{C_{class}}{\tau_T} \epsilon_{app}$.
We substitute the average parametric regularized performance gap from Theorem~\ref{thm:average-iterate}:
$$ \frac{1}{T} \sum_{t=0}^{T-1} \mathbb{E}[\text{Gap}_t^\dag] \le \underbrace{\mathcal{O}\left( \frac{1}{(1-\gamma)^5 \kappa^6 \lambda^{5/2} \tau_T^2} \frac{\log T}{T} \right)}_{\text{Finite-Time Optimization Error}} + \underbrace{\mathcal{O}\left( \frac{\epsilon_{app}}{(1-\gamma) \lambda^2 \tau_T} \right) + \frac{C_{class}}{\tau_T} \epsilon_{app}}_{\text{Total Approximation Bias}}. $$
Because $C_{class} \triangleq \frac{B_\phi^4}{(1-\gamma) \lambda^2} + \frac{C_{joint}^*}{1-\gamma}= \mathcal{O}\big(\frac{1}{(1-\gamma)\lambda^2}\big)$, the representation penalty is absorbed into the approximation bias order.
Feeding this into the translation bound isolates the three competing components:
$$ \text{Total Error} \le \underbrace{ \mathcal{O}\left( \frac{A_{max}}{\Delta(1-\gamma)^5 \kappa^6 \lambda^{5/2} \tau_T^2} \frac{\log T}{T} \right) }_{\text{Finite-Time Optimization Error}} + \underbrace{ \mathcal{O}\left( \frac{A_{max} \epsilon_{app}}{\Delta(1-\gamma) \lambda^2 \tau_T} \right) }_{\text{Total Approximation Bias}} + \underbrace{ C_{tail} \exp\left(-\frac{\Delta}{2\tau_T}\right) }_{\text{Entropy Tail}}. $$

(iii) \textit{Regime 1: Sample-Limited Phase.}
When $\log T \le \log(1+\epsilon_{app}^{-1})$, the temperature is tuned to $\tau_T = \frac{\Delta}{2 \log(C_\gamma T)}$. Substituting this alongside $A_{max} \le \frac{1}{1-\gamma}$, $C_{tail} \triangleq \frac{A_{max} C_\gamma}{1-\gamma} \le \frac{C_\gamma}{(1-\gamma)^2}$, and $\log C_\gamma = \mathcal{O}(\frac{1}{1-\gamma})$ into the bounded components yields:
\begin{align*}
\text{Finite-Time Error}
&= \mathcal{O}\Big( \frac{1}{(1-\gamma)^6 \kappa^6 \lambda^{5/2} \Delta^3} (\log(C_\gamma T))^2 (\log T) T^{-1} \Big) \\
&= \mathcal{O}\Big( \frac{1}{(1-\gamma)^6 \kappa^6 \lambda^{5/2} \Delta^3} \Big( \frac{\log T}{(1-\gamma)^2} + (\log T)^3 \Big) T^{-1} \Big), \\
\text{Approximation Bias}
&= \mathcal{O}\Big( \frac{1}{(1-\gamma)^2 \lambda^2 \Delta^2} \log(C_\gamma T) \epsilon_{app} \Big) \\
&\le \mathcal{O}\Big( \frac{1}{(1-\gamma)^2 \lambda^2 \Delta^2} \Big( \frac{1}{1-\gamma} + \log(1+\epsilon_{app}^{-1}) \Big) \epsilon_{app} \Big), \\
\text{Entropy Tail}
&\le \frac{C_\gamma}{(1-\gamma)^2} \exp\big(-\log(C_\gamma T)\big) = \frac{1}{(1-\gamma)^2 T}.
\end{align*}
The regime condition $\log T \le \log(1+\epsilon_{app}^{-1})$ guarantees
$$ \log(C_\gamma T) = \log C_\gamma + \log T \le \log C_\gamma + \log(1+\epsilon_{app}^{-1}), $$
which gives the stated approximation-bias envelope.
The $C_\gamma$ prefactor cancels in the exponential tail, yielding a pure $\mathcal{O}(T^{-1})$ tail that is absorbed by the dominant $\tilde{\mathcal{O}}(T^{-1})$ optimization error.

(iv) \textit{Regime 2: Approximation-Limited Phase.}
When $\log T > \log(1+\epsilon_{app}^{-1})$, the temperature locks to the approximation-dependent floor:
$\tau_T = \frac{\Delta}{2\log(C_\gamma(1+\epsilon_{app}^{-1}))}$.
Substituting this fixed temperature into the bounded components yields:
\begin{align*}
\text{Finite-Time Error}
&= \mathcal{O}\Big( \frac{1}{(1-\gamma)^6 \kappa^6 \lambda^{5/2} \Delta^3} \big(\log(C_\gamma(1+\epsilon_{app}^{-1}))\big)^2 (\log T) T^{-1} \Big) \\
&\le \mathcal{O}\Big( \frac{1}{(1-\gamma)^6 \kappa^6 \lambda^{5/2} \Delta^3} \Big( \frac{\log T}{(1-\gamma)^2} + (\log T)^3 \Big) T^{-1} \Big), \\
\text{Approximation Bias}
&= \mathcal{O}\Big( \frac{1}{(1-\gamma)^2 \lambda^2 \Delta^2} \log(C_\gamma(1+\epsilon_{app}^{-1})) \epsilon_{app} \Big) \\
&= \mathcal{O}\Big( \frac{1}{(1-\gamma)^2 \lambda^2 \Delta^2} \Big( \frac{1}{1-\gamma} + \log(1+\epsilon_{app}^{-1}) \Big) \epsilon_{app} \Big), \\
\text{Entropy Tail}
&\le \frac{C_\gamma}{(1-\gamma)^2} \exp\big(-\log(C_\gamma(1+\epsilon_{app}^{-1}))\big)
= \frac{\epsilon_{app}/(1+\epsilon_{app})}{(1-\gamma)^2}
\le \frac{\epsilon_{app}}{(1-\gamma)^2}.
\end{align*}
For the finite-time error, the regime condition $\log(1+\epsilon_{app}^{-1}) < \log T$ implies
$$ \log(C_\gamma(1+\epsilon_{app}^{-1})) \le \log C_\gamma + \log T , $$
which yields the stated finite-time envelope.
The fixed temperature floor prevents the $\tau_T^{-1}$ approximation bias from growing with $T$, leading to the $\tilde{\mathcal{O}}(\epsilon_{app})$ rate. Simultaneously, the exponential entropy tail is bounded by the linear scale of the approximation error, yielding the $\mathcal{O}(\epsilon_{app})$ rate.

(v) \textit{Synthesis.}
Summing these evaluated components across the two regimes bounds the overall expected global unregularized performance gap by their respective envelopes, yielding:
\begin{align*}
\frac{1}{T} \sum_{t=0}^{T-1} \mathbb{E}\big[J_0(\pi^*_0)-J_0(\pi_t)\big]
&\le \mathcal{O}\left( \frac{(\log T)/(1-\gamma)^2 + (\log T)^3}{(1-\gamma)^6 \kappa^6 \lambda^{5/2} \Delta^3 T} \right) \\
&\quad + \mathcal{O}\left( \frac{1/(1-\gamma)+\log(1+\epsilon_{app}^{-1})}{(1-\gamma)^2 \lambda^2 \Delta^2} \epsilon_{app} \right).
\end{align*}
The approximation term is proportional to $\epsilon_{app}$ up to the logarithmic factor $\log(1+\epsilon_{app}^{-1})$. Thus, subject to the admissible range of $\epsilon_{app}$ in
Remark \ref{rem:admissible-eps-app}, the overall convergence rate can be summarized as $\tilde{\mathcal{O}}(T^{-1})+\tilde{\mathcal{O}}(\epsilon_{app})$, completing the proof.
\end{prf}

\begin{prf}[Proof of Theorem~\ref{thm:unreg-last-iterate-optimal} (Unregularized Last-Iterate Convergence)]
We consider $\epsilon_{app}>0$ throughout, as required by the Part-I compatibility condition in Lemma~\ref{lem:part-I-compatibility}.

(i) \textit{Global Translation Bound.} We apply the Universal Unregularized Suboptimality Bound (Corollary~\ref{cor:universal-suboptimal-mass}) directly to the terminal policy $\pi_T$. This guarantees the following inequality holds for any $\tau_T > 0$:
$$ J_0(\pi^*_0) - J_0(\pi_T) \le \frac{4 A_{max}}{\Delta} \text{Gap}_T^\dag(\tau_T) + C_{tail} \exp\left(-\frac{\Delta}{2\tau_T}\right). $$
Taking the expectation bounds the last-iterate global unregularized performance gap:
$$ \mathbb{E} \big[ J_0(\pi^*_0) - J_0(\pi_T) \big] \le \frac{4 A_{max}}{\Delta} \mathbb{E}[\text{Gap}_T^\dag(\tau_T)] + C_{tail} \exp\left(-\frac{\Delta}{2\tau_T}\right). $$

(ii) \textit{Regularized Order Substitution.} By the Global Regularized Performance Bound (Corollary~\ref{cor:global-reg-convergence}), the expected global regularized gap is bounded by the parametric gap plus the representation penalty: $\mathbb{E}[\text{Gap}_T^\dag] \le \mathbb{E}[\text{Gap}_T] + \frac{C_{class}}{\tau_T} \epsilon_{app}$.
We substitute the last-iterate parametric regularized performance gap from Theorem~\ref{thm:last-iterate}:
$$ \mathbb{E}[\text{Gap}_T^\dag] \le \underbrace{\mathcal{O}\left( \frac{1}{(1-\gamma)^5 \kappa^6 \lambda^{5/2} \tau_T^2 T} \right)}_{\text{Finite-Time Optimization Error}} + \underbrace{\mathcal{O}\left( \frac{\epsilon_{app}}{(1-\gamma) \lambda^2 \tau_T} \right) + \frac{C_{class}}{\tau_T} \epsilon_{app}}_{\text{Total Approximation Bias}}. $$
Because $C_{class} = \mathcal{O}\big(\frac{1}{(1-\gamma)\lambda^2}\big)$, the representation penalty is absorbed into the approximation bias order.
Feeding this into the translation bound isolates the three competing components:
$$ \text{Total Error} \le \underbrace{ \mathcal{O}\left( \frac{A_{max}}{\Delta(1-\gamma)^5 \kappa^6 \lambda^{5/2} \tau_T^2 T} \right) }_{\text{Finite-Time Optimization Error}} + \underbrace{ \mathcal{O}\left( \frac{A_{max} \epsilon_{app}}{\Delta(1-\gamma) \lambda^2 \tau_T} \right) }_{\text{Total Approximation Bias}} + \underbrace{ C_{tail} \exp\left(-\frac{\Delta}{2\tau_T}\right) }_{\text{Entropy Tail}}. $$

(iii) \textit{Regime 1: Sample-Limited Phase.}
When $\log T \le \log(1+\epsilon_{app}^{-1})$, the temperature is tuned to $\tau_T = \frac{\Delta}{2 \log(C_\gamma T)}$. Substituting this alongside $A_{max} \le \frac{1}{1-\gamma}$, $C_{tail} \triangleq \frac{A_{max} C_\gamma}{1-\gamma} \le \frac{C_\gamma}{(1-\gamma)^2}$, and $\log C_\gamma = \mathcal{O}(\frac{1}{1-\gamma})$ into the bounded components yields:
\begin{align*}
\text{Finite-Time Error}
&= \mathcal{O}\Big( \frac{1}{(1-\gamma)^6 \kappa^6 \lambda^{5/2} \Delta^3} (\log(C_\gamma T))^2 T^{-1} \Big) \\
&= \mathcal{O}\Big( \frac{1}{(1-\gamma)^6 \kappa^6 \lambda^{5/2} \Delta^3} \Big( \frac{1}{(1-\gamma)^2} + (\log T)^2 \Big) T^{-1} \Big), \\
\text{Approximation Bias}
&= \mathcal{O}\Big( \frac{1}{(1-\gamma)^2 \lambda^2 \Delta^2} \log(C_\gamma T) \epsilon_{app} \Big) \\
&\le \mathcal{O}\Big( \frac{1}{(1-\gamma)^2 \lambda^2 \Delta^2} \Big( \frac{1}{1-\gamma} + \log(1+\epsilon_{app}^{-1}) \Big) \epsilon_{app} \Big), \\
\text{Entropy Tail}
&\le \frac{C_\gamma}{(1-\gamma)^2} \exp\big(-\log(C_\gamma T)\big) = \frac{1}{(1-\gamma)^2 T}.
\end{align*}
The regime condition $\log T \le \log(1+\epsilon_{app}^{-1})$ guarantees
$$ \log(C_\gamma T) = \log C_\gamma + \log T \le \log C_\gamma + \log(1+\epsilon_{app}^{-1}), $$
which gives the stated approximation-bias envelope.
The $C_\gamma$ prefactor cancels in the exponential tail, which scales as $\mathcal{O}(T^{-1})$ and is absorbed by the dominant $\tilde{\mathcal{O}}(T^{-1})$ optimization error. Notably, the final optimization error is sharper than the average-iterate bound by a factor of $\log T$.

(iv) \textit{Regime 2: Approximation-Limited Phase.}
When $\log T > \log(1+\epsilon_{app}^{-1})$, the temperature locks to the approximation-dependent floor:
$\tau_T = \frac{\Delta}{2\log(C_\gamma(1+\epsilon_{app}^{-1}))}$.
Substituting this fixed temperature into the bounded components yields:
\begin{align*}
\text{Finite-Time Error}
&= \mathcal{O}\Big( \frac{1}{(1-\gamma)^6 \kappa^6 \lambda^{5/2} \Delta^3} \big(\log(C_\gamma(1+\epsilon_{app}^{-1}))\big)^2 T^{-1} \Big) \\
&\le \mathcal{O}\Big( \frac{1}{(1-\gamma)^6 \kappa^6 \lambda^{5/2} \Delta^3} \Big( \frac{1}{(1-\gamma)^2} + (\log T)^2 \Big) T^{-1} \Big), \\
\text{Approximation Bias}
&= \mathcal{O}\Big( \frac{1}{(1-\gamma)^2 \lambda^2 \Delta^2} \log(C_\gamma(1+\epsilon_{app}^{-1})) \epsilon_{app} \Big) \\
&= \mathcal{O}\Big( \frac{1}{(1-\gamma)^2 \lambda^2 \Delta^2} \Big( \frac{1}{1-\gamma} + \log(1+\epsilon_{app}^{-1}) \Big) \epsilon_{app} \Big), \\
\text{Entropy Tail}
&\le \frac{C_\gamma}{(1-\gamma)^2} \exp\big(-\log(C_\gamma(1+\epsilon_{app}^{-1}))\big)
= \frac{\epsilon_{app}/(1+\epsilon_{app})}{(1-\gamma)^2}
\le \frac{\epsilon_{app}}{(1-\gamma)^2}.
\end{align*}
For the finite-time error, the regime condition $\log(1+\epsilon_{app}^{-1}) < \log T$ implies
$$ \log(C_\gamma(1+\epsilon_{app}^{-1})) \le \log C_\gamma + \log T, $$
which yields the stated finite-time envelope.
The fixed temperature floor prevents the $\tau_T^{-1}$ approximation bias from growing with $T$, preserving the $\tilde{\mathcal{O}}(\epsilon_{app})$ rate. Simultaneously, the exponential entropy tail is bounded by the linear scale of the approximation error, yielding the $\mathcal{O}(\epsilon_{app})$ rate.

(v) \textit{Synthesis.}
Summing these evaluated components across the two regimes bounds the expected last-iterate global unregularized performance gap by their respective envelopes, yielding:
\begin{align*}
\mathbb{E}\big[J_0(\pi^*_0)-J_0(\pi_T)\big]
&\le \mathcal{O}\left( \frac{(1-\gamma)^{-2} + (\log T)^2}{(1-\gamma)^6 \kappa^6 \lambda^{5/2} \Delta^3 T} \right) \\
&\quad + \mathcal{O}\left( \frac{(1-\gamma)^{-1}+\log(1+\epsilon_{app}^{-1})}{(1-\gamma)^2 \lambda^2 \Delta^2} \epsilon_{app} \right).
\end{align*}
The approximation term is proportional to $\epsilon_{app}$ up to the logarithmic factor $\log(1+\epsilon_{app}^{-1})$. Thus, subject to the admissible range of $\epsilon_{app}$ in
Remark \ref{rem:admissible-eps-app}, the overall convergence rate can be summarized as $\tilde{\mathcal{O}}(T^{-1})+\tilde{\mathcal{O}}(\epsilon_{app})$, completing the proof.
\end{prf}

\subsection{Proofs for Section~\ref{sec:pmd-geometry}}

\begin{prf}[Proof of Lemma~\ref{lem:pmd-progress} (PMD Algorithmic Progress)]
(i) \textit{Bregman Identity.} The standard Bregman Three-Point identity over the probability simplex dictates that for any state $s$:
$$ D_{KL}(\pi^*_\tau \| \pi_{t+1}) = D_{KL}(\pi^*_\tau \| \pi_t) + D_{KL}(\pi_t \| \pi_{t+1}) - \langle \log \pi_{t+1} - \log \pi_t, \pi^*_\tau - \pi_t \rangle_{\mathcal{A}}. $$
By substituting the log-linear policy parameterization, the log-probability update is $\log \pi_{t+1}(a|s) - \log \pi_t(a|s) = \eta_t (g_t^{ac})^\top\phi(s,a) - \log \frac{Z_{t+1}(s)}{Z_t(s)}$. Substituting this into the inner product eliminates the state-dependent normalizer, as both policies sum to $1$ over the action space:
\begin{align*}
\langle \log \pi_{t+1} - \log \pi_t, \pi^*_\tau - \pi_t \rangle_{\mathcal{A}} &= \sum_a \big( \pi^*_\tau(a|s) - \pi_t(a|s) \big) \left( \eta_t (g_t^{ac})^\top\phi(s,a) - \log \frac{Z_{t+1}(s)}{Z_t(s)} \right) \\
&= \eta_t \langle (g_t^{ac})^\top\phi, \pi^*_\tau - \pi_t \rangle_{\mathcal{A}}.
\end{align*}
Rearranging the Bregman identity isolates the scaled algorithmic progress:
$$ \eta_t \langle (g_t^{ac})^\top\phi, \pi^*_\tau - \pi_t \rangle_{\mathcal{A}} = D_{KL}(\pi^*_\tau \| \pi_t) - D_{KL}(\pi^*_\tau \| \pi_{t+1}) + D_{KL}(\pi_t \| \pi_{t+1}). $$

(ii) \textit{Local Step Distance.} To establish the lemma, it suffices to upper-bound the local step distance $D_{KL}(\pi_t \| \pi_{t+1})$. By definition of the forward KL divergence between successive policies:
$$ D_{KL}(\pi_t \| \pi_{t+1}) = \sum_a \pi_t(a|s) \big( \log \pi_t(a|s) - \log \pi_{t+1}(a|s) \big). $$
Substituting the exact log-probability update yields:
$$ D_{KL}(\pi_t \| \pi_{t+1}) = - \eta_t \langle (g_t^{ac})^\top\phi, \pi_t \rangle_{\mathcal{A}} + \log \frac{Z_{t+1}(s)}{Z_t(s)}. $$

(iii) \textit{Hoeffding's Bound.} We express the partition function ratio as an expectation over the current policy $\pi_t$:
$$ \frac{Z_{t+1}(s)}{Z_t(s)} = \sum_a \pi_t(a|s) \exp\big(\eta_t (g_t^{ac})^\top \phi(s,a)\big) = \mathbb{E}_{a \sim \pi_t} \big[ \exp\big(\eta_t (g_t^{ac})^\top \phi(s,a)\big) \big]. $$
Because the step size satisfies $\eta_t \le 1/\tau$, Lemma~\ref{lem:structural-bounds}(ii) guarantees that the actor update direction is bounded under Assumptions~\ref{ass:bounded-features} and \ref{ass:pd-uncentered}. Then the functional update direction is bounded by a temperature-independent constant: $|(g_t^{ac})^\top \phi(s,a)| \le \|g_t^{ac}\|_2 \|\phi(s,a)\|_2 \le G_{ac} B_\phi$.
Therefore, we apply Hoeffding's Lemma to bound the log-moment generating function:
$$ \log \frac{Z_{t+1}(s)}{Z_t(s)} \le \eta_t \mathbb{E}_{a \sim \pi_t}\big[(g_t^{ac})^\top \phi(s,a)\big] + \frac{\eta_t^2 (2 G_{ac} B_\phi)^2}{8} = \eta_t \langle (g_t^{ac})^\top \phi, \pi_t \rangle_{\mathcal{A}} + \frac{\eta_t^2}{2} G_{ac}^2 B_\phi^2. $$

(iv) \textit{Synthesis.} Substituting this log-partition bound back into Step (ii) cancels the inner product, bounding the local step distance purely by the worst-case update magnitude (the PMD Local Error):
$$ D_{KL}(\pi_t \| \pi_{t+1}) \le - \eta_t \langle (g_t^{ac})^\top \phi, \pi_t \rangle_{\mathcal{A}} + \left( \eta_t \langle (g_t^{ac})^\top \phi, \pi_t \rangle_{\mathcal{A}} + \frac{\eta_t^2}{2} G_{ac}^2 B_\phi^2 \right) = \frac{\eta_t^2}{2} G_{ac}^2 B_\phi^2. $$
Substituting this upper bound into the Bregman identity from Step (i) and dividing by $\eta_t$ completes the proof.
\end{prf}

\subsection{Proofs for Section~\ref{sec:regularized-conv-deterministic}}

\begin{prf}[Proof of Lemma~\ref{lem:pmd-telescoping-l2} (PMD Actor Progress Bound)]
By the Actual Critic Decomposition with $\pi=\pi^*_\tau$ in Lemma~\ref{lem:sub-decomp}(ii), the gap evaluates via the PMD geometry to:
\begin{align*}
(1-\gamma) \text{Gap}_t^\dag &= -\tau D_t^\dag + \underbrace{\mathbb{E}_{d^{\pi^*_\tau}} \Big[ \langle (\theta_{t+1} - \tau \omega_t)^\top \phi, \pi^*_\tau - \pi_t \rangle_{\mathcal{A}} \Big]}_{\text{Algorithmic Progress}} \\
&\quad + \underbrace{\mathbb{E}_{d^{\pi^*_\tau}} \Big[ \langle -e_t^\top \phi, \pi^*_\tau - \pi_t \rangle_{\mathcal{A}} \Big]}_{\text{Critic Tracking Error}} +
\underbrace{ \mathbb{E}_{d^{\pi^*_\tau}} \big[ \langle \epsilon_t, \pi^*_\tau - \pi_t \rangle_{\mathcal{A}} \big]}_{\text{Approximation Bias}}.
\end{align*}
We analyze the three components on the right-hand side independently.

(i) \textit{PMD Algorithmic Progress.}
Because $\eta_t \le 1/\tau$, Lemma~\ref{lem:pmd-progress} applies to bound the progress inner product driven by the log-linear update, yielding:
$$\mathbb{E}_{d^*_\tau} \Big[ \langle (\theta_{t+1} - \tau \omega_t)^\top \phi, \pi^*_\tau - \pi_t \rangle_{\mathcal{A}} \Big] \le \frac{D_t^\dag - D_{t+1}^\dag}{\eta_t} + \frac{\eta_t}{2} G_{ac}^2 B_\phi^2.$$
Substituting this into the gap decomposition gives:
\begin{align}
(1-\gamma) \text{Gap}_t^\dag \le -\tau D_t^\dag + \frac{D_t^\dag - D_{t+1}^\dag}{\eta_t} + \frac{\eta_t}{2} G_{ac}^2 B_\phi^2 + \mathbb{E}_{d^*_\tau} \Big[ \langle -e_t^\top \phi + \epsilon_t, \pi^*_\tau - \pi_t \rangle_{\mathcal{A}} \Big]. \label{eq:PMD-progress}
\end{align}

(ii) \textit{Approximation Bias Transfer.}
For any state $s$, define the local policy mismatch $W^\dag(s) \triangleq \max_a \frac{\pi^*_\tau(a|s)}{\pi_t(a|s)}$. By Action-Space Cauchy-Schwarz and the $\chi^2$ bound (Lemma~\ref{lem:chi2-KL}), the approximation bias satisfies:
$$\langle \epsilon_t, \pi^*_\tau - \pi_t \rangle_{\mathcal{A}} \le \sqrt{ 2 W^\dag(s) D_{KL}(\pi^*_\tau(\cdot|s) \| \pi_t(\cdot|s)) } \sqrt{ \mathbb{E}_{a \sim \pi_t}[\epsilon_t(s,a)^2] }.$$
Taking the expectation over $d^*_\tau(s)$ and applying State-Space Cauchy-Schwarz gives:
\begin{align*}
\mathbb{E}_{d^*_\tau} \big[ \langle \epsilon_t, \pi^*_\tau - \pi_t \rangle_{\mathcal{A}} \big] &\le \mathbb{E}_{d^*_\tau} \Big[ \sqrt{ 2 D_{KL}(\pi^*_\tau(\cdot|s) \| \pi_t(\cdot|s)) } \sqrt{ W^\dag(s) \mathbb{E}_{a \sim \pi_t}[\epsilon_t(s,a)^2] } \Big] \\
&\le \sqrt{ 2 D_t^\dag } \sqrt{ \mathbb{E}_{d^*_\tau} \big[ W^\dag(s) \mathbb{E}_{a \sim \pi_t}[\epsilon_t(s,a)^2] \big] }.
\end{align*}
To control the second root, we change the state measure to $d^{\pi_t}$ and apply the joint concentrability bound $\frac{d^*_\tau(s)}{d^{\pi_t}(s)} W^\dag(s) = \max_a \frac{d^*_\tau(s) \pi^*_\tau(a|s)}{d^{\pi_t}(s) \pi_t(a|s)} \le C_{joint}^\dag$ (Assumption~\ref{ass:density-ratios}):
\begin{align*}
\mathbb{E}_{d^*_\tau} \big[ W^\dag(s) \mathbb{E}_{a \sim \pi_t}[\epsilon_t(s,a)^2] \big] &= \mathbb{E}_{d^{\pi_t}} \left[ \left( \frac{d^*_\tau(s)}{d^{\pi_t}(s)} W^\dag(s) \right) \mathbb{E}_{a \sim \pi_t}[\epsilon_t(s,a)^2] \right] \\
&\le C_{joint}^\dag \mathbb{E}_{d^{\pi_t}} \Big[ \mathbb{E}_{a \sim \pi_t}[\epsilon_t(s,a)^2] \Big] \\
&= C_{joint}^\dag \mathbb{E}_{d^{\pi_t}, \pi_t}[\epsilon_t^2] \le C_{joint}^\dag \epsilon_{app},
\end{align*}
where the last step uses Assumption~\ref{ass:approximation-error}.
Applying Young's Inequality ($xy \le \frac{\tau}{8}x^2 + \frac{2}{\tau}y^2$) with $x = \sqrt{D_t^\dag}$ and $y = \sqrt{2 C_{joint}^\dag \epsilon_{app}}$, this term is bounded by $\frac{\tau}{8}D_t^\dag + \frac{4 C_{joint}^\dag}{\tau} \epsilon_{app}$.

(iii) \textit{Tracking Error Bound.}
For the tracking error, the bounded features assumption guarantees $|e_t^\top \phi(s,a)| \le \|e_t\|_2 \|\phi(s,a)\|_2 \le B_\phi \|e_t\|_2$. Applying H\"{o}lder's inequality over the action space and Pinsker's inequality ($\|\pi^*_\tau - \pi_t\|_1 \le \sqrt{2 D_{KL}(\pi^*_\tau \| \pi_t)}$) yields:
$$ \mathbb{E}_{d^*_\tau} \Big[ \langle -e_t^\top \phi, \pi^*_\tau - \pi_t \rangle_{\mathcal{A}} \Big] \le \mathbb{E}_{d^*_\tau} \Big[ \max_a |e_t^\top \phi(s,a)| \|\pi^*_\tau - \pi_t\|_1 \Big] \le \mathbb{E}_{d^*_\tau} \Big[ B_\phi \|e_t\|_2 \sqrt{2 D_{KL}(\pi^*_\tau \| \pi_t)} \Big]. $$
Applying Jensen's inequality over the state distribution yields $B_\phi \|e_t\|_2 \sqrt{2 D_t^\dag}$. Crucially, because $\|e_t\|_2$ is a global scalar and we use Pinsker's inequality, this term is free of any joint concentrability.
Applying Young's Inequality ($xy \le \frac{\tau}{8}x^2 + \frac{2}{\tau}y^2$) with $x = \sqrt{D_t^\dag}$ and $y = \sqrt{2} B_\phi \|e_t\|_2$ bounds the tracking penalty:
$$ B_\phi \|e_t\|_2 \sqrt{2 D_t^\dag} \le \frac{\tau}{8}D_t^\dag + \frac{2}{\tau} \big(2 B_\phi^2 \|e_t\|_2^2 \big) = \frac{\tau}{8}D_t^\dag + \frac{4 B_\phi^2}{\tau} \|e_t\|_2^2. $$

(iv) \textit{Synthesis.}
Summing the bounds from (ii) and (iii) yields $\frac{\tau}{4} D_t^\dag + \frac{4 C_{joint}^\dag}{\tau} \epsilon_{app} + \frac{4 B_\phi^2}{\tau} \|e_t\|_2^2$. Substituting this into inequality \eqref{eq:PMD-progress}, and combining the $D_t^\dag$ terms, $-\tau D_t^\dag + \frac{\tau}{4} D_t^\dag = -\frac{3\tau}{4} D_t^\dag$, directly yields the single-step inequality:
$$(1-\gamma) \text{Gap}_t^\dag \le -\frac{3\tau}{4} D_t^\dag + \frac{D_t^\dag - D_{t+1}^\dag}{\eta_t} + \frac{\eta_t}{2} G_{ac}^2 B_\phi^2 + \frac{4 C_{joint}^\dag}{\tau} \epsilon_{app} + \frac{4 B_\phi^2}{\tau} \|e_t\|_2^2.$$
Taking the total expectation $\mathbb{E}[\cdot]$ over the algorithm trajectory and substituting $Z_t \triangleq \mathbb{E}[\|e_t\|_2^2]$ completes the proof.
\end{prf}

\begin{prf}[Proof of Lemma~\ref{lem:telescoping-pmd} (PMD Value-Telescoping Bound)]
Because the initial offset satisfies $t_0 \ge c_\eta \tau$ and the step-size schedule is decreasing, we have $\eta_t \le \eta_0 = \frac{c_\eta}{t_0} \le \frac{1}{\tau}$ for all $t \ge 0$. This satisfies the step-size condition required to apply the PMD Actor Progress Bound (Lemma~\ref{lem:pmd-telescoping-l2}). We rearrange this single-step recursive bound to isolate the next-step divergence:
$$ \mathbb{E}[D_{t+1}^\dag] \le \left(1 - \frac{3\tau}{4}\eta_t\right) \mathbb{E}[D_t^\dag] - \eta_t (1-\gamma) \mathbb{E}[\text{Gap}_t^\dag] + \eta_t \mathcal{E}_{noise}(t). $$
Substituting the diminishing step size $\eta_t = \frac{c_\eta}{t+t_0}$, the contraction coefficient evaluates to $1 - \frac{3\tau c_\eta / 4}{t+t_0}$. Because the step-size constant satisfies $\frac{3\tau}{4} c_\eta \ge 1$, we upper-bound this coefficient:
$$ 1 - \frac{3\tau c_\eta / 4}{t+t_0} \le 1 - \frac{1}{t+t_0} = \frac{t+t_0-1}{t+t_0}. $$
Substituting this upper bound into the recurrence gives:
$$ \mathbb{E}[D_{t+1}^\dag] \le \frac{t+t_0-1}{t+t_0} \mathbb{E}[D_t^\dag] - \frac{c_\eta}{t+t_0} (1-\gamma) \mathbb{E}[\text{Gap}_t^\dag] + \frac{c_\eta}{t+t_0} \mathcal{E}_{noise}(t). $$
We multiply the entire inequality by $(t+t_0)$ and rearrange the terms to isolate the gap on the left side:
$$ c_\eta (1-\gamma) \mathbb{E}[\text{Gap}_t^\dag] \le (t+t_0-1) \mathbb{E}[D_t^\dag] - (t+t_0) \mathbb{E}[D_{t+1}^\dag] + c_\eta \mathcal{E}_{noise}(t). $$
Summing this inequality from $t=0$ to $T-1$, the divergence terms telescope:
\begin{align*}
c_\eta (1-\gamma) \sum_{t=0}^{T-1} \mathbb{E}[\text{Gap}_t^\dag] &\le \sum_{t=0}^{T-1} \Big( (t+t_0-1) \mathbb{E}[D_t^\dag] - (t+t_0) \mathbb{E}[D_{t+1}^\dag] \Big) + c_\eta \sum_{t=0}^{T-1} \mathcal{E}_{noise}(t) \\
&= (t_0-1) D_0^\dag - (T+t_0-1) \mathbb{E}[D_T^\dag] + c_\eta \sum_{t=0}^{T-1} \mathcal{E}_{noise}(t),
\end{align*}
where we used $\mathbb{E}[D_0^\dag]=D_0^\dag$ because the initial actor parameter $\omega_0$ is fixed.
Dropping the non-positive terminal term $-(T+t_0-1)\mathbb{E}[D_T^\dag]$ and dividing the entire inequality by $c_\eta T$ isolates the average sequence and completes the proof.
\end{prf}

\begin{prf}[Proof of Lemma~\ref{lem:fractional-chung} (Fractional Chung's Lemma)]
Let $s = t+t_0$. Because $t_0 \ge c^{3/2}$, we have $s \ge c^{3/2}$, which guarantees the contraction multiplier $\left(1 - \frac{c}{s^{2/3}}\right)$ is non-negative. We prove $x_t \le \frac{v}{s^{2/3}}$ by induction.

The base case at $t=0$ holds trivially by the definition of $v$: $x_0 \le \frac{v}{t_0^{2/3}}$.
Assume the bound holds at step $t$. Substituting the inductive hypothesis into the recurrence yields:
$$ x_{t+1} \le \left(1 - \frac{c}{s^{2/3}}\right) \frac{v}{s^{2/3}} + \frac{K}{s^{4/3}} = \frac{v}{s^{2/3}} - \frac{cv - K}{s^{4/3}}. $$
By the convexity of $f(s) = s^{-2/3}$, the function is lower-bounded by its first-order Taylor expansion $f(s+1) \ge f(s) + f'(s)(s+1-s)$:
$$ \frac{1}{(s+1)^{2/3}} \ge \frac{1}{s^{2/3}} - \frac{2}{3 s^{5/3}}. $$
Therefore, to guarantee $x_{t+1} \le \frac{v}{(s+1)^{2/3}}$, it is sufficient to show:
$$ \frac{cv - K}{s^{4/3}} \ge \frac{2v}{3s^{5/3}} \Longleftrightarrow cv - K \ge \frac{2v}{3s^{1/3}}. $$
Because the initial offset ensures $s \ge t_0 \ge (\frac{4}{3c})^3$, the right-hand side is bounded by: $\frac{2v}{3s^{1/3}} \le \frac{2v}{3} \left( \frac{3c}{4} \right) = \frac{c}{2}v$.
A sufficient condition is then $cv - K \ge \frac{c}{2}v$ or equivalently $\frac{c}{2}v \ge K$. This matches our definition $v \ge \frac{2K}{c}$, completing the induction.
\end{prf}

\begin{prf}[Proof of Lemma~\ref{lem:decoupled-critic} (Pointwise Critic Tracking Error)]
(i) \textit{Step-Size Validity.} We verify that the initial offset $t_0$ guarantees the required structural conditions. First, $t_0 \ge c_\eta \tau$ guarantees that $\eta_t \le \eta_0 = \frac{c_\eta}{t_0} \le 1/\tau$ for all $t \ge 0$, which ensures the actor update direction is bounded by $G_{ac}$ via Lemma~\ref{lem:structural-bounds}(ii). Second, to ensure the SGD contraction multiplier is non-negative ($1 - 2\alpha_{t+1}\kappa \ge 0$) for the Young's inequality substitution, we require $\alpha_1 \le \frac{1}{2\kappa}$, which is guaranteed by $t_0 \ge (2 c_\alpha \kappa)^{3/2}$. Third, to apply the Fractional Chung's Lemma (Lemma~\ref{lem:fractional-chung}) with contraction parameter $c = c_\alpha \kappa$, we require $t_0 \ge \max\left(c^{3/2}, (\frac{4}{3c})^3\right)$.
The first term evaluates to $c^{3/2} = (c_\alpha \kappa)^{3/2}$, and the second term evaluates to $(\frac{4}{3c})^3 = (\frac{4}{3 c_\alpha \kappa})^3$. Both terms are enveloped by the assumed lower bound on $t_0$.

(ii) \textit{SGD Contraction and Young's Inequality.} Following the identical SGD contraction and Young's inequality as derived in Steps (i)--(iii) of the Coupled Critic Tracking proof (Lemma~\ref{lem:critic-tracking-pl}), the tracking error conditioned on $\mathcal{F}_{t+1}$ satisfies:
$$ \mathbb{E}[\|e_{t+1}\|_2^2 \mid \mathcal{F}_{t+1}] \le (1 - \alpha_{t+1} \kappa) \|e_t\|_2^2 + \alpha_{t+1}^2 G_{cr}^2 + \frac{1}{\alpha_{t+1} \kappa} \|\theta^*_\tau(\omega_t) - \theta^*_\tau(\omega_{t+1})\|_2^2. $$
Taking the total expectation $\mathbb{E}[\cdot]$ over the algorithmic trajectory yields the single-step recurrence:
$$ Z_{t+1} \le \underbrace{ (1 - \alpha_{t+1} \kappa) Z_t + \alpha_{t+1}^2 G_{cr}^2 }_{\text{SGD Error}} + \underbrace{ \frac{1}{\alpha_{t+1} \kappa} \mathbb{E} \big[ \|\theta^*_\tau(\omega_t) - \theta^*_\tau(\omega_{t+1})\|_2^2 \big] }_{\text{Target Drift}} . $$

(iii) \textit{Controlling Target Drift.} In Part I, we dynamically coupled this drift to the objective performance gap.
Here, because $\lambda_\tau \to 0$ prevents that coupling, we evaluate it as an uncoupled worst-case bound.
We utilize the Lipschitz continuity of the ideal critic mapping (Lemma~\ref{lem:lipschitz-critic}) and the uncoupled parameter step bound $\|\omega_{t+1} - \omega_t\|_2 = \eta_t \|g_t^{ac}\|_2 \le \eta_t G_{ac}$ (Lemma~\ref{lem:structural-bounds}(ii)):
$$ \mathbb{E} \big[ \|\theta^*_\tau(\omega_t) - \theta^*_\tau(\omega_{t+1})\|_2^2 \big] \le L_\theta^2 \mathbb{E} \big[ \|\omega_{t+1} - \omega_t\|_2^2 \big] \le L_\theta^2 \eta_t^2 G_{ac}^2. $$

(iv) \textit{Pointwise Recurrence.} Substituting the target drift bound into the tracking recurrence yields:
$$ Z_{t+1} \le (1 - \alpha_{t+1} \kappa) Z_t + \alpha_{t+1}^2 G_{cr}^2 + \frac{1}{\alpha_{t+1} \kappa} L_\theta^2 \eta_t^2 G_{ac}^2. $$
Substituting the step sizes $\alpha_{t+1} = \frac{c_\alpha}{(t+t_0)^{2/3}}$ and $\eta_t = \frac{c_\eta}{t+t_0}$ into the additive noise terms exactly extracts our defined constant $K$:
$$ \frac{c_\alpha^2 G_{cr}^2}{(t+t_0)^{4/3}} + \frac{L_\theta^2 c_\eta^2 G_{ac}^2}{\kappa c_\alpha (t+t_0)^{4/3}} = \frac{K}{(t+t_0)^{4/3}}. $$
The sequence satisfies the recurrence $Z_{t+1} \le \left(1 - \frac{c_\alpha \kappa}{(t+t_0)^{2/3}}\right) Z_t + \frac{K}{(t+t_0)^{4/3}}$.
Applying the Fractional Chung's Lemma (Lemma~\ref{lem:fractional-chung}) with contraction coefficient $c = c_\alpha \kappa$ yields the stated pointwise bound and completes the proof.
\end{prf}

\begin{prf}[Proof of Theorem~\ref{thm:reg-conv-l2} (Average-Iterate Regularized Convergence)]
(i) \textit{Telescoping Expansion.} By the PMD Value-Telescoping Bound (Lemma~\ref{lem:telescoping-pmd}), the average gap is bounded by:
$$ (1-\gamma) \frac{1}{T} \sum_{t=0}^{T-1} \mathbb{E}[\text{Gap}_t^\dag] \le \frac{t_0-1}{c_\eta T} D_0^\dag + \frac{1}{T} \sum_{t=0}^{T-1} \mathcal{E}_{noise}(t), $$
where the noise envelope is $\mathcal{E}_{noise}(t) = \frac{\eta_t}{2} G_{ac}^2 B_\phi^2 + \frac{4 C_{joint}^\dag}{\tau} \epsilon_{app} + \frac{4 B_\phi^2}{\tau} Z_t$.

(ii) \textit{Evaluating Sums.} By the Pointwise Critic Tracking Error (Lemma~\ref{lem:decoupled-critic}), the expected tracking error is bounded by $Z_t \le v (t+t_0)^{-2/3}$.
Substituting $\eta_t$ and $Z_t$ into the noise sum yields:
\begin{align*}
\frac{1}{T} \sum_{t=0}^{T-1} \mathcal{E}_{noise}(t) &\le \frac{c_\eta G_{ac}^2 B_\phi^2}{2 T} \sum_{t=0}^{T-1} \frac{1}{t+t_0} + \frac{4 C_{joint}^\dag}{\tau} \epsilon_{app} + \frac{4 B_\phi^2 v}{\tau T} \sum_{t=0}^{T-1} \frac{1}{(t+t_0)^{2/3}}.
\end{align*}
Because $t_0 \ge c_\eta \tau \ge \frac{4}{3} > 1$,
we bound the harmonic sum by $\sum_{t=0}^{T-1} \frac{1}{t+t_0} \le 1 + \log T$.
For the fractional tracking sum, we bound it via the integral $\int_0^T (x+t_0-1)^{-2/3} dx = 3(T+t_0-1)^{1/3} - 3(t_0-1)^{1/3} \le 3(T+t_0)^{1/3}$.
Dividing this tracking sum by $T$ evaluates to $3 T^{-2/3} \left(1 + \frac{t_0}{T}\right)^{1/3}$.
Therefore, the tracking penalty is bounded by $\frac{12 B_\phi^2 v}{\tau} T^{-2/3} \left(1 + \frac{t_0}{T}\right)^{1/3}$.

(iii) \textit{Synthesis and Asymptotic Order.} Substituting these bounded sums back into the PMD expansion yields the bound:
\begin{align*}
(1-\gamma) \frac{1}{T} \sum_{t=0}^{T-1} \mathbb{E}[\text{Gap}_t^\dag] &\le \underbrace{ \frac{t_0-1}{c_\eta T} D_0^\dag }_{\text{PMD Initialization Error}} + \underbrace{ \frac{c_\eta G_{ac}^2 B_\phi^2}{2 T} (1 + \log T) }_{\text{PMD Local Error}} \\
&\quad + \underbrace{ \frac{12 B_\phi^2 v}{\tau} T^{-2/3} \left(1 + \frac{t_0}{T}\right)^{1/3} }_{\text{Critic Tracking Error}} + \underbrace{ \frac{4 C_{joint}^\dag}{\tau} \epsilon_{app} }_{\text{Approximation Bias}}.
\end{align*}

To evaluate the structural orders of the bound, we explicitly track the dependencies on the temperature $\tau$, the feature moment minimum eigenvalue $\kappa$, and the horizon $(1-\gamma)^{-1}$.
We choose $c_\eta=\Theta(\tau^{-1})$ subject to $c_\eta\ge\frac{4}{3\tau}$.
By choosing the critic step-size constant $\zeta = \Theta( \frac{1}{\kappa})$, the product $\zeta \kappa = \Theta(1)$.
Consequently, the lower bound on $t_0$ is, up to constant factors,
$\max\left(c_\eta\tau,c_\eta,\frac{64}{27c_\eta^2}\right)$, eliminating the conflicting $\kappa$ dependencies.
Because $c_\eta = \Theta(\tau^{-1})$, the dominant term in the lower bound for $t_0$ is $c_\eta$. This allows us to satisfy the initial offset condition by choosing $t_0 = \Theta(\tau^{-1})$.

\begin{itemize}
    \item \textit{Noise Envelope ($K$):} We evaluate $K  = c_\alpha^2 G_{cr}^2 + \frac{L_\theta^2 c_\eta^2 G_{ac}^2}{\kappa c_\alpha} = (\zeta c_\eta^{2/3})^2 G_{cr}^2 + \frac{L_\theta^2 c_\eta^{4/3} G_{ac}^2}{\kappa \zeta}$.
        By Lemmas~\ref{lem:structural-bounds} and \ref{lem:lipschitz-critic}, the structural bounds scale as $G_{cr}^2 = \Theta(\kappa^{-2}(1-\gamma)^{-2})$,
        $G_{ac} = \Theta(\kappa^{-1}(1-\gamma)^{-1})$, and $L_\theta = \Theta(\kappa^{-2}(1-\gamma)^{-2})$.
        Substituting these alongside $c_\eta^{4/3} = \Theta(\tau^{-4/3})$ and $\zeta \kappa = \Theta(1)$, the first term (SGD error) evaluates to $\Theta(\tau^{-4/3} \kappa^{-4} (1-\gamma)^{-2})$, while the second term (target drift) evaluates to $\Theta(\tau^{-4/3} \kappa^{-6} (1-\gamma)^{-6})$.
        Because both terms share an identical $\Theta(\tau^{-4/3})$ temperature dependence, the second term dominates in its dependence on $\kappa$ and $(1-\gamma)^{-1}$, yielding $K = \Theta(\tau^{-4/3} \kappa^{-6} (1-\gamma)^{-6})$.

    \item \textit{Tracking Bounding Coefficient ($v$):} The bounding coefficient is $v = \max\left(t_0^{2/3} Z_0, \frac{2K}{\zeta c_\eta^{2/3} \kappa}\right)$.
        Since $\zeta \kappa = \Theta(1)$ and $c_\eta^{2/3} = \Theta(\tau^{-2/3})$, the second argument scales as $\Theta(K \tau^{2/3}) = \Theta(\tau^{-2/3} \kappa^{-6} (1-\gamma)^{-6})$.
        As shown in the proof of Theorem~\ref{thm:average-iterate}, $Z_0 \le (2B_\theta)^2 = G_{ac}^2 = \Theta(\kappa^{-2}(1-\gamma)^{-2})$.
        Thus, the initial error term evaluates to $t_0^{2/3} Z_0 = \mathcal{O}(\tau^{-2/3} \kappa^{-2}(1-\gamma)^{-2})$.
        This is dominated by the second argument, yielding $v = \Theta(\tau^{-2/3} \kappa^{-6} (1-\gamma)^{-6})$. Multiplying by $\frac{1}{\tau}$ gives the tracking multiplier $\frac{v}{\tau} = \Theta(\tau^{-5/3} \kappa^{-6} (1-\gamma)^{-6})$.

    \item \textit{Critic Tracking Error:} Because $T$ appears inside the cube root, we apply the subadditivity inequality $(1 + t_0/T)^{1/3} \le 1 + (t_0/T)^{1/3}$.
        Since $t_0 = \Theta(\tau^{-1})$, the root evaluates to $(t_0/T)^{1/3} = \Theta(\tau^{-1/3} T^{-1/3})$.
        Multiplying this by the leading tracking multiplier decomposes the tracking error into two asymptotic upper bounds:
        \begin{align*}
        &\quad \mathcal{O}(\tau^{-5/3} \kappa^{-6} (1-\gamma)^{-6}) T^{-2/3} \left( 1 + \Theta(\tau^{-1/3} T^{-1/3}) \right) \\
        & = \mathcal{O}(\tau^{-5/3} \kappa^{-6} (1-\gamma)^{-6} T^{-2/3}) + \mathcal{O}(\tau^{-2} \kappa^{-6} (1-\gamma)^{-6} T^{-1}).
        \end{align*}

    \item \textit{PMD Initialization and Local Errors:} By assumption $\|\omega_0\|_2 \le \frac{B_\theta}{\tau_{max}}$, the initial policy log-probability is bounded by $-\log \pi_0(a|s) = -\omega_0^\top \phi(s,a) + \log \sum_{a'} \exp(\omega_0^\top \phi(s,a')) \le 2\|\omega_0\|_2 B_\phi + \log|\mathcal{A}| \le \frac{2B_\theta B_\phi}{\tau_{max}} + \log|\mathcal{A}|$. Thus, the initial KL divergence is bounded by $D_0^\dag \le \frac{2B_\theta B_\phi}{\tau_{max}} + \log|\mathcal{A}| = \mathcal{O}(\kappa^{-1}(1-\gamma)^{-1})$.
    Because $\frac{t_0-1}{c_\eta} = \Theta(1)$, the initialization error evaluates to $\frac{t_0-1}{c_\eta T}D_0^\dag = \mathcal{O}(\kappa^{-1}(1-\gamma)^{-1}T^{-1})$.
    The PMD local error evaluates to $\frac{c_\eta G_{ac}^2 B_\phi^2}{2T}\log T = \mathcal{O}(\tau^{-1}\kappa^{-2}(1-\gamma)^{-2}T^{-1}\log T)$.
    Both terms are subsumed by the $\mathcal{O}(\tau^{-2}\kappa^{-6}(1-\gamma)^{-6}T^{-1})$ transient critic tracking error (second term above), up to a $\log T$ factor.
\end{itemize}
Dividing the entire expression by $(1-\gamma)$ distributes the $(1-\gamma)^{-1}$ factor across all asymptotic terms. This elevates the $(1-\gamma)^{-6}$ dependencies to $(1-\gamma)^{-7}$, yielding the final stated expected gap bound and completing the proof.
\end{prf}

\begin{prf}[Proof of Lemma~\ref{lem:fractional-chung-bias} (Chung's Lemma with Fractional Noise)]
Let $s = t+t_0$. Since $c \ge 1$ and $t_0 \ge c$, we have $s \ge c \ge 1$, guaranteeing the contraction multiplier $(1 - c/s)$ is non-negative. We proceed by induction to prove $x_t \le \frac{u}{s^{2/3}} + \frac{B}{c}$.

The base case at $t=0$ holds trivially by the definition of $u$: $x_0 \le \frac{u}{t_0^{2/3}} \le \frac{u}{t_0^{2/3}} + \frac{B}{c}$. Assume the bound holds at step $t$. Substituting the inductive hypothesis into the recurrence yields:
\begin{align*}
x_{t+1} &\le \left(1 - \frac{c}{s}\right) \left(\frac{u}{s^{2/3}} + \frac{B}{c}\right) + \frac{A}{s^{5/3}} + \frac{B}{s} \\
&= \frac{u}{s^{2/3}} - \frac{cu - A}{s^{5/3}} + \frac{B}{c} - \frac{B}{s} + \frac{B}{s} \\
&= \frac{u}{s^{2/3}} - \frac{cu - A}{s^{5/3}} + \frac{B}{c}.
\end{align*}
By the convexity of $f(s) = s^{-2/3}$, the function is lower-bounded by its first-order Taylor expansion: $\frac{1}{(s+1)^{2/3}} \ge \frac{1}{s^{2/3}} - \frac{2}{3s^{5/3}}$. Therefore, to guarantee $x_{t+1} \le \frac{u}{(s+1)^{2/3}} + \frac{B}{c}$, it is sufficient to show:
$$ \frac{cu - A}{s^{5/3}} \ge \frac{2u}{3s^{5/3}} \Longleftrightarrow cu - A \ge \frac{2}{3}u \Longleftrightarrow \left(c - \frac{2}{3}\right)u \ge A. $$
Because $c \ge 1$, we are guaranteed $c - 2/3 \ge 1/3 > 0$. This sufficient condition matches our definition $u \ge \frac{A}{c - 2/3} = \frac{3A}{3c-2}$, completing the induction.
\end{prf}

\begin{prf}[Proof of Lemma~\ref{lem:last-iterate-dt} (Last-Iterate Forward KL Convergence)]
(i) \textit{Pointwise Recurrence.} We isolate the forward KL divergence from the single-step PMD Actor Progress Bound (Lemma~\ref{lem:pmd-telescoping-l2}) by dropping the non-negative $(1-\gamma)\mathbb{E}[\text{Gap}_t^\dag]$ term:
$$ \mathbb{E}[D_{t+1}^\dag] \le \left(1 - \frac{3\tau}{4}\eta_t\right) \mathbb{E}[D_t^\dag] + \eta_t \mathcal{E}_{noise}(t). $$
Let $s = t+t_0$. Substituting $\eta_t = c_\eta s^{-1}$ and the pointwise tracking bound $Z_t \le v s^{-2/3}$ from Lemma~\ref{lem:decoupled-critic}, the additive noise evaluates to:
$$ \eta_t \mathcal{E}_{noise}(t) = \frac{c_\eta^2 G_{ac}^2 B_\phi^2}{2 s^2} + \frac{4 c_\eta B_\phi^2 v}{\tau s^{5/3}} + \frac{4 c_\eta C_{joint}^\dag}{\tau s} \epsilon_{app}. $$

(ii) \textit{Applying Chung's Lemma with Fractional Noise.} For $s \ge 1$, we upper-bound the $s^{-2}$ PMD local error term by $s^{-5/3}$.
Substituting the definitions $A \triangleq \frac{c_\eta^2 G_{ac}^2 B_\phi^2}{2} + \frac{4 c_\eta B_\phi^2 v}{\tau}$ and $B \triangleq \frac{4 c_\eta C_{joint}^\dag}{\tau} \epsilon_{app}$, the recurrence simplifies to:
$$ \mathbb{E}[D_{t+1}^\dag] \le \left(1 - \frac{c}{s}\right) \mathbb{E}[D_t^\dag] + \frac{A}{s^{5/3}} + \frac{B}{s}, $$
where $c \triangleq \frac{3\tau}{4}c_\eta$.
The step-size condition $c_\eta \ge \frac{4}{3\tau}$ guarantees $c \ge 1$. Additionally, the initial offset condition (Theorem~\ref{thm:reg-conv-l2}) requires $t_0 \ge c_\eta \tau$, which ensures $t_0 \ge \frac{4}{3} c > c$ for Lemma~\ref{lem:fractional-chung-bias}.
We apply Chung's Lemma with Fractional Noise (Lemma~\ref{lem:fractional-chung-bias}) using the sequence $x_t = \mathbb{E}[D_t^\dag]$.
The bias fraction cancels the actor step-size constant: $\frac{B}{c} = \frac{4 c_\eta C_{joint}^\dag \epsilon_{app}}{\tau} \frac{4}{3\tau c_\eta} = \frac{16 C_{joint}^\dag}{3\tau^2} \epsilon_{app}$.
This establishes the pointwise bound $$\mathbb{E}[D_T^\dag] \le \frac{u}{(T+t_0)^{2/3}} + \frac{16 C_{joint}^\dag}{3\tau^2} \epsilon_{app},$$
where $u \triangleq \max\left(t_0^{2/3} D_0^\dag, \frac{3A}{3c-2}\right)$.

(iii) \textit{Asymptotic Order.} To evaluate the structural orders of the bounding constant $u$, we explicitly track the dependencies on the temperature $\tau$, the feature moment minimum eigenvalue $\kappa$, and the horizon $(1-\gamma)^{-1}$. As established in the proof of Theorem~\ref{thm:reg-conv-l2}, we satisfy the required conditions on the step sizes by choosing $c_\eta = \Theta(\tau^{-1})$, $\zeta = \Theta(\kappa^{-1})$, and $t_0 = \Theta(\tau^{-1})$. Under these choices, the tracking coefficient scales as $v = \Theta(\tau^{-2/3} \kappa^{-6} (1-\gamma)^{-6})$.

\begin{itemize}
    \item \textit{Initialization Term:} By the initial KL divergence bound $D_0^\dag=\mathcal{O}(\kappa^{-1}(1-\gamma)^{-1})$ established in the proof of Theorem~\ref{thm:reg-conv-l2} and $t_0=\Theta(\tau^{-1})$, the initialization term evaluates to $t_0^{2/3}D_0^\dag=\mathcal{O}(\tau^{-2/3}\kappa^{-1}(1-\gamma)^{-1})$.

    \item \textit{Fractional Noise Coefficient ($A$):} For the numerator $A = \frac{c_\eta^2 G_{ac}^2 B_\phi^2}{2} + \frac{4 c_\eta B_\phi^2 v}{\tau}$, we evaluate the two components.
        By Lemma~\ref{lem:structural-bounds}, the actor update bound scales as $G_{ac}^2 = \Theta(\kappa^{-2}(1-\gamma)^{-2})$.
        Thus, the PMD local error evaluates to $c_\eta^2 G_{ac}^2 = \Theta(\tau^{-2}) \Theta(\kappa^{-2}(1-\gamma)^{-2}) = \Theta(\tau^{-2} \kappa^{-2} (1-\gamma)^{-2})$.
        The critic tracking penalty evaluates to $\frac{c_\eta v}{\tau} = \frac{\Theta(\tau^{-1}) \Theta(\tau^{-2/3} \kappa^{-6} (1-\gamma)^{-6})}{\tau} = \Theta(\tau^{-8/3} \kappa^{-6} (1-\gamma)^{-6})$.
        Because $\tau^{-8/3}$ dominates $\tau^{-2}$ as $\tau \to 0$, the critic tracking term dominates, yielding $A = \Theta(\tau^{-8/3} \kappa^{-6} (1-\gamma)^{-6})$.

    \item \textit{Bounding Constant ($u$):} The choice $c_\eta = \Theta(\tau^{-1})$ guarantees $c \triangleq \frac{3\tau}{4}c_\eta = \Theta(1)$, meaning the denominator $3c-2 = \Theta(1)$.
        Therefore, $\frac{3A}{3c-2}=\Theta(\tau^{-8/3}\kappa^{-6}(1-\gamma)^{-6})$. Taking the maximum with the initialization term $t_0^{2/3}D_0^\dag=\mathcal{O}(\tau^{-2/3}\kappa^{-1}(1-\gamma)^{-1})$ yields $u=\Theta(\tau^{-8/3}\kappa^{-6}(1-\gamma)^{-6})$.
\end{itemize}
Substituting $u$ into the pointwise bound yields the stated asymptotic convergence rate and completes the proof.
\end{prf}

\begin{prf}[Proof of Lemma~\ref{lem:pinsker-gap} (Forward KL Bound on Regularized Suboptimality)]
(i) \textit{Extracting the KL Divergence.} By the first step of the Regularized Suboptimality Decompositions (Lemma~\ref{lem:sub-decomp}) evaluated for the training policy $\pi_t$ against the optimal comparator $\pi^*_\tau$, the exact performance difference is:
$$(1-\gamma) \big( J_\tau(\pi^*_\tau) - J_\tau(\pi_t) \big) = -\tau \mathbb{E}_{d^*_\tau}[D_{KL}(\pi^*_\tau \| \pi_t)] + \mathbb{E}_{d^*_\tau} \Big[ \langle Q_\tau^{\pi_t} - \tau \log \pi_t, \pi^*_\tau - \pi_t \rangle_{\mathcal{A}} \Big].$$

(ii) \textit{H\"{o}lder's and Pinsker's Bounds.} Because the KL divergence is non-negative, dropping $-\tau D_{KL}(\pi^*_\tau \| \pi_t)$ establishes an upper bound:
$$ J_\tau(\pi^*_\tau) - J_\tau(\pi_t) \le \frac{1}{1-\gamma} \mathbb{E}_{d^*_\tau} \Big[ \langle Q_\tau^{\pi_t} - \tau \log \pi_t, \pi^*_\tau - \pi_t \rangle_{\mathcal{A}} \Big]. $$
Applying H\"{o}lder's inequality ($L_1 \times L_\infty$) to the summation and then applying Pinsker's inequality ($\|\pi^*_\tau - \pi_t\|_1 \le \sqrt{2 D_{KL}(\pi^*_\tau \| \pi_t)}$) yields:
\begin{align*}
J_\tau(\pi^*_\tau) - J_\tau(\pi_t) & \le \frac{1}{1-\gamma} \mathbb{E}_{d^*_\tau} \left[ \|\pi^*_\tau(a|s) - \pi_t(a|s)\|_1 \big|Q_\tau^{\pi_t}(s,a) - \tau \log \pi_t(a|s)\big| \right] \\
& \le \frac{\sqrt{2}}{1-\gamma} \mathbb{E}_{d^*_\tau} \Big[ \max_a \big|Q_\tau^{\pi_t}(s,a) - \tau \log \pi_t(a|s)\big| \sqrt{D_{KL}(\pi^*_\tau(\cdot|s) \| \pi_t(\cdot|s))} \Big].
\end{align*}

(iii) \textit{Bounding $Q$-value and Log-Policy via Parameter Capacity.} We bound the magnitude of the $Q$-function and log-policy term for all actions. By Lemma~\ref{lem:bounded-values}, the regularized Q-value is bounded by $0 \le Q^{\pi_t}_\tau(s,a) \le V_{max}$.
To bound the log-policy magnitude $|\tau \log \pi_t(a|s)|$, we substitute the log-linear identity $\tau \log \pi_t(a|s) = \tau \omega_t^\top \phi(s,a) - \tau \log Z_t(s)$.
By Lemma~\ref{lem:structural-bounds}(ii), the actor iterate is bounded by $\|\omega_t\|_2 \le B_\theta / \tau$. Thus, the feature inner product is bounded by:
$$ |\tau \omega_t^\top \phi(s,a)| \le \tau \|\omega_t\|_2 \|\phi(s,a)\|_2 \le \tau \left(\frac{B_\theta}{\tau}\right) B_\phi = B_\theta B_\phi. $$
The log-partition function is then bounded by:
$$ \tau \log Z_t(s) = \tau \log \sum_{a'} \exp(\omega_t^\top \phi(s,a')) \le \tau \log \left( |\mathcal{A}| \exp\left(\frac{B_\theta B_\phi}{\tau}\right) \right) = B_\theta B_\phi + \tau \log |\mathcal{A}|. $$
Similarly, $\tau \log Z_t(s) \ge \tau \log \left( \exp\left(-\frac{B_\theta B_\phi}{\tau}\right) \right) = -B_\theta B_\phi$. Therefore, $|\tau \log Z_t(s)| \le B_\theta B_\phi + \tau \log |\mathcal{A}|$.
By the triangle inequality, the log-policy magnitude is bounded by $|\tau \log \pi_t(a|s)| \le 2 B_\theta B_\phi + \tau \log |\mathcal{A}|$.
Consequently, the magnitude of the $Q$-function and log-policy term is bounded by:
$$ \big|Q_\tau^{\pi_t}(s,a) - \tau \log \pi_t(a|s)\big| \le V_{max} + 2 B_\theta B_\phi + \tau \log |\mathcal{A}| \triangleq M_\tau. $$

(iv) \textit{Jensen's Inequality and $M_{max}$ Relaxation.} Applying Jensen's inequality
$ \mathbb{E}_{d^*_\tau} [\sqrt{D_{KL}(\pi^*_\tau(\cdot|s) \| \pi_t(\cdot|s))} ] \le \sqrt{\mathbb{E}_{d^*_\tau} [D_{KL}(\pi^*_\tau(\cdot|s) \| \pi_t(\cdot|s)) ] } $
establishes the $\tau$-dependent bound. Relaxing $M_\tau \le M_{max}$ via $\tau \le \tau_{max}$ yields the $\tau$-independent bound and completes the proof.
\end{prf}

\begin{prf}[Proof of Theorem~\ref{thm:reg-last-iterate-gap} (Last-Iterate Regularized Convergence)]
(i) \textit{Forward KL Bound.} By the Forward KL Bound on Regularized Suboptimality (Lemma~\ref{lem:pinsker-gap}), for each realized policy $\pi_T$ we have
$$ \text{Gap}_T^\dag \le \frac{\sqrt{2} M_{max}}{1-\gamma} \sqrt{D_T^\dag}. $$
Taking the expectation over the algorithmic trajectory and applying Jensen's inequality yields:
$$ \mathbb{E}[\text{Gap}_T^\dag] \le \frac{\sqrt{2} M_{max}}{1-\gamma} \mathbb{E}\big[\sqrt{D_T^\dag}\big] \le \frac{\sqrt{2} M_{max}}{1-\gamma} \sqrt{\mathbb{E}[D_T^\dag]}. $$
Substituting the explicit bound for the terminal divergence from Lemma~\ref{lem:last-iterate-dt} yields:
$$ \mathbb{E}[\text{Gap}_T^\dag] \le \frac{\sqrt{2} M_{max}}{1-\gamma} \sqrt{ \frac{u}{(T+t_0)^{2/3}} + \frac{16 C_{joint}^\dag}{3\tau^2} \epsilon_{app} }. $$
Applying square-root subadditivity ($\sqrt{x+y} \le \sqrt{x} + \sqrt{y}$) separates the bound into two terms:
$$ \mathbb{E}[\text{Gap}_T^\dag] \le \frac{\sqrt{2 u} M_{max}}{1-\gamma} \frac{1}{(T+t_0)^{1/3}} + \frac{4 M_{max} \sqrt{2 C_{joint}^\dag}}{\sqrt{3} \tau (1-\gamma)} \sqrt{\epsilon_{app}}. $$

(ii) \textit{Asymptotic Order.} As established in Lemma~\ref{lem:last-iterate-dt}, the fractional noise bounding constant scales as $u = \Theta(\tau^{-8/3} \kappa^{-6} (1-\gamma)^{-6})$. Consequently, its square root guarantees $\sqrt{u} = \Theta(\tau^{-4/3} \kappa^{-3} (1-\gamma)^{-3})$.
Recall from Lemma~\ref{lem:pinsker-gap} that
$ M_{max} = V_{max} + 2 B_\theta B_\phi + \tau_{max} \log |\mathcal{A}|$.
By Lemma~\ref{lem:structural-bounds}, $B_\theta = \frac{B_\phi V_{max}}{\kappa}$, while $V_{max} = \Theta((1-\gamma)^{-1})$. Therefore,
$ M_{max} = \Theta ((1-\gamma)^{-1}\kappa^{-1}) $.
Hence, the leading coefficients for the finite-time term and approximation bias evaluate to:
$$ \frac{\sqrt{u} M_{max}}{1-\gamma} = \Theta\left( \frac{1}{(1-\gamma)^5 \kappa^4 \tau^{4/3}} \right), \quad \text{and} \quad \frac{M_{max}}{\tau (1-\gamma)} = \Theta\left( \frac{1}{(1-\gamma)^2 \kappa \tau} \right). $$
Because $t_0 \ge 1$, we bound $(T+t_0)^{-1/3} \le T^{-1/3}$. Substituting these structural orders yields the stated asymptotic bounds, completing the proof.
\end{prf}

\subsection{Proofs for Section~\ref{sec:unregularized-conv-deterministic}}

\begin{prf}[Proof of Theorem~\ref{thm:PMD-unreg-average} (Unregularized Average-Iterate Convergence)]
We first consider $\epsilon_{app}>0$. When $\epsilon_{app}=0$, the approximation bias term vanishes and the result follows from the sample-limited calculation with $\tau_T=\frac{\Delta}{2\log(C_\gamma T^{2/3})}$.

(i) \textit{Global Unregularized Bound.}
Because the finite-time schedule $\tau_T$ may initially exceed the temperature threshold required in Theorem~\ref{thm:suboptimal-mass}, we apply the Universal Unregularized Suboptimality Bound (Corollary~\ref{cor:universal-suboptimal-mass}).
This guarantees the following global inequality holds for any $\tau_T > 0$:
$$ \frac{1}{T} \sum_{t=0}^{T-1} \mathbb{E} \big[ J_0(\pi^*_0) - J_0(\pi_t) \big] \le \frac{4 A_{max}}{\Delta} \left( \frac{1}{T} \sum_{t=0}^{T-1} \mathbb{E} [\text{Gap}_t^\dag(\tau_T)] \right) + C_{tail} \exp\left(-\frac{\Delta}{2\tau_T}\right). $$
Substituting the explicit average regularized gap bound from Theorem~\ref{thm:reg-conv-l2} isolates the competing algorithmic components:
\begin{align*}
\text{Total Error} \le &\underbrace{ \mathcal{O}\left( \frac{A_{max}}{\Delta (1-\gamma)^7 \kappa^6 \tau_T^{5/3}} T^{-2/3} \right) }_{\text{Primary Tracking Error}} + \underbrace{ \mathcal{O}\left( \frac{A_{max}}{\Delta (1-\gamma) \tau_T} \epsilon_{app} \right) }_{\text{Approximation Bias}} \\
&+ \underbrace{ C_{tail} \exp\left(-\frac{\Delta}{2\tau_T}\right) }_{\text{Entropy Tail}} + \underbrace{ \tilde{\mathcal{O}}\left( \frac{A_{max}}{\Delta (1-\gamma)^7 \kappa^6 \tau_T^2} T^{-1} \right) }_{\text{Transient Convergence Error}}.
\end{align*}

(ii) \textit{Regime 1: Sample-Limited Phase.}
When $\log(T^{2/3}) \le \log(1+\epsilon_{app}^{-1})$, the temperature is tuned to $\tau_T = \frac{\Delta}{2 \log(C_\gamma T^{2/3})}$.
Substituting this alongside $A_{max} \le \frac{1}{1-\gamma}$, $C_{tail} \triangleq \frac{A_{max} C_\gamma}{1-\gamma} \le \frac{C_\gamma}{(1-\gamma)^2}$, and $\log C_\gamma = \mathcal{O}(\frac{1}{1-\gamma})$ into the global inequality yields:
\begin{align*}
\text{Primary Tracking Error}
&= \mathcal{O}\Big( \frac{1}{(1-\gamma)^8 \kappa^6 \Delta^{8/3}} (\log(C_\gamma T^{2/3}))^{5/3} T^{-2/3} \Big) \\
&= \mathcal{O}\Big( \frac{1}{(1-\gamma)^8 \kappa^6 \Delta^{8/3}} \Big( \frac{1}{(1-\gamma)^{5/3}} + (\log T)^{5/3} \Big) T^{-2/3} \Big), \\
\text{Approximation Bias}
&= \mathcal{O}\Big( \frac{1}{(1-\gamma)^2 \Delta^2} \log(C_\gamma T^{2/3}) \epsilon_{app} \Big) \\
&\le \mathcal{O}\Big( \frac{1}{(1-\gamma)^2 \Delta^2} \Big( \frac{1}{1-\gamma} + \log(1+\epsilon_{app}^{-1}) \Big) \epsilon_{app} \Big), \\
\text{Entropy Tail}
&\le \frac{C_\gamma}{(1-\gamma)^2} \exp\big(-\log(C_\gamma T^{2/3})\big) = \frac{1}{(1-\gamma)^2 T^{2/3}}, \\
\text{Transient Error}
&= \tilde{\mathcal{O}}\Big( \frac{1}{(1-\gamma)^8 \kappa^6 \Delta^3} (\log(C_\gamma T^{2/3}))^2 T^{-1} \Big).
\end{align*}
The regime condition $\log(T^{2/3}) \le \log(1+\epsilon_{app}^{-1})$ guarantees
$$ \log(C_\gamma T^{2/3}) \le \log C_\gamma + \log(1+\epsilon_{app}^{-1}), $$
which gives the stated approximation-bias envelope.
The $C_\gamma$ prefactor cancels in the exponential tail, yielding a pure $\mathcal{O}(T^{-2/3})$ tail that is absorbed by the primary tracking error.
The transient convergence error (which encapsulates the transient tracking error and PMD initialization and local errors) decays as $\tilde{\mathcal{O}}(T^{-1})$, and is asymptotically dominated by the $\tilde{\mathcal{O}}(T^{-2/3})$ primary tracking error.

(iii) \textit{Regime 2: Approximation-Limited Phase.}
When $\log(T^{2/3}) > \log(1+\epsilon_{app}^{-1})$, the temperature locks to the approximation-dependent floor:
$$ \tau_T = \frac{\Delta}{2 \log(C_\gamma(1+\epsilon_{app}^{-1}))}. $$
Substituting this fixed temperature into the global inequality yields:
\begin{align*}
\text{Primary Tracking Error}
&= \mathcal{O}\Big( \frac{1}{(1-\gamma)^8 \kappa^6 \Delta^{8/3}} \big(\log(C_\gamma(1+\epsilon_{app}^{-1}))\big)^{5/3} T^{-2/3} \Big) \\
&\le \mathcal{O}\Big( \frac{1}{(1-\gamma)^8 \kappa^6 \Delta^{8/3}} \Big( \frac{1}{(1-\gamma)^{5/3}} + (\log T)^{5/3} \Big) T^{-2/3} \Big), \\
\text{Approximation Bias}
&= \mathcal{O}\Big( \frac{1}{(1-\gamma)^2 \Delta^2} \log(C_\gamma(1+\epsilon_{app}^{-1})) \epsilon_{app} \Big) \\
&= \mathcal{O}\Big( \frac{1}{(1-\gamma)^2 \Delta^2} \Big( \frac{1}{1-\gamma} + \log(1+\epsilon_{app}^{-1}) \Big) \epsilon_{app} \Big), \\
\text{Entropy Tail}
&\le \frac{C_\gamma}{(1-\gamma)^2} \exp\big(-\log(C_\gamma(1+\epsilon_{app}^{-1}))\big)
= \frac{\epsilon_{app}/(1+\epsilon_{app})}{(1-\gamma)^2}
\le \frac{\epsilon_{app}}{(1-\gamma)^2}, \\
\text{Transient Error}
&= \tilde{\mathcal{O}}\Big( \frac{1}{(1-\gamma)^8 \kappa^6 \Delta^3} \big(\log(C_\gamma(1+\epsilon_{app}^{-1}))\big)^2 T^{-1} \Big) \\
&\le \tilde{\mathcal{O}}\Big( \frac{1}{(1-\gamma)^8 \kappa^6 \Delta^3} (\log(C_\gamma T^{2/3}))^2 T^{-1} \Big).
\end{align*}
For the finite-time terms, the regime condition $\log(1+\epsilon_{app}^{-1}) < \log(T^{2/3})$ implies
$$ \log(C_\gamma(1+\epsilon_{app}^{-1})) < \log(C_\gamma) + \log( T^{2/3}), $$
which yields the stated finite-time envelopes.
The fixed temperature floor prevents the $\tau_T^{-1}$ approximation bias from growing with $T$, preserving the $\tilde{\mathcal{O}}(\epsilon_{app})$ rate.
Simultaneously, the exponential entropy tail is bounded by the linear scale of the approximation error, yielding $\mathcal{O}(\epsilon_{app})$.
The transient convergence error remains asymptotically dominated by the primary tracking error.

(iv) \textit{Synthesis.}
Summing the evaluated components across the two regimes bounds the expected global unregularized performance gap by their respective envelopes, yielding:
\begin{align*}
\frac{1}{T} \sum_{t=0}^{T-1} \mathbb{E} \big[ J_0(\pi^*_0) - J_0(\pi_t) \big]
&\le \mathcal{O}\left( \frac{(1-\gamma)^{-5/3} + (\log T)^{5/3}}{(1-\gamma)^8 \kappa^6 \Delta^{8/3} T^{2/3}} \right) \\
&\quad + \mathcal{O}\left( \frac{(1-\gamma)^{-1} + \log(1+\epsilon_{app}^{-1})}{(1-\gamma)^2 \Delta^2} \epsilon_{app} \right).
\end{align*}
Finally, as $\epsilon_{app}\to0$, $\log(1+\epsilon_{app}^{-1})=\Theta(\log(1/\epsilon_{app}))$ and $\epsilon_{app}\log(1+\epsilon_{app}^{-1})=\tilde{\mathcal{O}}(\epsilon_{app})$. Therefore, the overall convergence rate can be summarized as $\tilde{\mathcal{O}}(T^{-2/3})+\tilde{\mathcal{O}}(\epsilon_{app})$, completing the proof.
\end{prf}

\begin{prf}[Proof of Theorem~\ref{thm:PMD-unreg-last} (Unregularized Last-Iterate Convergence)]
We first consider $\epsilon_{app}>0$. When $\epsilon_{app}=0$, the approximation bias term vanishes and the result follows from the sample-limited calculation with $\tau_T=\frac{\Delta}{2\log(C_\gamma T^{2/3})}$.

(i) \textit{Global Unregularized Bound.}
Because the finite-time schedule $\tau_T$ may initially exceed the temperature threshold required in Theorem~\ref{thm:suboptimal-mass}, we apply the Universal Unregularized Suboptimality Bound (Corollary~\ref{cor:universal-suboptimal-mass}) directly to the terminal policy $\pi_T$.
This guarantees the following global inequality holds for any $\tau_T > 0$:
$$ \mathbb{E} \big[ J_0(\pi^*_0) - J_0(\pi_T) \big] \le \frac{4 A_{max}}{\Delta} \mathbb{E}[\text{Gap}_T^\dag(\tau_T)] + C_{tail} \exp\left(-\frac{\Delta}{2\tau_T}\right). $$
Substituting the explicit regularized last-iterate bound from Theorem~\ref{thm:reg-last-iterate-gap} isolates the three competing algorithmic components:
\begin{align*}
\text{Total Error} \le &\underbrace{ \mathcal{O}\left( \frac{A_{max}}{\Delta(1-\gamma)^5 \kappa^4 \tau_T^{4/3}} T^{-1/3} \right) }_{\text{Primary Optimization Error}} + \underbrace{ \mathcal{O}\left( \frac{A_{max}}{\Delta(1-\gamma)^2 \kappa \tau_T} \sqrt{\epsilon_{app}} \right) }_{\text{Approximation Bias}} \\
&+ \underbrace{ C_{tail} \exp\left(-\frac{\Delta}{2\tau_T}\right) }_{\text{Entropy Tail}}.
\end{align*}

(ii) \textit{Regime 1: Sample-Limited Phase.}
When $\log(T^{2/3}) \le \log(1+\epsilon_{app}^{-1})$, the temperature is tuned to $\tau_T = \frac{\Delta}{2 \log(C_\gamma T^{2/3})}$.
Substituting this alongside $A_{max} \le \frac{1}{1-\gamma}$, $C_{tail} \triangleq \frac{A_{max} C_\gamma}{1-\gamma} \le \frac{C_\gamma}{(1-\gamma)^2}$, and $\log C_\gamma = \mathcal{O}(\frac{1}{1-\gamma})$ into the global inequality yields:
\begin{align*}
\text{Primary Optimization Error}
&= \mathcal{O}\Big( \frac{1}{(1-\gamma)^6 \kappa^4 \Delta^{7/3}} (\log(C_\gamma T^{2/3}))^{4/3} T^{-1/3} \Big) \\
&= \mathcal{O}\Big( \frac{1}{(1-\gamma)^6 \kappa^4 \Delta^{7/3}} \Big( \frac{1}{(1-\gamma)^{4/3}} + (\log T)^{4/3} \Big) T^{-1/3} \Big), \\
\text{Approximation Bias}
&= \mathcal{O}\Big( \frac{1}{(1-\gamma)^3 \kappa \Delta^2} \log(C_\gamma T^{2/3}) \sqrt{\epsilon_{app}} \Big) \\
&\le \mathcal{O}\Big( \frac{1}{(1-\gamma)^3 \kappa \Delta^2} \Big( \frac{1}{1-\gamma} + \log(1+\epsilon_{app}^{-1}) \Big) \sqrt{\epsilon_{app}} \Big), \\
\text{Entropy Tail}
&\le \frac{C_\gamma}{(1-\gamma)^2} \exp\big(-\log(C_\gamma T^{2/3})\big) = \frac{1}{(1-\gamma)^2 T^{2/3}}.
\end{align*}
The regime condition $\log(T^{2/3}) \le \log(1+\epsilon_{app}^{-1})$ guarantees
$$ \log(C_\gamma T^{2/3}) \le \log C_\gamma + \log(1+\epsilon_{app}^{-1}), $$
which gives the stated approximation-bias envelope.
The $C_\gamma$ prefactor cancels in the exponential tail, which scales as $\mathcal{O}(T^{-2/3})$ and is asymptotically dominated by the $\tilde{\mathcal{O}}(T^{-1/3})$ primary optimization error.

(iii) \textit{Regime 2: Approximation-Limited Phase.}
When $\log(T^{2/3}) > \log(1+\epsilon_{app}^{-1})$, the temperature locks to the approximation-dependent floor:
$$ \tau_T = \frac{\Delta}{2 \log(C_\gamma(1+\epsilon_{app}^{-1}))}. $$
Substituting this fixed temperature into the global inequality yields:
\begin{align*}
\text{Primary Optimization Error}
&= \mathcal{O}\Big( \frac{1}{(1-\gamma)^6 \kappa^4 \Delta^{7/3}} \big(\log(C_\gamma(1+\epsilon_{app}^{-1}))\big)^{4/3} T^{-1/3} \Big) \\
&\le \mathcal{O}\Big( \frac{1}{(1-\gamma)^6 \kappa^4 \Delta^{7/3}} \Big( \frac{1}{(1-\gamma)^{4/3}} + (\log T)^{4/3} \Big) T^{-1/3} \Big), \\
\text{Approximation Bias}
&= \mathcal{O}\Big( \frac{1}{(1-\gamma)^3 \kappa \Delta^2} \log(C_\gamma(1+\epsilon_{app}^{-1})) \sqrt{\epsilon_{app}} \Big) \\
&= \mathcal{O}\Big( \frac{1}{(1-\gamma)^3 \kappa \Delta^2} \Big( \frac{1}{1-\gamma} + \log(1+\epsilon_{app}^{-1}) \Big) \sqrt{\epsilon_{app}} \Big), \\
\text{Entropy Tail}
&\le \frac{C_\gamma}{(1-\gamma)^2} \exp\big(-\log(C_\gamma(1+\epsilon_{app}^{-1}))\big)
= \frac{\epsilon_{app}/(1+\epsilon_{app})}{(1-\gamma)^2}
\le \frac{\epsilon_{app}}{(1-\gamma)^2}.
\end{align*}
For the finite-time error, the regime condition $\log(1+\epsilon_{app}^{-1}) < \log(T^{2/3})$ implies
$$ \log(C_\gamma(1+\epsilon_{app}^{-1})) < \log(C_\gamma) + \log( T^{2/3}), $$
which yields the stated finite-time envelope.
The fixed temperature floor prevents the $\tau_T^{-1}$ factor from growing with $T$, preserving the $\tilde{\mathcal{O}}(\sqrt{\epsilon_{app}})$ approximation-bias rate.
Simultaneously, the exponential entropy tail is bounded by $\mathcal{O}(\epsilon_{app})$ and is therefore asymptotically subsumed by the slower $\tilde{\mathcal{O}}(\sqrt{\epsilon_{app}})$ approximation bias.

(iv) \textit{Synthesis.}
Summing these evaluated components across the two regimes bounds the expected last-iterate global unregularized performance gap by their respective envelopes, yielding:
\begin{align*}
\mathbb{E} \big[ J_0(\pi^*_0) - J_0(\pi_T) \big]
&\le \mathcal{O}\left( \frac{(1-\gamma)^{-4/3} + (\log T)^{4/3}}{(1-\gamma)^6 \kappa^4 \Delta^{7/3} T^{1/3}} \right) \\
&\quad + \mathcal{O}\left( \frac{(1-\gamma)^{-1} + \log(1+\epsilon_{app}^{-1})}{(1-\gamma)^3 \kappa \Delta^2} \sqrt{\epsilon_{app}} \right).
\end{align*}
Finally, as $\epsilon_{app}\to0$, $\log(1+\epsilon_{app}^{-1})=\Theta(\log(1/\epsilon_{app}))$ and $\sqrt{\epsilon_{app}}\log(1+\epsilon_{app}^{-1})=\tilde{\mathcal{O}}(\sqrt{\epsilon_{app}})$. Therefore, the overall convergence rate can be summarized as $\tilde{\mathcal{O}}(T^{-1/3})+\tilde{\mathcal{O}}(\sqrt{\epsilon_{app}})$, completing the proof.
\end{prf}

\begin{prf}[Proof of Theorem~\ref{thm:tabular-convergence} (Unregularized Convergence in Tabular Setting)]
To apply Corollary~\ref{cor:linfty-nu-sampling}, we systematically verify the required structural and trajectory-dependent assumptions for the tabular setting with one-hot features.

(i) \textit{Verifying Assumption~\ref{ass:bounded-features} (Bounded Features):} Because $\phi(s,a)$ is the one-hot basis vector, with Euclidean norm $1$ for all $(s,a)$, Assumption~\ref{ass:bounded-features} holds trivially with $B_\phi = 1$.

(ii) \textit{Verifying Assumption~\ref{ass:bounded-q} (Bounded Q-value Second Moment):}
We construct the unbiased estimator $\hat{Q}_t$ using a geometrically stopped Monte Carlo rollout and evaluate the entropy contribution in its statewise form $H(\pi_t(\cdot|s))$. Let $H$ denote the random rollout horizon, with $\mathbb{P}(H\ge k)=\gamma^k$. Since the original reward is bounded in $[0,1]$ and $H(\pi_t(\cdot|s))\le\log|\mathcal{A}|$, every statewise regularized reward contribution is bounded by $1+\tau_{max}\log|\mathcal{A}|$. Hence,
$$ |\hat{Q}_t| \le \big(1+\tau_{max}\log|\mathcal{A}|\big)(H+1). $$
For the geometric stopping rule, $\mathbb{E}[(H+1)^2]=\frac{1+\gamma}{(1-\gamma)^2}\le\frac{2}{(1-\gamma)^2}$. Therefore,
$$ \mathbb{E}[\hat{Q}_t^2\mid\mathcal{F}_t,s_t,a_t] \le 2\left(\frac{1+\tau_{max}\log|\mathcal{A}|}{1-\gamma}\right)^2 = 2V_{max}^2. $$
Thus, Assumption~\ref{ass:bounded-q} holds with $Q_{max}=\sqrt{2}V_{max}=\mathcal{O}((1-\gamma)^{-1})$.

(iii) \textit{Verifying Assumption~\ref{ass:generalized-nu-exp} (Positive-Definite Moment under Exploratory Training Measure):}
Under the one-hot representation,
$$ \bar{\Sigma}_{unc}(d^{\pi_t}_\nu)
=\mathbb{E}_{d^{\pi_t}_\nu}\big[\phi(s,a)\phi(s,a)^\top\big]
=\operatorname{diag}\big(d^{\pi_t}_\nu(s,a)\big). $$
Because $d^{\pi_t}_\nu(s,a)\ge(1-\gamma)\nu(s,a)$, we have
$$ \bar{\Sigma}_{unc}(d^{\pi_t}_\nu)
\succeq (1-\gamma)\min_{s,a}\nu(s,a) I
=\kappa_{exp}I. $$
Thus, Assumption~\ref{ass:generalized-nu-exp} holds with $\kappa_{exp}=(1-\gamma)\min_{s,a}\nu(s,a)>0$.

(iv) \textit{Verifying Assumption~\ref{ass:linfty-error} ($L_\infty$ Critic Approximation Error):}
Because the one-hot features span the entire space of state-action functions, the regularized $Q$-function is represented exactly. For every training policy $\pi_t$, we have
for all $(s,a)$,
$$ \theta^*_\tau(\omega_t)_{(s,a)}=Q_\tau^{\pi_t}(s,a), $$
and hence
$ \epsilon_t(s,a)=Q_\tau^{\pi_t}(s,a)-\theta^*_\tau(\omega_t)^\top\phi(s,a)=0$.
Thus, Assumption~\ref{ass:linfty-error} holds with $\epsilon_\infty=0$.

(v) \textit{Synthesis:}
The preceding arguments verify Assumptions~\ref{ass:bounded-features}, \ref{ass:bounded-q}, \ref{ass:generalized-nu-exp}, and \ref{ass:linfty-error}, while Assumption~\ref{ass:action-gap} holds by hypothesis. Therefore, Corollary~\ref{cor:linfty-nu-sampling} applies. Since $\epsilon_\infty=0$, its two-stage last-iterate temperature reduces by convention to
$$ \tau_T^{last}=\frac{\Delta}{2\log(C_\gamma T^{2/3})}=\tau_T^{avg}.$$
Thus, the common temperature $\tau_T$ can be used as stated.
Substituting $\epsilon_\infty=0$ and $\kappa=\kappa_{exp}$ into Corollary~\ref{cor:linfty-nu-sampling} yields the stated convergence bounds, completing the proof.
\end{prf}

\section{Unregularized Convergence via Linear Entropy Penalty (Stochastic Regime)} \label{app:linear-entropy}

In this appendix, we provide the unregularized convergence analysis relying on the standard linear entropy penalty
for Algorithm~\ref{alg:uncentered-nac} in the Stochastic Regime. While these convergence rates ($\tilde{\mathcal{O}}(T^{-1/3})$) are slower than the optimal $\tilde{\mathcal{O}}(T^{-1})$ bounds established in the main text via the exponential translation bounds, this analysis remains theoretically valuable. Crucially, these bounds do not require the Minimal Action Gap condition (Assumption~\ref{ass:action-gap}). Therefore, in environments lacking a uniformly positive action margin---where our exponential translation bounds are unavailable---this linear entropy penalty framework provides a robust theoretical fallback to guarantee convergence for the parametric unregularized performance gap.

\subsection{Average-Iterate Unregularized Convergence}

We map the average-iterate regularized performance gap bound established in Section \ref{sec:pl-average} to the unregularized objective. We establish the unregularized convergence rate by tuning the fixed temperature $\tau$ to optimally balance the optimization error against the induced entropy penalty.

\begin{thm}[Average-Iterate Unregularized Convergence] \label{thm:unreg-average-iterate}
For Algorithm~\ref{alg:uncentered-nac}, suppose that Assumptions \ref{ass:bounded-features}--\ref{ass:approximation-error} and \ref{ass:pd-fim}--\ref{ass:parametric-density} hold.
Define the two-stage temperature:
$$ \tau_T \triangleq \max\left( \left( \frac{\log T}{(1-\gamma)^4 \kappa^6 \lambda^{5/2} T} \right)^{1/3}, \frac{\sqrt{\epsilon_{app}}}{\lambda} \right). $$
If $T$ is sufficiently large and $\epsilon_{app}$ is sufficiently small such that $\tau_T \le \tau_{max}$, then by choosing the step-size parameters as in Theorem~\ref{thm:average-iterate} evaluated at $\tau = \tau_T$ subject to the stated conditions, the expected average parametric unregularized performance gap satisfies the convergence rate:
\begin{align*}
\frac{1}{T} \sum_{t=0}^{T-1} \mathbb{E} \big[ J_0(\pi_{\omega^*_0}) - J_0(\pi_t) \big] &\le \mathcal{O}\left( \frac{(\log T)^{1/3}}{(1-\gamma)^{7/3} \kappa^2 \lambda^{5/6} T^{1/3}} \right) + \mathcal{O}\left( \frac{\sqrt{\epsilon_{app}}}{(1-\gamma) \lambda} \right) \\
&= \tilde{\mathcal{O}}\left( T^{-1/3} \right) + \mathcal{O}(\sqrt{\epsilon_{app}}).
\end{align*}
\end{thm}

\begin{prf}
(i) \textit{Structural Error Bound.} Let $\pi_{\omega^*_0}$ be the unregularized parametrically optimal policy.
Because the entropy is non-negative, evaluating this optimal policy under the regularized objective yields $J_0(\pi_{\omega^*_0}) \le J_\tau(\pi_{\omega^*_0})$.
By the definition of the regularized parametric optimum, we have $J_\tau(\pi_{\omega^*_0}) \le J_\tau(\pi_{\omega^*_\tau})$.
Thus, for any policy $\pi_t$, the parametric unregularized performance gap is bounded by the sum of the parametric regularized performance gap and the entropy penalty:
\begin{align*}
J_0(\pi_{\omega^*_0}) - J_0(\pi_t) &\le J_\tau(\pi_{\omega^*_\tau}) - J_0(\pi_t) \\
&= \Big( J_\tau(\pi_{\omega^*_\tau}) - J_\tau(\pi_t) \Big) + \Big( J_\tau(\pi_t) - J_0(\pi_t) \Big) \\
&= \text{Gap}_t + \frac{\tau}{1-\gamma} \mathbb{E}_{s \sim d^{\pi_t}} [H(\pi_t(\cdot|s))].
\end{align*}
Taking the expectation and averaging this bound over the trajectory from $t=0$ to $T-1$ yields:
$$\frac{1}{T} \sum_{t=0}^{T-1} \mathbb{E} \big[ J_0(\pi_{\omega^*_0}) - J_0(\pi_t) \big] \le \frac{1}{T} \sum_{t=0}^{T-1} \mathbb{E}[\text{Gap}_t] + \frac{\tau}{1-\gamma} \left( \frac{1}{T} \sum_{t=0}^{T-1} \mathbb{E} \Big[ \mathbb{E}_{s \sim d^{\pi_t}} [H(\pi_t(\cdot|s))] \Big] \right).$$
The policy entropy is universally bounded by $H(\pi_t) \le \log |\mathcal{A}| = \mathcal{O}(1)$.

(ii) \textit{Regularized Order Substitution.} Substituting the average regularized performance gap bound established in Theorem~\ref{thm:average-iterate}, the average parametric unregularized performance gap is bounded by:
$$\frac{1}{T} \sum_{t=0}^{T-1} \mathbb{E} \big[ J_0(\pi_{\omega^*_0}) - J_0(\pi_t) \big] \le \underbrace{\mathcal{O}\left( \frac{1}{(1-\gamma)^5 \kappa^6 \lambda^{5/2} \tau^2} \frac{\log T}{T} \right)}_{\text{Finite-Time Optimization Error}} + \underbrace{\mathcal{O}\left( \frac{\epsilon_{app}}{(1-\gamma) \lambda^2 \tau} \right)}_{\text{Approximation Bias}} + \underbrace{\mathcal{O}\left( \frac{\tau}{1-\gamma} \right)}_{\text{Entropy Penalty}}.$$

(iii) \textit{Two-Stage Hyperparameter Balancing.} To identify the appropriate temperature scale, we balance the dominant regularized optimization error against the Entropy Penalty.
\begin{itemize}
    \item \textit{Sample-Limited Regime:} When the average optimization error dominates the approximation bias, balancing the former against the entropy penalty,  $\frac{1}{(1-\gamma)^5 \kappa^6 \lambda^{5/2} \tau^2} \frac{\log T}{T} = \frac{\tau}{1-\gamma}$, yields the tuned temperature $\tau = \big(\frac{\log T}{(1-\gamma)^4 \kappa^6 \lambda^{5/2} T}\big)^{1/3}$. Substituting this choice gives an unregularized performance gap bound of $\mathcal{O}\big( \frac{(\log T)^{1/3}}{(1-\gamma)^{7/3} \kappa^2 \lambda^{5/6} T^{1/3}} \big)$.
    \item \textit{Approximation-Limited Regime:} When the approximation bias dominates the average optimization error, balancing $\frac{\epsilon_{app}}{(1-\gamma) \lambda^2 \tau} = \frac{\tau}{1-\gamma}$ yields the tuned temperature $ \tau=\frac{\sqrt{\epsilon_{app}}}{\lambda}$. Substituting this choice gives an unregularized performance gap bound of $\mathcal{O}\left( \frac{\sqrt{\epsilon_{app}}}{(1-\gamma) \lambda} \right)$.
\end{itemize}

(iv) \textit{Synthesis.} By setting the tuned temperature to the maximum of these two candidate temperatures, the expected average parametric unregularized performance gap is bounded by the sum of these respective envelopes, completing the proof.
\end{prf}

\subsection{Last-Iterate Unregularized Convergence}

We map the last-iterate regularized performance gap bound in Section \ref{sec:pl-last} to the unregularized objective. Because this bound avoids the $\log T$ penalty, we establish a sharper unregularized convergence rate (by a factor of $(\log T)^{1/3}$) when tuning the fixed temperature $\tau$ to balance this terminal optimization error against the induced entropy penalty.

\begin{thm}[Last-Iterate Unregularized Convergence] \label{thm:unreg-last-iterate}
For Algorithm~\ref{alg:uncentered-nac} with $T$ replaced by $T+1$, suppose that Assumptions \ref{ass:bounded-features}--\ref{ass:approximation-error} and \ref{ass:pd-fim}--\ref{ass:parametric-density} hold.
Consider the two-stage temperature:
$$ \tau_T \triangleq \max\left( \left( \frac{1}{(1-\gamma)^4 \kappa^6 \lambda^{5/2} T} \right)^{1/3}, \frac{\sqrt{\epsilon_{app}}}{\lambda} \right). $$
If $T$ is sufficiently large and $\epsilon_{app}$ is sufficiently small such that $\tau_T \le \tau_{max}$, then by choosing the step-size parameters as in Theorem~\ref{thm:last-iterate} evaluated at $\tau = \tau_T$ subject to the stated conditions, the expected last-iterate parametric unregularized performance gap satisfies the convergence rate:
\begin{align*}
\mathbb{E} \big[ J_0(\pi_{\omega^*_0}) - J_0(\pi_T) \big] &\le \mathcal{O}\left( \frac{1}{(1-\gamma)^{7/3} \kappa^2 \lambda^{5/6} T^{1/3}} \right) + \mathcal{O}\left( \frac{\sqrt{\epsilon_{app}}}{(1-\gamma) \lambda} \right) \\
&= \mathcal{O}\left( T^{-1/3} \right) + \mathcal{O}(\sqrt{\epsilon_{app}}).
\end{align*}
\end{thm}

\begin{prf}
(i) \textit{Structural Error Bound.} We apply the unregularized bound in Step (i) of the proof of Theorem~\ref{thm:unreg-average-iterate} to the terminal policy $\pi_T$.
Taking the expectation, the parametric unregularized performance gap is bounded by the sum of the regularized performance gap and the entropy penalty:
$$\mathbb{E} \big[ J_0(\pi_{\omega^*_0}) - J_0(\pi_T) \big] \le \mathbb{E}[\text{Gap}_T] + \frac{\tau}{1-\gamma} \mathbb{E} \Big[ \mathbb{E}_{s \sim d^{\pi_T}} [H(\pi_T(\cdot|s))] \Big].$$

(ii) \textit{Regularized Order Substitution.} We substitute the last-iterate convergence bound for $\mathbb{E}[\text{Gap}_T]$ established in Theorem~\ref{thm:last-iterate}.
Recalling that the policy entropy is bounded by $H(\pi_T) \le \log |\mathcal{A}| = \mathcal{O}(1)$, the final parametric unregularized performance gap is bounded by:
$$\mathbb{E} \big[ J_0(\pi_{\omega^*_0}) - J_0(\pi_T) \big] \le \underbrace{\mathcal{O}\left( \frac{1}{(1-\gamma)^5 \kappa^6 \lambda^{5/2} \tau^2 T} \right)}_{\text{Finite-Time Optimization Error}} + \underbrace{\mathcal{O}\left( \frac{\epsilon_{app}}{(1-\gamma) \lambda^2 \tau} \right)}_{\text{Approximation Bias}} + \underbrace{\mathcal{O}\left( \frac{\tau}{1-\gamma} \right)}_{\text{Entropy Penalty}}.$$

(iii) \textit{Two-Stage Hyperparameter Balancing.} We identify the appropriate temperature scale analogously to the average-iterate case:
\begin{itemize}
    \item \textit{Sample-Limited Regime:} When the last-iterate optimization error dominates the approximation bias, balancing at $\tau = \big((1-\gamma)^4 \kappa^6 \lambda^{5/2} T\big)^{-1/3}$ yields an unregularized performance gap bound of $\mathcal{O}\left( \frac{1}{(1-\gamma)^{7/3} \kappa^2 \lambda^{5/6} T^{1/3}} \right)$.
    \item \textit{Approximation-Limited Regime:} When the approximation bias dominates the last-iterate optimization error, balancing at $\tau = \frac{\sqrt{\epsilon_{app}}}{\lambda}$ yields an unregularized performance gap bound of $\mathcal{O}\left( \frac{\sqrt{\epsilon_{app}}}{(1-\gamma) \lambda} \right)$.
\end{itemize}

(iv) \textit{Synthesis.} By setting the tuned temperature to the maximum of these two candidate temperatures, the final expected parametric unregularized performance gap is bounded by the sum of these respective envelopes, completing the proof.
\end{prf}

\begin{rem}[Equivalence of Unregularized Average and Last-Iterate Rates] \label{rem:unreg-equivalence}
Echoing the equivalence for regularized bounds in Remark~\ref{rem:regularized-equivalence}, the unregularized last-iterate convergence rate (Theorem~\ref{thm:unreg-last-iterate}) asymptotically matches the average-iterate rate (Theorem~\ref{thm:unreg-average-iterate}) up to a logarithmic factor, both achieving a $\tilde{\mathcal{O}}(T^{-1/3})$ dependence.
\end{rem}

\subsection{Global Unregularized Convergence}

By invoking the Global-to-Parametric Joint Concentrability (Assumption~\ref{ass:global-parametric-density}), we extend the parametric unregularized convergence bounds to the true global optimum. Because the representation penalty from the Global Class Approximation Error (Lemma~\ref{lem:global-class-error}) shares the identical $\mathcal{O}(1/\tau)$ dependence as the standard approximation bias, incorporating this representation penalty leads to identical unregularized convergence rates as in Theorems~\ref{thm:unreg-average-iterate} and \ref{thm:unreg-last-iterate}.

\begin{cor}[Global Unregularized Convergence] \label{cor:global-unreg-convergence}
Suppose that Assumptions \ref{ass:bounded-features}--\ref{ass:approximation-error} and \ref{ass:pd-fim}--\ref{ass:global-parametric-density} hold.
By tuning the temperature as established in Theorems~\ref{thm:unreg-average-iterate} and \ref{thm:unreg-last-iterate}, the following unregularized convergence rates hold with respect to the global optimum. For the last-iterate convergence:
$$\mathbb{E} \big[ J_0(\pi^*_0) - J_0(\pi_T) \big] \le \mathcal{O}\left( \frac{1}{(1-\gamma)^{7/3} \kappa^2 \lambda^{5/6} T^{1/3}} \right) + \mathcal{O}\left( \frac{\sqrt{\epsilon_{app}}}{(1-\gamma) \lambda} \right).$$
For the average-iterate convergence:
$$\frac{1}{T} \sum_{t=0}^{T-1} \mathbb{E} \big[ J_0(\pi^*_0) - J_0(\pi_t) \big] \le \mathcal{O}\left( \frac{(\log T)^{1/3}}{(1-\gamma)^{7/3} \kappa^2 \lambda^{5/6} T^{1/3}} \right) + \mathcal{O}\left( \frac{\sqrt{\epsilon_{app}}}{(1-\gamma) \lambda} \right).$$
\end{cor}

\begin{prf}
Because the policy entropy is non-negative, $J_0(\pi^*_0) \le J_\tau(\pi^*_0)$. Since $\pi^*_\tau$ maximizes the regularized objective globally, $J_\tau(\pi^*_0) \le J_\tau(\pi^*_\tau)$.
Taking the expectation over the algorithmic trajectory, the global unregularized performance gap decomposes into the representation penalty, the regularized parametric gap, and the entropy penalty:
\begin{align*}
\mathbb{E} \big[ J_0(\pi^*_0) - J_0(\pi_T) \big] &\le \mathbb{E} \big[ J_\tau(\pi^*_\tau) - J_0(\pi_T) \big] \\
&= \Big( J_\tau(\pi^*_\tau) - J_\tau(\pi_{\omega^*_\tau}) \Big) + \mathbb{E}[\text{Gap}_T] + \frac{\tau}{1-\gamma} \mathbb{E} \Big[ \mathbb{E}_{s \sim d^{\pi_T}} [H(\pi_T(\cdot|s))] \Big].
\end{align*}
By Lemma~\ref{lem:global-class-error}, the representation penalty is bounded by $\frac{C_{class}}{\tau} \epsilon_{app}$.
Thus:
$$ \mathbb{E} \big[ J_0(\pi^*_0) - J_0(\pi_T) \big] \le \frac{C_{class}}{\tau} \epsilon_{app} + \mathbb{E}[\text{Gap}_T] + \frac{\tau}{1-\gamma} \mathbb{E} \Big[ \mathbb{E}_{s \sim d^{\pi_T}} [H(\pi_T(\cdot|s))] \Big]. $$
Because $C_{class}=\mathcal{O}(\frac{1}{(1-\gamma)\lambda^2})$,
the representation penalty $\frac{C_{class}}{\tau} \epsilon_{app}$ shares the same dependence on $\tau$ as well as on $(1-\gamma)^{-1}$ and $\lambda$
as the approximation bias within $\mathbb{E}[\text{Gap}_T]$.
The same two-stage temperature applies to balance the terms as in Theorem~\ref{thm:unreg-last-iterate}, preserving the last-iterate rate. The same argument applies to the averaged sequence, preserving the average-iterate rate from Theorem~\ref{thm:unreg-average-iterate}, completing the proof.
\end{prf}

\section{Unifying Single-Loop and Double-Loop via Warm-Start Critic} \label{app:warm-start-double-loop}

\begin{algorithm}[!t]
\caption{Warm-Start Double-Loop, Entropy-Regularized, Uncentered Natural Actor-Critic}
\label{alg:warm-start-nac}
\begin{algorithmic}[1]
\REQUIRE Initial actor parameter $\omega_0$ satisfying $\|\omega_0\|_2 \le \frac{B_\theta}{\tau_{max}}$, initial critic parameter $\theta_0$, temperature $\tau \in (0,\tau_{max}]$, critic projection radius $R_\theta \, (=B_\theta)$, inner-loop iterations $N$, outer-loop iterations $T$, step-size schedules $\{\eta_t\}_{t \ge 0}$ and $\{\alpha_t\}_{t \ge 0}$.
\FOR{$t = 0, 1, \dots, T-1$}
    \STATE \textbf{Actor Freeze:} Fix the current policy $\pi_t \triangleq \pi_{\omega_t}$.
    \STATE \textbf{Critic Warm-Start:} Initialize the inner critic parameter with the previous outer-loop output: $\theta_t^{(0)} = \theta_t$.
    \FOR{$k = 0, 1, \dots, N-1$}
        \STATE \textbf{Sampling:} Draw a state-action pair $(s_k, a_k) \sim d^{\pi_t} \times \pi_t$.
        \STATE \textbf{Evaluation:} Obtain an unbiased estimate $\hat{Q}_k$ of the regularized action-value $Q_\tau^{\pi_t}(s_k, a_k)$.  
        \STATE \textbf{Critic SGD:} Compute the stochastic critic gradient $g_t^{cr}(s_k, a_k) = \big( \theta_t^{(k)\top} \phi(s_k, a_k) - \hat{Q}_k \big) \phi(s_k, a_k)$ and update via constrained SGD:
        $$\theta_t^{(k+1)} = \Pi_{R_\theta} \Big[ \theta_t^{(k)} - \alpha_t g_t^{cr}(s_k, a_k) \Big].$$ \vspace*{-.3in}
    \ENDFOR
    \STATE \textbf{Critic Assignment:} Set the outer-loop critic parameter to the final inner-loop iterate: $\theta_{t+1} = \theta_t^{(N)}$.
    \STATE \textbf{Actor Update:} Update the actor parameter:
    $$\omega_{t+1}  = \omega_t + \eta_t \big( \theta_{t+1} - \tau \omega_t \big)$$  \vspace*{-.4in}
\ENDFOR
\ENSURE The output policy sequence $\{\pi_t\}_{t=0}^T$.
\end{algorithmic}
\end{algorithm}

In this section, we generalize our analysis to a warm-start double-loop architecture (Algorithm~\ref{alg:warm-start-nac}). In this setting, the algorithm performs $T$ outer-loop actor updates, and for each fixed actor parameter $\omega_t$, runs $N$ inner-loop critic SGD steps. Instead of independently re-initializing the critic at each outer loop (as in standard, cold-start double-loop algorithms like Algorithm~\ref{alg:double-loop-npg}), we \textit{warm-start} the inner critic using the final parameter of the previous loop ($\theta_t^{(0)} = \theta_t$). The total number of SGD updates is defined as $T_{total} = T \times N$.

This architecture mathematically unifies the single-loop ($N=1, T=T_{total}$) and double-loop ($T = \sqrt{T_{total}}, N = \sqrt{T_{total}}$) algorithms. By analyzing the geometric contraction over the $N$ inner steps, we prove a structural invariance across inner-loop sizes $N$,
based on our PMD analysis in Part II (which applies to the Deterministic Regime as well as the Stochastic Regime).
First, under our primary Exponential Translation mechanism, we demonstrate that any intermediate inner-loop choice $N \in [1, T_{total}^{2/9}]$ yields the same $\tilde{\mathcal{O}}(T_{total}^{-2/3})$ average-iterate unregularized convergence rate. Subsequently, to reveal the underlying barrier that the exponential translation bypasses, we analyze the algorithm under the standard linear entropy penalty. We establish that without the exponential mechanism, the PMD geometry leads to a $\tilde{\mathcal{O}}(T_{total}^{-1/4})$ convergence rate, remaining invariant across the shifted window $N \in [1, \sqrt{T_{total}}]$.

Throughout this appendix, when invoking trajectory-dependent assumptions from the main text, we assume their corresponding bounds hold along the training policies and inner-loop samples generated by Algorithm~\ref{alg:warm-start-nac}, with the same structural constants.

\subsection{Warm-Start Critic Tracking Error}

To analyze the warm-start double-loop algorithm, standard fractional variants of Chung's lemma (Lemma~\ref{lem:fractional-chung}) provide tight pointwise bounds but suffer from artificial horizon inflation when integrated directly over an outer-loop trajectory. The following lemma introduces a fractional summation bound to safely isolate the initialization footprint, a mechanism that is critical for establishing loop invariance across intermediate inner-loop sizes $N$.

\begin{lem}[Fractional Chung's Summation Bound] \label{lem:fractional-chung-sum}
Let $\{x_t\}_{t=0}^\infty$ be a non-negative sequence satisfying the recursive bound:
$$x_{t+1} \le \left(1 - \frac{c}{(t+t_0)^{2/3}}\right) x_t + \frac{K}{(t+t_0)^{4/3}}, \quad t \ge 0,$$
where $c > 0$, $K \ge 0$, and the initial offset satisfies $t_0 \ge \max\left(1, (\frac{4}{3c})^3 \right)$.
The average of the sequence is bounded by:
$$ \frac{1}{T} \sum_{t=0}^{T-1} x_t \le \frac{2 t_0^{2/3}}{c T} x_0 + \frac{6K}{c T} (T+t_0)^{1/3} . $$
\end{lem}

\begin{prf}
We rearrange the recurrence to isolate $x_t$ alongside the fractional step size:
$$ x_t \le \frac{(t+t_0)^{2/3}}{c} (x_t - x_{t+1}) + \frac{K}{c(t+t_0)^{2/3}}. $$
Summing this inequality from $t=0$ to $T-1$, we apply summation by parts to the first term:
\begin{align*}
\sum_{t=0}^{T-1} x_t &\le \frac{t_0^{2/3}}{c} x_0 - \frac{(T-1+t_0)^{2/3}}{c} x_T + \sum_{t=1}^{T-1} x_t \frac{(t+t_0)^{2/3} - (t-1+t_0)^{2/3}}{c} + \frac{K}{c} \sum_{t=0}^{T-1} \frac{1}{(t+t_0)^{2/3}}.
\end{align*}
We drop the non-positive $-x_T$ term. By the concavity of the fractional power, the difference is bounded by its derivative: $(t+t_0)^{2/3} - (t-1+t_0)^{2/3} \le \frac{2}{3} (t-1+t_0)^{-1/3}$. Because $t \ge 1$, this is bounded by $\frac{2}{3} t_0^{-1/3}$.
Furthermore, with $t_0\ge 1$, the fractional sum is bounded by:
$\sum_{t=0}^{T-1} (t+t_0)^{-2/3} \le \int_0^T (x+t_0-1)^{-2/3} dx \le 3(T+t_0)^{1/3}$. Substituting these bounds gives:
$$ \sum_{t=0}^{T-1} x_t \le \frac{t_0^{2/3}}{c} x_0 + \frac{2}{3 c t_0^{1/3}} \sum_{t=1}^{T-1} x_t + \frac{3K}{c} (T+t_0)^{1/3}. $$
By the initial offset condition $t_0 \ge (\frac{4}{3c})^3$, the multiplier on the right-hand sum satisfies $\frac{2}{3 c t_0^{1/3}} \le \frac{1}{2}$. Subtracting $\frac{1}{2} \sum_{t=0}^{T-1} x_t$ from both sides and multiplying the entire inequality by $\frac{2}{T}$ establishes the stated average bound.
\end{prf}

\begin{rem}[Comparison with Pointwise Summation] \label{rem:comparison-pointwise}
To understand the benefit of Lemma~\ref{lem:fractional-chung-sum}, consider the alternative of directly summing the pointwise bound $x_t \le v(t+t_0)^{-2/3}$ provided by the Fractional Chung's Lemma (Lemma~\ref{lem:fractional-chung}), where $v = \max(t_0^{2/3} x_0, 2K/c)$.
Summing this polynomial envelope via an integral bound ($\sum_{t=0}^{T-1} (t+t_0)^{-2/3} \le \int_0^T (x+t_0-1)^{-2/3} dx \le 3(T+t_0)^{1/3}$) yields $\sum_{t=0}^{T-1} x_t \le 3v(T+t_0)^{1/3}$. When the initial error term $t_0^{2/3} x_0$ dominates $v$, this standard approach bounds the initialization sum by $\mathcal{O}(x_0 t_0^{2/3} T^{1/3})$.
In contrast, the summation by parts in Lemma~\ref{lem:fractional-chung-sum} captures the exponential decay of the initial error, decoupling it from the harmonic integral and bounding it by $\mathcal{O}\big(\frac{x_0 t_0^{2/3}}{c}\big)$. Therefore, directly integrating the pointwise envelope artificially inflates the initialization footprint by a factor of $c T^{1/3}$.
\end{rem}

With this fractional summation tool established, we bound the critic's tracking error across the warm-start architecture. By unrolling the geometric contraction of the $N$ inner-loop SGD steps starting from the warm-started parameter $\theta_t$, we compress the inner-loop dynamics into a single effective outer-loop recurrence, allowing us to apply the fractional summation bound.

\begin{lem}[Warm-Start Critic Tracking Error] \label{lem:warm-start-critic}
Suppose that Assumptions~\ref{ass:bounded-features}--\ref{ass:pd-uncentered} hold.
For independent constants $c_\eta, c_\beta > 0$, let the actor step size be $\eta_t = \frac{c_\eta}{t+t_0}$, and let the inner critic step size be $\alpha_{t+1} = N^{-1}\beta_{t+1}$ and the effective outer schedule be $\beta_{t+1} =\frac{2c_\beta}{ (t+t_0)^{2/3}}$ for all $t \ge 0$, with an arbitrary initialization $\alpha_0$ and an initial offset
$t_0 \ge \max\left(1, c_\eta \tau, (4c_\beta \kappa)^{3/2}, (\frac{4}{3c_\beta \kappa} )^3 \right)$.
Then the expected tracking error $Z_t \triangleq \mathbb{E}[\|e_t\|_2^2]=\mathbb{E}[\|\theta_{t+1} - \theta^*_\tau(\omega_t)\|_2^2]$ is bounded pointwise for all $t \ge 0$ by:
$$ Z_t \le \frac{v}{(t+t_0)^{2/3}}, $$
where the bounding coefficient is $v \triangleq \max\left( t_0^{2/3} Z_0, \frac{2K_N}{c_\beta \kappa} \right)$, and the noise envelope constant is $K_N \triangleq \frac{4 c_\beta^2 G_{cr}^2}{N} + \frac{L_\theta^2 c_\eta^2 G_{ac}^2}{c_\beta \kappa}$. Furthermore, the average tracking error is bounded by:
$$ \frac{1}{T} \sum_{t=0}^{T-1} Z_t \le \frac{2 t_0^{2/3}}{c_\beta \kappa T} Z_0 + \frac{6 K_N}{c_\beta \kappa T} (T+t_0)^{1/3}. $$
\end{lem}

\begin{prf}
(i) \textit{Step-Size Validity.} We first verify that the initial offset $t_0$ guarantees the condition $\alpha_{t+1} \kappa \le \frac{1}{2N}$ for all $t \ge 0$, which is equivalent to $\beta_{t+1} \kappa \le \frac{1}{2}$. This condition holds because $\beta_{t+1} \kappa \le \beta_1 \kappa \le \frac{2c_\beta \kappa}{t_0^{2/3}} \le \frac{1}{2}$, due to the requirement $t_0 \ge (4c_\beta \kappa)^{3/2}$. Additionally, the condition $t_0 \ge c_\eta \tau$ guarantees that the actor step size satisfies $\eta_t \le \frac{c_\eta}{t_0} \le 1/\tau$ for all $t \ge 0$, which ensures the actor update direction is bounded by $G_{ac}$ via Lemma~\ref{lem:structural-bounds}(ii).

(ii) \textit{Exact Inner-Loop Contraction.}
Consider outer loop $t+1$, which uses the critic step size $\alpha_{t+1}$ and fixes the target $\theta^*_\tau(\omega_{t+1})$. Let
$e_{t+1}^{(j)} \triangleq \theta_{t+1}^{(j)} - \theta^*_\tau(\omega_{t+1})$ denote the tracking error at the $j$-th inner-loop SGD step, and let $\mathcal{F}_{t+1}^{(j)}$ denote the filtration including the history up to this step. As derived in Lemma~\ref{lem:critic-tracking-pl}, a single inner SGD step satisfies the exact contraction:
$$ \mathbb{E}[\|e_{t+1}^{(j+1)}\|_2^2 \mid \mathcal{F}_{t+1}^{(j)}] \le (1 - 2\alpha_{t+1} \kappa) \|e_{t+1}^{(j)}\|_2^2 + \alpha_{t+1}^2 G_{cr}^2. $$
We take the total expectation conditioned on the outer-loop filtration $\mathcal{F}_{t+1}$ (which fixes the actor parameter $\omega_{t+1}$ and the warm-start initialization $\theta_{t+1}^{(0)} = \theta_{t+1}$). Unrolling this exact SGD contraction over $N$ steps, the terminal expected error satisfies:
$$ \mathbb{E}[\|e_{t+1}^{(N)}\|_2^2 \mid \mathcal{F}_{t+1}] \le (1 - 2\alpha_{t+1}\kappa)^N \|\tilde{e}_t\|_2^2 + \alpha_{t+1}^2 G_{cr}^2 \sum_{j=0}^{N-1}(1 - 2\alpha_{t+1}\kappa)^j, $$
where $\tilde{e}_t \triangleq \theta_{t+1} - \theta^*_\tau(\omega_{t+1})$.
Note that $e_{t+1}^{(0)} = \theta_{t+1}^{(0)} - \theta^*_\tau(\omega_{t+1}) = \theta_{t+1} - \theta^*_\tau(\omega_{t+1}) = \tilde{e}_t$,
and  $e_{t+1}^{(N)} = \theta_{t+1}^{(N)} - \theta^*_\tau(\omega_{t+1}) = \theta_{t+2} - \theta^*_\tau(\omega_{t+1}) = e_{t+1}$.
Defining $\bar{\alpha}_{t+1} \triangleq \frac{1 - (1 - 2\alpha_{t+1} \kappa)^N}{\kappa}$, we have
$$ \sum_{j=0}^{N-1} (1 - 2\alpha_{t+1} \kappa)^j = \frac{1 - (1 - 2\alpha_{t+1} \kappa)^N}{2\alpha_{t+1}\kappa} = \frac{\bar{\alpha}_{t+1}}{2\alpha_{t+1}}, $$
so that the preceding inequality reduces to:
$$ \mathbb{E}[\|e_{t+1} \|_2^2 \mid \mathcal{F}_{t+1}] \le (1 - \bar{\alpha}_{t+1} \kappa) \|\tilde{e}_t\|_2^2 + \frac{\alpha_{t+1}\bar{\alpha}_{t+1}}{2} G_{cr}^2. $$

(iii) \textit{Young's Inequality and Controlling Target Drift.}
We separate the outer-loop tracking error $e_t = \theta_{t+1} - \theta^*_\tau(\omega_t)$ from the target drift:
$$ \tilde{e}_t = e_t + \big(\theta^*_\tau(\omega_t) - \theta^*_\tau(\omega_{t+1})\big). $$
Applying Young's inequality $\|x+y\|_2^2 \le (1+p)\|x\|_2^2 + (1+1/p)\|y\|_2^2$ with parameter $p = \bar{\alpha}_{t+1}\kappa/2$ yields:
$$ \|\tilde{e}_t\|_2^2 \le \left(1+\frac{\bar{\alpha}_{t+1}\kappa}{2}\right)\|e_t\|_2^2 + \left(1+\frac{2}{\bar{\alpha}_{t+1}\kappa}\right)\|\theta^*_\tau(\omega_t) - \theta^*_\tau(\omega_{t+1})\|_2^2. $$
The multiplier $1-\bar{\alpha}_{t+1}\kappa = (1-2\alpha_{t+1} \kappa)^N$ is non-negative, because
$2\alpha_{t+1} \kappa \le \frac{1}{N} \le 1$ by Step (i).
This allows us to substitute the Young bound into the contraction from Step (ii). Using
$(1-\bar{\alpha}_{t+1}\kappa) (1+\frac{\bar{\alpha}_{t+1}\kappa}{2}) \le 1-\frac{\bar{\alpha}_{t+1}\kappa}{2}$
and
$ (1-\bar{\alpha}_{t+1}\kappa) (1+\frac{2}{\bar{\alpha}_{t+1}\kappa} ) \le \frac{2}{\bar{\alpha}_{t+1}\kappa}$
together with the Lipschitz target-drift bound
$$ \|\theta^*_\tau(\omega_t) - \theta^*_\tau(\omega_{t+1})\|_2^2 \le L_\theta^2 \|\omega_{t+1}-\omega_t\|_2^2 \le L_\theta^2 \eta_t^2 G_{ac}^2, $$
and taking the total expectation $\mathbb{E} [\cdot]$ yields the single-step outer-loop recurrence:
$$ Z_{t+1} \le \left(1-\frac{\bar{\alpha}_{t+1}\kappa}{2}\right) Z_t + \frac{\alpha_{t+1}\bar{\alpha}_{t+1}}{2} G_{cr}^2 + \frac{2}{\bar{\alpha}_{t+1}\kappa} L_\theta^2 \eta_t^2 G_{ac}^2. $$

(iv) \textit{Bounding $\bar{\alpha}_{t+1}$ and Extracting $K_N$.} We bound $\bar{\alpha}_{t+1}$ against the effective outer schedule $\beta_{t+1} = N\alpha_{t+1}$. Let $x=2\alpha_{t+1}\kappa$, which, by Step (i), satisfies $Nx=2\beta_{t+1}\kappa\le1$.
Because $x\le1$, Bernoulli's inequality yields $(1-x)^N \ge 1-Nx$, ensuring
$$ \bar{\alpha}_{t+1}\kappa = 1-(1-x)^N \le Nx = 2\beta_{t+1}\kappa . $$
Hence $\bar{\alpha}_{t+1} \le 2\beta_{t+1}$.
Conversely, we bound the Binomial expansion $(1-x)^N = \sum_{k=0}^N \binom{N}{k} (-x)^k$. Letting $a_k = \binom{N}{k}x^k$, the ratio of adjacent magnitudes is $\frac{a_{k+1}}{a_k} = \frac{N-k}{k+1}x \le \frac{Nx}{k+1}$. Because $Nx \le 1$, this ratio satisfies $\frac{a_{k+1}}{a_k} \le \frac{1}{k+1} < 1$ for all $k \ge 1$. Since the terms decrease in magnitude, grouping the alternating remainder terms starting from $k=3$ yields a non-positive sum: $(-a_3 + a_4) + (-a_5 + a_6) + \dots \le 0$. Thus, $(1-x)^N \le 1 - Nx + \frac{N(N-1)}{2}x^2$. Bounding $N(N-1) \le N^2$ gives $(1-x)^N \le 1 - Nx + \frac{N^2 x^2}{2} = 1 - Nx(1 - \frac{Nx}{2})$.
Because $Nx\le 1$, we have $1 - \frac{Nx}{2} \ge \frac{1}{2}$, yielding
$$  \bar{\alpha}_{t+1}\kappa = 1 - (1-x)^N \ge \frac{1}{2} Nx = \beta_{t+1}\kappa .$$
Thus, $ \beta_{t+1} \le \bar{\alpha}_{t+1} \le 2\beta_{t+1}$.

Substituting these bounds into the recurrence yields:
$$ Z_{t+1} \le \left(1-\frac{1}{2}\beta_{t+1}\kappa\right) Z_t + \frac{\beta_{t+1}^2}{N} G_{cr}^2 + \frac{2}{\beta_{t+1}\kappa} L_\theta^2 \eta_t^2 G_{ac}^2. $$
Substituting the step-size schedules $\beta_{t+1} = \frac{2c_\beta}{(t+t_0)^{2/3}}$ and $\eta_t = \frac{c_\eta}{t+t_0}$, the contraction multiplier evaluates to $1-\frac{c_\beta\kappa}{(t+t_0)^{2/3}}$. The additive noise terms evaluate to:
$$ \frac{1}{N}\frac{4c_\beta^2}{(t+t_0)^{4/3}}G_{cr}^2 + \frac{2(t+t_0)^{2/3}}{2c_\beta\kappa}L_\theta^2\frac{c_\eta^2}{(t+t_0)^2}G_{ac}^2 = \frac{\frac{4c_\beta^2G_{cr}^2}{N}+\frac{L_\theta^2c_\eta^2G_{ac}^2}{c_\beta\kappa}}{(t+t_0)^{4/3}} = \frac{K_N}{(t+t_0)^{4/3}}, $$
explicitly extracting our defined constant $K_N$.

(v) \textit{Pointwise and Average Bounds.}
The sequence satisfies the single-step recurrence:
$$ Z_{t+1} \le \left(1-\frac{c_\beta\kappa}{(t+t_0)^{2/3}}\right) Z_t + \frac{K_N}{(t+t_0)^{4/3}}. $$
To invoke Fractional Chung's Lemma (Lemma~\ref{lem:fractional-chung}) with the contraction coefficient $c=c_\beta\kappa$, we verify that the initial offset satisfies $t_0 \ge \max\left((c_\beta\kappa)^{3/2}, (\frac{4}{3c_\beta\kappa} )^3\right)$. The term $(c_\beta\kappa)^{3/2}$ is subsumed by the stated $(4c_\beta\kappa)^{3/2}$ requirement, and $ (\frac{4}{3c_\beta\kappa} )^3$ is directly enveloped by the assumed lower bound on $t_0$. Thus, Lemma~\ref{lem:fractional-chung} yields the pointwise bound $Z_t \le \frac{v}{(t+t_0)^{2/3}}$ with $v = \max\left(t_0^{2/3} Z_0,\frac{2K_N}{c_\beta\kappa}\right)$.
Furthermore,
by the initial offset condition $t_0 \ge \max\left(1, (\frac{4}{3c_\beta\kappa} )^3\right)$,
applying the Fractional Chung's Summation Bound (Lemma~\ref{lem:fractional-chung-sum}) to the identical sequence directly yields the average bound:
$$ \frac{1}{T}\sum_{t=0}^{T-1}Z_t \le \frac{2t_0^{2/3}}{c_\beta\kappa T}Z_0 + \frac{6K_N}{c_\beta\kappa T}(T+t_0)^{1/3}, $$
completing the proof.
\end{prf}

\subsection{Warm-Start Average-Iterate Regularized Convergence}

We now integrate the warm-start tracking bound into our PMD analysis. By substituting the effective fractional noise envelope into the PMD Value-Telescoping Bound (Lemma~\ref{lem:telescoping-pmd}), we establish the average-iterate convergence rate for the regularized objective. This demonstrates how the warm-start architecture separates the total algorithmic error into an invariant primary tracking component and rapidly decaying transient optimization terms.

\begin{thm}[Warm-Start Average-Iterate Regularized Convergence] \label{thm:warm-start-reg}
For Algorithm~\ref{alg:warm-start-nac}, suppose that Assumptions~\ref{ass:bounded-features}--\ref{ass:approximation-error} and \ref{ass:density-ratios} hold.
For independent constants $c_\eta \ge \frac{4}{3\tau}$ and $\zeta >0$, define the base schedule constant $c_\beta \triangleq \frac{1}{2} \zeta c_\eta^{2/3} N^{1/3}$.
Let the actor step size be $\eta_t=\frac{c_\eta}{t+t_0}$. Let $\alpha_0$ be arbitrary, and for all $t\ge0$, let the subsequent inner critic step sizes satisfy $\alpha_{t+1}=N^{-1}\beta_{t+1}$ with the effective outer schedule $\beta_{t+1}=\frac{2c_\beta}{(t+t_0)^{2/3}}$, and let the initial offset satisfy
$ t_0 \ge \max\left(c_\eta \tau, (4c_\beta \kappa)^{3/2}, (\frac{4}{3c_\beta \kappa} )^3 \right)$.
By tuning the constants to the asymptotic orders $c_\eta = \Theta(\tau^{-1})$, $\zeta = \Theta(\kappa^{-1})$, and $t_0 = \Theta(N^{1/2} \tau^{-1})$ subject to the stated conditions, the expected average regularized performance gap for the warm-start double-loop algorithm satisfies:
\begin{align*}
\frac{1}{T} \sum_{t=0}^{T-1} \mathbb{E}[\text{Gap}_t^\dag] \le &\underbrace{ \mathcal{O}\left( \frac{1}{(1-\gamma)^7 \kappa^6 \tau^{5/3} T_{total}^{2/3}} \right) }_{\text{Invariant Tracking Error}} + \underbrace{ \mathcal{O}\left( \frac{\epsilon_{app}}{(1-\gamma)\tau} \right) }_{\text{Approximation Bias}} \\
&+ \underbrace{ \tilde{\mathcal{O}}\left( \frac{N^{1/2}}{(1-\gamma)^2 \kappa T} + \frac{N^{-1/2}}{(1-\gamma)^7 \kappa^6 \tau^2 T} + \frac{1}{(1-\gamma)^3 \kappa^2 \tau T} \right) }_{\text{Transient Convergence Error}}.
\end{align*}
\end{thm}

\begin{prf}
(i) \textit{Telescoping Expansion.} By substituting the average tracking bound from Lemma~\ref{lem:warm-start-critic} into the PMD Value-Telescoping Bound (Lemma~\ref{lem:telescoping-pmd}) with $c_\eta \ge \frac{4}{3\tau}$, the expected average regularized performance gap expands as:
\begin{align*}
(1-\gamma) \frac{1}{T} \sum_{t=0}^{T-1} \mathbb{E}[\text{Gap}_t^\dag] &\le \underbrace{ \frac{t_0-1}{c_\eta T} D_0^\dag }_{\text{PMD Initialization Error}} + \underbrace{ \frac{c_\eta G_{ac}^2 B_\phi^2}{2 T} (1 + \log T) }_{\text{PMD Local Error}} \\
&\quad + \underbrace{ \frac{4 C_{joint}^\dag}{\tau} \epsilon_{app} }_{\text{Approximation Bias}}
+ \underbrace{ \frac{8 B_\phi^2 t_0^{2/3}}{\tau c_\beta\kappa T} Z_0 }_{\text{Initialization Tracking Penalty}} + \underbrace{ \frac{24 B_\phi^2 K_N}{\tau c_\beta\kappa T} (T+t_0)^{1/3} }_{\text{Fractional Tracking Penalty}}.
\end{align*}
Since $(T+t_0)^{1/3} =\Theta( T^{1/3} + t_0^{1/3})$, the fractional tracking penalty decomposes into a $\Theta(\frac{K_N}{\tau c_\beta\kappa T^{2/3}})$ primary term and a $\Theta(\frac{K_N t_0^{1/3}}{\tau c_\beta\kappa T})$ secondary transient term.

(ii) \textit{Invoking Average Tracking Bound.} By choosing $\zeta = \Theta(\kappa^{-1})$ and setting $c_\eta = \Theta(\tau^{-1})$, the effective contraction coefficient scales as $c_\beta\kappa = \Theta(N^{1/3} \tau^{-2/3})$.
Note that $t_0 \ge 1$ is guaranteed by $t_0 \ge c_\eta \tau$ and $c_\eta \ge \frac{4}{3\tau}$.
To satisfy the initial offset condition $t_0 \ge \max \left(c_\eta \tau, (4c_\beta \kappa)^{3/2}, (\frac{4}{3c_\beta\kappa})^3 \right)$, we verify the dominant term. For $\tau \le \tau_{max}$, the term $(4c_\beta \kappa)^{3/2} = \Theta(N^{1/2} \tau^{-1})$ dominates $(\frac{4}{3c_\beta\kappa})^3 = \Theta(N^{-1} \tau^2)$ as well as $c_\eta \tau = \Theta(1)$. This allows us to fulfill the requirement by setting $t_0 = \Theta(N^{1/2} \tau^{-1})$.

(iii) \textit{Structural Order of Noise Envelope.} We evaluate the noise envelope $K_N$. By Lemmas~\ref{lem:structural-bounds} and \ref{lem:lipschitz-critic}, the structural bounds scale as $G_{cr}^2 = \Theta(\kappa^{-2}(1-\gamma)^{-2})$, $G_{ac} = \Theta(\kappa^{-1}(1-\gamma)^{-1})$, and $L_\theta^2 = \Theta(\kappa^{-4}(1-\gamma)^{-4})$.
Substituting these alongside our choices $c_\eta = \Theta(\tau^{-1})$ and $c_\beta = \Theta(\kappa^{-1} \tau^{-2/3} N^{1/3})$, we evaluate the two additive components of $K_N$ separately.
The inner-loop SGD noise component evaluates to:
$$ \frac{4 c_\beta^2 G_{cr}^2}{N} = \Theta\left( N^{-1/3} \tau^{-4/3} \kappa^{-4} (1-\gamma)^{-2} \right). $$
The actor target drift component evaluates to:
$$ \frac{L_\theta^2 c_\eta^2 G_{ac}^2}{c_\beta \kappa} = \Theta\left( N^{-1/3} \tau^{-4/3} \kappa^{-6} (1-\gamma)^{-6} \right). $$
While both terms share an identical $\Theta(N^{-1/3} \tau^{-4/3})$ dependence, the second term dominates in its dependence on $\kappa$ and $(1-\gamma)^{-1}$.
Combining them yields the final noise envelope:
$$ K_N = \Theta\left( N^{-1/3} \tau^{-4/3} \kappa^{-6} (1-\gamma)^{-6} \right). $$

(iv) \textit{Synthesis.} We substitute the evaluated orders back into the tracking and PMD penalties:
\begin{itemize}
    \item \textit{Initialization Tracking Penalty:} The geometric coefficient evaluates to $\frac{t_0^{2/3}}{c_\beta\kappa} = \frac{\Theta(N^{1/3} \tau^{-2/3})}{\Theta(N^{1/3} \tau^{-2/3})} = \Theta(1)$. Thus, the initialization tracking penalty evaluates to $\Theta\left( \frac{Z_0}{\tau T} \right)$.
        As shown in the proof of Theorem~\ref{thm:average-iterate},
        the initial tracking error $Z_0 = \mathbb{E}[\|\theta_1 - \theta^*_\tau(\omega_0)\|_2^2]$ is bounded by $Z_0 \le (2B_\theta)^2 = G_{ac}^2 = \Theta(\kappa^{-2}(1-\gamma)^{-2})$. Substituting this bound and dividing by $1-\gamma$, the penalty yields $\mathcal{O}\left( \frac{1}{(1-\gamma)^3 \kappa^2 \tau T} \right)$, which is subsumed by the PMD local error $\mathcal{O}\left( \frac{\log T}{(1-\gamma)^3 \kappa^2 \tau T} \right)$.

    \item \textit{Invariant Tracking Error:} The primary fractional tracking penalty evaluates to $\Theta\left( \frac{K_N}{\tau c_\beta\kappa T^{2/3}} \right) = \Theta\left( \frac{N^{-2/3}}{\tau^{5/3} \kappa^6 (1-\gamma)^6 T^{2/3}} \right)$.
    Because $N^{2/3} T^{2/3} = (N \times T)^{2/3} = T_{total}^{2/3}$, the individual loop sizes combine into the total size $T_{total}$. After dividing the entire inequality by $1-\gamma$, this yields the invariant tracking error envelope $\Theta\left( \frac{1}{(1-\gamma)^7 \kappa^6 \tau^{5/3} T_{total}^{2/3}} \right)$.

    \item \textit{Transient Convergence Error:} The secondary fractional tracking penalty divided by $1-\gamma$ evaluates to $\Theta\left( \frac{K_N t_0^{1/3}}{(1-\gamma) \tau c_\beta\kappa T} \right)= \Theta\left( \frac{ N^{-1/2} }{(1-\gamma)^7 \kappa^6 \tau^2 T} \right)$.
    As shown in the proof of Theorem \ref{thm:reg-conv-l2}, the initial KL divergence is bounded by $D_0^\dag \le \frac{2B_\theta B_\phi}{\tau_{max}} + \log|\mathcal{A}| = \mathcal{O}(\kappa^{-1}(1-\gamma)^{-1})$. Since $t_0/c_\eta=\Theta(N^{1/2})$, the PMD initialization error divided by $1-\gamma$ evaluates to $\mathcal{O}\left( \frac{t_0 D_0^\dag}{(1-\gamma)c_\eta T} \right)=\mathcal{O}\left(\frac{N^{1/2}}{(1-\gamma)^2\kappa T}\right)$.
    Finally, dividing the PMD local error by $1-\gamma$ yields $\Theta\left( \frac{\log T}{(1-\gamma)^3 \kappa^2 \tau T} \right)$.
\end{itemize}
Combining the invariant tracking error envelope with these transient optimization terms and the $\Theta(\tau^{-1} \epsilon_{app})$ approximation bias completes the proof.
\end{prf}

\begin{rem}[Advantage Over Pointwise Summation and Last-Iterate Collapse] \label{rem:pointwise-collapse}
The proof of Theorem~\ref{thm:warm-start-reg} fundamentally relies on the average tracking error derived from the Fractional Chung's Summation Bound (Lemma~\ref{lem:fractional-chung-sum}). Alternatively, if we were to sum the pointwise tracking error $Z_t \le v(t+t_0)^{-2/3}$ derived from the standard Fractional Chung's Lemma, the resulting bound would include a heavily inflated penalty.

Specifically, the pointwise bounding coefficient is
$v \triangleq
\max\left(
t_0^{2/3}Z_0,\,
\frac{2K_N}{c_\beta\kappa}
\right)$, with an initialization contribution $t_0^{2/3} Z_0$.
Summing the pointwise bound via an integral ($\sum_{t=0}^{T-1} (t+t_0)^{-2/3} \le 3(T+t_0)^{1/3}$) yields
$\frac{1}{T}\sum_{t=0}^{T-1} Z_t
=\mathcal{O}\big(
\frac{v(T+t_0)^{1/3}}{T}\big)$.
Using $(T+t_0)^{1/3}\le T^{1/3}+t_0^{1/3}$, the component proportional to $T^{1/3}$ gives an average tracking penalty of $\mathcal{O}\big(\frac{v}{\tau T^{2/3}}\big)$,
with an initialization contribution $\mathcal{O}\big(\frac{t_0^{2/3}Z_0}{\tau T^{2/3}}\big)$, in the performance gap bound.
As evaluated in Step (ii) of the proof of Theorem~\ref{thm:warm-start-reg},
the offset condition $t_0 \ge (4c_\beta \kappa)^{3/2}$ (ensuring a stable inner-loop contraction) requires an initial offset $t_0 = \Theta(N^{1/2} \tau^{-1})$.
By substituting this and the outer-loop size $T = T_{total}/N$, the initialization contribution to the average tracking penalty becomes
$$
\mathcal{O}\left(
\frac{N^{1/3}Z_0}{\tau^{5/3}T^{2/3}}
\right) =
\mathcal{O}\left(
\frac{NZ_0}
{\tau^{5/3}T_{total}^{2/3}}
\right).$$
In contrast, by utilizing summation by parts, Lemma~\ref{lem:fractional-chung-sum} captures the exponential decay of the initial error, decoupling it from the harmonic integral and bounding its contribution to the average penalty by $\mathcal{O}\big(\frac{t_0^{2/3} Z_0}{c_\beta\kappa \tau T}\big)$. Because $t_0^{2/3}$ and the effective contraction rate $c_\beta\kappa$ both scale as $\Theta(N^{1/3} \tau^{-2/3})$, the average penalty reduces to $\mathcal{O}\big(\frac{Z_0}{\tau T}\big) = \mathcal{O}\big(\frac{N Z_0}{\tau T_{total}}\big)$.
The pointwise approach inflates this penalty by a factor of $c_\beta\kappa T^{1/3} = \Theta(T_{total}^{1/3} \tau^{-2/3})$,
consistent with Remark~\ref{rem:comparison-pointwise}.

Under the Minimal Action Gap, in the sample-limited phase, the temperature is tuned to $\tau = \Theta(1/\log T_{total})$ to achieve the $\tilde{\mathcal{O}}(T_{total}^{-2/3})$ average-iterate unregularized rate.
The initialization component of the pointwise bound, $\mathcal{O}\big(N Z_0 \tau^{-5/3} T_{total}^{-2/3}\big)$, becomes $\tilde{\mathcal{O}}\big(N Z_0 T_{total}^{-2/3}\big)$. For a critic initialization with $Z_0$ bounded away from zero, keeping this term within the $\tilde{\mathcal{O}}(T_{total}^{-2/3})$ target rate requires $N = \tilde{\Theta}(1)$, causing the invariance window to collapse. In contrast, Lemma~\ref{lem:fractional-chung-sum} safely bounds the initialization penalty by $\tilde{\mathcal{O}}(N Z_0 T_{total}^{-1})$, allowing $N$ to grow unbounded without breaking the target rate envelope (Section~\ref{app:warm-start-exponential}).

Likewise, relying on the pointwise tracking error to establish the \textit{last-iterate} regularized convergence---which involves evaluating the terminal error $Z_T \le v T^{-2/3}$ directly to establish the convergence of the forward KL divergence---faces the same structural bottleneck, leading to a similar collapse of the last-iterate invariance window to $N = \tilde{\Theta}(1)$.
\end{rem}

\subsection{Warm-Start Unregularized Convergence (Minimal Action Gap)} \label{app:warm-start-exponential}

Building upon the regularized guarantees, we map the warm-start performance gap to the unregularized objective for Algorithm~\ref{alg:warm-start-nac}. By leveraging the Universal Exponential Translation Bound and deploying a two-stage temperature schedule, we balance the invariant tracking error against an exponentially small target tail. This mechanism bypasses the standard linear entropy bottleneck, extracting an accelerated $\tilde{\mathcal{O}}(T_{total}^{-2/3})$ average-iterate unregularized convergence rate that remains invariant for any intermediate inner-loop choice up to $N \le T_{total}^{2/9}$.

\begin{thm}[Warm-Start Exponential Translation Convergence] \label{thm:warm-start-exponential}
For Algorithm~\ref{alg:warm-start-nac}, suppose that Assumptions~\ref{ass:bounded-features}--\ref{ass:action-gap} and \ref{ass:density-ratios} hold.
Define the two-stage temperature:
$$ \tau_{T_{total}} \triangleq \max\left( \frac{\Delta}{2\log(C_\gamma T_{total}^{2/3})}, \frac{\Delta}{2\log(C_\gamma(1+\epsilon_{app}^{-1}))} \right). $$
For any intermediate inner-loop choice $N \in [1,T_{total}^{2/9}]$ and $T=T_{total}/N$, if $T_{total}$ is sufficiently large and $\epsilon_{app}$ is sufficiently small such that $\tau_{T_{total}}\le\tau_{max}$, then by choosing the step-size parameters as in Theorem~\ref{thm:warm-start-reg} evaluated at $\tau=\tau_{T_{total}}$, the expected average unregularized performance gap satisfies the convergence rate:
\begin{align*}
\frac{1}{T}\sum_{t=0}^{T-1}\mathbb{E}\big[J_0(\pi^*_0)-J_0(\pi_t)\big]
&\le \mathcal{O}\left(\frac{(1-\gamma)^{-5/3}+(\log T_{total})^{5/3}}{(1-\gamma)^8\kappa^6\Delta^{8/3}T_{total}^{2/3}}\right) \\
&\quad + \mathcal{O}\left(\frac{(1-\gamma)^{-1}+\log(1+\epsilon_{app}^{-1})}{(1-\gamma)^2\Delta^2}\epsilon_{app}\right) \\
&= \tilde{\mathcal{O}}\left(T_{total}^{-2/3}\right)+\tilde{\mathcal{O}}(\epsilon_{app}).
\end{align*}
In the case of $\epsilon_{app}=0$, we interpret $\tau_{T_{total}}=\frac{\Delta}{2\log(C_\gamma T_{total}^{2/3})}$ and $\epsilon_{app}\log(1+\epsilon_{app}^{-1})=0$.
\end{thm}

\begin{prf}
We first consider $\epsilon_{app}>0$. When $\epsilon_{app}=0$, the approximation bias term vanishes and the result follows from the sample-limited calculation with $\tau_{T_{total}}=\frac{\Delta}{2\log(C_\gamma T_{total}^{2/3})}$.

(i) \textit{Global Unregularized Bound.} By applying the Universal Unregularized Suboptimality Bound (Corollary~\ref{cor:universal-suboptimal-mass}) to the warm-start average-iterate regularized gap from Theorem~\ref{thm:warm-start-reg}, the global unregularized performance gap is bounded for any $\tau>0$ by:
\begin{align*}
&\quad \frac{1}{T} \sum_{t=0}^{T-1} \mathbb{E} \big[ J_0(\pi^*_0) - J_0(\pi_t) \big] \\
\le &\underbrace{ \mathcal{O}\left( \frac{1}{\Delta(1-\gamma)^8 \kappa^6 \tau^{5/3} T_{total}^{2/3}} \right) }_{\text{Invariant Tracking Error}} + \underbrace{ \mathcal{O}\left( \frac{\epsilon_{app}}{\Delta(1-\gamma)^2\tau} \right) }_{\text{Approximation Bias}} + \underbrace{ C_{tail} \exp\left(-\frac{\Delta}{2\tau}\right) }_{\text{Entropy Tail}} \\
&+ \underbrace{ \tilde{\mathcal{O}}\left( \frac{N^{1/2}}{\Delta(1-\gamma)^3 \kappa T} + \frac{N^{-1/2}}{\Delta(1-\gamma)^8 \kappa^6 \tau^2 T} + \frac{1}{\Delta(1-\gamma)^4 \kappa^2 \tau T} \right) }_{\text{Transient Convergence Error}}.
\end{align*}

(ii) \textit{Regime 1: Sample-Limited Phase.} When $\log(T_{total}^{2/3})\le\log(1+\epsilon_{app}^{-1})$, the temperature is tuned to
$ \tau_{T_{total}}=\frac{\Delta}{2\log(C_\gamma T_{total}^{2/3})}$.
Substituting this alongside $A_{max}\le\frac{1}{1-\gamma}$, $C_{tail}\triangleq\frac{A_{max}C_\gamma}{1-\gamma}\le\frac{C_\gamma}{(1-\gamma)^2}$, and $\log C_\gamma=\mathcal{O}((1-\gamma)^{-1})$ into the leading components yields:
\begin{align*}
\text{Invariant Tracking Error}
&=\mathcal{O}\left(\frac{(\log(C_\gamma T_{total}^{2/3}))^{5/3}}{(1-\gamma)^8\kappa^6\Delta^{8/3}T_{total}^{2/3}}\right) \\
&=\mathcal{O}\left(\frac{(1-\gamma)^{-5/3}+(\log T_{total})^{5/3}}{(1-\gamma)^8\kappa^6\Delta^{8/3}T_{total}^{2/3}}\right), \\
\text{Approximation Bias}
&=\mathcal{O}\left(\frac{\log(C_\gamma T_{total}^{2/3})}{(1-\gamma)^2\Delta^2}\epsilon_{app}\right) \\
&\le\mathcal{O}\left(\frac{(1-\gamma)^{-1}+\log(1+\epsilon_{app}^{-1})}{(1-\gamma)^2\Delta^2}\epsilon_{app}\right), \\
\text{Entropy Tail}
&\le\frac{C_\gamma}{(1-\gamma)^2}\exp\left(-\log(C_\gamma T_{total}^{2/3})\right)
=\frac{1}{(1-\gamma)^2T_{total}^{2/3}}.
\end{align*}
The regime condition $\log(T_{total}^{2/3})\le\log(1+\epsilon_{app}^{-1})$ guarantees
$$ \log(C_\gamma T_{total}^{2/3})\le\log C_\gamma+\log(1+\epsilon_{app}^{-1}), $$
which gives the stated approximation-bias envelope. The $C_\gamma$ prefactor cancels in the exponential tail, yielding an $\mathcal{O}(T_{total}^{-2/3})$ tail that is absorbed by the invariant tracking error.

We next evaluate the transient errors to verify loop invariance:
\begin{itemize}
    \item \textit{PMD Initialization Error:} $\tilde{\mathcal{O}}\big(\frac{N^{1/2}}{\Delta(1-\gamma)^3\kappa T}\big)=\tilde{\mathcal{O}}\big(\frac{N^{3/2}T_{total}^{-1}}{\Delta(1-\gamma)^3\kappa}\big)$. To ensure that this error remains within the $\tilde{\mathcal{O}}(T_{total}^{-2/3})$ target envelope, it suffices to require $N^{3/2}T_{total}^{-1}\le T_{total}^{-2/3}$, establishing the invariance boundary $N\le T_{total}^{2/9}$.

    \item \textit{Secondary Fractional Tracking Penalty:} Under the tuned temperature,
    $$ \tilde{\mathcal{O}}\left(\frac{N^{-1/2}}{\Delta(1-\gamma)^8\kappa^6\tau^2T}\right)
    =\tilde{\mathcal{O}}\left(\frac{N^{1/2}T_{total}^{-1}(\log(C_\gamma T_{total}^{2/3}))^2}{\Delta^3(1-\gamma)^8\kappa^6}\right). $$
    For $N\le T_{total}^{2/9}$, this term decays as $\tilde{\mathcal{O}}(T_{total}^{-8/9})$ with respect to $T_{total}$ and is therefore dominated by the $\tilde{\mathcal{O}}(T_{total}^{-2/3})$ target envelope.

    \item \textit{PMD Local Error:} Under the tuned temperature,
    $$ \tilde{\mathcal{O}}\left(\frac{1}{\Delta(1-\gamma)^4\kappa^2\tau T}\right)
    =\tilde{\mathcal{O}}\left(\frac{NT_{total}^{-1}\log(C_\gamma T_{total}^{2/3})}{\Delta^2(1-\gamma)^4\kappa^2}\right). $$
    For $N\le T_{total}^{2/9}$, this term decays as $\tilde{\mathcal{O}}(T_{total}^{-7/9})$ with respect to $T_{total}$ and is also dominated by the $\tilde{\mathcal{O}}(T_{total}^{-2/3})$ target envelope.
\end{itemize}
Thus, for any inner-loop choice $N\in[1,T_{total}^{2/9}]$, all transient errors remain within the target envelope, establishing the stated loop-invariance window.

(iii) \textit{Regime 2: Approximation-Limited Phase.} When $\log(T_{total}^{2/3})>\log(1+\epsilon_{app}^{-1})$, the temperature locks to the approximation-dependent floor:
$ \tau_{T_{total}}=\frac{\Delta}{2\log(C_\gamma(1+\epsilon_{app}^{-1}))}$.
Substituting this fixed temperature into the leading components yields:
\begin{align*}
\text{Invariant Tracking Error}
&=\mathcal{O}\left(\frac{(\log(C_\gamma(1+\epsilon_{app}^{-1})))^{5/3}}{(1-\gamma)^8\kappa^6\Delta^{8/3}T_{total}^{2/3}}\right) \\
&\le\mathcal{O}\left(\frac{(1-\gamma)^{-5/3}+(\log T_{total})^{5/3}}{(1-\gamma)^8\kappa^6\Delta^{8/3}T_{total}^{2/3}}\right), \\
\text{Approximation Bias}
&=\mathcal{O}\left(\frac{\log(C_\gamma(1+\epsilon_{app}^{-1}))}{(1-\gamma)^2\Delta^2}\epsilon_{app}\right) \\
&=\mathcal{O}\left(\frac{(1-\gamma)^{-1}+\log(1+\epsilon_{app}^{-1})}{(1-\gamma)^2\Delta^2}\epsilon_{app}\right), \\
\text{Entropy Tail}
&\le\frac{C_\gamma}{(1-\gamma)^2}\exp\left(-\log(C_\gamma(1+\epsilon_{app}^{-1}))\right) \\
&=\frac{\epsilon_{app}/(1+\epsilon_{app})}{(1-\gamma)^2}
\le\frac{\epsilon_{app}}{(1-\gamma)^2}.
\end{align*}
Because $\log(1+\epsilon_{app}^{-1})<\log(T_{total}^{2/3})$, we have
$$ \log(C_\gamma(1+\epsilon_{app}^{-1}))<\log(C_\gamma) + \log( T_{total}^{2/3}). $$
Therefore, the secondary fractional tracking penalty and PMD local error are bounded by their corresponding Regime~1 envelopes, while the PMD initialization error is independent of $\tau$. Hence, for any inner-loop choice $N\in[1,T_{total}^{2/9}]$, all transient errors remain safely within the $\tilde{\mathcal{O}}(T_{total}^{-2/3})$ target envelope.

(iv) \textit{Synthesis.} Summing the evaluated components across the two regimes yields:
\begin{align*}
\frac{1}{T}\sum_{t=0}^{T-1}\mathbb{E}\big[J_0(\pi^*_0)-J_0(\pi_t)\big]
&\le \mathcal{O}\left(\frac{(1-\gamma)^{-5/3}+(\log T_{total})^{5/3}}{(1-\gamma)^8\kappa^6\Delta^{8/3}T_{total}^{2/3}}\right) \\
&\quad + \mathcal{O}\left(\frac{(1-\gamma)^{-1}+\log(1+\epsilon_{app}^{-1})}{(1-\gamma)^2\Delta^2}\epsilon_{app}\right).
\end{align*}
Finally, as $\epsilon_{app}\to0$, $\log(1+\epsilon_{app}^{-1})=\Theta(\log(1/\epsilon_{app}))$ and $\epsilon_{app}\log(1+\epsilon_{app}^{-1})=\tilde{\mathcal{O}}(\epsilon_{app})$. Therefore, the overall convergence rate is $\tilde{\mathcal{O}}(T_{total}^{-2/3})+\tilde{\mathcal{O}}(\epsilon_{app})$, completing the proof.
\end{prf}

\subsection{Warm-Start Unregularized Convergence (Linear Entropy)} \label{app:warm-start-linear-entropy}

To expose the underlying barrier bypassed by the exponential translation mechanism, we provide a comparative unregularized convergence analysis relying on the standard linear entropy penalty for Algorithm~\ref{alg:warm-start-nac}.
Crucially, this result does not require the Minimal Action Gap assumption, providing a robust theoretical fallback for environments lacking a uniformly positive action margin. By balancing the invariant tracking error against the $\mathcal{O}(\tau)$ entropy penalty, we establish a $\tilde{\mathcal{O}}(T_{total}^{-1/4})$ average-iterate unregularized convergence rate, which maintains its loop invariance across the shifted inner-loop window $N \in [1, \sqrt{T_{total}}]$.

\begin{thm}[Warm-Start Linear Entropy Convergence] \label{thm:warm-start-linear-entropy}
For Algorithm~\ref{alg:warm-start-nac}, suppose that Assumptions~\ref{ass:bounded-features}--\ref{ass:approximation-error} and \ref{ass:density-ratios} hold.
Define the two-stage temperature:
$$ \tau_{T_{total}} \triangleq \max\left( \frac{1}{(1-\gamma)^{9/4} \kappa^{9/4} T_{total}^{1/4}}, \sqrt{\epsilon_{app}} \right). $$
For any intermediate inner-loop choice $N \in [1, \sqrt{T_{total}}]$ and $T = T_{total} / N$, if $T_{total}$ is sufficiently large and $\epsilon_{app}$ is sufficiently small such that $\tau_{T_{total}} \le \tau_{max}$, then by choosing the step-size parameters as in Theorem~\ref{thm:warm-start-reg} evaluated at $\tau = \tau_{T_{total}}$, the expected average unregularized performance gap satisfies the convergence rate:
\begin{align*}
\frac{1}{T} \sum_{t=0}^{T-1} \mathbb{E} \big[ J_0(\pi^*_0) - J_0(\pi_t) \big] &\le \tilde{\mathcal{O}}\left( \frac{1}{(1-\gamma)^{13/4} \kappa^{9/4} T_{total}^{1/4}} \right) + \mathcal{O}\left( \frac{\sqrt{\epsilon_{app}}}{1-\gamma} \right) \\
&= \tilde{\mathcal{O}}\left( T_{total}^{-1/4} \right) + \mathcal{O}(\sqrt{\epsilon_{app}}).
\end{align*}
\end{thm}

\begin{prf}
(i) \textit{Global Unregularized Bound.} Because the policy entropy is non-negative, evaluating the optimal policy under the regularized objective yields $J_0(\pi^*_0) \le J_\tau(\pi^*_0) \le J_\tau(\pi^*_\tau)$.
For any training policy $\pi_t$, the global unregularized performance gap is bounded by the global regularized performance gap plus the linear entropy penalty:
\begin{align*}
J_0(\pi^*_0) - J_0(\pi_t) &\le J_\tau(\pi^*_\tau) - J_0(\pi_t) \\
&= \Big( J_\tau(\pi^*_\tau) - J_\tau(\pi_t) \Big) + \Big( J_\tau(\pi_t) - J_0(\pi_t) \Big) \\
&\le \text{Gap}_t^\dag + \frac{\tau \log |\mathcal{A}|}{1-\gamma}.
\end{align*}
Taking the expectation and substituting the average warm-start regularized gap from Theorem~\ref{thm:warm-start-reg} explicitly yields the total unregularized error bound:
\begin{align*}
&\quad \frac{1}{T} \sum_{t=0}^{T-1} \mathbb{E} \big[ J_0(\pi^*_0) - J_0(\pi_t) \big] \\
\le &\underbrace{ \mathcal{O}\left( \frac{1}{(1-\gamma)^7 \kappa^6 \tau^{5/3} T_{total}^{2/3}} \right) }_{\text{Invariant Tracking Error}} + \underbrace{ \mathcal{O}\left( \frac{\epsilon_{app}}{(1-\gamma)\tau} \right) }_{\text{Approximation Bias}} + \underbrace{ \mathcal{O}\left( \frac{\tau}{1-\gamma} \right) }_{\text{Entropy Penalty}} \\
&+ \underbrace{ \tilde{\mathcal{O}}\left( \frac{N^{1/2}}{(1-\gamma)^2 \kappa T} + \frac{N^{-1/2}}{(1-\gamma)^7 \kappa^6 \tau^2 T} + \frac{1}{(1-\gamma)^3 \kappa^2 \tau T} \right) }_{\text{Transient Convergence Error}}.
\end{align*}

(ii) \textit{Regime 1: Sample-Limited Phase.} When $\frac{1}{(1-\gamma)^{9/4} \kappa^{9/4} T_{total}^{1/4}} \ge \sqrt{\epsilon_{app}}$,
we balance the Invariant Tracking Error against the Entropy Penalty. Setting $\frac{1}{(1-\gamma)^7 \kappa^6 \tau^{5/3} T_{total}^{2/3}} = \frac{\tau}{1-\gamma}$ yields the tuned temperature $\tau_{T_{total}} = \frac{1}{(1-\gamma)^{9/4} \kappa^{9/4} T_{total}^{1/4}}$. Under this tuning, both the Invariant Tracking Error and the Entropy Penalty evaluate to $\mathcal{O}\big( \frac{1}{(1-\gamma)^{13/4} \kappa^{9/4} T_{total}^{1/4}} \big)$,
and the Approximation Bias is dominated by this leading order. We evaluate the transient errors under this tuning to verify loop invariance:
\begin{itemize}
    \item \textit{PMD Initialization Error:} $\tilde{\mathcal{O}}\big( \frac{N^{1/2}}{(1-\gamma)^2\kappa T} \big) = \tilde{\mathcal{O}}\big( \frac{N^{3/2}T_{total}^{-1}}{(1-\gamma)^2\kappa} \big)$. To ensure that this error remains within the $\tilde{\mathcal{O}}(T_{total}^{-1/4})$ target envelope, it suffices to require $N^{3/2}T_{total}^{-1}\le T_{total}^{-1/4}$, establishing the invariance boundary $N\le\sqrt{T_{total}}$.

    \item \textit{Secondary Fractional Tracking Penalty:} $\tilde{\mathcal{O}}\big( \frac{N^{-1/2}}{(1-\gamma)^7\kappa^6\tau^2T} \big) = \tilde{\mathcal{O}}\big( \frac{N^{1/2}T_{total}^{-1}}{(1-\gamma)^7\kappa^6\tau^2} \big)$. Because $\tau=\Theta(T_{total}^{-1/4})$, this term evaluates to $\tilde{\mathcal{O}}(N^{1/2}T_{total}^{-1/2})$. For $N\le\sqrt{T_{total}}$, it is bounded by $\tilde{\mathcal{O}}(T_{total}^{-1/4})$ and therefore remains within the target envelope.

    \item \textit{PMD Local Error:} $\tilde{\mathcal{O}}\big( \frac{1}{(1-\gamma)^3\kappa^2\tau T} \big) = \tilde{\mathcal{O}}\big( \frac{NT_{total}^{-1}}{(1-\gamma)^3\kappa^2\tau} \big)$. Because $\tau=\Theta(T_{total}^{-1/4})$, this term evaluates to $\tilde{\mathcal{O}}(NT_{total}^{-3/4})$. For $N\le\sqrt{T_{total}}$, it is bounded by $\tilde{\mathcal{O}}(T_{total}^{-1/4})$ and therefore also remains within the target envelope.
\end{itemize}
Thus, for any inner-loop choice $N\in[1,\sqrt{T_{total}}]$, all transient errors remain within the $\tilde{\mathcal{O}}(T_{total}^{-1/4})$ target envelope, establishing the stated loop-invariance window.

(iii) \textit{Regime 2: Approximation-Limited Phase.} When $T_{total}$ grows such that $\frac{1}{(1-\gamma)^{9/4} \kappa^{9/4} T_{total}^{1/4}} < \sqrt{\epsilon_{app}}$, the temperature locks to the constant floor $\tau_{T_{total}} = \sqrt{\epsilon_{app}}$. The Approximation Bias and the Entropy Penalty evaluate to $\mathcal{O}\big( \frac{\sqrt{\epsilon_{app}}}{1-\gamma} \big)$, and the Invariant Tracking Error is dominated by this leading order.
Because $\frac{1}{(1-\gamma)^{9/4}\kappa^{9/4}T_{total}^{1/4}}<\sqrt{\epsilon_{app}}$, the secondary fractional tracking penalty and PMD local error are bounded by their corresponding Regime~1 envelopes, while the PMD initialization error is independent of $\tau$. Thus, for any inner-loop choice $N\in[1,\sqrt{T_{total}}]$, all transient errors remain safely within the $\tilde{\mathcal{O}}(T_{total}^{-1/4})$ target envelope.

(iv) \textit{Synthesis.} By setting the tuned temperature $\tau_{T_{total}}$ to the maximum of these two balanced bounds, the expected average unregularized performance gap is bounded by the sum of their respective envelopes, completing the proof.
\end{prf}

\end{document}